\documentclass[12pt]{article}

\usepackage{graphicx,psfrag,epsf}
\usepackage{enumerate}
\usepackage{natbib}
\usepackage{url} 
\usepackage[noblocks]{authblk}
\usepackage[margin=0.67in]{geometry}

\usepackage[export]{adjustbox}
\usepackage[ruled,vlined]{algorithm2e}
\usepackage{algpseudocode}
\usepackage{amsfonts}       
\usepackage{amsmath}
\usepackage{amsthm}
\usepackage{bbold}
\usepackage{xr}
\usepackage{xr-hyper}
\usepackage{bm, bbm}
\usepackage{booktabs}       
\usepackage{breakcites}
\allowdisplaybreaks
\usepackage{comment}
\usepackage{color}
\usepackage{dsfont}
\usepackage{enumitem}
\usepackage{extarrows}
\usepackage{float}
\usepackage[T1]{fontenc}    
\usepackage{hyperref}       
\usepackage[utf8]{inputenc} 
\usepackage{lscape} 
\usepackage{mathabx}
\usepackage{mathrsfs}  
\usepackage{mathtools}
\usepackage{microtype}      
\usepackage{multirow}
\usepackage{nicefrac}       
\usepackage[labelformat=simple]{subcaption}

\usepackage{url}            
\usepackage{wrapfig}
\usepackage[dvipsnames]{xcolor}

\newlist{condenum}{enumerate}{1}

\newtheorem{theorem}{Theorem}
\newtheorem{lemma}{Lemma}[section]
\newtheorem{proposition}[theorem]{Proposition} 
\newtheorem{remark}{Remark}[section]
\newtheorem{corollary}{Corollary}[theorem]
\newtheorem{definition}{Definition}

\newtheorem*{assumptions}{Assumptions}

\DeclareMathOperator*{\conv}{Conv}

\DeclareMathOperator*{\argmax}{arg\,max} 
\DeclareMathOperator*{\argmin}{arg\,min} 
\newcommand{\RomanNumeralCaps}[1]{\MakeUppercase{\romannumeral #1}}

\newcommand{\colorblue}{\color{black}}

\newcommand{\bU}{\mathbf{U}}
\newcommand{\bC}{\mathbf{C}}
\newcommand{\bW}{\mathbf{W}}

\newcommand\tC{\mathbf C^0}

\newcommand\tU{\mathbf U^0}
\newcommand{\tu}[1]{\mathbf u^{0}_{#1}}
\newcommand\tW{\mathbf W^0}

\newcommand{\tw}[1]{\mathbf w^{0}_{#1}}
\newcommand{\idfn}{\mathds 1}
\newcommand\dis{\mathcal{D}}
\newcommand\simplex{\Delta}
\newcommand{\smallC}{\mathbf C^*}
\newcommand{\setballs}{\mathcal{S}}
\newcommand{\hatCn}{\hat{\mathbf C}_n}

\newcommand*\samethanks[1][\value{footnote}]{\footnotemark[#1]}
\newcommand{\blind}{0}

\begin{document}

\if0\blind
{
\title{\bf Learning Topic Models:  Identifiability and Finite-Sample Analysis}
\author[1]{Yinyin Chen\thanks{Yinyin Chen and Shishuang He contributed equally to this work.}}
\affil{Meta Platforms, Inc.}
\author[2]{Shishuang He\samethanks}
\affil{Department of Statistics, University of Illinois at Urbana-Champaign}
\author[2]{Yun Yang}
\author[2]{Feng Liang}
\date{\vspace{-6ex}}
  \maketitle
} \fi

\if1\blind
{
\title{\bf Learning Topic Models:  Identifiability and Finite-Sample Analysis}
\author[]{}
\date{\vspace{-6ex}}
  \maketitle
} \fi


\begin{abstract}
 Topic models provide a useful text-mining tool for learning, extracting, and discovering latent structures in large text corpora. Although a plethora of methods have been proposed for topic modeling, lacking in the literature is a formal theoretical investigation of the statistical identifiability and accuracy of latent topic estimation. In this paper, we propose a maximum likelihood estimator (MLE) of latent topics based on a specific integrated likelihood that is naturally connected to the concept, in computational geometry,  of \emph{volume minimization}. Our theory introduces a new set of geometric conditions for topic model identifiability, conditions that are weaker than conventional separability conditions, which typically rely on the existence of pure topic documents or of anchor words. Weaker conditions allow a wider and thus potentially more fruitful investigation. We conduct finite-sample error analysis for the proposed estimator and discuss connections between our results and those of previous investigations. We conclude with empirical studies employing both simulated and real datasets.
\end{abstract}

{Keywords:} Topic models, Identifiability, Sufficiently scattered, Volume minimization, Maximum likelihood, Finite-sample analysis. 



\section{Introduction}\label{sec:intro}

Topic models, such as Latent Dirichlet Allocation  \citep{blei2003latent} models and probabilistic Latent Semantic Analysis  \citep{hofm99}, have been widely used in natural language processing, text mining,  information retrieval, etc. The purpose of those models is to learn a lower-dimensional representation of the data, in which each document can be expressed as a convex combination of a set of latent topics. 

Consider a corpus of $d$ documents with vocabulary size $V$. A topic model with $k$ latent topics can be summarized as the following matrix factorization: 
\begin{equation} \label{eq:matrix-decomp}
\bU_{V \times d} = \bC_{V \times k} \bW_{k \times d},
\end{equation}
where all matrices are column-stochastic\footnote{We say a matrix is column-stochastic if its entries are non-negative and columns sum to one.}. In particular, $\bU_{V \times d}$ is the true term-document matrix whose columns are the true underlying word frequencies for the $d$ documents; $\bC_{V \times k}$ is the \emph{topic matrix} whose columns are the multinomial parameters (i.e., word frequencies) for the $k$ topics; and $\bW_{k \times d}$ is the \emph{mixing matrix} whose columns present the mixing weights over $k$ topics for $d$ documents.  

The primary interest here is to reveal the latent structure of a collection of documents, i.e., to estimate the collection's topic matrix $\mathbf{C}$. Despite the popularity and success of topic models,  work on the estimation accuracy of $\bC$ is scarce. An obstacle to rigorous analysis of that important question is that the factorization~\eqref{eq:matrix-decomp} may not be unique up to permutation (throughout we ignore any non-uniqueness due to permutations of the $k$ topics). The non-uniqueness issue can be easily understood via the following geometric interpretation of Equation~\eqref{eq:matrix-decomp}: recovering $\bC$ based on $\bU$ is equivalent to finding a $k$-vertex convex polytope that encloses all columns of $\bU$; the vertices of this $k$-vertex convex polytope form the columns of $\bC$. Apparently, such a convex polytope may not be unique; see Figure \ref{fig:vol1}. In statistical language, topic models parameterized by $(\bC, \bW)$ without any further constraints are not identifiable (modulo column permutations).  

This leads to the following two questions that we aim to address in this paper.
 \begin{enumerate}
    \item \emph{Identifiability}. Under what conditions is a topic model parameterized by $(\bC,\bW)$ identifiable up to permutation? It is easy to achieve identifiability by imposing stringent conditions that significantly limit the usefulness of the result. Our goal is to develop a set of identifiability conditions that are weaker than ones proposed in prior studies but whose accuracy may nevertheless be well estimated.

    \item \emph{Finite-sample error}. For an identifiable topic model, can we provide an estimator of $\bC$ whose finite-sample error leads to the desired rate of convergence?  The rate will depend on the number of documents $d$ and/or the number of words per document $n$ (which, without loss of generality, is assumed to be the same for all documents). Throughout, we assume the vocabulary size $V$ and the number of topics $k$ to be known and fixed. 
\end{enumerate}

\begin{figure}[h]
    \centering
    \begin{subfigure}[t]{0.24\textwidth}
     \centering
       \includegraphics[width=0.75\textwidth]{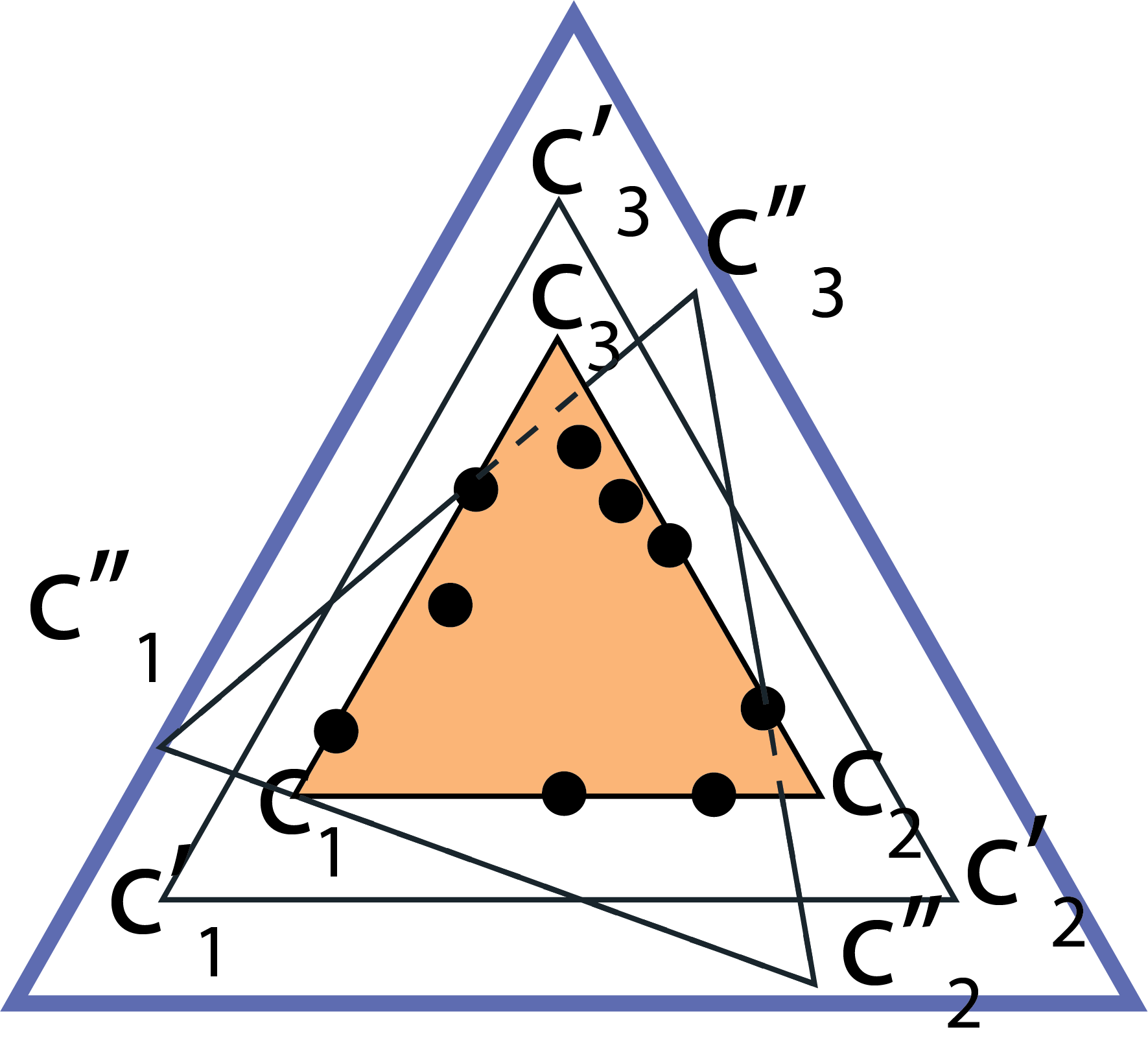}
        \subcaption{} 
        \label{fig:vol1} 
    \end{subfigure} 
    \begin{subfigure}[t]{0.24\textwidth}
     \centering
        \includegraphics[width=0.75\textwidth]{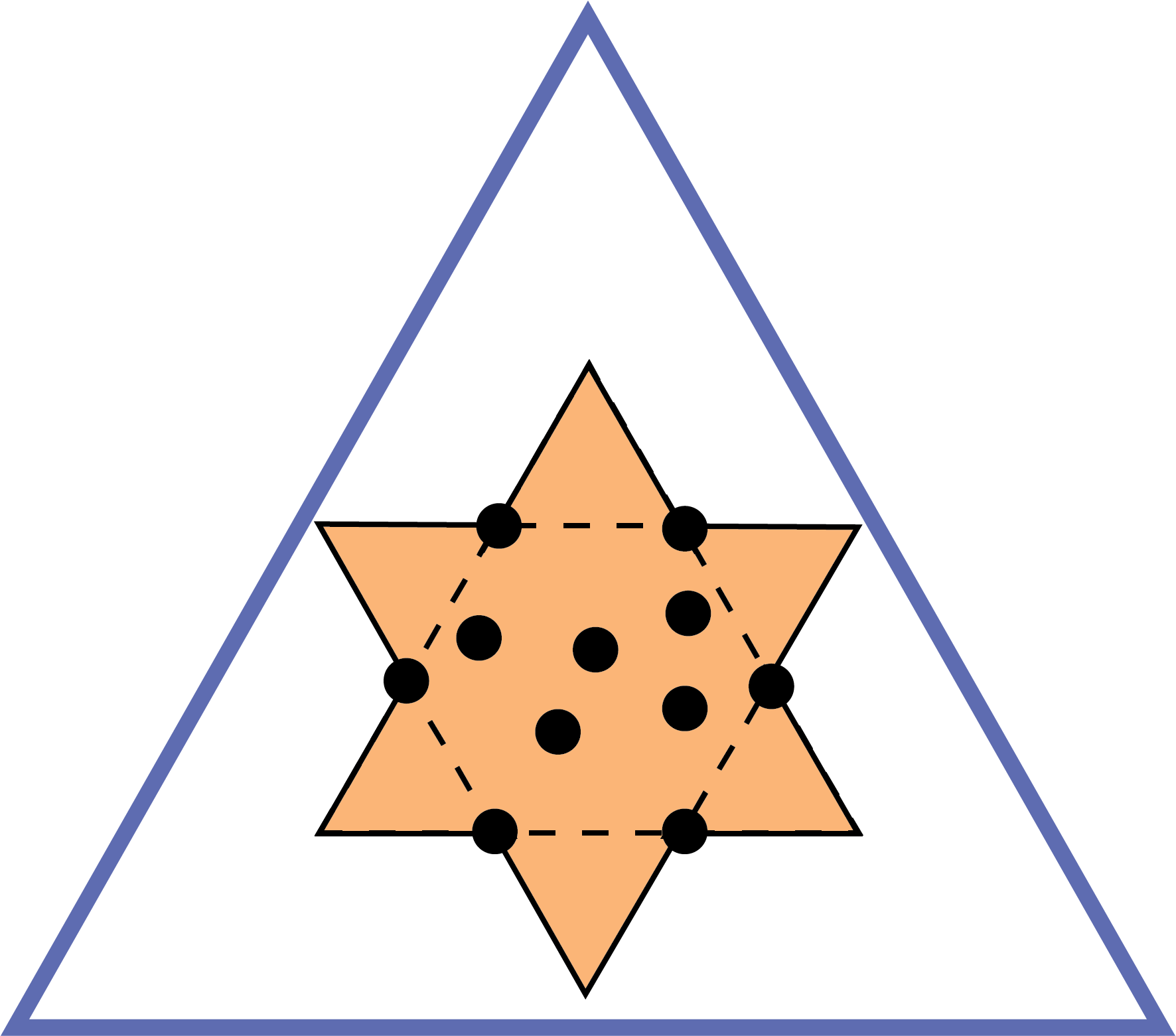}
        \subcaption{} 
        \label{fig:sm_vol1}
    \end{subfigure} 
    \begin{subfigure}[t]{0.24\textwidth}
     \centering
        \includegraphics[width=0.75\textwidth]{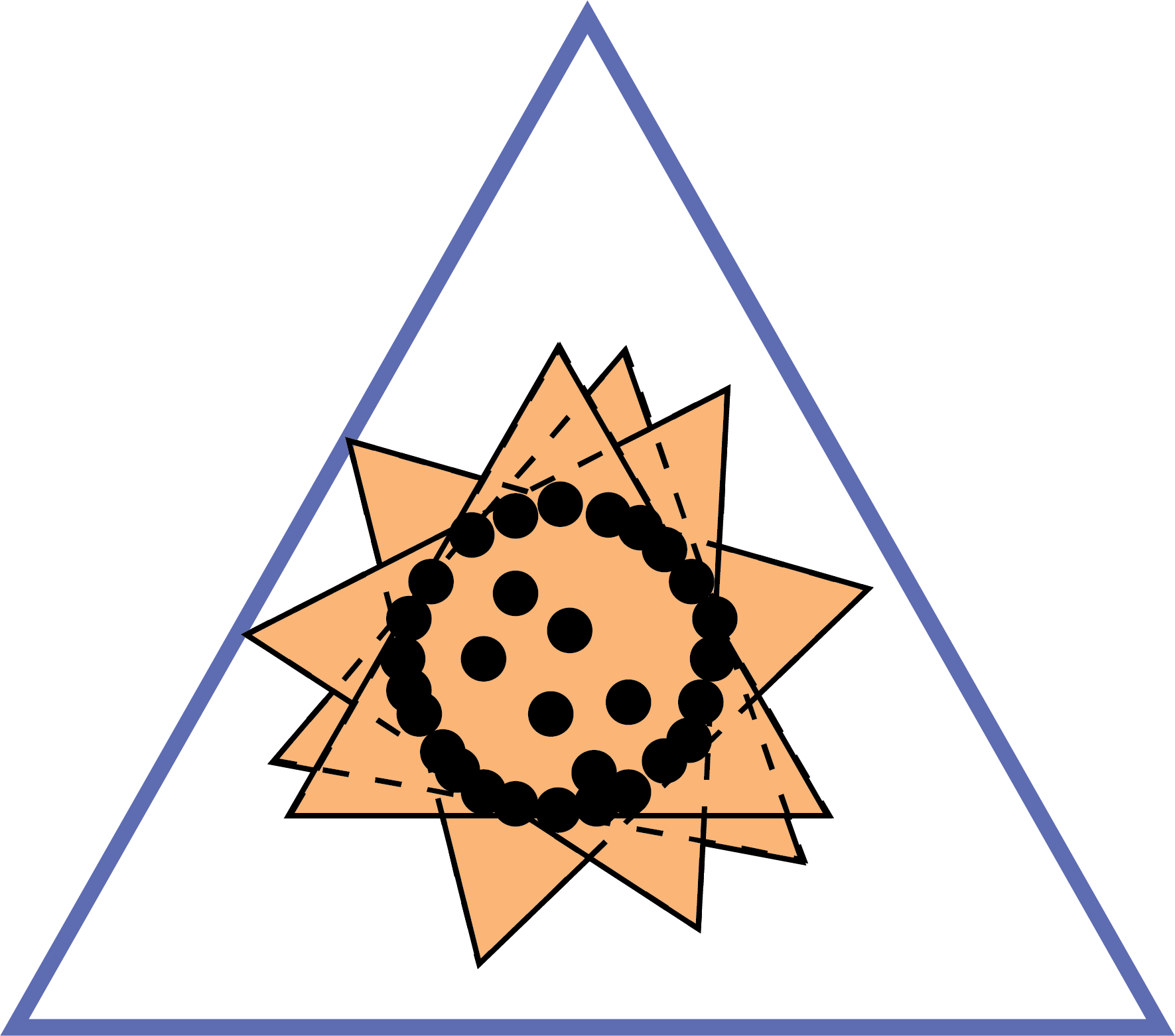}
        \subcaption{}
        \label{fig:sm_vol2} 
    \end{subfigure}
    \begin{subfigure}[t]{0.24\textwidth}
     \centering
        \includegraphics[width=0.75\textwidth]{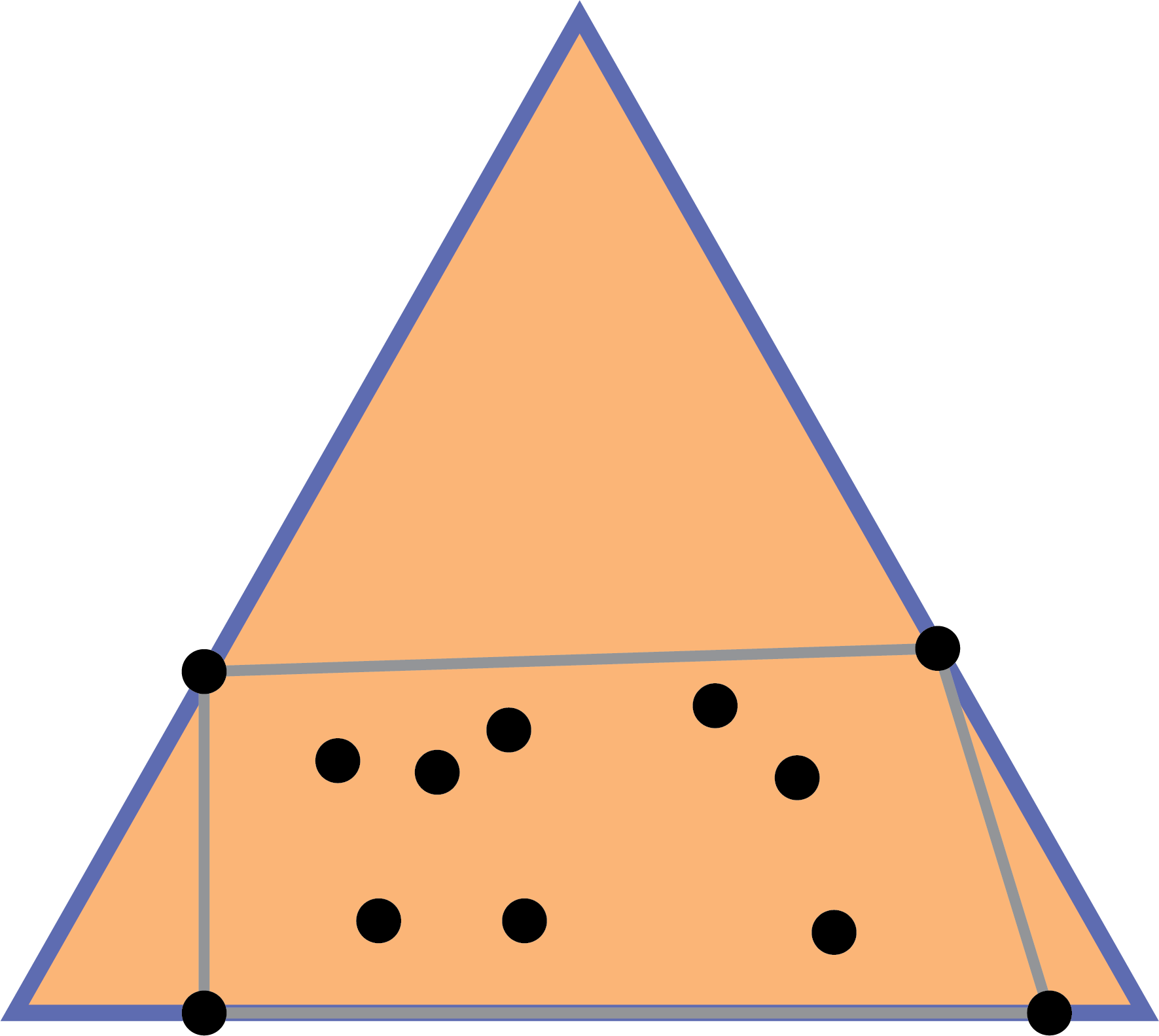}
        \subcaption{}
        \label{fig:sm_vol3} 
    \end{subfigure}
    \caption{Geometric view of the simplex $\simplex^{V-1} (k=V=3)$. Black dots are columns of $\mathbf U$. Black-lined triangles are $k$-vertex convex polygons; the shaded triangles are those with minimum volume.}
    \label{fig:vol}
\end{figure}

\subsection{Related Work}\label{sec:review}
Topic models have been studied under two settings: one in which the mixing weights, columns of $\bW$, are assumed to be stochastically generated from some distribution; the other in which they are assumed to be fixed but unknown. The Bayesian approach, for example, focuses on the former. 

\subsubsection{The Bayesian Approach}
In the Bayesian setting, the mixing weights are often assumed to be stochastically generated from a \emph{known} distribution with a full support on the simplex $\Delta^{k-1}$. Therefore, identifiability can be guaranteed under very mild conditions; for example, one such condition is just that $\bC$ be of full rank \citep{anandkumar2012spectral}. Under such Bayesian settings, \citet{nguyen2015posterior} and \citet{tang2014understanding} established posterior concentration rates; \citet{anandkumar2012spectral,anandkumar2014tensor} and \citet{Wang2019} established convergence rates for the maximum likelihood estimator (MLE). 

In this paper, we focus on a more general setting, in which the mixing weights may not be stochastically generated; if they are, moreover, we do not assume any knowledge of the corresponding distribution. Identifiability and estimation accuracy turn out to be much more challenging under this general setting. 

\subsubsection{The Separability Condition}

Several earlier investigations have addressed identifiability by imposing the so-called \emph{separability condition} or its generalization \citep{donoho2004does, arora2012learning,azar2001spectral,kleinberg2008using,kleinberg2003convergent,recht2012factoring,ge2015intersecting,ke2017new,papadimitriou2000latent,mcsherry2001spectral,anandkumar2012method}. The separability condition can be imposed either on rows of $\mathbf{C}$ or on columns of $\mathbf{W}$, due to the symmetry between these two matrices  in the factorization \eqref{eq:matrix-decomp}.

When imposed on the  topic matrix $\mathbf{C}_{V \times k}$, this condition assumes that, after the rows of $\mathbf{C}$ have been re-arranged, its top $k$ rows will form a diagonal matrix. Words associated with those rows are called \emph{anchor words}; anchor words can be used to identify topics since they appear only in one particular topic. 

When imposed on the mixing matrix $\mathbf{W}_{k \times d}$, this condition again assumes that, after the columns of $\mathbf{W}$ have been re-arranged, the first $k$ columns will form a diagonal matrix. We can further conclude that that diagonal matrix must be an identity matrix since $\mathbf{W}$ is column-stochastic; therefore, there are $k$ documents that belong to one and only one topic \citep{nascimento2005vertex,javadi2020nonnegative}. A geometric interpretation of this condition is that we can use the convex hull of $k$ columns of $\mathbf{U}$ to form the $k$-vertex polytope that contains all other columns of $\mathbf{U}$. In other words, the topic matrix $\mathbf{C}_{V \times k}$ can be recovered by identifying the corresponding subset of $k$ documents.

The separability condition can be easily violated, however, in real applications. In practice it is commonly the case that topics are correlated, tend to share keywords, and therefore are not separable. 

Nevertheless, several algorithms have been proposed to estimate $\mathbf{C}$ with a convergence rate of the order  $1/\sqrt{nd}$ \citep{arora2012learning, ke2017new}, but they assume separability. This rate of convergence would indicate that such algorithms can pool  information in the $d$ documents, each with $n$ words, to estimate $\mathbf{C}$; therefore they have an effective sample size of $nd$, instead of $n$ or $d$. However, as discussed in Section~\ref{section:theory_comparsion}, such a fast convergence rate is achievable only under the stringent separability assumption. This is because the strong separability condition greatly simplifies the statistical and computational hardness of the topic matrix estimation problem and turns it into a searching problem. As a consequence, such separability-condition-based methods circumvent the hidden non-regular statistical problem of boundary estimation (c.f.~Section~\ref{section:theory_comparsion}), which often leads to an extremely slow rate of convergence. See Section~\ref{subsec:comparison} for a review of separability-condition-based methods and how they relate to ours, from a two-stage estimation perspective.

\subsubsection{Beyond the Separability Condition}

To relax the separability assumption, the aforementioned connection between estimating a topic model and finding a $k$-vertex convex polytope that encloses all columns of $\bU$ has led researchers to start looking at geometric conditions.

When there are multiple $k$-vertex convex polytopes  enclosing columns of $\mathbf{U}$, it is natural to restrict our attention to the ones with minimum volume, that is, convex polytopes that circumscribe the data as compactly as possible. Many \emph{volume minimization} algorithms have been proposed \citep{craig1994minimum, nascimento2005vertex, miao2007endmember, fu2015blind} for nonnegative matrix factorization similar to~\eqref{eq:matrix-decomp}. However, most of these methods consider the noiseless setting. Blindly applying them to topic model estimation fails to respect the error structure in the counting data and may lead to a loss of statistical efficiency. Moreover, little theoretical work has been conducted on model identifiability and estimation accuracy beyond the limited context of topic modeling that assumes the separability condition. In particular, it is important to acknowledge that the minimum volume constraint alone does not guarantee uniqueness; see examples in Figure \ref{fig:sm_vol1}\subref{fig:sm_vol2}.

Recently, a set of geometric conditions known as the \emph{sufficiently scattered} (SS) condition, which is weaker than the separability condition, has been introduced to study identifiability of topic models \citep{huang2016anchor,jang2019minimum}.   \citet{huang2016anchor} ensure identifiability under the SS condition by adding the constraint that the determinant of $\mathbf{W} \mathbf{W}^T$ is minimized. \citet{jang2019minimum} have proved that the SS condition, along with volume minimization on the convex hull of $\mathbf{C}$, ensures identifiability when $V=k$ (vocabulary size is the same as topic size); their analysis is valid only for $V=k$ since it is built on the assumption that the volume of the convex hull of $\mathbf{C}$ is equal to the determinant of $\mathbf{C}$ (or to a monotonic function of the  determinant of $\mathbf{C}^T\mathbf{C}$) which holds true only when $V=k$. In addition, neither  \cite{huang2016anchor} nor \citet{jang2019minimum} provided a theoretical analysis of estimation errors for their proposed estimators, which are based on minimizing a squared loss based objective rather than on maximizing the multinomial likelihood associated with counting data.

\citet{javadi2020nonnegative} is the only study we are aware of that provides a theoretical analysis of estimation errors without assuming the separability condition. They proposed to estimate the $k$ columns of $\mathbf{C}$ by minimizing their distance to the convex hull of the data points, and established a convergence rate for their estimator. In their setting, model identifiability is equivalent to the uniqueness of the minimizer in the noiseless setting; that is, they assume that a unique set of $k$ columns (of $\mathbf{C}$) is closest to the convex hull formed by the columns of $\mathbf{U}$. They show that the minimizer is indeed unique when the separability condition is imposed on $\mathbf{W}$; other than that, they do not provide any checkable conditions for identifiability.

\subsection{Summary of Our Contribution}

First, we resolve the non-identifiability issue by focusing on convex hulls (of $\mathbf C$) of the smallest volume, and show that under volume minimization, the SS condition ensures identifiability regardless of the values of $V$ and $k$ (Section~\ref{sec:idfy}). 

Although volume minimization helps to ensure model identifiability, since the volume of a low-dimensional simplex in a high-dimensional space does not take a simple form \citep{miao2007endmember}, it is difficult to  incorporate volume minimization into an estimation procedure. This difficulty explains why many prior investigations have either assumed $V=k$ or used an approximation formula.

Our second contribution is to establish
the connection between volume minimization and maximization of a particular integrated likelihood (Section~\ref{sec:MLE_VM}). Specifically, we propose an estimator as the MLE of the topic matrix $\mathbf C$, based on an integrated likelihood, in which the mixing weights (i.e., columns of $\mathbf W$) are profiled out by integrating with respect to a uniform distribution over $(k-1)$-simplex. 
A geometric consequence of the use of uniform distribution is that, while maximizing the integrated likelihood, we implicitly minimize the volume of the convex hull of $\mathbf{C}$ without explicitly evaluating its volume. Here we emphasize that the uniform distribution is  used only to integrate over nuisance parameters (i.e., the mixing weights), and that our theoretical analysis does not require the mixing weights to be  generated stochastically from a uniform distribution.

Our third contribution is to  establish a finite-sample error bound of the proposed estimator of $\mathbf{C}$, of the order $\sqrt{\log (n\vee d)/n}$ under the fixed design setting where the mixing weights $\mathbf W$ can be arbitrarily allocated---as long as the SS condition pertains (Section~\ref{subsec:fixed}). As a consequence, our result implies asymptotic consistency as the number of documents $d$ and/or the number of words $n$ (in each document) increases to infinity. 
In the stochastic setting, where the mixing weights $\mathbf W$ are independently generated according to some unknown underlying distribution over the simplex, we show that, for sufficiently large $d$, $\mathbf W$ still satisfies a perturbed version of the SS condition with high probability---as long as the support of the weight generating distribution satisfies the SS condition. Based on this observation, we also provide a finite-sample error bound in the stochastic (or random design) setting (Section~\ref{Sec:random_design_error} in the supplementary material). Furthermore, by drawing a connection between our estimating approach and some representative existing methods, through a two-stage perspective (Section~\ref{subsec:comparison}), we illustrate that the separability condition greatly simplifies the topic matrix estimation problem by circumventing the highly nontrivial and non-regular statistical problem of boundary estimation (Section~\ref{section:theory_comparsion}). This explains why our finite-sample error bound is similar to that of \cite{javadi2020nonnegative} which is based on an archetypal analysis that, like ours, does not assume the separability condition; however, our error bound is (not surprisingly) worse (in terms of the dependence on $d$) than those \citep{ke2017new,arora2012learning} arrived at under the separability condition.

As a byproduct, our work provides a theoretical justification for the empirical success of Latent Dirichlet Allocation (LDA) \citep{blei2003latent} models, since the proposed estimator is essentially the maximum likelihood estimator of $\mathbf{C}$ from the LDA model, with a particular choice of prior on $\mathbf W$. More generally, the LDA model with other prior choices on $\mathbf W$ can be interpreted as maximizing the data likelihood while minimizing a weighted volume in which a non-uniform volume element is integrated over the convex hull of $\mathbf C$ when defining the volume (see Section~\ref{subsub:mis_prior} for some numerical comparisons).

Although presented in the context of topic modeling, our results can be adapted to many other applications by using the data-specific likelihood. For example, the decomposition $\bU = \bC  \bW$ plays an important role in \emph{hyperspectral imaging analysis}, in which each column of $\bU$ represents the  intensity levels over $V$ channels at a pixel. Due to the low spatial resolution of hyperspectral images, pixel spectra are usually mixtures of spectra from several pure materials, known as endmembers. So a key step in hyperspectral imaging analysis is to separate (or unmix) the pixel spectra into  convex combinations of endmember spectra; endmember spectra are essentially columns of $\mathbf{C}$ \citep{winter1999n}. Similar models also arise in \emph{reinforcement learning} \citep{singh1995reinforcement,duan2019state} as a way to compress the transition matrix of an underlying Markov decision process; a detailed discussion is given in Section \ref{subsubsec:real_taxi}. 

\subsection{Notation and Organization}
Let $\mathbf 1_k$ denote the all-ones vector of length $k$, and $\mathbf e_f$ the $f$-th column of the $k\times k$ identity matrix $\mathbf I_k$. 
Let $\simplex^{k-1} = \{\mathbf x\in \mathds{R}^{k}: 0\leq x_i \leq 1, \sum_{i=1}^k x_i = 1\}$ denote the $(k-1)$-dimensional probability simplex. 
For a matrix $\mathbf{A}_{p \times q} = (\mathbf A_1, \cdots, \mathbf A_q)$, 
let
\begin{eqnarray*}
\conv(\mathbf A) &=& \{\mathbf x\in \mathds{R}^{p}:  \mathbf{x} =  \mathbf A \bm\lambda, \bm\lambda  \in \simplex^{q-1} \}, \\
cone(\mathbf A)  &=& \{\mathbf x\in \mathds{R}^{p}: \mathbf x = \mathbf A \bm \lambda, \bm \lambda \geq 0 \},\\
\mbox{and}\quad \text{aff}(\mathbf A) &= & \{\mathbf x\in \mathds{R}^{p}:  \mathbf{x} =  \mathbf A \bm\lambda, \bm\lambda^T \mathbf 1_q = 1, \bm\lambda\in \mathds{R}^q\},
\end{eqnarray*}
denote the \emph{convex polytope},  \emph{simplicial cone} and  \emph{affine space} generated by (the $q$ columns of) $\mathbf A$, respectively. 
{\colorblue For $\mathbf A\in \mathds{R}^{p\times q} (p\geq q)$, we define $|\conv(\mathbf A)|$ as the $(q-1)$-dimensional volume of $\conv(\mathbf A)$ on $\text{aff}(\mathbf A)$, which can be computed by the Cayley–Menger determinant or Lemma~\ref{lem:vol_det} in Appendix \ref{app:pf_thm_1}.} 
For any vector $\mathbf x$, $\mathbf x \geq a$ means $\mathbf x$ is element-wisely greater than or equal to $a$. 
Denote $a\vee b$ and $a\wedge b$ as the larger and smaller number between $a$ and $b$, respectively.
For any cone $\mathcal C$, let $\mathcal C^\ast = \{\mathbf x: \mathbf x^T\mathbf y \geq 0, \forall \mathbf y \in \mathcal C\}$ denote its \emph{dual cone}. Recall some useful facts of dual cones \citep{donoho2004does}: (i) $cone (\mathbf A)^\ast = \{\mathbf x \in \mathds{R}^p: \mathbf x^T \mathbf A \geq 0\}$; (ii) if $\mathcal A$ and $\mathcal {\bar A}$ are convex cones, and $\mathcal A \subseteq \mathcal {\bar A}$, then $\mathcal {\bar A}^\ast \subseteq \mathcal A^\ast$. 
Unless stated otherwise, all the constants in the paper are independent of number of words per document $n$ and number of documents $d$.

 The rest of the paper is organized as follows. 
 In Section \ref{sec:idfy}, we discuss identifiability under volume minimization as well as a set of sufficient conditions. In Section \ref{sec:MLE}, we propose the MLE based on an integrated likelihood, establish its connection with volume minimization, and describe its computation. Theoretical analysis of the proposed estimator is presented in Section \ref{sec:est}.  Finally, empirical evidence is reported in Section \ref{sec:emp}. Proofs and technical results are included in the supplementary material.

\section{Identifiability of Topic Models}\label{sec:idfy}
In this section we start with a formal definition of topic model identifiability under the minimum volume constraint. After that, we  describe two sufficient conditions that lead to the identifiability, namely the separability condition and the sufficiently scattered condition. Finally, for the latter condition, which is weaker and less stringent than conventional separability, we provide a geometric interpretation.

\subsection{Identifiability under Volume Minimization}
We have observed (see Figure \ref{fig:vol1}) that without any constraint, a topic model is almost always non-identifiable. We thus focus on identifiability under the minimum volume volume minimization constraint, due to its natural interpretation as finding the most parsimonious topic model that explains the documents in the corpus data, or equivalently, the most compact $k$-vertex convex polytope in which the documents reside. 

We begin by defining the following distance metric between two topic matrices $\mathbf C$ and $\mathbf {\bar C}$:
\begin{equation} \label{eq:distance_W}
\dis(\mathbf C , \mathbf {\bar C}) = \min_{\mathbf \Pi}\|\mathbf {\bar C} -\mathbf C \mathbf \Pi\|_2,
\end{equation}
where $\|\cdot \|_2$ denotes the spectral norm and $\mathbf \Pi$ is a permutation matrix. Note that $\dis(\mathbf C, \mathbf {\bar C}) = 0$ if and only if $\mathbf {\bar C} = \mathbf C\mathbf \Pi$, that is, $\mathbf C$ and $\mathbf {\bar C}$ are identical up to a permutation of columns. Since $k$ and $V$ are fixed, the spectral norm in (\ref{eq:distance_W}) is not important because all matrix norms are equivalent. In particular, if the Frobenius norm is employed instead of the spectral norm, then the distance metric $\dis$ coincides with the $2$-Wasserstein distance between column vectors of $\mathbf{C}$ and $\mathbf{\bar C}$.

Next, we state the definition of identifiability under the minimum volume constraint:

\begin{definition}[Identifiability] \label{def:identify}
A topic model associated with parameters $(\mathbf C, \mathbf W)$ is identifiable, if for any other set of parameters $(\mathbf {\bar C}, \mathbf {\bar W})$, the following conditions hold,
\begin{align} \label{eq:identify}
    \mathbf C \mathbf W = \mathbf {\bar C} \mathbf {\bar W} \,  \text{ and } \, |\conv(\mathbf {\bar C})| \leq |\conv(\mathbf C)|,
\end{align}
 if and only if $\dis (\mathbf C, \mathbf {\bar C}) = 0$. 
\end{definition}

It is easy to check that model identifiability is achieved under the separability condition on columns of $\mathbf{W}$, as it  implies that $\mathbf{W}$ contains a $k\times k$ identity matrix after a proper column permutation; that is, there exist $k$ columns in $\mathbf{U}$ that are the $k$ corners of $\conv(\mathbf C)$. Therefore, no other $k$-vertex convex polytope of smaller or equal volume can still enclose all columns in $\mathbf{U}$.

\begin{proposition}
If the separability condition is satisfied on $\mathbf{W}$, then $(\mathbf{C}, \mathbf{W})$ is identifiable. 
\end{proposition}

Since the separability condition can be overly stringent in practice, 
we next show that a condition weaker than the separability condition can also achieve model identifiability. Our analysis is related to the following geometric condition, known as  \emph{sufficiently scattered} (SS). Its definition relies on the second order cone $\mathcal{K}$,  its boundary $bd\mathcal{K}$, and its dual cone $\mathcal{K}^\ast$, which are defined below:
\begin{eqnarray*}
\mathcal{K} &=&  \{\mathbf x \in \mathds{R}^k: \|\mathbf x\|_2 \leq \mathbf x^T \mathbf 1_k\}, \\
bd\mathcal{K} &=& \{\mathbf x \in \mathds{R}^k: \|\mathbf x\|_2 = \mathbf x^T \mathbf 1_k\}, \\
\mbox{and}\quad \mathcal{K}^\ast  &=& \{\mathbf x \in \mathds{R}^k: \mathbf x^T \mathbf 1_k \geq \sqrt{k-1}\|\mathbf x\|_2\}.
\end{eqnarray*}


\begin{definition}[SS Condition] \label{def:W_suff_scat}
A matrix $\mathbf W$ is {\bf sufficiently scattered}, if it satisfies:
  \begin{enumerate}
      \item[(S1).]\label{con:contain_circle} $cone(\mathbf W)^\ast \subseteq \mathcal{K}$, or equivalently, $cone(\mathbf W) \supseteq \mathcal{K}^\ast$;
      \item[(S2).]  $cone(\mathbf W)^{\ast} \bigcap bd \mathcal{K} \subseteq \{\lambda \mathbf e_f,  f=1,\cdots, k, \lambda \geq 0\}$. 
  \end{enumerate} 
  \end{definition}

It is easy to verify that the separability condition on $\mathbf W$ implies $\mathbf W$ to be sufficiently scattered. In fact, the separability condition on $\mathbf W$ means that $\conv(\mathbf{W})= \simplex^{k-1}$ fills up the entire simplex, and that $cone(\mathbf W)^\ast = cone(\simplex^{k-1})$ is the most extreme cone (smallest possible cone, corresponding to the solid triangle in Figure \ref{fig:two_suff_scatt}; see the following section for details) that satisfies (S1) - (S2) in the SS condition.  

\begin{theorem}\label{thm:idf}
 If $\mathbf W$ is sufficiently scattered and $\mathbf C$ is of rank $k$ (full column rank), then $(\mathbf C, \mathbf W)$ is identifiable.
\end{theorem}

Proof of Theorem \ref{thm:idf} is given in the supplementary material (Section~\ref{supp:proof_thm_idf}). Here we give a sketch of the proof. Suppose $\mathbf C \mathbf W = \mathbf {\bar C} \mathbf {\bar W}$. We have $\mathbf C  = \mathbf {\bar C} \mathbf B$, where ${\mathbf B} = \mathbf {\bar W}\mathbf W^{T}(\mathbf W \mathbf W^{ T})^{-1}$. It suffices to show $\mathbf B $ is a permutation matrix, which we prove by verifying that any row of $\mathbf B$ is in $cone(\mathbf W)^\ast\bigcap bd\mathcal{K}  = \{\lambda \mathbf e_f, \lambda \geq 0\} $ and is also of unit length.

{\colorblue
\begin{remark}[Comparison with definition in \citet{javadi2020nonnegative}] 
\label{rmk:compare_javadi}
\normalfont
The model identifiability defined in \citet{javadi2020nonnegative} is different from ours. 
They define a model to be identifiable if there is a unique convex polytope that minimizes the sum of distances from vertices of $\conv(\mathbf C)$ (i.e., columns of $\mathbf C$) to the convex hull of $\bU$. Their notion of identifiability is easier than ours to be formulated into a statistical estimator that minimizes an empirical evaluation of the distance sum from data. In our approach, the volume of our low-dimensional polytope does not take a simple form, which greatly complicates the estimator construction. Fortunately, we find that maximizing a particular integrated likelihood leads to an estimator that implicitly minimizes the volume. (See Appendix \ref{Sec:example_idf} for further discussion of this topic.)
\end{remark}
}

{\colorblue
\begin{remark}[SS condition is not a necessary condition]
\normalfont
The SS condition is not necessary for identifiability --- one reason is that it does not take into account additional parameter constraints  (e.g., in the topic model, each column of topic matrix $\mathbf C$ should be a probability weight vector belonging to the simplex). See Figure \ref{fig:sm_vol3} for an example ($V=k=3$ and $\mathbf{C} = \mathbf{I}_3$) where the SS condition does not hold but the model is identifiable. Since any alternative topic matrix $\mathbf C$ as a convex polytope with three vertices must be inside $\simplex^{2},$ due to the parameter constraint, $\mathbf{I}_3$ is the only topic matrix enclosing all columns of $\bU$ and is within simplex $\simplex^{2}$.
However, the SS condition does not hold since, apparently, $cone(\mathbf W) \supseteq \mathcal{K}^\ast$ is not true.
\end{remark}
}

\subsection{Geometrical Interpretation of Sufficiently Scattered Condition}

We provide a geometric interpretation of the SS condition in Figure \ref{fig:two_suff_scatt} with $k=3$. Since the mixing weights are all on $\Delta^{2}$, 
what is shown in Figure \ref{fig:two_suff_scatt} is the intersection of the cones with the hyperplane $\mathbf x^T \mathbf 1_{3} = 1$. 
The mixing weights, $\mathbf{w}_1, \dots, \mathbf{w}_d$, are represented as blue dots. Other items related to Definition \ref{def:W_suff_scat} are:
$bd \mathcal{K}$ is the red circle, $\mathcal{K}^*$ is the dark brown ball inscribed in the triangle, and $cone(\mathbf W)^{\ast}$ is the yellow convex region with dashed boundary. 

We illustrate three different scenarios: ``SS'' means that the SS condition is satisfied, ``not SS'' means that the SS condition is violated, and ``sub-SS'' means that (S1) is satisfied but  (S2) is not.

An equivalent form of Condition (S1) is $cone(\mathbf W) \supseteq \mathcal{K}^\ast$. So (S1) has a simple and intuitive interpretation: the mixing weights (blue dots) should form a convex polytope that contains the dual cone $\mathcal{K}^\ast$, the inner ball inscribed in the triangle. See Figure \ref{fig:suff_scatter_non1} for a violation of (S1). In particular, the separability condition on $\mathbf W$ implies that the three vertices (blue circles) of the triangle are included in $\mathbf W$. As a consequence, $cone(\mathbf W)^\ast=cone(\mathbf W)$ is the entire triangle, which is the most extreme/superfluous instance that satisfies the SS condition.

Condition (S1) ensures that $\conv(\mathbf C)$ has the smallest possible volume, but such minimum volume convex polytopes may not be unique. The purpose of condition (S2) is to determine the ``orientation'' of the convex polytope and consequently to ensure that it is unique. 
When (S2) is violated, it is possible to rotate the convex polytope to produce different feasible convex polytopes of the same volume; see Figure \ref{fig:suff_scatter_non2}\subref{fig:suff_scatter_non3}. 

  \begin{figure}
    \begin{subfigure}[t]{0.3\textwidth}
        \includegraphics[width=0.95\textwidth]{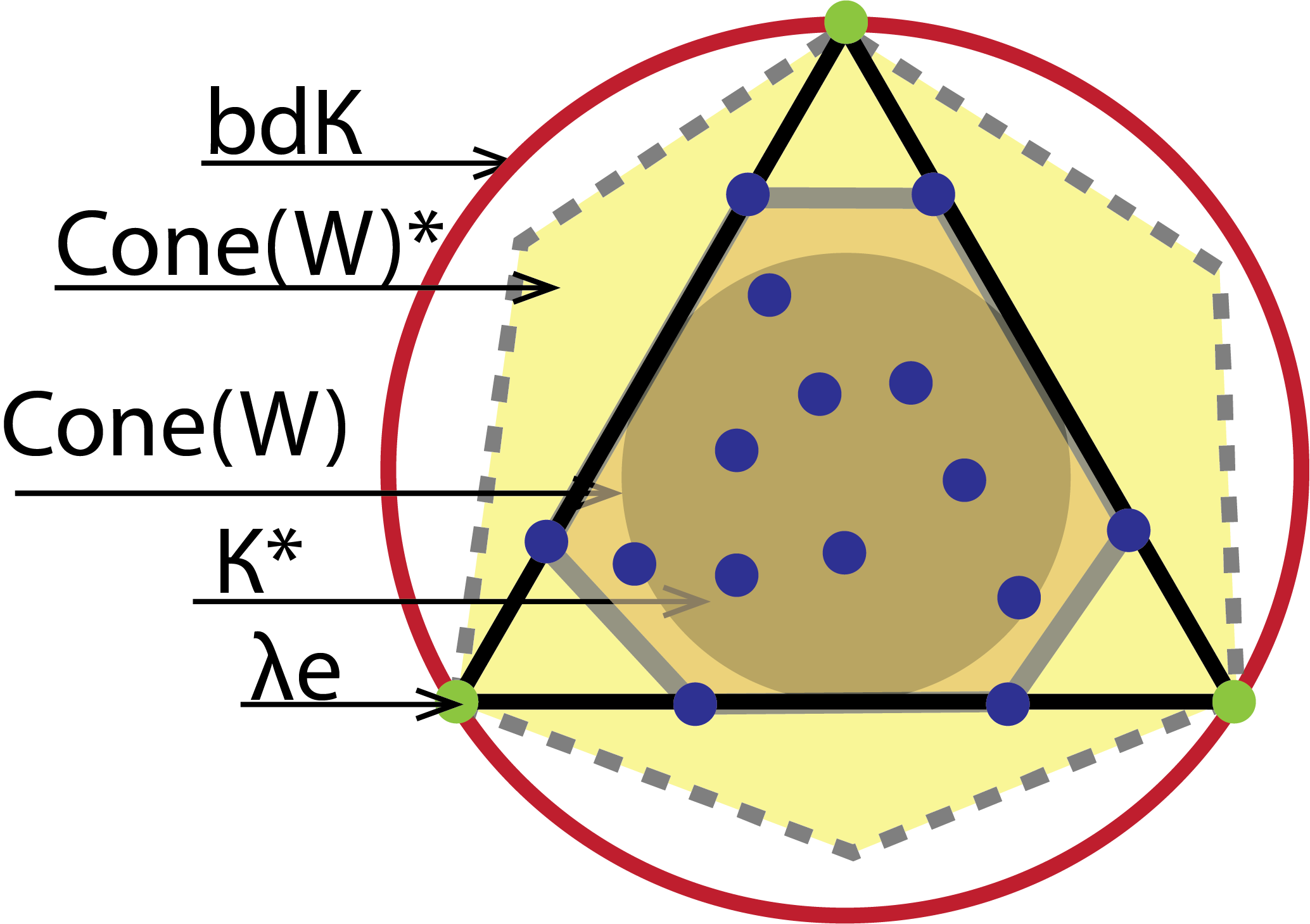}%
        \subcaption{SS}%
           \label{fig:suff_scatt} %
    \end{subfigure}%
   \begin{subfigure}[t]{0.25\textwidth}
      \centering
      \includegraphics[width=0.9\textwidth]{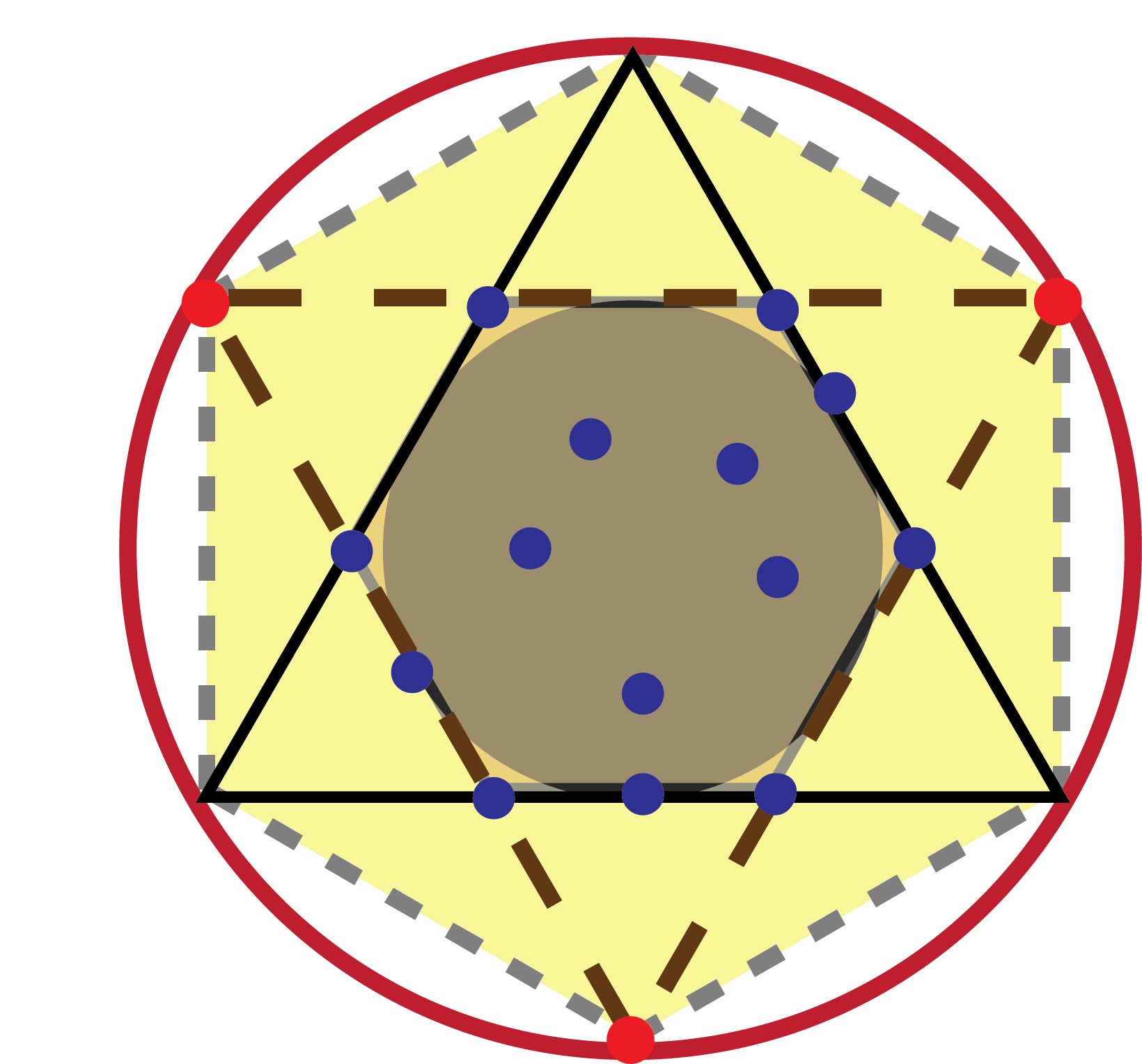}%
         \subcaption{sub-SS}%
      \label{fig:suff_scatter_non2}%
   \end{subfigure}
   \begin{subfigure}[t]{0.25\textwidth}
         \centering
         \includegraphics[width=0.9\textwidth]{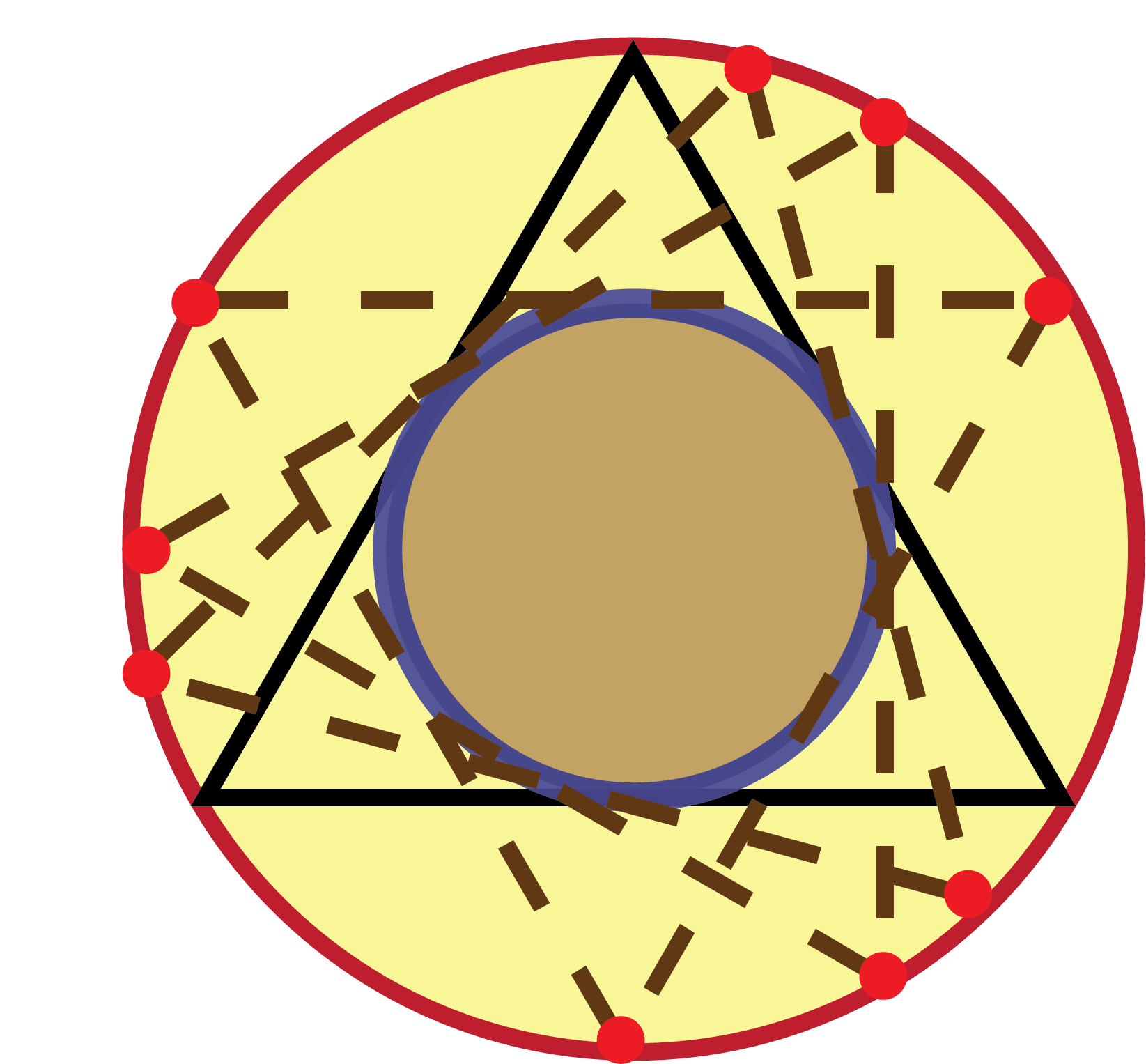}
         \subcaption{\centering sub-SS
}%
\label{fig:suff_scatter_non3}  %
   \end{subfigure}%
     \begin{subfigure}[t]{0.23\textwidth}
      \includegraphics[width=0.85\textwidth]{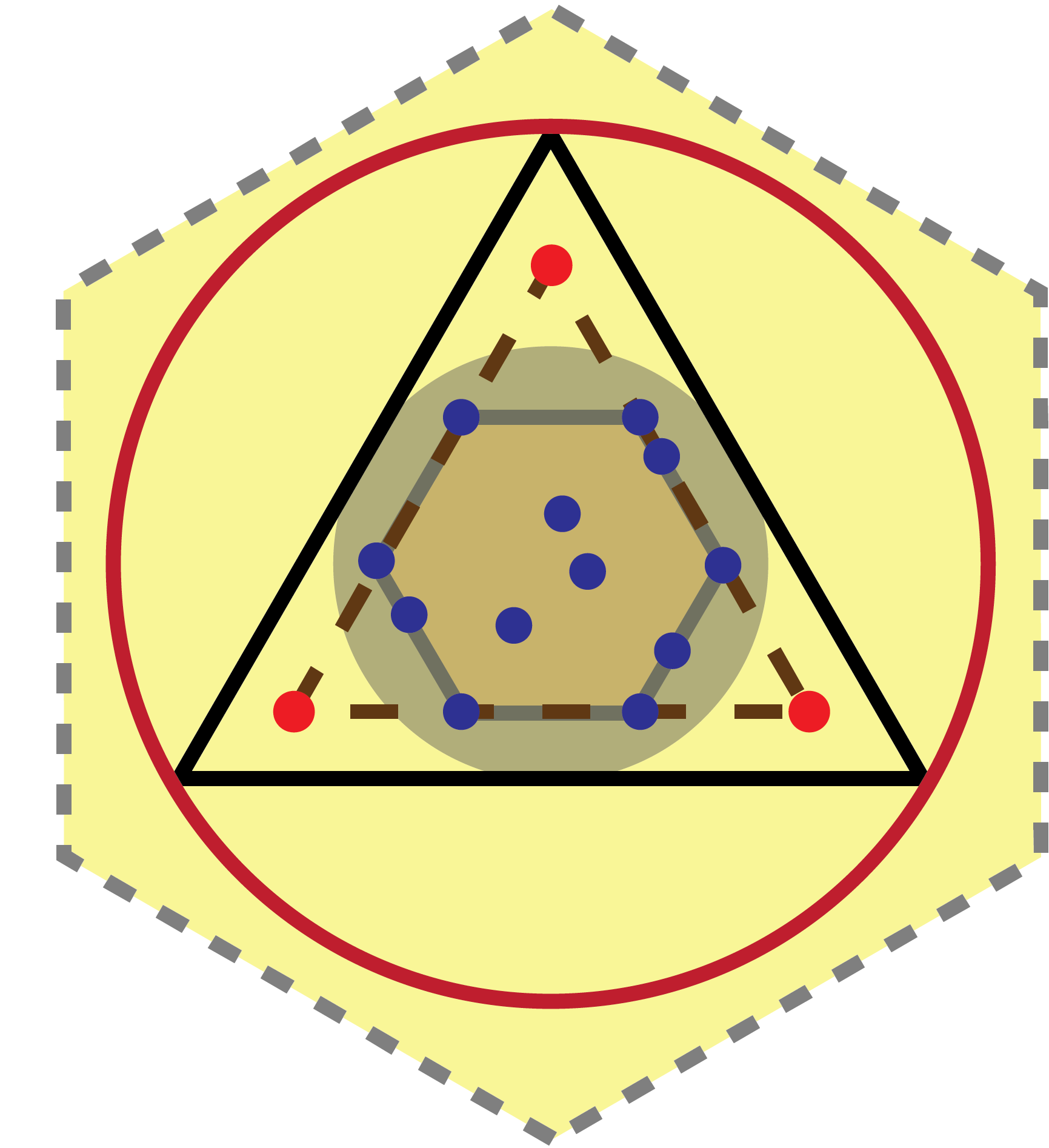}
      \subcaption{not SS}%
      \label{fig:suff_scatter_non1}%
   \end{subfigure}%
    \caption{Geometric views of the SS condition shown on the hyperplane $\mathbf x^T \mathbf 1_k = 1$ ($k=3$). Mixing weights $\mathbf w$ are represented as blue dots; blue dots in (c) are all on the boundary of the inner circle. Any dashed triangle in (b)(c)(d) is an alternative $3$-vertex convex polytope that contains all $\mathbf w$'s and is of a volume no larger than $\simplex^{k-1}$.}
    \label{fig:two_suff_scatt}%
\end{figure}

The SS condition was first introduced by \citet{huang2016anchor} to study the identifiability of topic models, where identifiability is ensured under the SS condition along with a minimal determinant on $\bW \bW^T$. This condition is used differently in their work and ours: \citet{huang2016anchor} impose the  SS condition on  rows of $\mathbf C$; we impose this condition on columns of $\mathbf{W}$. Although volume is not discussed  in \citet{huang2016anchor}, imposing the SS condition on rows of $\mathbf C$ in fact leads to a convex polytope of maximum volume; in contrast, we seek a convex polytope of the smallest volume. 

{\colorblue
\begin{remark}[Algorithm for checking SS condition]
\label{rmk:check_SS}
\normalfont Checking $cone(\mathbf W) \supseteq \mathcal{K}^\ast$ in the SS condition is equivalent to verifying whether a convex polytope contains a ball (after being projected to $\simplex^{k-1}$), which is in general an NP-complete problem in computational geometry \citep{freund1985complexity, huang2014non}. Consequently, it can be computationally difficult to provide a definitive conclusion as to whether or not the SS condition holds in high dimensions. However, if making a small probability mistake is allowed, then we propose that the following randomized algorithm to check the SS condition will give the correct answer with acceptable high probability. Since it suffices to verify that $\conv(\mathbf W) \supseteq bd \mathcal K^\ast \bigcap \simplex^{k-1}$, we can independently choose $M$ sample points uniformly from $bd \mathcal K^\ast \bigcap \simplex^{k-1}$ and check whether all of them are in $\conv(\mathbf W)$. 
If $\mathbf{W}$ satisfies the SS condition, then the $M$ sampled points should belong to $\conv(\mathbf W)$; if $\mathbf{W}$ does not satisfy the SS condition, then, since the probability of each sampled point falling in $\conv(\mathbf W)$ is a fixed number, the probability of making a mistake decays exponentially in $M$. For real datasets where $\conv(\mathbf W)$ is not observed, we can use an estimator of it to empirically check the SS condition by reporting the frequency of sampled points not falling into the estimated $\conv(\mathbf W)$.
\end{remark}
}



\section{Maximum Integrated Likelihood Estimation}\label{sec:MLE}

 Before introducing the proposed estimator for topic matrix $\mathbf C$, let us describe some more notations and the data generating process.
 Let $\mathbf X= (\mathbf x^{(1)},\cdots, \mathbf x^{(d)})$ denote the observed data as a collection of word sequences.   Without loss of generality, we assume each document has the same number of words,  denoted by $n$. Given parameters $(\mathbf C, \mathbf W)$, word sequences from different documents are independent, with the word sequence from the $i$-th document, $\mathbf x^{(i)} = (x_{i,1}, \dots, x_{i, n})$, being $n$ i.i.d.~samples from the categorical distribution Cat$(\mathbf u_i)$, where $\mathbf u_i = \mathbf C\mathbf w_i$ is the $V$-dimensional probability vector in $\simplex^{V-1}$, and $\mathbf w_i=(w_{i,1},\dots,w_{i,k})$ denotes the $i$-th column of matrix $\mathbf W$. We use $f_n(\cdot \mid \mathbf u_i)$ to denote the multinomial likelihood function of the $i$-th document.
 Let $\mathbf c_j$ denote the $j$-th topic vector, i.e.,~the $j$-th column of matrix $\mathbf C$, for $j=1,2,\ldots,k$.
 Under this notation, we can express the word frequency vector $\mathbf u_i = \sum_{j=1}^kw_{i,j}\mathbf c_j$ associated with the $i$-th document as a convex combination of the topic vectors, where $\mathbf w_i$ serves as the mixing weight vector.

\subsection{Implicit Volume Minimization}\label{sec:MLE_VM}
Since our primary interest is on the topic matrix $\mathbf C$, we can profile out the nuisance parameters $\mathbf w_i$'s by integrating them with respect to some distribution, resulting an integrated likelihood function of $\mathbf C$. After that, we can estimate $\mathbf C$ by maximizing the integrated likelihood \citep{berger1999integrated}. We propose to integrate out $\mathbf w_i$'s with respect to the \emph{uniform distribution} over simplex $\simplex^{k-1}$, which induces a uniform distribution on $\mathbf u_i = \mathbf C\mathbf w_i$ over $\conv(\mathbf C)$. This is because the linear transformation $\mathbf w\mapsto\mathbf C \mathbf w$ has a constant Jacobian. The integrated likelihood can be formally written as follows:
\begin{align}
    F_{n\times d}(\mathbf C;\mathbf X) &
     = \prod_{i=1}^d  \int_{\conv(\mathbf C)}  \frac{f_n(\mathbf x^{(i)}\,|\, \mathbf u)}{|\conv(\mathbf C)|} \,d \mathbf u, \label{eq:lkh_fn}
\end{align}
where $|\conv(\mathbf C)|$ denotes the $(k-1)$-dimensional volume of the set $\conv(\mathbf C)$. The corresponding maximum likelihood estimator (MLE) is defined to be 
\begin{equation} \label{eq:mle}
\mathbf{\hat C}_n = \underset{\mathbf C}{\argmax}   F_{n\times d}(\mathbf C;\mathbf X),
\end{equation}
where the maximum is over all $V$-by-$k$ column-stochastic matrices.

Although the integrated likelihood (\ref{eq:lkh_fn}) is equivalent to the marginal likelihood from an LDA model after integrating out the mixing weight $\mathbf{w}$ with respect to a $\text{Dirichlet}(\mathbf{1}_k)$ prior, we emphasize again that the uniform prior is just used to profile out the nuisance parameters so that we can derive an MLE for the topic matrix. In our theoretical analysis below, we do not assume data to be generated from the LDA model with a uniform prior on $\mathbf w$. 

\emph{Why uniform distribution?}  To understand the motivation behind the use of a uniform distribution in  (\ref{eq:lkh_fn}), let us consider the noiseless case (corresponding to the limiting case as $n\to\infty$), in which  we ``observe'' the true word-frequency vectors for the $d$ documents: $\tu{1}, \cdots, \tu{d}$. In this ideal setting, from a standard Laplace approximation argument, the $i$-th integral inside the product in~\eqref{eq:lkh_fn} after rescaling by a factor of order $n^{(V-1)/2}$ converges to $\idfn (\tu{i}  \in \conv(\mathbf C))$, and the MLE $\hat{\mathbf C}$ becomes:
\begin{align}\label{eq:mle_uniform}
    &  \underset{\mathbf C}{\argmax} \ \prod_{i=1}^d \frac{\idfn (\tu{i}  \in \conv(\mathbf C))}{|\conv(\mathbf C)|} =  \underset{\mathbf C}{\argmax} \ \frac{ \idfn (\tu{1},\cdots,\tu{d} \in \conv(\mathbf C))}{|\conv(\mathbf C)|} ,
\end{align}
 where $\mathds 1 (\cdot)$ is the indicator function. Therefore, maximizing the integrated likelihood function (\ref{eq:lkh_fn}) is asymptotically equivalent to minimizing the volume of $\conv(\mathbf C)$ subject to the constraint that $\conv(\mathbf C)$ contains all true word-frequency vectors.
 
 In the rest of this section we first provide an alternative interpretation of our approach as a two-stage estimation procedure. We compare it with some  representative topic learning methods designed under the separability condition that can also be cast as two-stage procedures. After that, we describe an MCMC-EM algorithm designed for implementing the optimization problem of maximizing the integrated likelihood. 
 



 \begin{figure}[h]
    \begin{subfigure}[t]{0.5\textwidth}
      \centering
      \includegraphics[trim=0 2em 0 0, clip, width=0.95\linewidth]{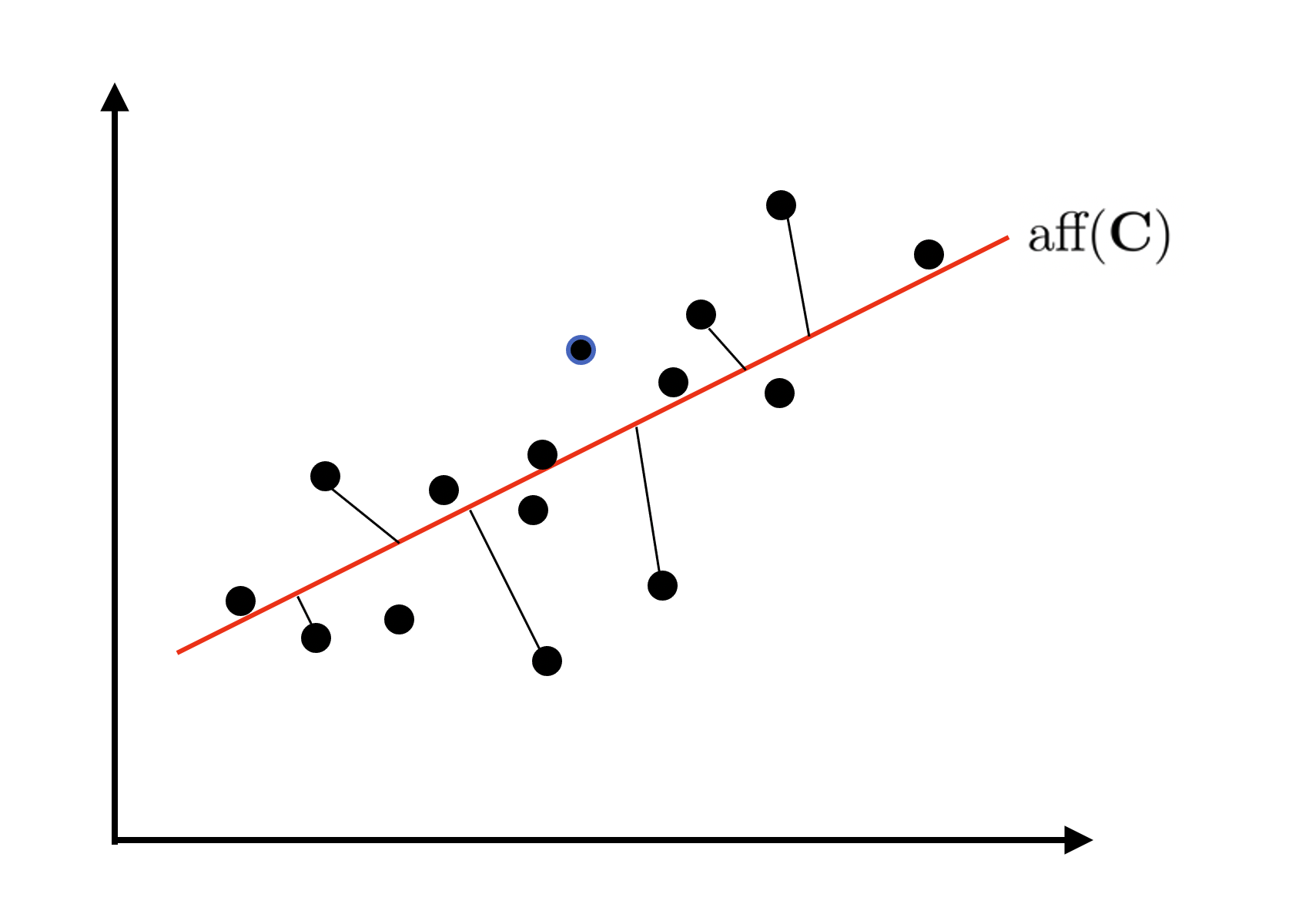}
         \subcaption{First stage}%
      \label{fig:Step1}%
   \end{subfigure}
   \begin{subfigure}[t]{0.5\textwidth}
         \centering
         \includegraphics[trim=0 -1em 0 0, clip, width=1.08\linewidth]{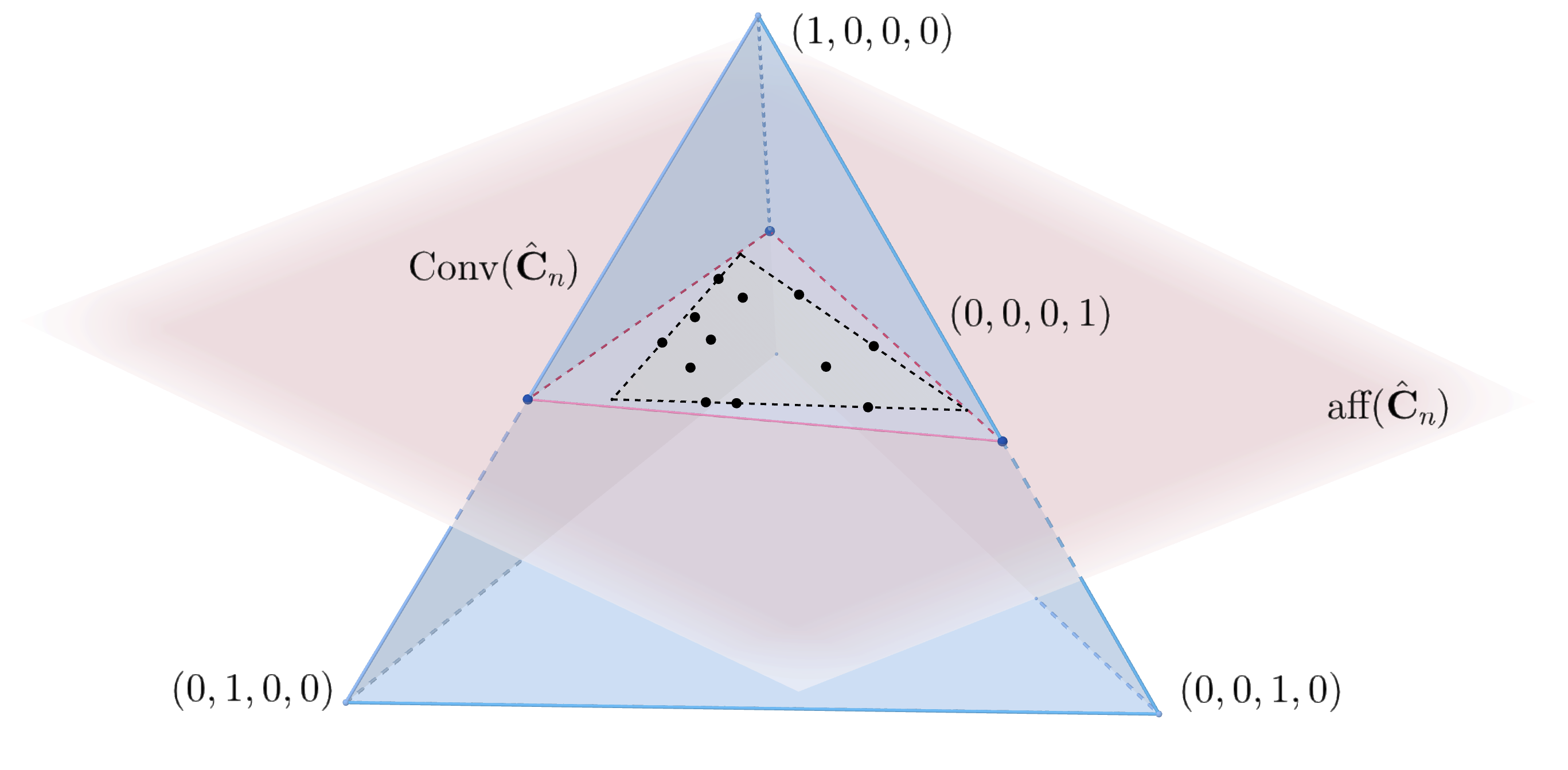}
      \subcaption{Second stage}%
      \label{fig:Intersection}%
   \end{subfigure}%
    \caption{Illustration of the two-stage perspective of maximizing the integrated likelihood~\eqref{eq:lkh_fn}. The left figure illustrates the first stage when $V=k=2$. Black dots are the sample word frequency vectors $\mathbf {\hat u}^{(i)}$'s. $\text{aff}(\mathbf C)$ is the red line. We target to minimize the sum of the squared distances, where each distance is induced from its own local norm $\|\cdot\|_{i}$ (see main text for details about the norm). That is why the black lines correspond to the projection directions are not necessarily parallel to each other. The right figure illustrates of the second stage when $V=4, k=3$. The blue tetrahedron is the simplex $\simplex^{V-1}$ and the red hyperplane is the estimated $\text{aff}(\hatCn)$ from the first stage. 
The black dots are the projections of the sample word frequency vectors $\mathbf {\hat u}^{(i)}$'s on $\text{aff}(\hatCn)$. The black dashed triangle is our estimator $\hatCn$, whose convex hull is roughly the $3$-vertex convex polytope that encloses all the black dots and has the minimal volume.}
    \label{fig:two_stage}%
\end{figure}


\subsection{Interpretation as Two-Stage Optimization}\label{subsec:comparison}

Our method of estimating $\mathbf{C}$ can be viewed as a two-stage procedure: in the first stage, we estimate the $(k-1)$-dimensional hyperplane $\text{aff}(\mathbf{C})$ in which the convex polytope of $\mathbf{C}$ lies; then in the second stage, we determine the boundary of $\text{Conv}(\mathbf{C})$ by estimating its $k$ vertices within the estimated hyperplane obtained in the first stage. See Figure \ref{fig:two_stage} for an illustration, and the following for a heuristic derivation.

It is worth mentioning that many recent separability condition based topic modeling methods in the literature (such as~\cite{arora2012learning,azar2001spectral,kleinberg2008using,kleinberg2003convergent,ke2017new,papadimitriou2000latent,mcsherry2001spectral,anandkumar2012method}) can be explained under this general two-stage framework. For example, some papers
\citep{azar2001spectral,kleinberg2008using,kleinberg2003convergent} aim only at recovering the column span of topic matrix $\mathbf C$ using singular value decomposition (SVD), which suffices for their applications. This corresponds to solving the hyperplane estimation problem in our first stage. Some papers~\citep{arora2012learning,papadimitriou2000latent,mcsherry2001spectral,anandkumar2012method} directly search for a subset of words (separability condition on anchor words,~\cite{arora2012learning}) or documents (separability condition on pure topic documents,~\cite{papadimitriou2000latent,mcsherry2001spectral,anandkumar2012method}) in their first stage, and then in their second stage recover the population-level term-document matrix (or the hyperplane $\text{aff}(\mathbf{C})$) based on the estimated anchor words/pure topic documents. This corresponds to our two-stage procedure, in reverse order. Others such as~\cite{ke2017new} also use a two-stage procedure based, first, on projecting a certain transformation of the sample term-document matrix onto a lower-dimensional hyperplane via SVD, and then searching for the anchor words over that hyperplane. Notice that all aforementioned methods reply crucially  on the separability condition, which greatly simplifies the statistical and computational hardness of the problem and turns it into a searching problem; thus they are able to circumvent the hidden non-regular statistical problem of boundary estimation (c.f.~Section~\ref{section:theory_comparsion}).


To illustrate the two-stage interpretation of our method, we observe that the integrated likelihood \eqref{eq:lkh_fn} is equivalent to the following expression:
\begin{equation} \label{eq:KL:two-stage}
\frac{1}{|\conv(\mathbf{ C})|^d} \prod_{i=1}^d \int_{\conv(\mathbf{C})}  
\exp \big\{-n\, D_{\rm KL}(\mathbf{\hat{u}}^{(i)}\, ||\, \mathbf u)\big\}\,
d\mathbf u,
\end{equation}
where $\mathbf {\hat u}^{(i)}$ denotes the sample word frequency vector for document $i$. Here, we use  $D_{\rm KL}(\mathbf p\,||\,\mathbf q)=\sum_{v=1}^V p_v\log(p_v/q_v)$ to denote the Kullback-Leibler divergence between two categorical distributions with parameters $\mathbf p=(p_1,\ldots,p_V)$ and $\mathbf q=(q_1,\ldots,q_V)$.
When $n$ is large, the classical Laplace approximation to the integral in~\eqref{eq:KL:two-stage} uses a nonnegative quadratic form $\|\mathbf u-\mathbf{\hat{u}}^{(i)}\|_{i}^2:\,=(\mathbf u-\mathbf{\hat{u}}^{(i)})^T\mathbf H_i(\mathbf u-\mathbf{\hat{u}}^{(i)})$ to approximate the exponent $D_{\rm KL}(\mathbf{\hat{u}}^{(i)}\,||\,\mathbf u)$ in a local neighborhood of $\mathbf{\hat{u}}^{(i)}$. Since such a quadratic form defines the norm $\|\cdot\|_{i}$, we can decompose it into $\|\mathbf u-\mathbf{\hat{u}}^{(i)}\|_{i}^2 = \|\mathbf u-\mathbb P_{\mathbf C}^{(i)} \,\mathbf{\hat{u}}^{(i)}\|_{i}^2+ \|\big(\mathbf I_{V}-\mathbb P_{\mathbf C}^{(i)}\big)\,\mathbf{\hat{u}}^{(i)}\|_{i}^2$, where $\mathbb P_{\mathbf C}^{(i)}$ denotes the projection operator onto the $(k-1)$-dimensional hyperplane $\text{aff}(\mathbf C)$ with respect to the distance induced from $\|\cdot\|_{i}$. Finally, we can  approximate the integrated likelihood in the preceding display as
\begin{align}\label{Eqn:likelihood_approx}
\exp\Big\{-n \LaTeXunderbrace{\sum_{i=1}^d \|\big(\mathbf I_{V}-\mathbb P_{\mathbf C}^{(i)}\big)\,\mathbf{\hat{u}}^{(i)}\|_{i}^2}_{\text{residual sum of squares}}\Big\}
\, \cdot\, \frac{1}{|\conv(\mathbf{ C})|^d}\, \prod_{i=1}^d \LaTeXunderbrace{\int_{\conv(\mathbf{C})}  
\exp \big\{-n\, \|\mathbf u-\mathbb P_{\mathbf C}^{(i)}\,\mathbf{\hat{u}}^{(i)}\|_{i}^2\big\}\,
d\mathbf u}_{\approx\,  C_i \, n^{-(k-1)/2}\, \idfn (\mathbb P_{\mathbf C}^{(i)}\,\mathbf{\hat{u}}^{(i)}  \in \conv(\mathbf C))},
\end{align}
where the display underneath the second curly bracket is due to the Laplace approximation to the $(k-1)$-dimensional integral, and the constants $C_i$  depends only on $\mathbf{\hat{u}}^{(i)}$.

We see from this approximation that the maximization of integrated likelihood~\eqref{eq:KL:two-stage} can be approximately cast into a two-stage sequential optimization problem. In the first stage, we find an optimal $(k-1)$-dimensional hyperplane spanned by $\mathbf C$ that is closest to $\mathbf {\hat u}^{(i)}$'s by minimizing the
residual sum of squares in~\eqref{Eqn:likelihood_approx} (see Figure~\ref{fig:Step1}). 
This corresponds to the SVD approach for estimating the true topic supporting hyperplane adopted by \citet{azar2001spectral,kleinberg2008using,kleinberg2003convergent,ke2017new}, and several others under the separability condition. In the second stage, we find the most compact (i.e., minimal volume) $k$-vertex convex polytope $\text{Conv}(\mathbf{C})$ that encloses the projections of $\mathbf {\hat u}^{(i)}$'s onto the hyperplane $\text{aff}(\mathbf{C})$, so that the second term in~\eqref{Eqn:likelihood_approx} is maximized. 

With the separability condition on $\mathbf{C}$ or $\mathbf{W}$, the vertex search in the second stage can be greatly simplified and restricted to a small number of choices. For example, the anchor-word assumption implies that each column of $\mathbf{C}$ has at least $(k-1)$ zeros; consequently, columns of $\mathbf{C}$ should be chosen from the intersection of $\text{aff}(\mathbf{C})$ and the simplex $\simplex^{V-1}$ in the second stage (as shown in Figure~\ref{fig:Intersection}). 

Our second stage, in the absence of a separability condition,  is essentially the much more challenging non-regular statistical problem of boundary estimation.
To see this, consider the same toy example of $(V,k)=(4,3)$ as illustrated in Figure~\ref{fig:Intersection}. The separability condition on $\mathbf C$ implies that once the hyperplane aff$(\mathbf C)$ (red hyperplane) is determined, the only candidate topic matrix $\mathbf{C}$ is the one whose columns are the intersections (blue circles) of this hyperplane and the three $1$-dimensional edges of the simplex $\simplex^3$ (blue tetrahedron), making the second stage trivial. On the contrary, the statistical problem in our setting is to estimate the minimal volume $k$-vertex convex polytope (black dashed triangle as our estimator) that encloses all true underlying word probability vectors of the documents, which is highly nontrivial (see Section~\ref{section:theory_comparsion} for a more detailed comparison). Fortunately, our computational algorithm described in the following subsection circumvents this difficulty directly maximizing the integrated likelihood via a variant of the expectation maximization (EM) algorithm, which implicitly constructs such an estimator.

\subsection{Computing Maximum Integrated Likelihood Estimator}
\label{app:MCMCEM}
 For computation, we employ an MCMC-EM algorithm to find the maximizer $\hatCn$ of the integrated likelihood objective~\eqref{eq:lkh_fn} by augmenting the model with a set of latent variables $\mathbf Z =\{Z_{ij}:\,i=1,2,\ldots,d,\, j=1,2,\ldots,n\}$, where, given the mixing weights $\mathbf w_i$, $Z_{ij}\in\{1,2,\ldots,k\}$ follows Cat$(\mathbf w_i)$ and is interpreted as the topic indicating variable for the $j$-th word $\mathbf x^{(i)}_j$ in the $i$-th document. Our MCMC-EM algorithm proceeds in a  manner similar to that of the classical EM algorithm with, first, an {\bf E}-step of computing the expected log-likelihood function $\log p(\mathbf X,\,\mathbf Z\,|\,\mathbf C)$, where the expectation is with respect to the distribution of latent variable $\mathbf Z$
 after marginalizing out $\mathbf W$, and then an {\bf M}-step of maximizing the expected log-likelihood function over topic matrix $\mathbf C$. An MCMC scheme is introduced in the {\bf E}-step for sampling $(\mathbf Z,\,\mathbf W)$ pairs from the joint conditional distribution of $p(\mathbf Z,\,\mathbf W\,|\,\mathbf X,\,\mathbf C)$ in order to compute the expected log-likelihood function via Monte-Carlo approximation.

 As discussed before, our proposed estimator is essentially the MLE estimator from the LDA model \citep{blei2003latent} with a particular choice of priors on $\mathbf{W}$. Many algorithms have been proposed for the LDA model, such as the Gibbs sampler  \citep{griffiths2004finding},  partially collapsed Gibbs samplers \citep{magnusson2018sparse,terenin2018polya}, and various variational algorithms \citep{blei2003latent}. The use of MCMC-EM here is a personal preference.  
 Our MCMC-EM algorithm is a stochastic EM algorithm similar to the Gibbs sampler in \citet{griffiths2004finding}, and to the partially collapsed Gibbs samplers in \citet{magnusson2018sparse,terenin2018polya}. According to the asymptotic results of stochastic EM algorithms in \citet{nielsen2000stochastic}, the estimation of the topic matrix produced by our algorithm is guaranteed to converge to the proposed MLE, provided that $\mathbf W^0$ is sufficiently scattered. In Section \ref{sec:real_data}, we compare our algorithm with the algorithms mentioned above and find all very similar in performance. Since computation is not the main focus of this paper, we confine the details, including derivations for the full algorithm, to the supplementary material.

\section{Finite-Sample Error Analysis}\label{sec:est}
In this section, we study the finite-sample error bound and its implied asymptotic consistency of the proposed estimator $\hatCn$. We consider the fixed design setting where columns of $\mathbf{W}$ can take arbitrary positions in $\simplex^{k-1}$ as long as a perturbed version of the SS condition described in the following is satisfied. For the stochastic setting where columns of $\mathbf{W}$ are generated from some distribution, the error analysis and consistency can be found from Section~\ref{Sec:random_design_error} in the supplementary material.
To avoid ambiguity, we use $\tC$, $\tW$, $\tU$ to denote the ground truth, and leave $\mathbf C$, $\mathbf W$, $\mathbf U$ as generic notations for parameters.

\subsection{Noise Perturbed SS Condition}\label{subsec:alpha-beta-SS}
Before introducing our results from the error analysis, it is helpful to introduce a perturbed version of the SS condition, called $(\alpha, \beta)$-SS condition, which characterizes the robustness/stability of the (population level) SS condition against random noise perturbation due to the finite sample size.

\begin{definition}[$(\alpha, \beta)$-SS Condition] \label{def:W_suff_catt_a} A matrix $\mathbf W$ is {\bf $(\bm \alpha, \bm \beta)$-\bf sufficiently scattered} for some $\alpha, \beta \geq 0 $, if it satisfies (S1) and 

\begin{enumerate}
      \item[(S3).] $[cone(\mathbf W)^{\ast}]^\alpha \bigcap [bd \mathcal K]^{\alpha} \subseteq \{\mathbf x: \|\mathbf x - \lambda \mathbf e_f \|_2 \leq \beta\lambda, \lambda \geq 0 \}, $ where 
      
      $[cone(\mathbf W)^\ast]^{ \alpha}=  \{\mathbf x: \mathbf x^T \mathbf W \geq -\alpha\|\mathbf x\|_2\}$ and $[bd\mathcal{K}]^\alpha = \{\mathbf x:  |\|\mathbf x\|_2 - \mathbf x^T \mathbf 1_k| \leq \alpha  \|\mathbf x\|_2 \}$ are the $\alpha$-enlargements of $cone(\mathbf W)^\ast$ and $bd \mathcal K$, respectively.
\end{enumerate} 
\end{definition}

We provide a geometric view of the $(\alpha, \beta)$-SS condition in Figure \ref{fig:suff_scatt_a}. Similar to the setting of Figure \ref{fig:two_suff_scatt}, everything is projected onto the hyperplane $\mathbf x^T \mathbf 1_k = 1$: blue dots denote columns of $\mathbf{W}$, the inner brown ball inscribed in the triangle denotes $\mathcal{K}^\ast$, and the shaded yellow region denotes $cone(\mathbf W)^\ast$ along with the dashed gray line as its boundary. The boundary of the enlarged cone of $cone(\mathbf W)^\ast$, $[cone(\mathbf W)^\ast]^{ \alpha}$, is marked by the solid gray line, and the thickened boundary of $\mathcal K$, $[bd\mathcal{K}]^\alpha$, is the outside ring in red. The set $\{\mathbf x: \|\mathbf x - \lambda \mathbf e_f \|_2 \leq \beta\lambda, \lambda \geq 0, f \in [k]\}$, when being projected to the hyperplane $\mathbf x^T \mathbf 1_k = 1$, corresponds to the green balls centered at the vertices of $\Delta^{k-1}$ with radius $\beta$.

  \begin{figure}[h]
    \begin{subfigure}[t]{0.42\textwidth}
        \includegraphics[width=0.93\textwidth]{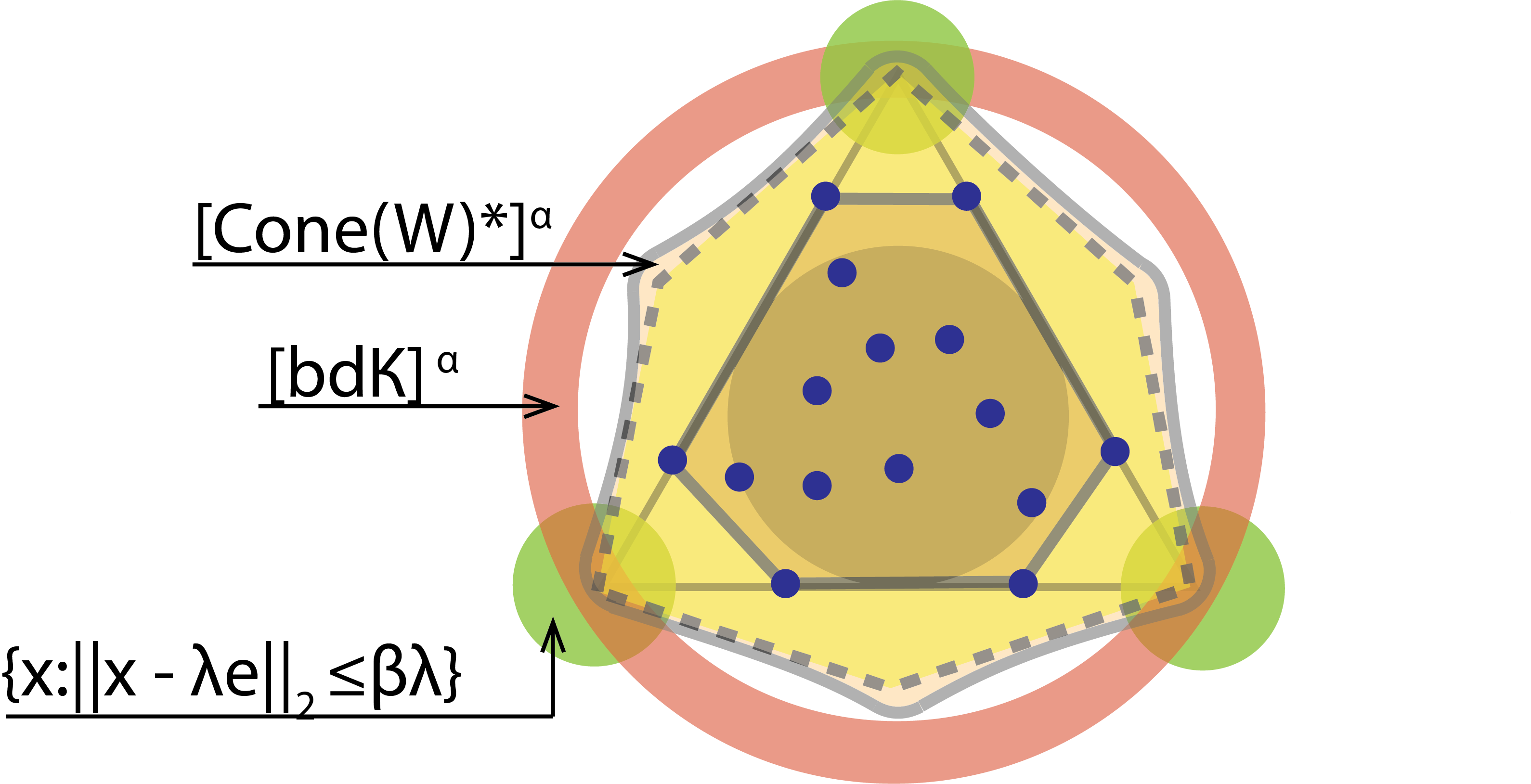}%
        \subcaption{$(\alpha, \beta)$-SS}%
           \label{fig:suff_scatt_alpha} %
    \end{subfigure}%
      \begin{subfigure}[t]{0.29\textwidth}
      \includegraphics[width=0.68\textwidth]{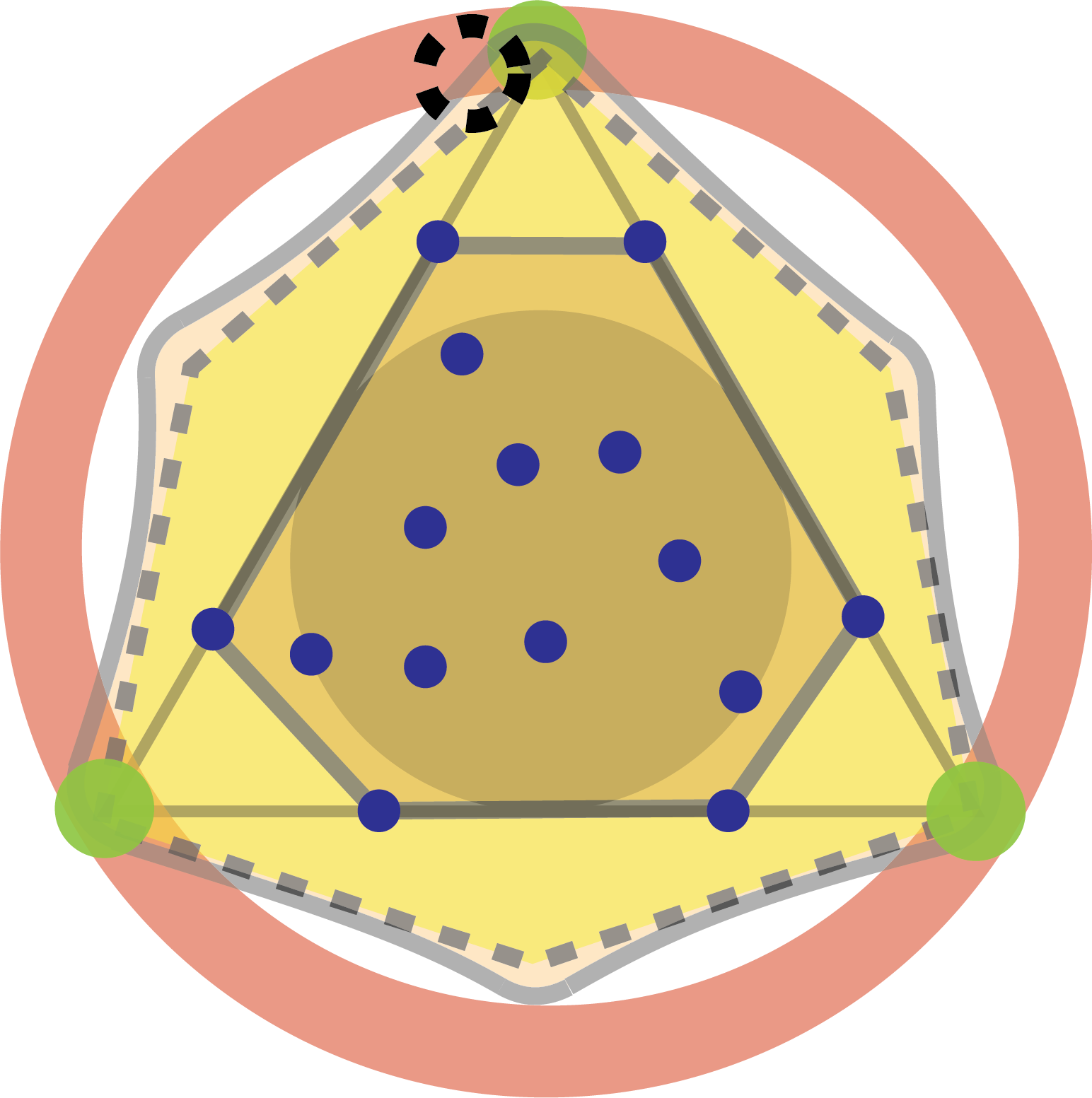}%
         \subcaption{not ($\alpha_2$, $\beta_2$)-SS\\ 
         ($\alpha_2 = \alpha, \beta_2 < \beta$)}%
      \label{fig:suff_scatter_alpha2}%
   \end{subfigure}
      \begin{subfigure}[t]{0.27\textwidth}
         \centering
         \includegraphics[width=0.8\textwidth]{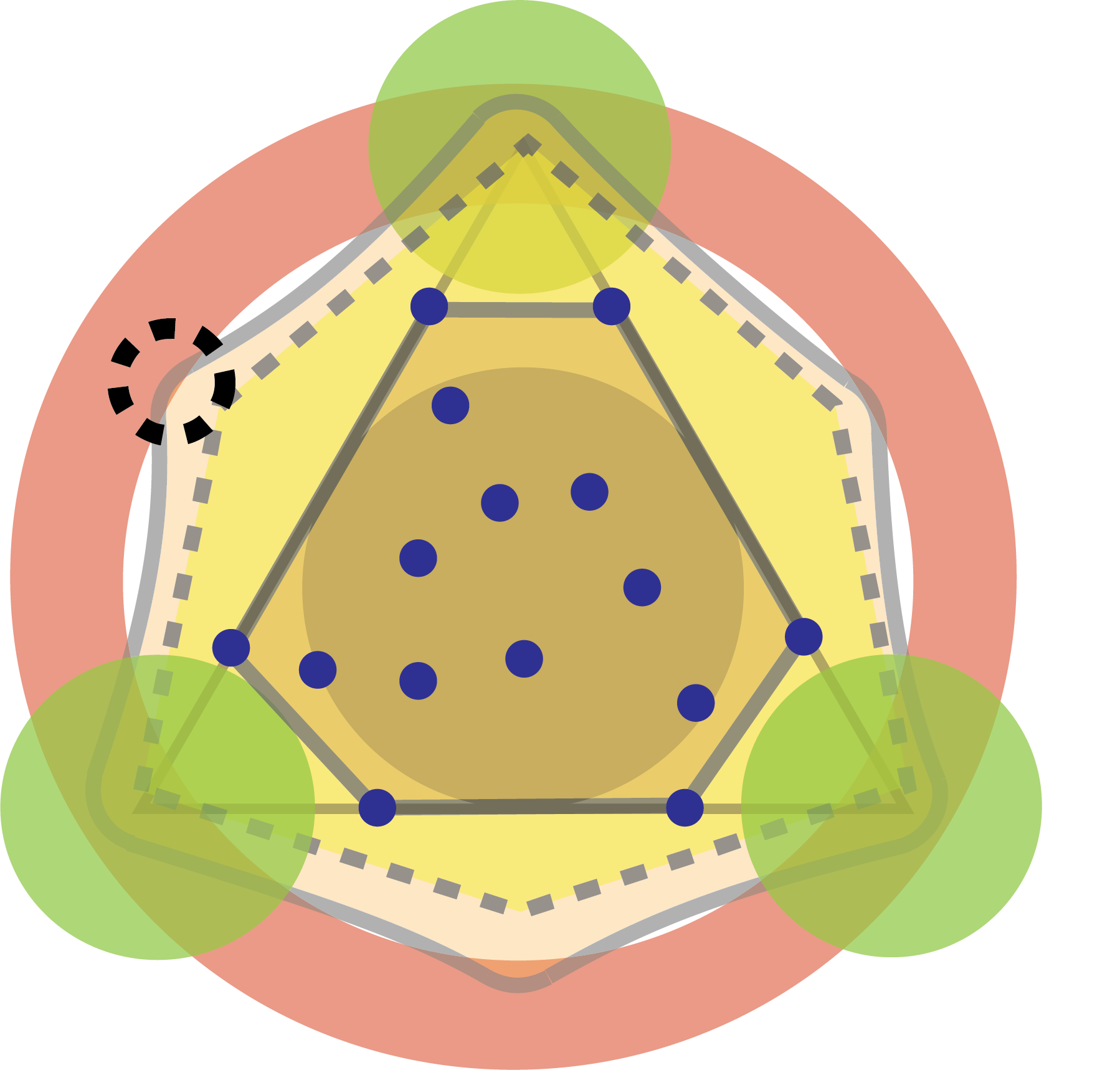}
         \subcaption{ not ($\alpha_3$, $\beta_3$)-SS\\ ($\alpha_3 > \alpha, \beta_3 > \beta$)
}%
\label{fig:suff_scatter_alpha3}  %
   \end{subfigure}%
    \caption{Geometric view of $( \alpha, \beta)$-SS sliced at the hyperplane $\mathbf x^T \mathbf 1_k = 1$ ($k=3$). $\mathbf W$ is the same in (a)(b)(c) while the values of $\alpha$ and $\beta$ are different.  In (b) and (c), we highlight the region (the dashed circle) that are in $[cone(\mathbf W)^{\ast}]^\alpha \bigcap [bd \mathcal K]^{\alpha}$ but not in $\{\mathbf x: \|\mathbf x -\lambda \mathbf e_f \|_2 \leq \beta\lambda  \}$. 
    }
    \label{fig:suff_scatt_a}%
\end{figure}

For a matrix $\mathbf{W}$ to satisfy the $(\alpha, \beta)$-SS condition, the corresponding convex hull of the blue dots need to contain $\mathcal{K}^\ast$, the inner brown ball. In addition, the intersection of the red ring, $[bd\mathcal{K}]^\alpha$, and the region enclosed by the solid gray line, $[cone(\mathbf W)^\ast]^{ \alpha}$, must be inside the green balls; see Figure \ref{fig:suff_scatt_alpha}. In other words, $[cone(\mathbf W)^\ast]^{ \alpha}$ only touches $[bd\mathcal{K}]^\alpha$ near the $k$ vertices of the simplex $\Delta^{k-1}$.

The $(\alpha, \beta)$-SS condition can be viewed as a generalization of the SS condition with the two parameters $(\alpha, \beta)$ quantifying the robustness of $cone(\mathbf W)^\ast$ under noise perturbation. In particular, $\alpha$ characterizes the \textit{tolerable noise level}, and $\beta$, which we refer to as the vertices \textit{sensitivity coefficient}, represents the maximum estimation error induced by noises below level $\alpha$. Due to this interpretation, the $(\alpha,\beta)$-SS condition becomes stronger as $\alpha$ increases and $\beta$ decreases (c.f.~Proposition~\ref{prop:ab_ss_property}). In particular, the minimal allowable $\beta$ under (S3) should increase as $\alpha$ increase. In most examples, $\beta$ should be proportional to $\alpha$ up to some constant depending on the geometric structure of $cone(\mathcal K)$ (for a concrete example, c.f.~Proposition~\ref{prop:suff_SS_new}).

While the SS condition requires $cone(\mathbf W)^\ast$ and $bd\mathcal{K}$ to intersect exactly at the positive semi-axis rays $\{\lambda \mathbf e_f, \lambda \geq 0\}$, the $(\alpha, \beta)$-SS condition requires the intersection of $[cone(\mathbf W)^\ast]^{ \alpha}$ and $[bd\mathcal{K}]^\alpha$---the perturbed versions of $cone(\mathbf W)^\ast$ and $bd\mathcal{K}$, respectively, with noise level $\alpha$---to be within distance $\beta$ away from the semi-axis rays. Note that $(\alpha, \beta)$-SS degenerates to the SS condition when $\alpha = \beta = 0.$  

Intuitively, if a matrix $\mathbf W$ has vertices sensitivity coefficient $\beta$ under noise level $\alpha$, then condition~(S3) remains valid at the same sensitivity coefficient as we decrease the noise level and at the same tolerable noise level as we increase the sensitivity coefficient. The following proposition provides a more general picture about the relation of the $(\alpha,\beta)$-SS conditions under different combinations of $(\alpha,\beta)$.

\begin{proposition}\label{prop:ab_ss_property}
The followings are some properties of $(\alpha, \beta)$-SS condition and SS condition.
\begin{itemize}
\item[(i)] If $\alpha \geq \alpha'$ and $\beta \leq \beta'$, then $(\alpha, \beta)$-SS implies $(\alpha', \beta')$-SS.
 \item[(ii)] If $\mathbf W$ is $(\alpha, \beta)$-SS and $\conv(\mathbf W) \subseteq \conv(\mathbf{\bar{W}})$, then $\mathbf{\bar{W}}$ is also $(\alpha, \beta)$-SS.
 \item[(iii)] If $\mathbf W$ is SS and $cone(\mathbf W) \subseteq cone(\mathbf{\bar{W}})$, then $\mathbf{\bar{W}}$ is also SS.
\end{itemize}
\end{proposition}


By Proposition \ref{prop:ab_ss_property}(i), the $(\alpha, \beta)$-SS condition gets more stringent if we increase the tolerable noise level $\alpha$ and/or reduce the vertices sensitivity coefficient $\beta$. This is because when $\alpha$ gets larger, the intersection $[cone(\mathbf W)^{\ast}]^\alpha \bigcap [bd \mathcal K]^{\alpha}$ gets larger and consequently may not be packed inside the green ball with radius $\beta$. 
Similarly, when $\beta$ gets smaller, the green balls may not be large enough to contain the intersection. See Figure \ref{fig:suff_scatter_alpha2}\subref{fig:suff_scatter_alpha3} for illustration.
Since $\conv(\mathbf W) \subseteq \conv(\mathbf{\bar{W}})$ implies $cone(\mathbf W) \subseteq cone(\mathbf{\bar{W}})$,  we provide a more general sufficient condition for SS in Proposition \ref{prop:ab_ss_property}(iii) compared to that in Proposition \ref{prop:ab_ss_property}(ii), where SS is a special case of $(\alpha, \beta)$-SS. However, in this paper, the columns of $\mathbf W$ we consider are all on the hyperplane $\mathbf x^T \mathbf 1_k = 1$, so $\conv(\mathbf W) \subseteq \conv(\mathbf{\bar{W}})$ is equivalent to $cone(\mathbf W) \subseteq cone(\mathbf{\bar{W}})$. As a direct consequence of Proposition \ref{prop:ab_ss_property}(ii), if some columns of $\mathbf W$ is $(\alpha, \beta)$-SS, then $\mathbf W$ is $(\alpha, \beta)$-SS.

The maximal allowable tolerable noise level $\alpha$ is determined by the geometric structure of $cone(\mathbf W)$. Given $\alpha$, the $(\alpha, \beta)$-SS condition can be satisfied by almost any $\mathbf W$ when $\beta$ is large enough. However, such a condition is meaningless since $\beta$ will appear as one of the error terms later in Theorem \ref{thm:main}. So we would like to set $\beta$ as small as possible in order to derive a tight error bound. For example, we need $\beta$ to have an order of $\sqrt{\frac{\log (n\vee d)}{n}}$ in Theorem \ref{thm:main} to ensure a desired error rate that matches the order of our $\alpha$ choice reflecting the effective noise level in the data.

\subsection{Error Analysis and Consistency}\label{subsec:fixed}
In this subsection, we consider the setting where columns of $\mathbf W$ are fixed, and satisfy a set of conditions related to the noise perturbed SS condition discussed in the previous subsection. Note that the results in this subsection also apply to randomly generated mixing weights, as long as we can verify that the set of conditions below holds for the random mixing weights with high probability (c.f.~Section~\ref{Sec:random_design_error} in the supplementary material).
Before presenting our main results on the finite-sample error bound of the estimator $\mathbf{\hat C}_n$, let us first state our assumptions.

\begin{assumptions} Assume the following: 
\begin{itemize}
\item[(A1)]
 $\tC$ is of rank $k$ and its columns are bounded away from the boundary of $\simplex^{V-1}$. 
 \item[(A2)] Eigenvalues of $\frac{1}{d}\mathbf{W}_c{\mathbf{W}_c}^{T}$ are lower bounded by a positive constant, where $\mathbf{W}_c = \tW - \frac{1}{d}\tW \mathbf{1}_d\mathbf{1}_d^T$ is the centered version of $\tW$. In addition, there exist $k$ affinely independent columns of $\tW$ with minimum positive singular value larger than a positive constant.
\item[(A3)] There exist $s$ columns of $\tW$ which are ($\alpha$, $\beta$)-SS with $\alpha \geq C_1\sqrt{\frac{s\log (n\vee d)}{n}}$, where $s$ and $C_1$ are constants.

\end{itemize}
\end{assumptions}

Now we are ready to present our main result on the estimation accuracy. 

\begin{theorem}\label{thm:main}
Under Assumptions (A1)-(A3), with probability at least $(1-3/(n\vee d)^{c})^d$, 
\begin{equation}\label{eq:result}
    \dis(\mathbf{\hat C}_n, \tC)  \leq  D_1 \sqrt{\frac{s \log (n\vee d)}{n}} + D_2 \sqrt{s} \beta,
\end{equation}
where $c,D_1$ and $D_2$ are positive constants. 
In particular, if $\beta \leq C_2\sqrt{\frac{\log (n\vee d)}{n}}$ where $C_2$ is a constant, then 
\begin{equation}\label{eq:result_new}
    \dis(\mathbf{\hat C}_n, \tC)  \le D_1' \sqrt{\frac{s\log (n\vee d)}{n}}.
\end{equation}
\end{theorem}


\noindent In the theorem, constants $D_1$ and $c$ have the relation that $D_1 = C_3\cdot\sqrt{c}+C_4$ where $C_3$ and $C_4$ are constants independent of $(n,d)$. Some remarks about the assumptions are in order.

 (A1) is commonly imposed for technical reasons in other related work, such as \citet{nguyen2015posterior} and \citet{Wang2019}, to avoid singularity issues.
The geometric interpretation of the assumption in (A2) on $\mathbf{W}_c$ is that $\conv(\mathbf {U}_0)$ should contain a ball of a constant radius, which is again imposed to avoid singularity issues when a large proportion of the mixing weight vectors are too concentrated. Similar assumptions are also made in \citet{ke2017new, javadi2020nonnegative}.  

Next, we discuss Assumption (A3) in detail. First, note that a subset of columns of $\mathbf W^0$ satisfying the $(\alpha,\beta)$-SS condition immediately implies the full matrix $\mathbf W^0$ itself to satisfy the same condition, due to Proposition~\ref{prop:ab_ss_property}(ii).
Second, note that to attain the error bound~(\ref{eq:result_new}) we need the existence of a sub-matrix $\mathbf{W}^0$ to satisfy condition (A3) with $\beta$ of the same order as $\alpha$. 
The following proposition provides a sufficient condition for fulfilling this requirement. For example, when $k=3$ as illustrated in Figure~\ref{fig:suff_scatt_alpha}, all we need are two data points on each of the three line segments connecting $\mathbf e_i$ and $\mathbf e_j$ $(i\neq j)$ (i.e., totally six points) with the distance from each data point to the nearest vertex is less than $1/3$.

\begin{proposition}\label{prop:suff_SS_new}
Suppose for all $1\leq i \neq j \leq k$, there exists a column of $\tW$ that can be represented as $(1-x_{ij})\mathbf{e}_i+x_{ij}\mathbf{e}_j$ where $0\leq x_{ij}< 1/k$, then $\tW$ is ($\epsilon$, $C\epsilon$)-SS for all $\epsilon>0$, where $C$ is constant only depending on the geometry of $\mathbf W^0$.
\end{proposition}

Third, we discuss the parameter $s$,  the smallest number of columns in $\tW$ that are ($\alpha$, $\beta$)-SS, in Assumption (A3). The following proposition shows that when the columns of $\tW$ are stochastically generated according to some underlying distribution over $\simplex^{k-1}$ with appropriate properties, then $s$ can be chosen as a constant with high probability. 
Note that even if $s$ is not a constant,  the error bound in \eqref{eq:result_new} still goes to zero as long as $s$ is of a smaller order of $\frac{n}{\log(n \vee d)}$ in the asymptotic setting where $(n,d)\to\infty$.
\begin{proposition}\label{prop:const_SS_columns}
Suppose the columns of $\tW$ are i.i.d.~samples from a probability density function that is uniformly larger than a positive constant on neighborhoods of the vertices of $\simplex^{k-1}$. If $C\cdot n^{\frac{k-1}{2}}\leq d\leq e^{n^c}$, then with probability at least $1-C_0\cdot k/d$, there exist $k$ columns in $\tW$ that are $\left(C_1 \sqrt{\frac{\log (n\vee d)}{n}}, 
C_2 \sqrt{\frac{\log (n\vee d)}{n}} \right)$-SS, where $c\in(0,1)$, $C$, $C_0$, $C_1$ and $C_2$ are positive constants.
\end{proposition}


    



Next, we show the asymptotic consistency of $\hat{\mathbf C}_n$, that is, $\hat{\mathbf C}_n \to \mathbf C^0$ in probability as $(n,d)\to\infty$. In particular, we assume the existence of a sequence of $\alpha$ and $\beta$ values along which the $(\alpha,\beta)$-SS conditions are satisfied, which is summarized in the following.



\begin{assumptions} Assume the following: 
\begin{itemize}
\item[(A3')] For any sufficiently small $\epsilon>0$, there exists some $\beta_\epsilon$ such that $\beta_\epsilon \to 0$ when $\epsilon \to 0$, and there are $s$ columns of $\tW$ satisfying the ($\epsilon$, $\beta_\epsilon$)-SS condition, where $s$ is a bounded constant.
\item[(A4)] $\log d/n\to 0$ as $(n,d)\to\infty$. 
\end{itemize} 
\end{assumptions}

\begin{theorem}[Estimation Consistency]\label{Coro:consistency}
Under Assumptions (A1), (A2) and (A3') with a fixed $d$, we have
\begin{equation}\label{eq:result_con_new}
    \dis(\mathbf{\hat C}_n, \tC)  \to 0 \quad\mbox{in probability as $n\to\infty$.}
\end{equation}
If $d$ is also increasing in $n$ in a way such that Assumption (A4) holds, then
\begin{equation}\label{eq:result_con_new_n&d}
    \dis(\mathbf{\hat C}_n, \tC)  \to 0 \quad\mbox{in probability as $(n, d)\to\infty$.}
\end{equation}
\end{theorem}



 
 Note that Proposition~\ref{prop:suff_SS_new} again provides a set of sufficient conditions for  Assumption (A3'). However, our current condition on $\mathbf W^0$ in Proposition~\ref{prop:suff_SS_new} is stronger than the SS condition on $\mathbf W^0$.
 We conjecture that Assumption (A3') is equivalent to the SS condition on $\mathbf W^0$, and leave a formal proof to future work.

\subsection{Comparison with Existing Theoretical Results}\label{section:theory_comparsion}
Our error bound in Theorem~\ref{thm:main} does not decay as the number of documents $d$ increases, which is seemingly weaker than some existing results, such as \citet{arora2012learning}, \citet{bansal2014provable}, \citet{anandkumar2014tensor}, \citet{ke2017new}, and \citet{Wang2019}. In particular, under the anchor word assumption, \citet{arora2012learning} and \citet{ke2017new} showed an error upper bound as $1/\sqrt{nd}$.


As discussed in Section \ref{subsec:comparison}, many algorithms for estimating the topic matrix can be explained through a two-stage optimization, corresponding to either a single stage or both. Under this perspective, each stage will incur an error. With the anchor word assumption, the main source of errors comes from the first stage of applying an SVD approach \citep{azar2001spectral,kleinberg2008using,kleinberg2003convergent,ke2017new} to find a $(k-1)$-dimensional hyperplane best approximating the data whose error bound is $1/\sqrt{nd}$. In fact, the anchor word assumption greatly reduces the search space in the second stage of identifying columns of $\mathbf C$ as either a subset of anchor words or a subset of pure topic documents, yielding negligible estimation error. 
For example, the vertex hunting algorithm adopted in \citet{ke2017new} directly focuses on all the $k$ combinations of the noisy data points in the $(k-1)$-dimensional hyperplane obtained in the first stage, and chooses the combination that minimizes the predetermined criterion. With the separability condition, they show that the estimated vertices are all close to their corresponding true vertices in a $(k-1)$-dimensional hyperplane, from which they draw the conclusion that the estimation error of the second stage is no larger than that of the first stage (see Lemma A.3, \citet{ke2017new}).

Without the anchor word (or separability) assumption, errors incurred in the second stage become dominant. Consider the toy examples illustrated in Figures \ref{fig:vol} and \ref{fig:two_suff_scatt} with $K=V=3$. The first stage is trivial since the data are already in $(k-1)$-dimension and projection to a hyperplane is not needed. In the second stage, we need to estimate a $k$-vertex convex polytope enclosing all true  word probability vectors of the documents that generates the data, which can be formulated as the non-regular statistical problem of boundary estimation. As pointed out by \citet{goldenshluger2004estimating, brunel2021estimation}, estimation of convex supports from noisy measurements as in our second stage is an extremely difficult problem. For example, in the one-dimensional case, even with the knowledge that the noises are homogeneous and follow a known Gaussian distribution, the minimax rate of boundary estimation based on $d$ observations is as slow as $1/\sqrt{\log d}$, let alone the more complex situation where the noise distribution is heterogeneous and only partly known. For example, in our case the projection $\mathbb P_{\mathbf C}^{(i)} \, \mathbf{\hat{u}}^{(i)}$ onto aff$(\mathbf C)$ of the sample word frequency vector $\mathbf{\hat{u}}^{(i)}$ for document $i$, for $i=1,\ldots,d$, plays the role of a noisy measurement from the convex polytope $\conv(\mathbf C)$. Note that a typical noise level in our second stage is of order $1/\sqrt{n}$ due to $n$ number of words within each document; however, the error distribution depends on both the position of the hyperplane aff$(\mathbf C)$ obtained in the first stage as well as the location of $\mathbb P_{\mathbf C}^{(i)} \mathbf{\hat{u}}^{(i)}$ on the data simplex $\simplex^{V-1}$. Therefore, we cannot expect to achieve the $1/\sqrt{nd}$ error bound as those separability condition based methods.
It is an interesting open problem of determining the precise minimax-optimal rate in topic models without separability condition and whether our error bound is optimal, which we leave as a future direction.



\section{Empirical Studies}\label{sec:emp}
In this section, we describe numerical studies we have performed to test our theoretical results. We report the performance of our model on two real datasets. 
 
\subsection{Simulation Studies}  \label{sec:simstu}

We have conducted three simulation studies to verify our theoretical results and to test the performance of our proposed algorithms.
In Section \ref{subsub:effect_SS}, we apply the MCMC-EM algorithm to the data generated by non-identifiable and identifiable models, and compare the recovered convex polytopes with the truth, to show the importance of the SS condition. In Section \ref{subsub:mis_prior}, we compare the proposed uniform prior $\bm \beta_0 = \mathbf 1_k$ with other priors, using data generated from different distributions, to  demonstrate empirically the robust performance of our estimator. In Section \ref{subsub:conver}, we apply Monte Carlo simulation to visualize the convergence of the proposed MLE.

\subsubsection{Effect of the SS Condition}\label{subsub:effect_SS}

Data are generated from a simple setup: $k=V=3$, $\tC = \mathbf I_3$, and the number of words for each document is sampled from $\text{Poisson} (2000)$. For the true matrix $\mathbf{W}^0$, we consider four different configurations for $\tw{i}$: (a) concentrated in the center of $\simplex^{2}$; (b) concentrated in the bottom right; (c) satisfying the SS condition; (d) spread around three vertices. 
The four configurations are displayed in Figure \ref{fig:sim1}, where the black dots denote $\tw{i}$ and  the large black triangle represents $\conv(\tC) =\simplex^{2}$. In cases (a)(b)(d), we set the number of documents  $d=1000$, while in case (c) we set $d=6.$

\begin{figure}
    \centering
    \begin{subfigure}[t]{0.24\textwidth}
        \centering
         \includegraphics[width=.8\textwidth]{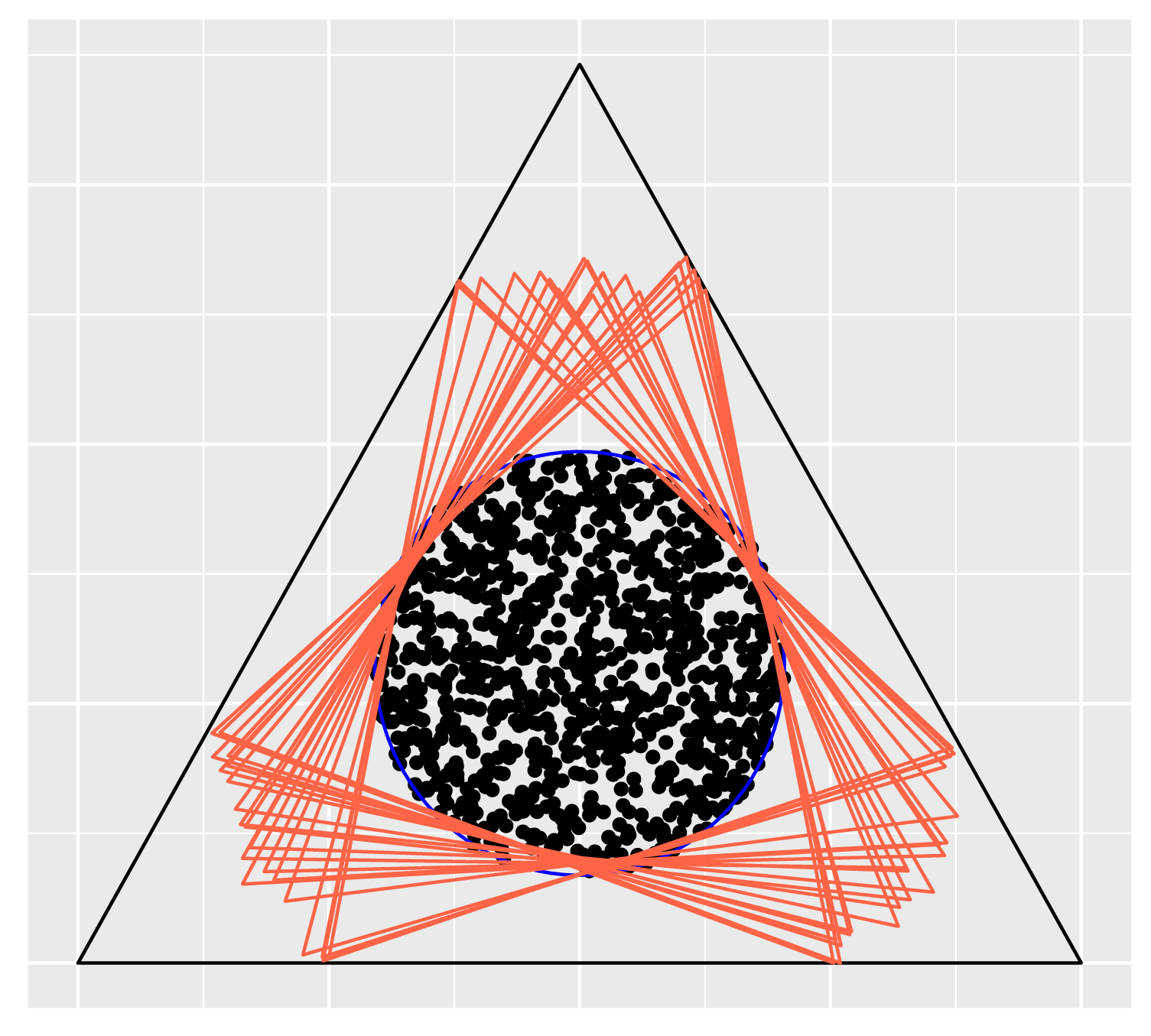}
    \label{fig:inner_circle}
    \caption{non-identifiable}
    \end{subfigure}
    \begin{subfigure}[t]{0.24\textwidth}
        \centering
         \includegraphics[width=.8\textwidth]{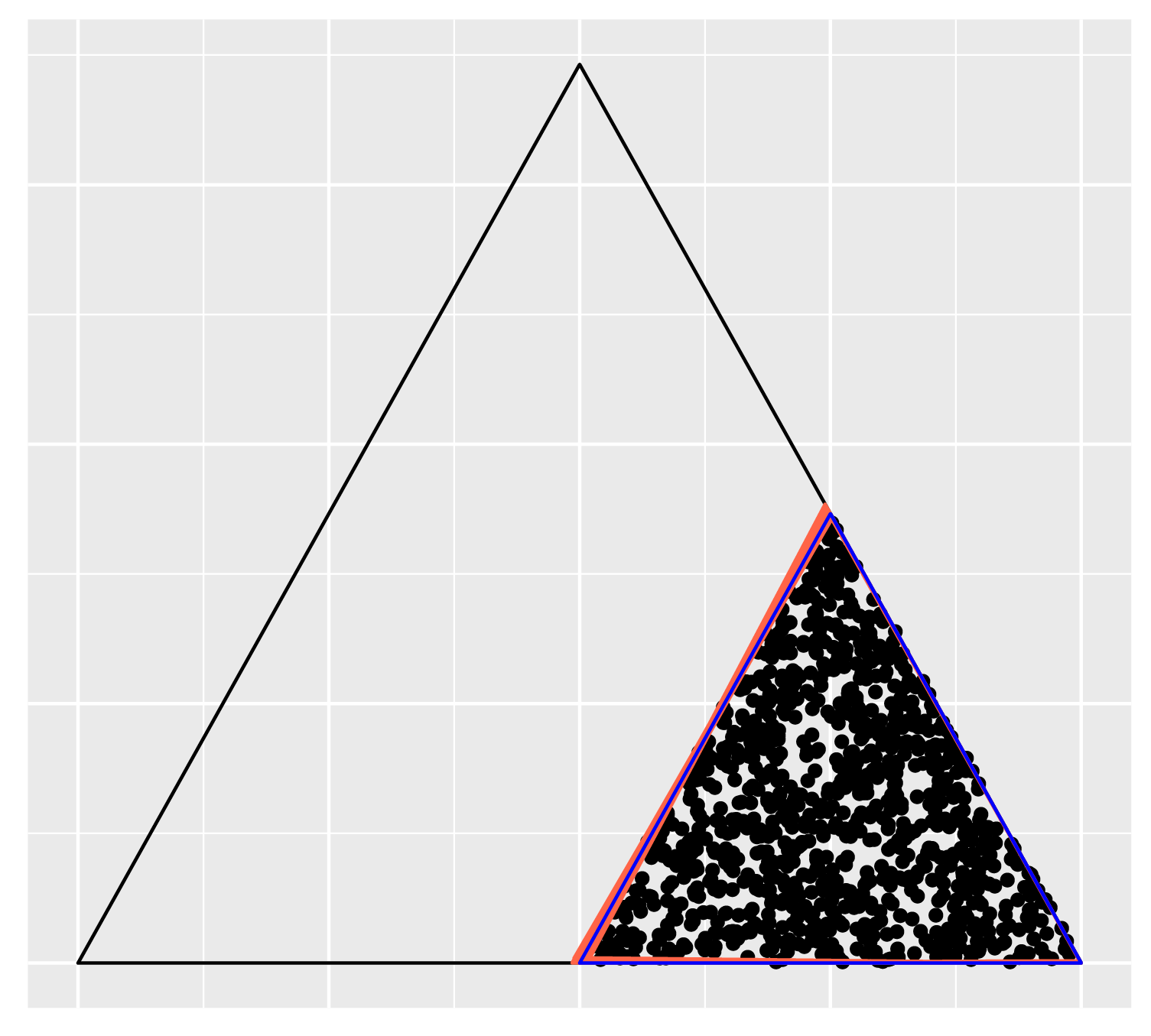}
     \label{fig:s_tri}
     \caption{non-identifiable}
     \end{subfigure}
            \begin{subfigure}[t]{0.24\textwidth}
        \centering
         \includegraphics[width=.8\textwidth]{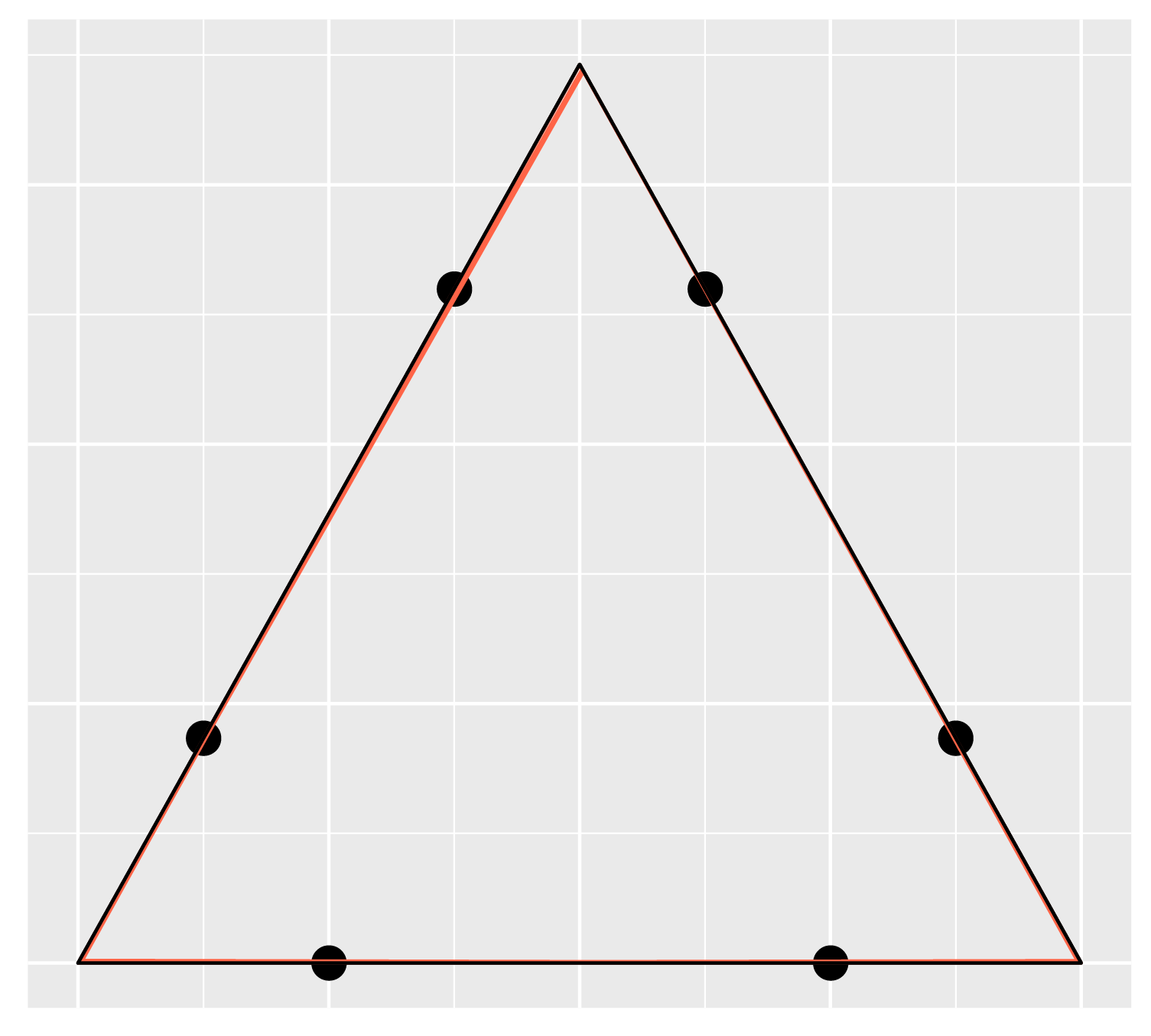}
    \label{fig:idf_1}
    \caption{identifiable}
    \end{subfigure}
    \begin{subfigure}[t]{0.24\textwidth}
        \centering
         \includegraphics[width=.8\textwidth]{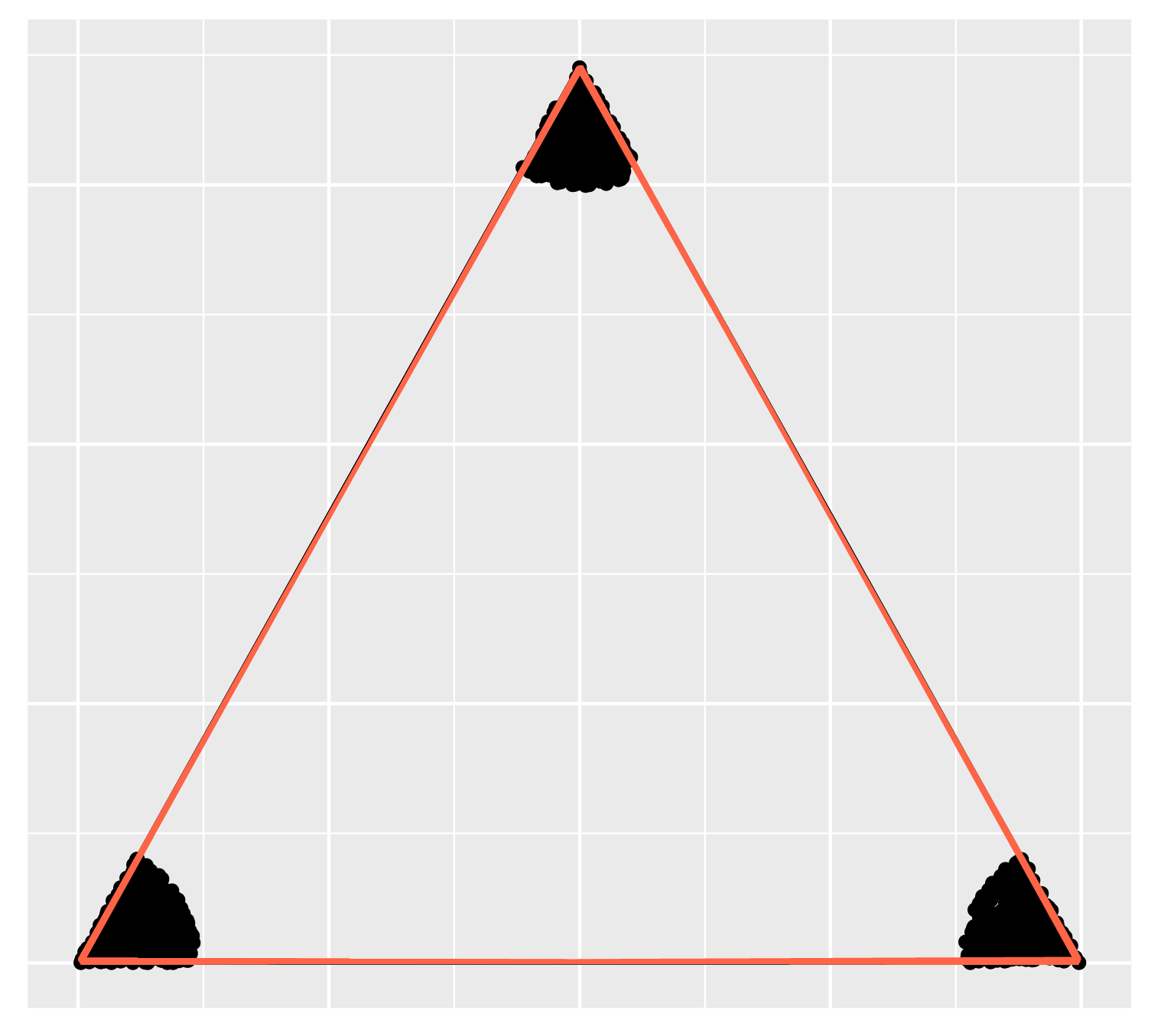}
     \label{fig:idf_2}
     \caption{identifiable}
    \end{subfigure}
    \caption{Results of the simulation in Section \ref{subsub:effect_SS}. Black dots are columns of $\tW{}$; the black triangle is the ground truth $\conv(\tC) = \simplex^2$; red triangles are estimates of $\conv(\hatCn)$.}
    \label{fig:sim1}
\end{figure}

 We run our MCMC-EM algorithm $20$ times with different initialization; Figure \ref{fig:sim1} displays the estimates of $\conv(\hatCn)$ as red triangles. Our simulation results demonstrate that if the SS condition is not satisfied, even when the sample size $d$ is fairly large ($d = 1000$ in (a) and (b)), $\conv(\tC)$ cannot be correctly recovered. However, when SS is satisfied, even with just a few samples ($d = 6$ in (c)), our algorithm can accurately recover the ground truth. Identifiability is thus determined primarily by the scatteredness of $\tw{i}$ rather than by the number of documents $d$.



%


\subsubsection{Performance under Prior Misspecification}\label{subsub:mis_prior}

When deriving our estimator, we choose to integrate over the mixing weights with respect to the uniform prior. A natural question is how our estimator would perform when the true mixing weight $\tW$ is  stochastically generated from a  distribution other than uniform. 

In this simulation study we consider the following setup: $k = 3$, $V = 1000$, $d=200$, $\tC \sim \text{Dirichlet}_V(\mathbf 1)$, and the number of words for each document is generated from $\text{Poisson} (20000)$. The true mixing weights $\tW$ are stochastically generated from the following distributions: (a) $\text{Dirichlet}_3(\mathbf 1)$; (b) uniformly from 10 Euclidean balls whose centers satisfy the SS condition; (c) a mixture of Dirichlet distributions: 
     $0.2 \times \text{Dir}_3(10,1,1) + 0.2 \times  \text{Dir}_3(0.1,1,1) + 0.2 \times \text{Dir}_3(10,10,1) + 0.2 \times  \text{Dir}_3(0.1,0.1,1) + 0.2\times  \text{Dir}_3(1,2,3)$.

\begin{table}[h]
\centering
\caption{Relative RMSE (Simulation 2). }

\begin{tabular}{c|ccccccccc}
\hline
priors       & \footnotesize (1, 1, 1)  & \footnotesize (0.1, 0.1, 0.1) & \footnotesize (10, 1, 1) & \footnotesize (0.1, 1, 1) &  \footnotesize (0.1, 0.1, 1) & \footnotesize  (10, 1, 0.1) & \footnotesize (1, 2, 3) & \footnotesize  (3, 3, 3) \\
\hline \hline
case (a)     & \textbf{0.048}    & 0.064           & 0.058      & 0.059       & 0.061         & 0.068        & \textbf{0.048}     & 0.049     \\
case (b)& \textbf{0.053}     & 0.065           & 0.062      & 0.060       & 0.061         & 0.075        & 0.056     & 0.057     \\
case (c)      & \textbf{0.040}    & 0.042           & 0.048      & \textbf{0.040}       & 0.041         & 0.049        & 0.042     & 0.044 \\
\hline
\end{tabular}
    \label{tab:sim_mis_prior}
\end{table}

We compare our estimator and estimators based on other Dirichlet priors using the averaged Relative RMSE (i.e., RMSE divided by the average of RMSE of random guesses) of $\mathbf{\hat C}_n$ over $100$ replications. The results are reported in Table \ref{tab:sim_mis_prior}. We can see that in all three cases, our proposed estimator outperforms other estimators.

\subsubsection{Convergence of the Estimation}\label{subsub:conver}

{\colorblue
We use the Monte Carlo simulation to show the convergence of the integrated likelihood $ F_{n\times d}(\mathbf C)$ and the MLE $\hatCn$.

In the first experiment, we consider the setup where $V=9$, $k=3$, and the sample size $n$ and number of documents $d$ increase simultaneously. 
The sample size $n$ varies as $n = 50, 200, 400, 1600$ and $d = n/5$. 
Let
$$\tC = \begin{bmatrix} 
2/3 & 1/6 & 1/6 \\
1/6 & 2/3 & 1/6\\
1/6 & 1/6 & 2/3 \\
\end{bmatrix}, \quad \tW = \begin{bmatrix}
5/6 & 0 & 1/6 & 5/6 & 1/6 & 0 \\ 
1/6 & 5/6 & 0 & 0 & 5/6 & 1/6 \\
0 & 1/6 & 5/6 & 1/6 & 0 & 5/6 \\
\end{bmatrix}.$$
We generate the ``noiseless'' data, i.e., $\mathbf X = n\mathbf{C}^1\mathbf{W}^1$, where $\mathbf{C}^1 = \frac{1}{3} \left(\mathbf{C}^{0T}, \mathbf{C}^{0T}, \mathbf{C}^{0T}\right)^T$, the first six columns of $\mathbf{W}^1$ are $\tW$, and the rest of the columns are randomly generated from $\text{Dir}_k (\mathbf 1)$. 
We compare the integrated likelihood among candidate topic matrices of the form  $\mathbf{C} = \frac{1}{3} \left(\mathbf{A}^{T}, \mathbf{A}^{T}, \mathbf{A}^{T}\right)^T$, where $\mathbf{A}$ is 
\begin{align}
    \label{mat:c}
\begin{bmatrix}
c & (1-c)/2 & (1-c)/2\\
(1-c)/2 & c & (1-c)/2\\
(1-c)/2 & (1-c)/2 & c\\
\end{bmatrix},
\end{align}
with $c$ taking values from $[0.5,1]$. We use the Monte Carlo method to evaluate the integrated likelihood \eqref{eq:lkh_fn}: 
$$\hat{F}_{n\times d, T}(\mathbf C) \approx  \prod_{i=1}^d  \left[\frac{1}{T} \sum_{t=1}^T f_n(\mathbf x^{(i)}| \mathbf u = \mathbf C\mathbf w_t)\right],$$
where $\mathbf w_1,\cdots, \mathbf w_T$ are i.i.d. random samples from $\text{Dir}_k (\mathbf 1)$ and $T = 100,000.$

Figure \ref{fig:lkh_vs_C_change_d} shows $\hat{F}_{n\times d, T}(\mathbf C)/\max_{\mathbf C}\hat{F}_{n\times d, T}(\mathbf C)$, the relative value of the estimated integrated likelihood. From the plot we can see that the integrated likelihood converges quickly to the truth as both $n$ and $d$ increase. That is because $n$ is the sample size, and the integrated likelihood is the product of $d$ terms. As $d$ increases, the product is more concentrated.

\begin{figure}[h!]
    \centering
    \includegraphics[width=0.6\textwidth]{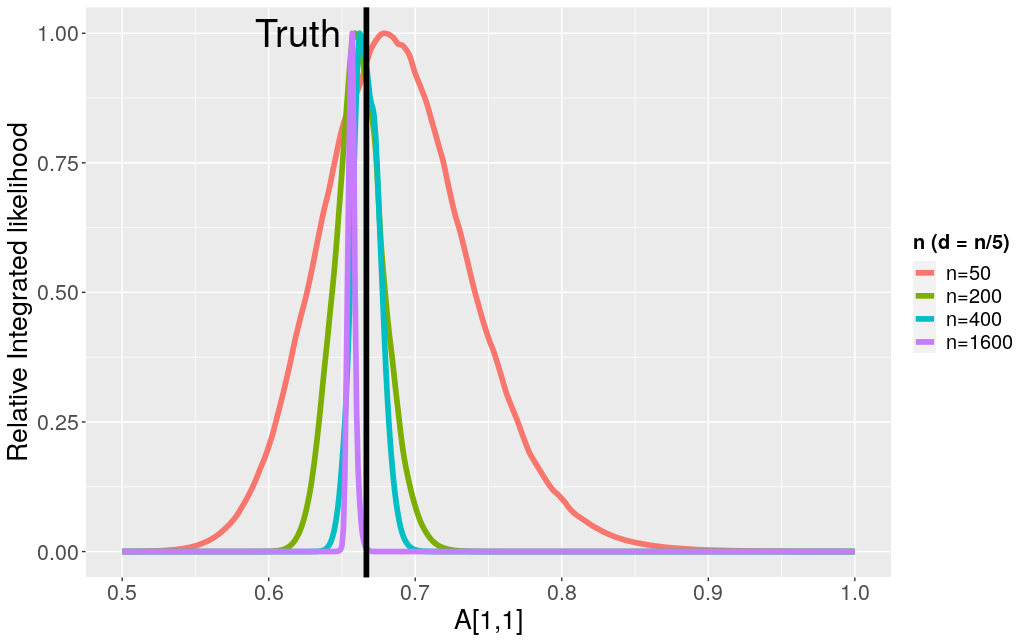}
    \caption{Results of the first experiment in Section \ref{subsub:conver}. The curves show the relative integrated likelihood of ``noiseless'' data when $n$ and $d$ increase simultaneously.}
\label{fig:lkh_vs_C_change_d}
\end{figure}

}

In the second experiment, we consider the case where $V=k=3$ and $d=6$. We  add some noise to the data, i.e., $\mathbf x^{(i)} \sim \text{Multi}(n, \tC\mathbf w^{0(i)})$. In Figure \ref{fig:sim_3_2} we plot the  multinomial likelihood density function $f_n(\mathbf u; \mathbf x^{(i)})$ (represented by the purple clusters) for the $d$ documents and the estimated $\conv(\hatCn)$ (represented by the red triangle).

\begin{figure}[h]
   \centering
    \begin{subfigure}[t]{0.24\textwidth}
        \centering
         \includegraphics[width=.7\textwidth]{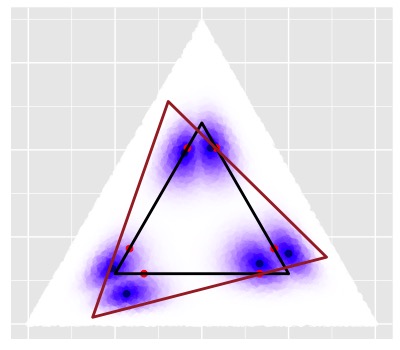}
    \label{fig:lkh_n60}
    \caption{$n=60$}
    \end{subfigure}
    \begin{subfigure}[t]{0.24\textwidth}
        \centering
         \includegraphics[width=.7\textwidth]{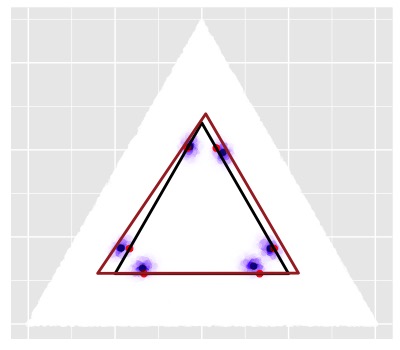}
     \label{fig:lkh_n600}
     \caption{$n=600$}
     \end{subfigure}
     \begin{subfigure}[t]{0.24\textwidth}
        \centering
         \includegraphics[width=.7\textwidth]{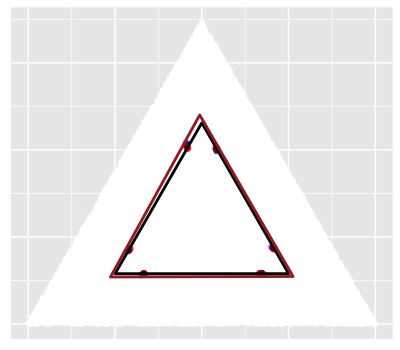}
    \label{fig:lkh_n6000}
    \caption{$n=6000$}
    \end{subfigure}
      \begin{subfigure}[t]{0.24\textwidth}
        \centering
         \includegraphics[width=.7\textwidth]{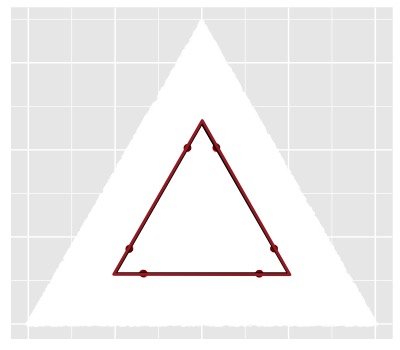}
    \label{fig:lkh_n60000}
    \caption{$n=60000$}
    \end{subfigure}
    \caption{Results of the second experiment in Section \ref{subsub:conver}. The likelihood density $f_n(\mathbf u; \mathbf x^{(i)})$ over $\simplex^2$ for different $n$. The colored circles represent the values of $f_n(\mathbf u; \mathbf x^{(i)})$: the darker the color is, the higher the likelihood is. The black triangle is $\conv(\tC)$; the dark red triangle is $\conv(\mathbf {\hat C}_n)$ produced by MCMC-EM. The red dots are the true means $\mathbf u^{0(i)}$, and the black dots are the sample means $\mathbf {\hat u}^{(i)}$.}
    \label{fig:sim_3_2}
\end{figure}

We observe that $\conv(\hatCn)$ tends to cover these density balls while maintaining its volume small.  Recall that $\hatCn = \argmax_{\mathbf C}\prod_{i=1}^d  \int_{\conv(\mathbf C)}  \frac{f_n(\mathbf x^{(i)}| \mathbf u)}{|\conv(\mathbf C)|} d \mathbf u$. $\conv(\hatCn)$ can be considered to be the convex polytope that has the highest value of the averaged likelihood density, as well as the smallest convex polytope containing the sample means $\mathbf{\hat u}^{(i)}$. Therefore, $\conv(\hatCn)$ tends to trade off its volume for a larger coverage of the density balls. In this case, the true means $\mathbf u^{0(i)}$ are all located on the boundary of $\conv(\tC)$; to fulfill the SS condition, a fraction of each circle thus lies outside $\conv(\tC)$. Consequently, the averaged likelihood density over $\conv(\hatCn)$ is larger than that of $\conv(\tC)$, though $|\conv(\hatCn)| > |\conv(\tC)|$. As proved in Theorem \ref{thm:main}, the convergence rate of $\hatCn$, in the order of $\sqrt{\log (n\vee d)/n}$, is slightly slower than that of $\mathbf {\hat u}^{(i)}$, which is in the order of $\sqrt{1/n}$.

\subsection{Real Applications}\label{sec:real_data}

We next apply our algorithm to some real-world datasets. In Section \ref{subsubsec:real_quant} we compare the quantitative performance of our algorithms,  and of several baseline methods, on two text datasets: an NIPS dataset that contains long academic documents, and the Daily Kos dataset that contains short news documents. In Section \ref{subsubsec:real_taxi} we analyze a taxi-trip dataset that contains  New York City (NYC) taxi trip records, including pick-up and drop-off locations.


\subsubsection{Text Data sets}
\label{subsubsec:real_quant}

The NIPS dataset\footnote{\href{ https://archive.ics.uci.edu/ml/datasets/NIPS+Conference+Papers+1987-2015}{https://archive.ics.uci.edu/ml/datasets/NIPS+Conference+Papers+1987-2015}} contains $V=11463$ unique words and $d=5811$ NIPS conference papers, with an average document length of $1902$ words. The Daily Kos dataset\footnote{\href{ https://archive.ics.uci.edu/ml/machine-learning-databases/bag-of-words/}{https://archive.ics.uci.edu/ml/machine-learning-databases/bag-of-words/}} contains $V=6906$ unique words and $d=3430$ Daily Kos blog entries, with an average document length of $136$ words. As the two datasets are formatted in document-term matrices without stop words or rarely occurring words, we do not apply any pre-processing procedures.

We compare the performance of our algorithm ($\text{MC}^2$-EM) with the following baseline algorithms: Anchor Free (AnchorF) \citep{huang2016anchor}, Geometric Dirichlet Means (GDM) \citep{yurochkin2016geometric}, and two MCMC algorithms---one based on Gibbs sampler (Gibbs) \citep{griffiths2004finding}, and the other based on a partially collapsed Gibbs sampler (pcLDA) \citep{magnusson2018sparse, terenin2018polya}. The hyper-parameters of the baselines are set as their default, except that the prior of the mixing weights in Gibbs and pcLDA is set as uniform as ours.  For our algorithm, the number of MCMC samples is $100$ without burn-in; the stopping criterion is that the relative change of likelihood goes below $10^{-9}$ or that $200$ EM iterations are completed, whichever comes first. 

To evaluate the results, we employ the following three metrics. \emph{Topic Coherence} is used to measure the single-topic quality, defined as $\sum_{l=1}^k\sum_{v_1, v_2 \in \mathcal V_l} \log \left(\nicefrac{\text{freq}(v_1, v_2) + \epsilon} {\text{freq}(v_2)}\right)$, where $\mathcal V_l$ is the leading 20 words for topic $l$, $\text{freq}(\cdot)$ is the occurrence count, and $\epsilon$ is a small constant added to avoid numerical issues. \emph{Similarity Count} is used to measure similarity between topics \citep{arora2013practical,huang2016anchor}; it is obtained simply by adding up the overlapped words across $\mathcal V_l$.  \emph {Perplexity Score} is used to measure  goodness of fit, which is the multiplicative inverse of the likelihood, normalized by the number of words. For the first metric, the larger the better; for the latter two, the smaller the better. (Detailed definition of these three metrics can be found in Appendix F.)

{\colorblue
In practice, the number of topics $k$ is unknown. We propose a procedure to select $k$ based on the \emph{effective rank} of the sample document-term matrix $\hat{\mathbf U}$. Since the topic matrix $\mathbf{C}$ is assumed to have full rank (Theorem \ref{thm:idf}), the true term-document matrix $\mathbf{U}$ has rank $k$. By Weyl's inequality \citep{weyl1912asymptotische}, the singular values of $\hat{\mathbf{U}}$ are expected to be close to those of $\mathbf{U}$. Therefore we can plot the ordered singular values of $\hat{\mathbf{U}}$ versus its index, and then select $k$ by detecting the location of a significant drop of the curve.  See Appendix F for a simulation illustrating this approach.
}

\begin{table}[h!]
\caption{Experiment results on the NIPS and the Daily Kos Datasets.}
\centering
 \small{\begin{tabular}{lrrrrr|rrrrr}
 \hline
 & \multicolumn{5}{c|}{NIPS}  & \multicolumn{5}{c}{Daily Kos}\\
 & AnchorF & GDM   & Gibbs &   pcLDA  & $\text{MC}^2$-EM & AnchorF & GDM   & Gibbs &  pcLDA  & $\text{MC}^2$-EM       \\
  \hline 
  \multicolumn{11}{c}{Topic Coherence }\\
 \hline \hline
$k=5$         & -904  & -501  & -365   & -355  & \textbf{-342} & -699  & \textbf{-643}  & -752  & -709    & -723   \\
$k=10$        & -1954 & -1083 & -960   & \textbf{-942}  & -975& -1659 & \textbf{-1551} & -1708 & -1609   & -1614 \\
$k=15$        & -2935 & -1770 & -1648  & -1599 & \textbf{-1573} & -2727 & \textbf{-2307} & -2465 & -2380   & -2411 \\
$k=20$        & -3664 & -2409 & -2314  & -2373 & \textbf{-2254}  & -3942 & \textbf{-3182} & -3840 & -3115   & -3299 \\
 \hline 
  \multicolumn{11}{c}{Similarity Counts}\\
  \hline \hline
$k=5$         & 24  & \textbf{10}    & 25    & 26    &24 & 24  & \textbf{14}  & 23  & 25  & 25  \\
$k=10$        & 69  & \textbf{44}    & 63    & 67    &63& 85  & \textbf{55}  & \textbf{55}  & 66  & 57  \\
$k=15$        & 102 & \textbf{98}    & 99    & 99      &102  & 151 & 111 & \textbf{78}  & 103  & 90  \\
$k=20$        & 154 & 161   & \textbf{134}   & 155     &147  & 224 & 175 & \textbf{116} & 153 & 143 \\
 \hline 
  \multicolumn{11}{c}{Perplexity Score}\\
  \hline \hline
$k=5$         & 4431 & 2955  & 2256  & 2183  & \textbf{2182} & 2252 & 2252 & 1755 & 1758 & \textbf{1724} \\
$k=10$        & 4317  & 2479  & 2067 & \textbf{1973}    & \textbf{1973} & 2124 & 2004 & 1546 & 1532 & \textbf{1507} \\
$k=15$        & 4176  & 2273  & 1975 & \textbf{1870}  & 1874& 2061 & 1912 & 1452 & 1438 & \textbf{1404} \\
$k=20$        & 3877  & 2166  & 1918 & 1801  & \textbf{1800}& 2012 & 1791 & 1405 & 1384 & \textbf{1342}\\
\hline
\end{tabular}
\label{tab:real_data_results}}
\end{table}

\begin{table}[h!]
\caption{Results on the Daily Kos dataset based on $k=7$ chosen by the singular values plot.}
\centering
 \small{\begin{tabular}{lrrrrr}
 \hline
  & \multicolumn{5}{c}{Daily Kos ($k=7$)}\\
 & AnchorF & GDM   & Gibbs &  pcLDA  & $\text{MC}^2$-EM       \\
  \hline 
Topic Coherence & \textbf{-998}  & -1007  & -1095  & -1090    & -1053  \\
 \hline 
Similarity Counts & 47  & \textbf{36}  & 40  & 40    & 40   \\
 \hline 
Perplexity Score & 2190  & 2147  & 1649  & 1643    & \textbf{1607}   \\
 \hline 
\end{tabular}
\label{tab:kos_7}}
\end{table}

The results are summarized in Table \ref{tab:real_data_results} and Table \ref{tab:kos_7}, where $k=5$ and $k=7$, respectively, are the recommended number of topics for NIPS and Daily Kos dataset, chosen by the procedure mentioned above (the singular values plots can be found in Appendix F).
The best score in each case is highlighted in boldface. Overall, our estimator ($\text{MC}^2$-EM) gives promising results. 
For all three metrics in both datasets, it gives the highest score or a score close to the highest. For topic coherence, it is the best for $k=5,15,$ and $20$ in NIPS. 
For similarity counts, it performs similarly to Gibbs and pcLDA in both datasets, and in Daily Kos largely outperforms AnchorF and GDM for $k=10, 15$, and $20$. For perplexity score, it is consistently the best in Daily Kos, and in NIPS except for $k=15$; its scores are  very close  to the best one given by pcLDA.

The leading $10$ topic words given by $\text{MC}^2$-EM can be found in the supplementary material.

\subsubsection{New York Taxi-trip Dataset}
\label{subsubsec:real_taxi}
Reinforcement learning algorithms have been widely used in solving real-world Markov decision problems. Use of a compact representation of the underlying states, known as state aggregation, is crucial for those algorithms to scale with large datasets. As shown below, learning a soft state aggregation \citep{singh1995reinforcement} is equivalent to estimating a  topic model.  

We say that a Markov chain $X_0, X_1, \cdots, X_T$ admits a \emph{soft state aggregation} with $k$ meta-states, if there exist random variables $Z_0, Z_1, \cdots, Z_{n-1}\in \{1,\cdots, k\}$ such that 
\begin{equation} \label{eq:soft:state}
\mathds P (X_{t+1}|X_t) = \sum_{l=1}^k \mathds P(Z_t=l|X_t)\cdot \mathds P (X_{t+1}|Z_t = l),
\end{equation}
for all $t$ with probability 1 \citep{singh1995reinforcement}. Here, $\mathds P(Z_t=l|X_t)$ and $\mathds P(X_{t+1}|Z_t = l)$ are independent of $t$ and are referred to as the \emph{aggregation distributions} and \emph {disaggregation distributions}. Let $\mathbf U \in \mathds R^{V\times V}$ denote the transition matrix with $U_{ji} = \mathds P(X_{t+1}=j | X_t = i)$. Let  $\mathbf{C} \in \mathds R^{V\times k}$ and $\mathbf{W} \in \mathds{R}^{k\times V}$ denote the disaggregation  and aggregation distribution matrices, respectively, with $C_{jl} = \mathds P(X_{t+1}=j|Z_t = l)$ and $W_{li} = \mathds P(Z_t=l|X_t=i)$. Then (\ref{eq:soft:state}) can be written as $\mathbf{U} = \mathbf{C} \mathbf{W}$, the same as the matrix form for topic modelling. 

In this section, we consider a New York taxi-trip\footnote{\href{ https://www1.nyc.gov/site/tlc/about/tlc-trip-record-data.page}{https://www1.nyc.gov/site/tlc/about/tlc-trip-record-data.page}} dataset. This dataset contains $\sum_{i=1}^d n_i  = 7,667,792$ New York City yellow cab trips in January 2019. The location information is discretized into $V=263$ taxi zones with $69$ in Manhattan, $69$ in Queens, $61$ in Brooklyn, $43$ in Bronx, $20$ in Staten Island, and 1 in EWR. For each trip, we are given its  pick-up  and drop-off zones. On the left of Figure \ref{fig:taxi_example}, we plot  $30$ example trips from the  data. Following a similar analysis of this dataset from \cite{duan2019state}, we aim to merge the $V=263$ taxi zones into meta-states via soft state aggregation.

\begin{figure}[h!]
\begin{minipage}{0.3\textwidth}
    \centering
    \includegraphics[width=0.8\textwidth]{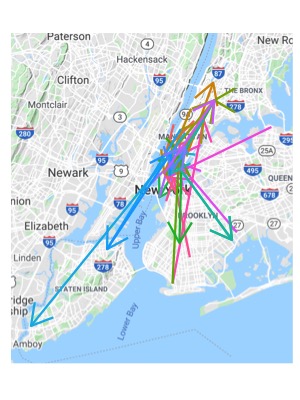}
\end{minipage}
\begin{minipage}{0.69\textwidth}
    \centering
    \includegraphics[width=0.99\textwidth]{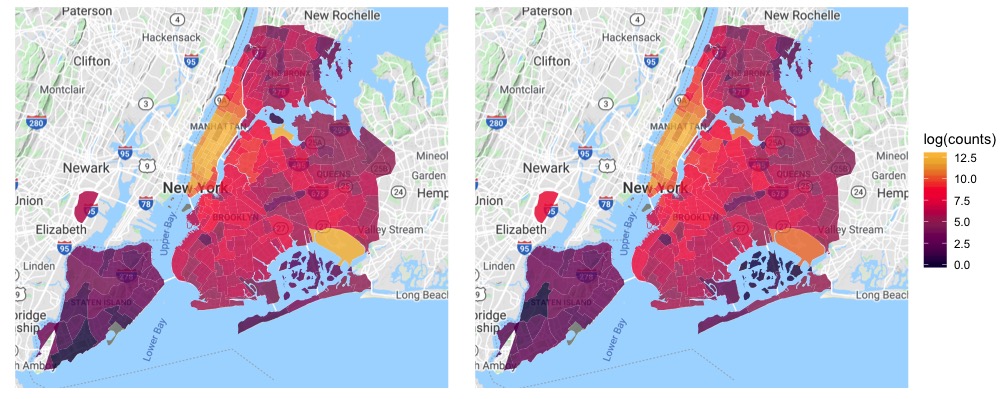}
\end{minipage}
    \caption{NYC taxi-trip data glance. Left: 30 example trips with arrows pointing from pick-up zones to drop-off zones. Middle: the pick-up distribution. Right: the drop-off distribution.}
    \label{fig:taxi_example}
\end{figure}

\begin{figure}
    \centering
    \includegraphics[width=0.9\textwidth]{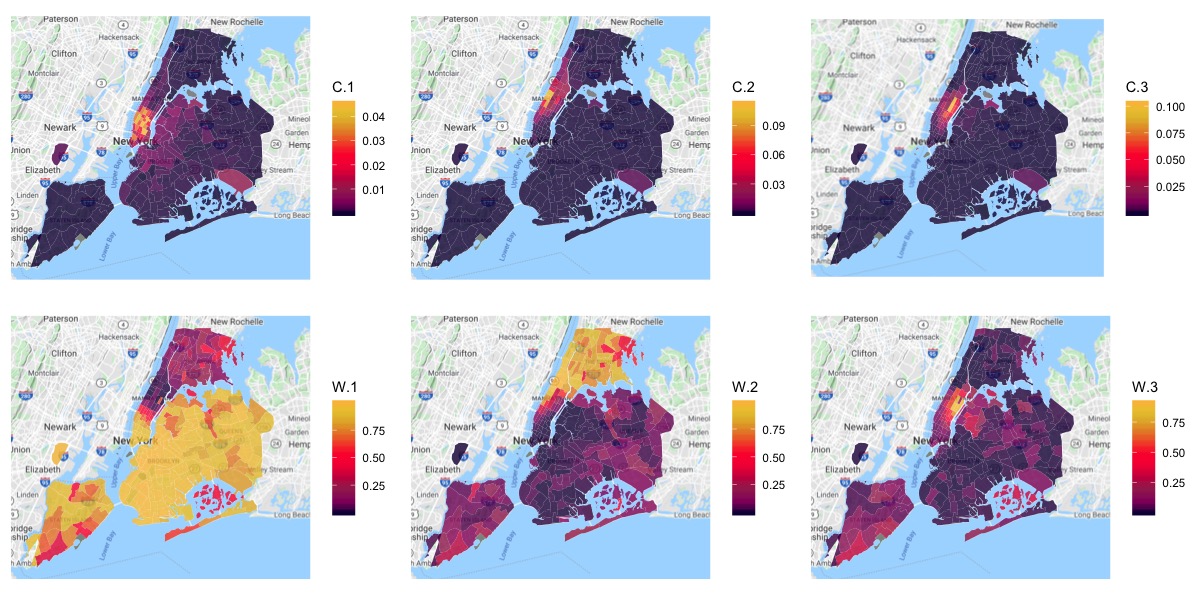}
    \caption{Estimation results for NYC taxi-trip data for $k=3$. The top three plots represent the estimated disaggregation distributions (topic vectors) $\mathbf {\hat C}_1, \mathbf {\hat C}_2, \mathbf {\hat C}_3 \in \mathds{R}^V$, where $\mathbf {\hat C}_l = \mathds P(X_{t+1}|Z_t = l)$. The bottom three plots represent the estimated aggregation distributions $\mathbf {\hat W}_1, \mathbf {\hat W}_2, \mathbf {\hat W}_3 \in \mathds{R}^V$, where $\mathbf {\hat W}_l = \mathds P(Z_t=l|X_t)$.}
    \label{fig:taxi_k3}
\end{figure}

In the middle and the right of Figure \ref{fig:taxi_example}, we use heat maps to visualize the distributions of the trip counts for pick-up and drop-off over $V=263$ zones. Most of the traffic concentrates in midtown and downtown Manhattan, as well as at the JFK airport on the southeast side of Queens, for both pick-up and drop-off. 

At the top of Figure \ref{fig:taxi_k3} we plot the estimation results for the drop-off distributions conditioned on the meta-state, $\mathds P(X_{t+1}|Z_t=l)$. We observe that the drop-off traffic is decomposed into three clusters, (1) downtown Manhattan, (2) west midtown Manhattan, and (3) east midtown Manhattan, for each of the three meta states (topics); this implies that people dropped off in downtown Manhattan may come from the first meta state, and that people dropped off in midtown east and west may come from the second and the third meta states, respectively. The JFK airport has a relatively high probability mass in all three states but is not on the top list for any of them, which implies that people arriving at JFK may come from anywhere in NYC. 

At the bottom of Figure \ref{fig:taxi_k3} we plot the conditional probability over the meta-state (topics), given the pick-up zone, $\mathds P(Z_t =l | X_{t})$. The three meta states consist of (1) Staten Island, Brooklyn, Queens, and downtown Manhattan; (2) uptown Manhattan and Bronx; and (3) east midtown Manhattan. Note that the scales of the estimates for $\mathbf C$ and $\mathbf W$ are quite different. In specific, the sum of values over each map of the top three is 1 since $\sum_{v=1}^V \mathds P(X_{t+1} = v|Z_t=l) =1, \;l=1,2,3$, while the sum of values for each zone over the bottom three maps is 1 since $\sum_{l=1}^3 \mathds P(Z_t =l | X_{t} = v) =1 , \;v = 1,\cdots, V$. The interpretation of, say, the second meta state, is that the destinations of trips starting from uptown Manhattan and Bronx are likely to be in midtown Manhattan. We observe that the pick-up and the drop-off locations in the same meta state are generally close regionally; this result is reasonable, as people tend to take a taxi for short trips, preferring less expensive public transportation for longer trips.

The estimated disaggregation and aggregation distributions plots for $k=9$ can be found in the supplementary material. They reveal that the traffic in the first eight meta states is within Manhattan, which is the most heavy-traffic place in NYC, and that the partition is more fine-grained compared with the results for $k=3$. Similar to the results for $k=3$, the pick-up and drop-off locations for each meta state are regionally close at this time. It is interesting that  such a strong regional relationship emerges, since the data fed into our algorithm do not contain any regional information. 

{\colorblue
\section{Discussion}
In this paper, we introduce a new set of geometric conditions for topic model identifiability under volume minimization, a weaker set than the commonly used separability conditions. For computation, we propose a maximum likelihood estimator of the latent topics matrix, based on an integrated likelihood. Our approach implicitly promotes volume minimization. 
We conduct finite-sample error analysis for the estimator and discuss the connection of our results to existing ones. 
Experiments on simulated and real datasets demonstrate the strength of our method.
Our work makes an important contribution to the general theory of estimation of latent structures arising for topic models. 
Some interesting future work might include: (1)~exploring a sufficient and necessary condition for model identifiability, as the SS condition is not necessary; (2)~providing explicit verifiable sufficient conditions for the $(\alpha, \beta)$-SS condition --- we conjecture that the $(\alpha, \beta)$-SS condition can be implied by the SS condition; (3)~establishing the minimax rate of convergence of topic matrix estimation, and verifying  whether the proposed estimator is (nearly) optimal. 
Although presented in the context of topic models, results from our work are immediately applicable to a wide range of mixed membership models arising from various machine learning applications. In addition, we may incorporate additional low-dimensional structures into the model, such as (group) sparsity, to enhance the estimation accuracy.
}

\bibliographystyle{chicago}

\bibliography{ref}

\begin{thebibliography}{}

\bibitem[\protect\citeauthoryear{Anandkumar, Foster, Hsu, Kakade, and
  Liu}{Anandkumar et~al.}{2012}]{anandkumar2012spectral}
Anandkumar, A., D.~P. Foster, D.~J. Hsu, S.~M. Kakade, and Y.-K. Liu (2012).
\newblock A spectral algorithm for latent dirichlet allocation.
\newblock In {\em Advances in Neural Information Processing Systems}, pp.\
  917--925.

\bibitem[\protect\citeauthoryear{Anandkumar, Ge, Hsu, Kakade, and
  Telgarsky}{Anandkumar et~al.}{2014}]{anandkumar2014tensor}
Anandkumar, A., R.~Ge, D.~Hsu, S.~M. Kakade, and M.~Telgarsky (2014).
\newblock Tensor decompositions for learning latent variable models.
\newblock {\em Journal of Machine Learning Research\/}~{\em 15\/}(1),
  2773--2832.

\bibitem[\protect\citeauthoryear{Anandkumar, Hsu, and Kakade}{Anandkumar
  et~al.}{2012}]{anandkumar2012method}
Anandkumar, A., D.~Hsu, and S.~M. Kakade (2012).
\newblock A method of moments for mixture models and hidden markov models.
\newblock In {\em Conference on Learning Theory}, pp.\  33--1. JMLR Workshop
  and Conference Proceedings.

\bibitem[\protect\citeauthoryear{Arora, Ge, Halpern, Mimno, Moitra, Sontag, Wu,
  and Zhu}{Arora et~al.}{2013}]{arora2013practical}
Arora, S., R.~Ge, Y.~Halpern, D.~Mimno, A.~Moitra, D.~Sontag, Y.~Wu, and M.~Zhu
  (2013).
\newblock A practical algorithm for topic modeling with provable guarantees.
\newblock In {\em International Conference on Machine Learning}, pp.\
  280--288.

\bibitem[\protect\citeauthoryear{Arora, Ge, and Moitra}{Arora
  et~al.}{2012}]{arora2012learning}
Arora, S., R.~Ge, and A.~Moitra (2012).
\newblock Learning topic models--going beyond {SVD}.
\newblock In {\em 2012 IEEE 53rd Annual Symposium on Foundations of Computer
  Science}, pp.\  1--10.

\bibitem[\protect\citeauthoryear{Azar, Fiat, Karlin, McSherry, and Saia}{Azar
  et~al.}{2001}]{azar2001spectral}
Azar, Y., A.~Fiat, A.~Karlin, F.~McSherry, and J.~Saia (2001).
\newblock Spectral analysis of data.
\newblock In {\em Proceedings of the thirty-third annual ACM symposium on
  Theory of computing}, pp.\  619--626.

\bibitem[\protect\citeauthoryear{Bansal, Bhattacharyya, and Kannan}{Bansal
  et~al.}{2014}]{bansal2014provable}
Bansal, T., C.~Bhattacharyya, and R.~Kannan (2014).
\newblock A provable svd-based algorithm for learning topics in dominant
  admixture corpus.
\newblock {\em Advances in Neural Information Processing Systems\/}~{\em 27},
  1997--2005.

\bibitem[\protect\citeauthoryear{Berger, Liseo, and Wolpert}{Berger
  et~al.}{1999}]{berger1999integrated}
Berger, J.~O., B.~Liseo, and R.~L. Wolpert (1999).
\newblock Integrated likelihood methods for eliminating nuisance parameters.
\newblock {\em Statistical science\/}~{\em 14\/}(1), 1--28.

\bibitem[\protect\citeauthoryear{Blei, Ng, and Jordan}{Blei
  et~al.}{2003}]{blei2003latent}
Blei, D.~M., A.~Y. Ng, and M.~I. Jordan (2003).
\newblock Latent dirichlet allocation.
\newblock {\em Journal of Machine Learning Research\/}~{\em 3\/}(Jan),
  993--1022.

\bibitem[\protect\citeauthoryear{Boyd and Vandenberghe}{Boyd and
  Vandenberghe}{2004}]{boyd2004convex}
Boyd, S. and L.~Vandenberghe (2004).
\newblock {\em Convex optimization}.
\newblock Cambridge University Press.

\bibitem[\protect\citeauthoryear{Brunel, Klusowski, and Yang}{Brunel
  et~al.}{2021}]{brunel2021estimation}
Brunel, V.-E., J.~M. Klusowski, and D.~Yang (2021).
\newblock Estimation of convex supports from noisy measurements.
\newblock {\em Bernoulli\/}~{\em 27\/}(2), 772--793.

\bibitem[\protect\citeauthoryear{Chen, Chen, and Li}{Chen
  et~al.}{2016}]{chen2016perturbation}
Chen, Y.~M., X.~S. Chen, and W.~Li (2016).
\newblock On perturbation bounds for orthogonal projections.
\newblock {\em Numerical Algorithms\/}~{\em 73\/}(2), 433--444.

\bibitem[\protect\citeauthoryear{Craig}{Craig}{1994}]{craig1994minimum}
Craig, M.~D. (1994).
\newblock Minimum-volume transforms for remotely sensed data.
\newblock {\em IEEE Transactions on Geoscience and Remote Sensing\/}~{\em
  32\/}(3), 542--552.

\bibitem[\protect\citeauthoryear{Davis and Kahan}{Davis and
  Kahan}{1970}]{davis1970rotation}
Davis, C. and W.~M. Kahan (1970).
\newblock The rotation of eigenvectors by a perturbation. {III}.
\newblock {\em SIAM Journal on Numerical Analysis\/}~{\em 7\/}(1), 1--46.

\bibitem[\protect\citeauthoryear{Devroye et~al.}{Devroye
  et~al.}{1983}]{devroye1983equivalence}
Devroye, L. et~al. (1983).
\newblock The equivalence of weak, strong and complete convergence in $ l\_1 $
  for kernel density estimates.
\newblock {\em The Annals of Statistics\/}~{\em 11\/}(3), 896--904.

\bibitem[\protect\citeauthoryear{Donoho and Stodden}{Donoho and
  Stodden}{2004}]{donoho2004does}
Donoho, D. and V.~Stodden (2004).
\newblock When does non-negative matrix factorization give a correct
  decomposition into parts?
\newblock In {\em Advances in Neural Information Processing Systems}, pp.\
  1141--1148.

\bibitem[\protect\citeauthoryear{Duan, Ke, and Wang}{Duan
  et~al.}{2019}]{duan2019state}
Duan, Y., T.~Ke, and M.~Wang (2019).
\newblock State aggregation learning from markov transition data.
\newblock In {\em Advances in Neural Information Processing Systems}, pp.\
  4488--4497.

\bibitem[\protect\citeauthoryear{Freund and Orlin}{Freund and
  Orlin}{1985}]{freund1985complexity}
Freund, R.~M. and J.~B. Orlin (1985).
\newblock On the complexity of four polyhedral set containment problems.
\newblock {\em Mathematical programming\/}~{\em 33\/}(2), 139--145.

\bibitem[\protect\citeauthoryear{Fu, Ma, Huang, and Sidiropoulos}{Fu
  et~al.}{2015}]{fu2015blind}
Fu, X., W.-K. Ma, K.~Huang, and N.~D. Sidiropoulos (2015).
\newblock Blind separation of quasi-stationary sources: Exploiting convex
  geometry in covariance domain.
\newblock {\em IEEE Transactions on Signal Processing\/}~{\em 63\/}(9),
  2306--2320.

\bibitem[\protect\citeauthoryear{Ge and Zou}{Ge and
  Zou}{2015}]{ge2015intersecting}
Ge, R. and J.~Zou (2015).
\newblock Intersecting faces: Non-negative matrix factorization with new
  guarantees.
\newblock In {\em International Conference on Machine Learning}, pp.\
  2295--2303.

\bibitem[\protect\citeauthoryear{Goldenshluger and Tsybakov}{Goldenshluger and
  Tsybakov}{2004}]{goldenshluger2004estimating}
Goldenshluger, A. and A.~Tsybakov (2004).
\newblock Estimating the endpoint of a distribution in the presence of additive
  observation errors.
\newblock {\em Statistics \& probability letters\/}~{\em 68\/}(1), 39--49.

\bibitem[\protect\citeauthoryear{G{\"o}tze, Sambale, and Sinulis}{G{\"o}tze
  et~al.}{2019}]{gotze2019higher}
G{\"o}tze, F., H.~Sambale, and A.~Sinulis (2019).
\newblock Higher order concentration for functions of weakly dependent random
  variables.
\newblock {\em Electronic Journal of Probability\/}~{\em 24}, 1--19.

\bibitem[\protect\citeauthoryear{Griffiths and Steyvers}{Griffiths and
  Steyvers}{2004}]{griffiths2004finding}
Griffiths, T.~L. and M.~Steyvers (2004).
\newblock Finding scientific topics.
\newblock {\em Proceedings of the National Academy of Sciences\/}~{\em
  101\/}(suppl 1), 5228--5235.

\bibitem[\protect\citeauthoryear{Hofmann}{Hofmann}{1999}]{hofm99}
Hofmann, T. (1999).
\newblock Probabilistic latent semantic analysis.
\newblock In {\em Uncertainty in Artificial Intelligence}, pp.\  289--296.

\bibitem[\protect\citeauthoryear{Huang, Fu, and Sidiropoulos}{Huang
  et~al.}{2016}]{huang2016anchor}
Huang, K., X.~Fu, and N.~D. Sidiropoulos (2016).
\newblock Anchor-free correlated topic modeling: Identifiability and algorithm.
\newblock In {\em Advances in Neural Information Processing Systems}, pp.\
  1786--1794.

\bibitem[\protect\citeauthoryear{Huang, Sidiropoulos, and Swami}{Huang
  et~al.}{2014}]{huang2014non}
Huang, K., N.~D. Sidiropoulos, and A.~Swami (2014).
\newblock Non-negative matrix factorization revisited: Uniqueness and algorithm
  for symmetric decomposition.
\newblock {\em IEEE Transactions on Signal Processing\/}~{\em 62\/}(1),
  211--224.

\bibitem[\protect\citeauthoryear{Jang and Hero}{Jang and
  Hero}{2019}]{jang2019minimum}
Jang, B. and A.~Hero (2019).
\newblock Minimum volume topic modeling.
\newblock In {\em International Conference on Artificial Intelligence and
  Statistics}, pp.\  3013--3021.

\bibitem[\protect\citeauthoryear{Javadi and Montanari}{Javadi and
  Montanari}{2020}]{javadi2020nonnegative}
Javadi, H. and A.~Montanari (2020).
\newblock Nonnegative matrix factorization via archetypal analysis.
\newblock {\em Journal of the American Statistical Association\/}~{\em
  115\/}(530), 896--907.

\bibitem[\protect\citeauthoryear{Ke and Wang}{Ke and Wang}{2017}]{ke2017new}
Ke, Z.~T. and M.~Wang (2017).
\newblock A new svd approach to optimal topic estimation.
\newblock {\em arXiv preprint arXiv:1704.07016\/}.

\bibitem[\protect\citeauthoryear{Kleinberg and Sandler}{Kleinberg and
  Sandler}{2003}]{kleinberg2003convergent}
Kleinberg, J. and M.~Sandler (2003).
\newblock Convergent algorithms for collaborative filtering.
\newblock In {\em Proceedings of the 4th ACM conference on Electronic
  commerce}, pp.\  1--10.

\bibitem[\protect\citeauthoryear{Kleinberg and Sandler}{Kleinberg and
  Sandler}{2008}]{kleinberg2008using}
Kleinberg, J. and M.~Sandler (2008).
\newblock Using mixture models for collaborative filtering.
\newblock {\em Journal of Computer and System Sciences\/}~{\em 74\/}(1),
  49--69.

\bibitem[\protect\citeauthoryear{Magnusson, Jonsson, Villani, and
  Broman}{Magnusson et~al.}{2018}]{magnusson2018sparse}
Magnusson, M., L.~Jonsson, M.~Villani, and D.~Broman (2018).
\newblock Sparse partially collapsed mcmc for parallel inference in topic
  models.
\newblock {\em Journal of Computational and Graphical Statistics\/}~{\em
  27\/}(2), 449--463.

\bibitem[\protect\citeauthoryear{McSherry}{McSherry}{2001}]{mcsherry2001spectral}
McSherry, F. (2001).
\newblock Spectral partitioning of random graphs.
\newblock In {\em Proceedings 42nd IEEE Symposium on Foundations of Computer
  Science}, pp.\  529--537. IEEE.

\bibitem[\protect\citeauthoryear{Miao and Qi}{Miao and
  Qi}{2007}]{miao2007endmember}
Miao, L. and H.~Qi (2007).
\newblock Endmember extraction from highly mixed data using minimum volume
  constrained nonnegative matrix factorization.
\newblock {\em IEEE Transactions on Geoscience and Remote Sensing\/}~{\em
  45\/}(3), 765--777.

\bibitem[\protect\citeauthoryear{Nascimento and Dias}{Nascimento and
  Dias}{2005}]{nascimento2005vertex}
Nascimento, J.~M. and J.~M. Dias (2005).
\newblock Vertex component analysis: A fast algorithm to unmix hyperspectral
  data.
\newblock {\em IEEE transactions on Geoscience and Remote Sensing\/}~{\em
  43\/}(4), 898--910.

\bibitem[\protect\citeauthoryear{Nguyen}{Nguyen}{2015}]{nguyen2015posterior}
Nguyen, X. (2015).
\newblock Posterior contraction of the population polytope in finite admixture
  models.
\newblock {\em Bernoulli\/}~{\em 21\/}(1), 618--646.

\bibitem[\protect\citeauthoryear{Nielsen et~al.}{Nielsen
  et~al.}{2000}]{nielsen2000stochastic}
Nielsen, S.~F. et~al. (2000).
\newblock The stochastic em algorithm: Estimation and asymptotic results.
\newblock {\em Bernoulli\/}~{\em 6\/}(3), 457--489.

\bibitem[\protect\citeauthoryear{Papadimitriou, Raghavan, Tamaki, and
  Vempala}{Papadimitriou et~al.}{2000}]{papadimitriou2000latent}
Papadimitriou, C.~H., P.~Raghavan, H.~Tamaki, and S.~Vempala (2000).
\newblock Latent semantic indexing: A probabilistic analysis.
\newblock {\em Journal of Computer and System Sciences\/}~{\em 61\/}(2),
  217--235.

\bibitem[\protect\citeauthoryear{Perrone, Jenkins, Spano, and Teh}{Perrone
  et~al.}{2016}]{perrone2016poisson}
Perrone, V., P.~A. Jenkins, D.~Spano, and Y.~W. Teh (2016).
\newblock Poisson random fields for dynamic feature models.
\newblock {\em arXiv preprint arXiv:1611.07460\/}.

\bibitem[\protect\citeauthoryear{Recht, Re, Tropp, and Bittorf}{Recht
  et~al.}{2012}]{recht2012factoring}
Recht, B., C.~Re, J.~Tropp, and V.~Bittorf (2012).
\newblock Factoring nonnegative matrices with linear programs.
\newblock In {\em Advances in Neural Information Processing Systems}, pp.\
  1214--1222.

\bibitem[\protect\citeauthoryear{Singh, Jaakkola, and Jordan}{Singh
  et~al.}{1995}]{singh1995reinforcement}
Singh, S.~P., T.~Jaakkola, and M.~I. Jordan (1995).
\newblock Reinforcement learning with soft state aggregation.
\newblock {\em Advances in neural information processing systems\/}, 361--368.

\bibitem[\protect\citeauthoryear{Tang, Meng, Nguyen, Mei, and Zhang}{Tang
  et~al.}{2014}]{tang2014understanding}
Tang, J., Z.~Meng, X.~Nguyen, Q.~Mei, and M.~Zhang (2014).
\newblock Understanding the limiting factors of topic modeling via posterior
  contraction analysis.
\newblock In {\em International Conference on Machine Learning}, pp.\
  190--198.

\bibitem[\protect\citeauthoryear{Terenin, Magnusson, Jonsson, and
  Draper}{Terenin et~al.}{2018}]{terenin2018polya}
Terenin, A., M.~Magnusson, L.~Jonsson, and D.~Draper (2018).
\newblock Polya urn latent dirichlet allocation: A doubly sparse massively
  parallel sampler.
\newblock {\em IEEE Transactions on Pattern Analysis and Machine
  Intelligence\/}~{\em 41\/}(7), 1709--1719.

\bibitem[\protect\citeauthoryear{Wang}{Wang}{2019}]{Wang2019}
Wang, Y. (2019).
\newblock Convergence rates of latent topic models under relaxed
  identifiability conditions.
\newblock {\em Electronic Journal of Statistics\/}~{\em 13\/}(1), 37--66.

\bibitem[\protect\citeauthoryear{Weyl}{Weyl}{1912}]{weyl1912asymptotische}
Weyl, H. (1912).
\newblock Das asymptotische verteilungsgesetz der eigenwerte linearer
  partieller differentialgleichungen (mit einer anwendung auf die theorie der
  hohlraumstrahlung).
\newblock {\em Mathematische Annalen\/}~{\em 71\/}(4), 441--479.

\bibitem[\protect\citeauthoryear{Winter}{Winter}{1999}]{winter1999n}
Winter, M.~E. (1999).
\newblock N-findr: An algorithm for fast autonomous spectral end-member
  determination in hyperspectral data.
\newblock In {\em Imaging Spectrometry V}, Volume 3753, pp.\  266--275.
  International Society for Optics and Photonics.

\bibitem[\protect\citeauthoryear{Yurochkin and Nguyen}{Yurochkin and
  Nguyen}{2016}]{yurochkin2016geometric}
Yurochkin, M. and X.~Nguyen (2016).
\newblock Geometric {D}irichlet means algorithm for topic inference.
\newblock In {\em Advances in Neural Information Processing Systems}, pp.\
  2505--2513.

\end{thebibliography}

\renewcommand\thesection{\arabic{section}}

\pagebreak
\begin{center}
\textbf{\LARGE Supplementary Material: \\
        Learning Topic Models: Identifiability and Finite-Sample Analysis}
\end{center}

\setcounter{equation}{0}
\setcounter{figure}{0}
\setcounter{table}{0}
\setcounter{section}{0}
\makeatletter
\renewcommand{\theequation}{\Alph{section}.\arabic{equation}}
\renewcommand{\thefigure}{S\arabic{figure}}
\renewcommand{\bibnumfmt}[1]{[S#1]}
\renewcommand{\citenumfont}[1]{S#1}
\renewcommand{\thesection}{\Alph{section}}

The supplementary material is organized as follows. 
\begin{itemize}
    \item Section \ref{Sec:example_idf}: Discussion on identifiability related to Remark \ref{rmk:compare_javadi}.
    \item Section \ref{Sec:random_design_error}: Error analysis and consistency under stochastic mixing weights.
    \item Section \ref{app:computation}: Derivation of the MCMC-EM algorithm. 
    \item Section \ref{app:pf_thm_1}: Proofs of main theorems. 
    \item Section \ref{app:prf_lemmas}: Proofs of technical lemmas and propositions.
    \item Section \ref{app:add_sim}: Additional simulations and experiments.
    \item Sections \ref{app:nips_topwords} \&   \ref{app:kos_topwords}: Top $10$ words of the latent topics returned by our algorithm for the two  real applications.
    \item Section \ref{app:taxi_trips}: Mined meta states for the taxi-trip dataset.
\end{itemize}

\newpage

{\colorblue
\section{Discussion on Identifiability Related to Remark \ref{rmk:compare_javadi}}\label{Sec:example_idf}

\citet{javadi2020nonnegative} and we both follow the same principle to address the non-identifiability issue ---  among all equivalent parameters that lead to the same statistical model, the one that minimizes a chosen criterion function is used to represent the equivalence class (therefore the most parsimonious representation). However, the adopted criterion functions are different: ours is the volume of $\conv(\mathbf C)$, while  theirs  is the sum of  distances from the vertices of $\conv(\mathbf C)$ (i.e.,  columns of $\mathbf C$)  to the convex hull of $\mathbf{U}$. The criterion function adopted \citet{javadi2020nonnegative} is easier to be formulated into a statistical estimator that minimizes an empirical evaluation of it. However, as we discussed in Section 1.2, our criterion function as the volume of a low-dimensional polytope in a high-dimensional space does not take a simple form, which greatly complicates the estimator construction. Fortunately, we find that maximizing a particular integrated likelihood leads to an estimator that implicitly minimizes the volume.
   
Regarding the two notions of identifiability, minimizers of the two criterion functions are usually different --- except for some special cases, such as when the pure topic documents condition hold so that vertices of $\conv(\mathbf C)$ are data points. Therefore, the two notions of identifiability are not directly comparable. Figure \ref{fig:reply_javadi} helps to illustrate this point. In Figure \ref{fig:reply_javadi}, the grey region is $\conv(\mathbf{U})$ and the black triangle $ABC$ is the unique volume minimizer among all three-vertex convex polytopes enclosing $\conv(\mathbf{U})$. However, triangle $ABC$ is not the minimizer of the criterion function in \citet{javadi2020nonnegative} with the Euclidean distance as the distance function:  it is easy to verify that when the ratio of the height to the base of triangle $ABC$ is larger than 6, the red triangle $FGH$ has a smaller summation of distances to the gray region than $ABC$ (see the caption that describes how we construct the red triangle $FGH$). 

   \begin{figure}[H]
    \centering
    \includegraphics[width=0.5\textwidth]{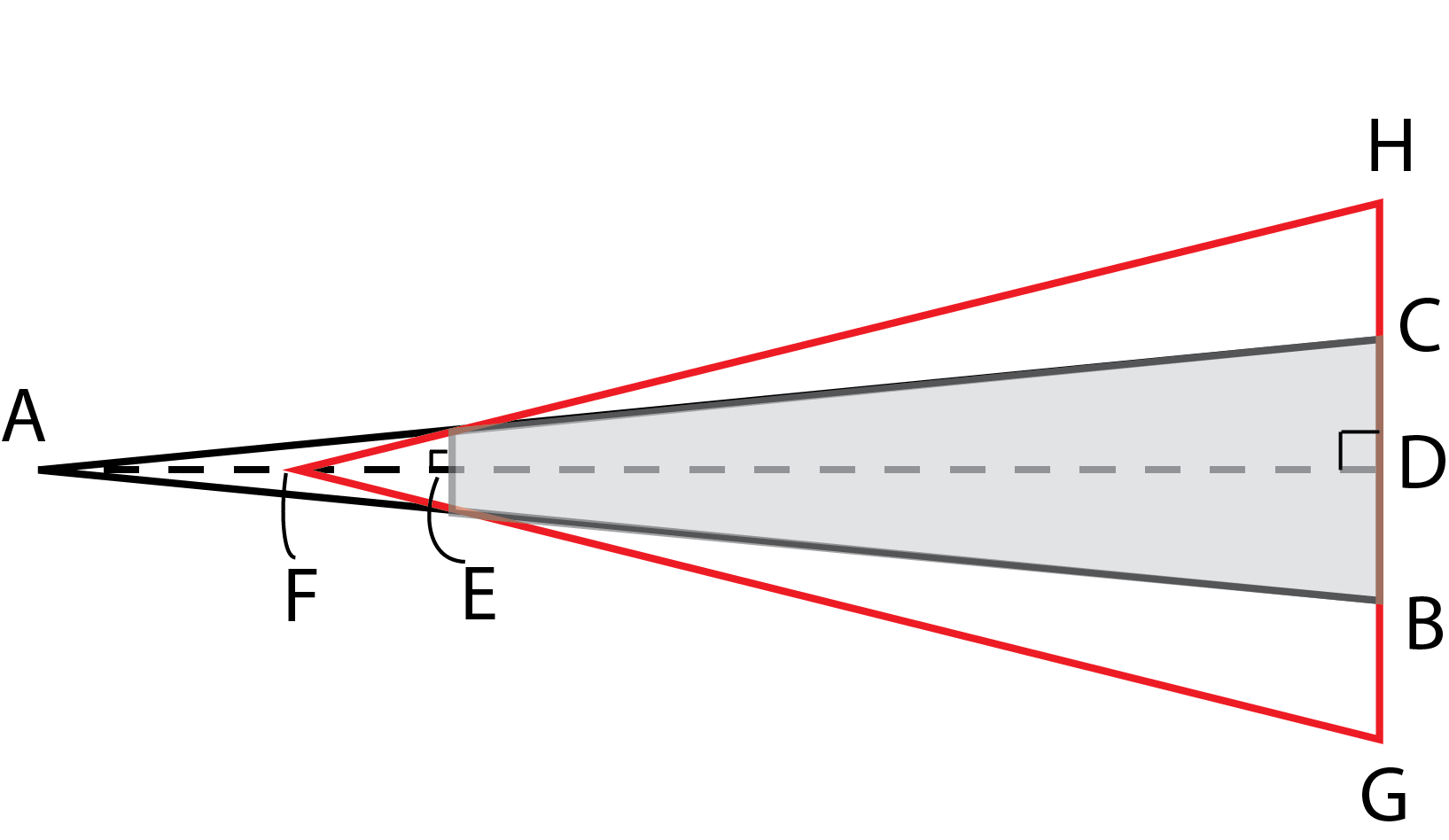}
    \caption{An example ($V=k=3$), in which both  $ABC$ and $FGH$ are isosceles triangles enclosing  $\conv(\mathbf{U})$ (grey region). In addition, $BC = b$, $AD = h$, $AE = h/4$, and $F$ is the midpoint of $AE$.}
    \label{fig:reply_javadi}
    \end{figure}
    
   Under the principle of using the minimizer to represent the whole equivalence class, a trivial identifiability condition is to assume the uniqueness of the minimizer, which is exactly the identifiability condition given in \citet{javadi2020nonnegative}. The drawback, however, is that it is often not trivial, if not impossible, to check whether the minimizer of a criterion function is unique. In  \citet{javadi2020nonnegative}, uniqueness is checked only for a simple case when the vertices of $\conv(\mathbf{C})$ are data points (their Remark 3.1). In contrast, our identifiability condition, the SS condition, is a set of explicit, verifiable conditions. Consider the example given in Figure \ref{fig:reply_javadi}. By our Theorem 2, the model is identifiable with respect to our volume minimization. But, it is difficult to verify whether the model is identifiable in \citet{javadi2020nonnegative}: We do not know whether the triangle $FGH$, although shown to be a better choice than triangle $ABC$, indeed  minimizes the criterion function; even if it does, we do not know whether it is unique. 
   
   In summary, neither definition of identifiability is more general than the other. Since the identification condition in \citet{javadi2020nonnegative} is difficult to check, we are not able to provide an example where the model is identifiable under one notion but not under the other. Due to the same reason, it is unclear whether our SS condition implies their definition of identifiability. Although the two notions of identifiability are not comparable, we would like to highlight that an advantage of our volume minimization criterion is that it helps to justify the empirical success of the Latent Dirichlet Allocation (LDA) model, because the proposed estimator is essentially the maximum likelihood estimator of $\mathbf{C}$ from the LDA model with the prior of $\mathbf{W}$ being the uniform distribution. LDA models with general priors can be interpreted as maximizing the data likelihood while minimizing a weighted volume where a non-uniform volume element is integrated over the convex hull of $\mathbf{C}$ when defining the volume.

}
\newpage

\section{Error Analysis and Consistency under Stochastic Mixing Weights}\label{Sec:random_design_error}

In this appendix, we explore cases in which  $\tw{1}, \cdots \tw{d}$ are random i.i.d. samples from some unknown distribution $\mathcal{P}$ over $\simplex^{k-1}$ (the theoretical result in the main manuscript considers the fixed mixing weights setting). In such cases, we will apply Theorem~\ref{thm:main} to this set of stochastic mixing weights by showing that under a suitable set of conditions to be described below, Assumptions (A1)-(A3) hold with high probability.

\begin{figure}[H]
   \centering
   \begin{subfigure}[t]{0.3\textwidth}
      \includegraphics[width=0.79\textwidth]{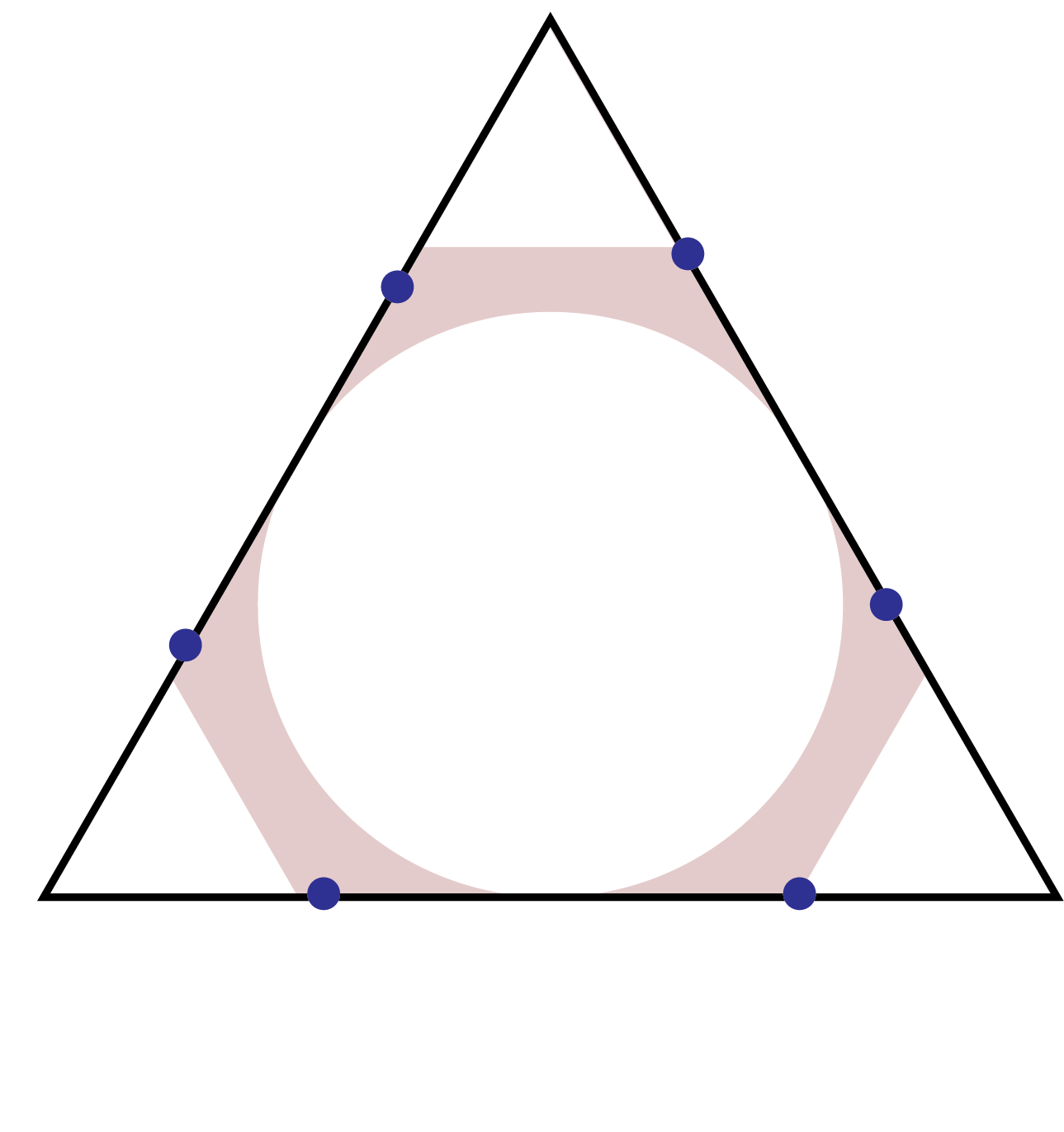}
      \label{fig:supp_2}   
   \end{subfigure}
   \begin{subfigure}[t]{0.3\textwidth}
      \centering
      \includegraphics[width=0.79\textwidth]{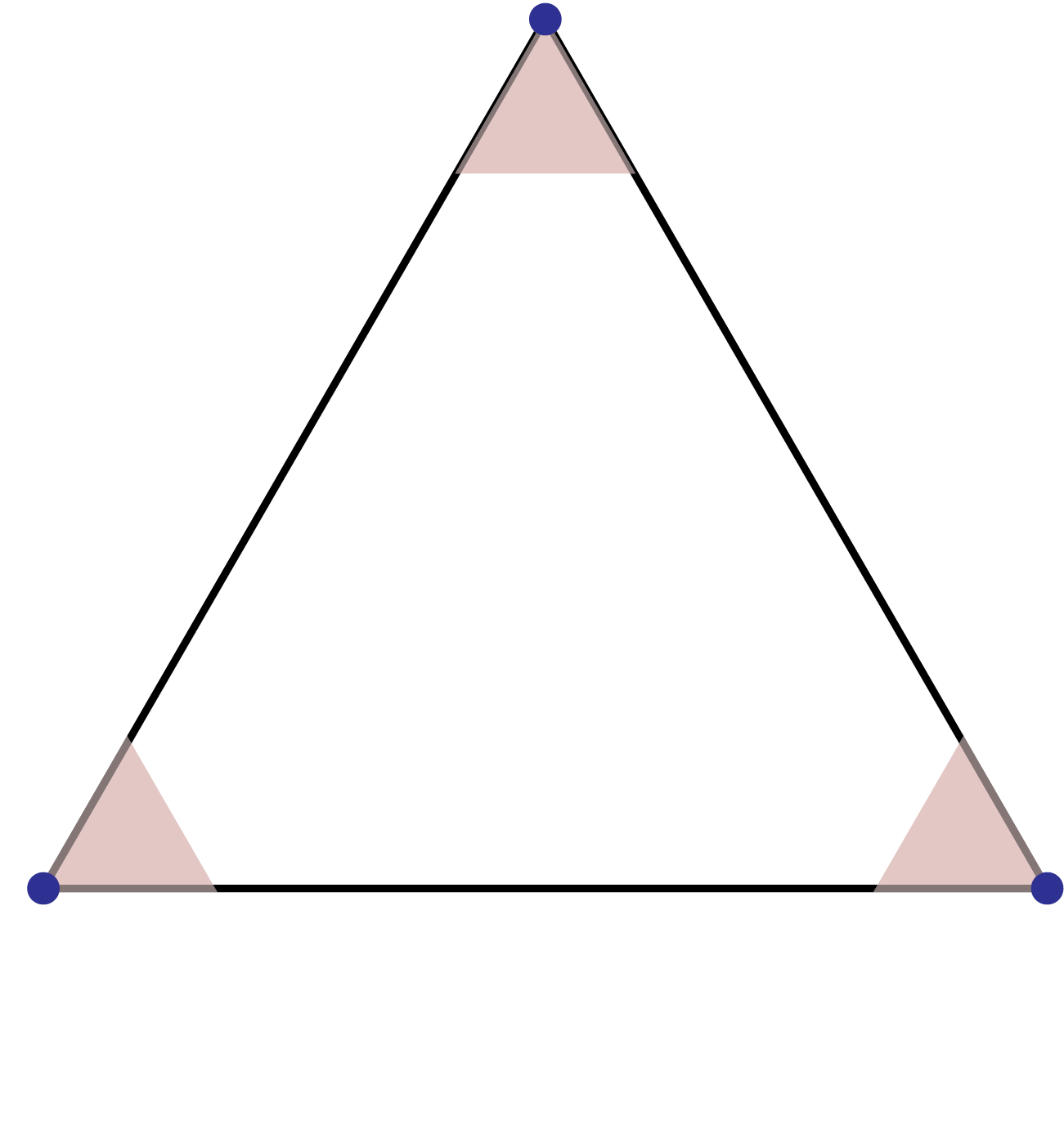}
      \label{fig:supp_3}  
   \end{subfigure}
   \begin{subfigure}[t]{0.3\textwidth}
      \centering
      \includegraphics[width=0.79\textwidth]{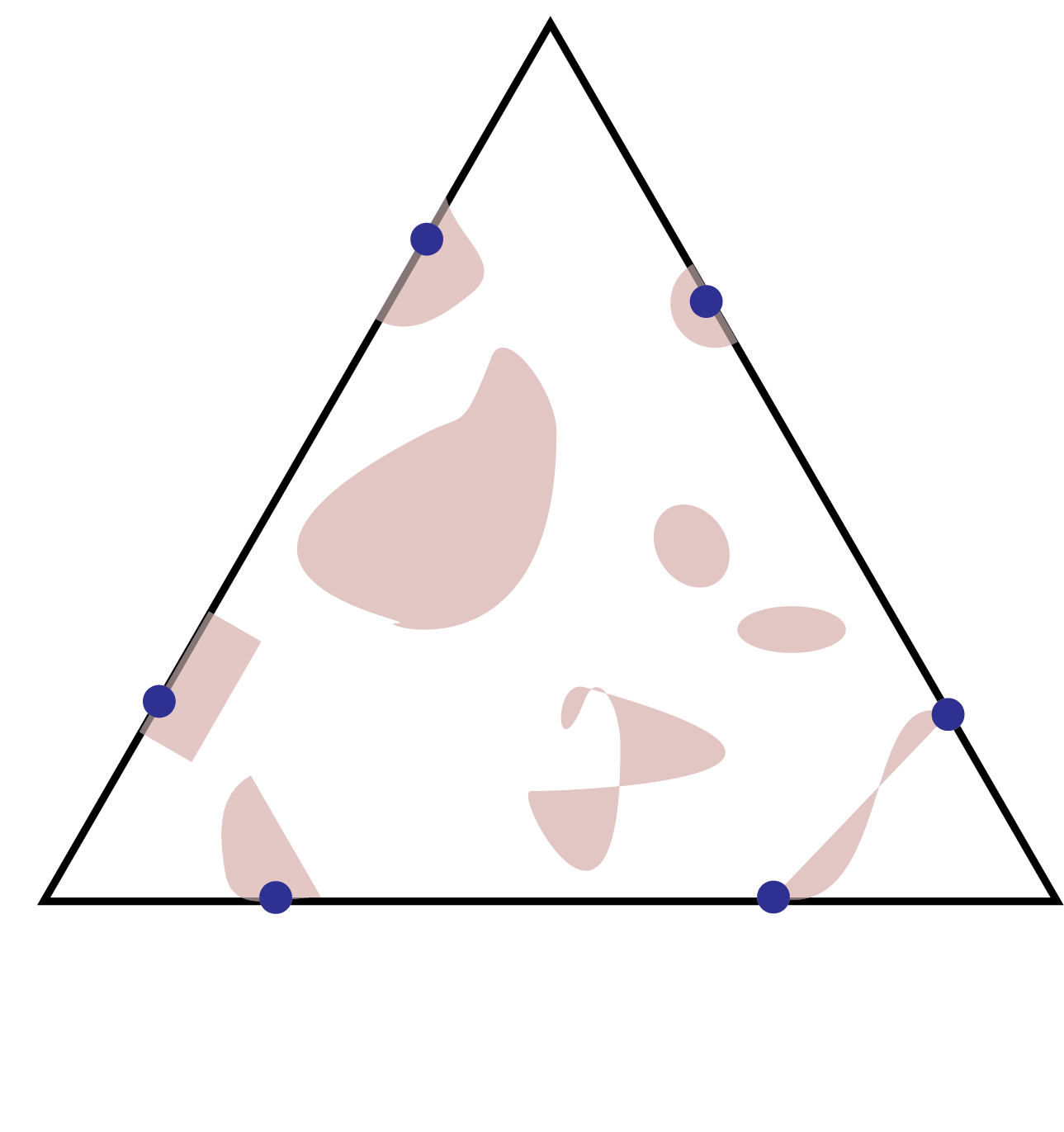}
      \label{fig:supp_4}  
   \end{subfigure}
   \caption{Examples of $(\alpha, \beta)$-SS distributions for $k=3$: $\mathbf w^{\sharp}$'s (blue dots) from $supp(\mathcal P)$ (pink area) are $(\alpha, \beta)$-SS on $\simplex^{k-1}$(the triangle).}
   \label{fig:supp_SS}
\end{figure}

Formally, we introduce a ``stochastic'' version of the SS condition on $\mathcal P$, called \emph{$(\alpha, \beta)$-SS distribution}, to ensure the $(\alpha, \beta)$-SS condition to hold for $\mathbf W$ with high probability as long as the number of documents $d$ is sufficiently large.

\begin{definition}[$(\alpha,\beta)$-SS distribution]\label{con:SS_dist}
A distribution $\mathcal P$ is an {\bf $(\bm \alpha, \bm \beta)$-SS distribution}, if there exist $s$ distinct points in its support, $\mathbf w^\sharp_1, \cdots, \mathbf w^\sharp_s \in supp(\mathcal P)$, and some positive constants $r_0$, $c_0$,
such that $\mathbf W^\sharp = \{\mathbf w^\sharp_i\}_{i=1}^s$ is $(\alpha, \beta)$-SS,  and for each $i\in[s]$, 
\begin{equation*}
    \mathcal P(\|\mathbf w - \mathbf w_i^\sharp\|_2 \leq r) \geq (k-1)!\cdot c_0 \cdot r^{k-1},\quad \forall\, 0 <r \leq r_0.
\end{equation*}
\end{definition}

The condition in Definition \ref{con:SS_dist} is mild and can be satisfied by many commonly encountered distributions over the simplex $\simplex^{k-1}$. For example, any distribution whose density function does not vanish on $\simplex^{k-1}$,
such as the uniform distribution and Dirichlet distributions, is $(\epsilon,C\epsilon)$-SS for any sufficiently small $\epsilon>0$, where $C$ is some constant depends on the distribution. In addition, an $(\alpha,\beta)$-SS distribution does not need to have a full support over $\simplex^{k-1}$---as long as a distribution has positive density values around a set of $(\alpha, \beta)$-SS points, then it is $(\alpha, \beta)$-SS. See Figure \ref{fig:supp_SS} for some examples of SS distributions whose supports are sparsely scattered over the simplex.



 
 Next, we state our assumption on the true underlying distribution $\mathcal P^0$ that generates the stochastic mixing weights.
 
\begin{assumptions} Assume the following: 
\begin{itemize}
\item[(A5)] $\tw{1}, \cdots, \tw{d}$ are i.i.d.~random samples from an $(\alpha, \beta)$-SS distribution $\mathcal P^0$, with $\alpha \geq C_1^\prime \sqrt{\frac{\log (n\vee d)}{n}} + \left(\frac{\log d}{d}  \right)^{\frac{1}{k-1}} $, where $C_1^\prime$ is a constant.
\end{itemize} 
\end{assumptions}

 The following Theorem \ref{thm:sample} establishes the finite-sample error bound when $\mathbf W$ is stochastically generated. 
 
\begin{theorem}\label{thm:sample}
Under Assumptions (A1), (A2) and (A5), it holds
with probability at least $1-D_1' s/d - D_2'd/(n\vee d)^c$ that
\begin{equation}\label{eq:sample_rate}
    \dis(\mathbf{\hat C}_n, \tC) \leq D_3' \sqrt{\frac{\log (n\vee d)}{n}} + D_4' \beta,
\end{equation}
where $c, D_1', D_2',D_3', D_4'$ are positive constants. In particular, if $\beta \leq C_2'\big(\sqrt{\log (n\vee d)/n} +(\log d/d)^{1/(k-1)} \big)$ for some constant $C_2'$, then
\begin{equation}\label{eq:result_new_dist}
    \dis(\mathbf{\hat C}_n, \tC)  \leq  D''_3 \sqrt{\frac{\log (n\vee d)}{n}} + D''_4 \left(\frac{\log d}{d}\right)^{\frac{1}{k-1}}.
\end{equation}
\end{theorem}

Similar to the remark of Assumption (A3), in most cases the parameter $\beta$ can be chosen as the same order as $\alpha$ in the $(\alpha, \beta)$-SS condition in Theorem \ref{thm:sample}. 
For example, according to Proposition \ref{prop:suff_SS_new}, if the support of $\mathcal P^0$ contains the point $(1-x_{ij})\mathbf{e}_i+x_{ij}\mathbf{e}_j$, where $0\leq x_{ij}< 1/k$, for all $1\leq i \neq j \leq k$, and $\mathcal P^0$ has positive density values around these points, then $\mathcal P$ is $(\epsilon, C\epsilon)$-SS for all $\epsilon>0$.

It is important to emphasize that our method does not require any prior knowledge about the distribution $\mathcal P^0$ (albeit our theory requires it to be SS). In comparison, in most Bayesian latent variable mixture model literature such as \cite{tang2014understanding},~\cite{nguyen2015posterior} and~\cite{Wang2019}, $\mathcal P^0$ is assumed to be known and have a full support over the simplex $\simplex^{k-1}$. 

Similar to Theorem~\ref{Coro:consistency}, we provide conditions for the estimator $\mathbf{\hat C}_n$ to have the estimation consistency under the double asymptotic setting by letting $(n,d)\to\infty$ in a suitable manner in Theorem~\ref{thm:sample}.
 
\begin{assumptions} Assume the following: 
\begin{itemize}
\item[(A5')] For all sufficiently small $\epsilon>0$, there exist some $\beta_\epsilon>0$, such that  $\beta_\epsilon \to 0$ as $\epsilon \to 0$, and $ \tw{1}, \cdots, \tw{d}$ are i.i.d.~random samples from distribution $\mathcal P^0$ that is $(\epsilon, \beta_\epsilon)$-SS.
\end{itemize} 
\end{assumptions}
 
\begin{theorem}[Estimation consistency]\label{cor:sample}
   Under Assumptions (A1), (A2), (A4) and (A5'), we have
   $$\dis(\mathbf{\hat C}_n, \tC)  \to 0 \quad\mbox{in probability as $(n,d)\to\infty$.}$$
\end{theorem}
 
\newpage


\section{Derivation of the MCMC-EM Algorithm}\label{app:computation}
We use an MCMC-EM algorithm to compute the MLE of the integrated likelihood function \eqref{eq:lkh_fn}.
First we introduce a set of latent variables  $\mathbf Z = \{Z_{ij}\}$, where $Z_{ij}$ is the topic label for $x_{i, j}$. Then express the LDA model  as follows:
\begin{align*}
    x_{i,j}|\mathbf C, Z_{ij}  = l & \sim \text{Multi}_V(\mathbf C_l)\\
    Z_{ij}|\mathbf w_i & \sim \text{Multi}_k(\mathbf w_i)\\
    \mathbf w_i | \bm \beta_0 & \sim \text{Dir}_k(\bm \beta_0), 
\end{align*}
where 
$$ i = 1, \dots, d; \quad j =1, \dots, n; \quad l =1, \dots, k.$$ 
We fix $\bm \beta_0 = \mathbf 1_k$ throughout, since we consider a uniform ``prior'' on $\mathbf W$.
The integrated likelihood \eqref{eq:lkh_fn} can be written as 
\begin{align*}
     F_{n\times d}(\mathbf C;\mathbf X) & =  p(\mathbf X \mid \mathbf C)  = \int p(\mathbf X, \mathbf Z \mid \mathbf C) d\mathbf Z\\
     & \propto \prod_{i=1}^d \int \left[ \int \prod_{j=1}^n p(x_{i,j}|\mathbf C, Z_{ij} )p(Z_{ij}|\mathbf w) p(\mathbf w|\bm\beta_0) d \mathbf w\right] d\mathbf Z_{i\cdot}\\
     & \propto \prod_{i=1}^d \int \prod_{j=1}^n p(x_{i,j}|\mathbf C, Z_{ij} ) p(\mathbf Z_{i\cdot} = \mathbf z|\bm \beta_0) d \mathbf z
\end{align*}
where $\mathbf Z_{i\cdot} = (Z_{i1}, \cdots, Z_{in})$.

\paragraph{E-step} Define $Q(\mathbf C|\mathbf C^{(0)})$ as the  expected value of the log likelihood function of $\mathbf C$, with respect to $\mathbf Z$ given $\mathbf X$ and $\mathbf C^{(0)}$, where $\mathbf C^{(0)}$ is the estimated topic matrix obtained from the last EM iteration. 
\begin{align*}
    Q(\mathbf C|\mathbf C^{(0)}) & =
    \mathds E_{\mathbf Z|\mathbf C^{(0)}} \log [F_{n\times d}(\mathbf C;  \mathbf X, \mathbf Z)]\\
    & =\mathds E_{\mathbf Z|\mathbf C^{(0)}} \sum_{i=1}^d \sum_{j=1}^n\log  p(x_{i,j}|\mathbf C, Z_{ij} ) + Const
\end{align*}
We ignore the constant term in the following derivation. Since the marginal probability $p(\mathbf Z_{i\cdot} = \mathbf z|\bm \beta_0)$ is infeasible, we apply MCMC to and iteratively sample $\mathbf Z = \{Z_{ij}\}_{i,j}$ and $\mathbf W = \{\mathbf w_i\}_{i=1}^d$ as follows:
   \begin{align*}
   Z_{ij} | \mathbf C,  x_{i,j} = v & \sim \text{Multi}_k\left(
   \frac{c_{vl} w_{li}}{\sum_{l=1}^k c_{vl} w_{li}}
   \right)_{l =1, \dots, k}\\
   \mathbf w_i | \mathbf Z_{i\cdot} & \sim \text{Dir}_k\left(
   \beta_{0l} + \sum_{j=1}^n\mathds 1(Z_{ij} = l)
   \right)_{l =1, \dots, k}.
\end{align*}
We approximate $ Q(\mathbf C|\mathbf C^{(0)})$ function by the samples of $\mathbf Z$,
\begin{align*}
    Q(\mathbf C|\mathbf C^{(0)}) & = \mathds E_{\mathbf Z|\mathbf C^{(0)}}
 \sum_{v=1}^V\sum_{l=1}^k \log c_{vl} \left[\sum_{i=1}^d\sum_{j=1}^n \mathds 1(Z_{ij}=l, x_{i,j} = v)\right]\\
    & \approx  \frac{1}{T} \sum_{t=b+1}^{b+T} \sum_{v=1}^V \sum_{l=1}^k \left[\log c_{vl} \sum_{i=1}^d\sum_{j=1}^n\mathds 1(Z^{(t)}_{ij}=l,x_{i,j} = v)\right].
\end{align*}
where $c_{vl}$ is the $(v,l)$-th element of $\mathbf C$. Here the $Z^{(t)}_{ij}$ denotes the sample of $Z_{ij}$ at $t$-th MCMC iteration, $b$ denotes the burn-in period and $T$ denotes the number of the samples after burn-in. 

\paragraph{M-step} We maximize the approximated $ Q(\mathbf C|\mathbf C^{(0)})$ with respect to $\mathbf C$ by the following closed-form solution: 
   $$c_{vl} = \frac{\sum_{i,j,t} \mathds 1(Z^{(t)}_{ij} = l, x_{i,j} = v)}{\sum_{i,j,t} \mathds 1(Z^{(t)}_{ij} = l)}. $$

The algorithm of the E-step is given in Algorithm \ref{alg:E_step}. Here we use, $\mathcal Z, \mathscr Z \in \mathds R^{d\times V\times k}$, to denote the counts of the samples of $\mathbf Z$. Specifically, $\mathscr Z[i, v, l] = \sum_j\mathds 1(Z^{(t)}_{ij} = l, x_{i,j}=v)$ is the count of $\mathbf Z$ at $t$-th MCMC iteration, and $\mathcal Z [i,v,l] = \sum_{t=b+1}^{b+T}\sum_j\mathds 1(Z^{(t)}_{ij} = l, x_{i,j}=v)$ is the sum of count of $\mathbf Z$ over $T$ iterations. 
\begin{algorithm}
\KwIn{$\mathbf C$;}
$\mathcal Z[:, :, :] \gets \mathbf 0_{d\times V\times k};$\Comment{Initialize $\mathcal Z$}\\
$\mathbf W[:,i] \gets \text{Dir}_k(\mathbf 1)$, $i = 1, \cdots, d;$\Comment{$\mathbf W[:,i]$ is the $i$-th column of $\mathbf W$}\\
 \For{$t = 1, \cdots, b, b+1, \cdots, b+T$}{
 $\mathscr Z[:, :, :] \gets \mathbf 0_{d\times V\times k};$\Comment{Initialize $\mathscr Z$}\\
 \For{$i =  1, \cdots, d$}{
 \For{$v  =  1, \cdots, V$}{
 $\bm p \gets \mathbf C[v,:]\odot \mathbf W[:, i]$ 
 \Comment{$\mathbf C[v,:]$ is the $v$-th row of $\mathbf C$}\\
 \Comment{$\odot$ denotes an element-wise multiplication}\\
 $\bm p \gets \bm p/\sum_{l=1}^k \bm p[l]$\Comment{$\bm p[l]$ is the $l$-th element of $\bm p$}\\
 $\mathscr Z[i, v, :] \gets \text{Multi}(n = x_v^{(i)}, p =p);$ \Comment{$x_v^{(i)}$ is the count of $v$-th word in the $i$-th doc}
 }
 $\mathbf W[:, i] \gets \text{Dir}_k(\sum_v \mathscr  Z[i, v, :] + \bm \beta_0);$
 }
 \If{$t> b$}
{$\mathcal Z \gets \mathcal Z + \mathscr Z;$}
}
{$\mathcal Z \gets \frac{1}{T}\mathcal Z;$}

\KwOut{$\mathcal Z$.}
 \caption{The E-step of the MCMC-EM Algorithm}
    \label{alg:E_step}
\end{algorithm}

Empirically, since $\mathcal Z$ and $\mathscr Z$ are sparse, to save the computation space, we recommend to use two 2-dim arrays instead, namely $\mathscr  C = \sum_{i=1}^d \mathcal Z[i, :, :]$ and $\mathscr  W = \sum_{v=1}^V \mathcal Z[:, v, :]$, and $\mathscr C, \mathscr  W$ can be used efficiently in updating $\mathbf C$ and $\mathbf W$, respectively. In addition, the operations in the two nested for-loops over $i$ and $v$ in Algorithm \ref{alg:E_step} can be paralleled, as they are independent with each other.

The full algorithm is given in Algorithm \ref{alg:main}.

\begin{algorithm}
\KwIn{Data $\mathbf X = \{\mathbf x^{(i)}\}_{i=1}^d$; number of topics $k$;}
$\mathbf C[l,:] \gets \text{Dir}_V(\mathbf 1)$, $l = 1, \cdots, k;$ \Comment{Initialize $\mathbf C$}\\
 \Repeat{convergence}{
Obtain $\mathcal Z$ using Algorithm \ref{alg:E_step};\Comment{E-step}\\

$\mathbf C[v, l] \gets \sum_{i=1}^d \mathcal Z[i, v, l]/\sum_{v=1}^V\sum_{i=1}^d \mathcal Z[i, v, l],$\\
$~~~~~~~~v=1,\cdots, V, \, l= 1,\cdots k;$\Comment{M-step}\\
}

$\mathbf{W}[l, i] \gets \sum_{v=1}^V \mathcal Z[i,v,l]/\sum_{l=1}^k\sum_{v=1}^V \mathcal Z[i,v,l],
$\\ $~~~~~~~~ l= 1,\cdots k, \, i=1,\cdots, d;$\Comment{Estimate $\mathbf W$}

\KwOut{$\mathbf{C};\mathbf {W}$.}
 \caption{The MCMC-EM Algorithm}
    \label{alg:main}
\end{algorithm}

\newpage

\section{Proofs of Main Theorems}\label{app:pf_thm_1}

\setcounter{subsection}{-1}
\subsection{Notation}
For a vector $\mathbf{x}$, we denote by $\|\mathbf x\|_2 = \sqrt{\sum_i x^2_i}$ its $L_2$ norm and $\|\mathbf x\|_1 = \sum_i |x_i|$
its $L_1$ norm. Write $\mathbf x \geq a$ to indicate that $\mathbf x$ is element-wisely no smaller than $a$. In particular, $\mathbf 1_k$ denotes the all-ones vector of length $k$, and $\mathbf e_f$ the $f$-th column of the $k\times k$ identity matrix $\mathbf I_k$.  

For a matrix $\mathbf{A}_{p \times q}$,  $\mathbf A(i,:)$ and $\mathbf A(:,j)$ are $i$-th row and $j$-th column vectors, respectively. We use $\sigma_{\max}(\mathbf A)$ to denote the square root of the largest eigenvalue of $\mathbf A^T\mathbf A$, and $\sigma_{\min}^+(\mathbf A)$ the square root of the smallest nonzero eigenvalue of  $\mathbf A^T \mathbf A$. We denote by $\| \mathbf{A}\|_2 = \sigma_{\max}(\mathbf A)$ the spectral norm and $\| \mathbf A\|_1 = \max_{j=1}^q \sum_{i=1}^p |\mathbf A_{ij}|$ the $L_1$ matrix norm. Some useful facts we will use in the proof: 
(i) $\sigma_{\max}(\mathbf A\mathbf B) \leq \sigma_{\max}(\mathbf A)\sigma_{\max}(\mathbf B)$; 
(ii) $\sigma_{\min}^+(\mathbf A\mathbf B) \geq \sigma_{\min}^+(\mathbf A)\sigma_{\min}^+(\mathbf B)$;
(iii) $\|\mathbf A\|_2 \leq \sqrt{q}\|\mathbf A\|_1$; 
(iv) if $p \geq q$ and $\mathbf A^T\mathbf A$ is invertible, then $\sigma_{\max}((\mathbf A^T\mathbf A)^{-1}\mathbf A^T)=1/{\sigma^+_{\min}(\mathbf A)}$.

We denote by $\simplex^{k-1} = \{\mathbf x\in \mathds{R}^{k}: 0\leq x_i \leq 1, \sum_{i=1}^k x_i = 1\}$ the standard $(k-1)$-dimensional simplex. For a matrix $\mathbf{A}_{p \times q}$, let
\begin{eqnarray*}
\conv(\mathbf A) &=& \{\mathbf x\in \mathds{R}^{p}:  \mathbf{x} =  \mathbf A \bm\lambda, \bm\lambda  \in \simplex^{q-1} \} \\
\text{cone}(\mathbf A)  &=& \{\mathbf x\in \mathds{R}^{p}: \mathbf x = \mathbf A \bm \lambda, \bm \lambda \geq 0 \}\\
\text{aff}(\mathbf A) &= & \{\mathbf x\in \mathds{R}^{p}:  \mathbf{x} =  \mathbf A \bm\lambda, \bm\lambda^T \mathbf 1_q = 1\}
\end{eqnarray*}
denote the \emph{convex polytope},  the \emph{simplicial cone} and  the \emph{affine space} generated by the $q$ columns of $\mathbf A$, respectively. 

For any cone $\mathcal C$, let $\mathcal C^\ast = \{\mathbf x: \mathbf x^T\mathbf y \geq 0, \forall \mathbf y \in \mathcal C\}$ denote its \emph{dual cone}. In particular, let $\mathcal{K}= \{\mathbf x \in \mathds{R}^k: \|\mathbf x\|_2 \leq \mathbf x^T \mathbf 1_k\}$. The boundary of $\mathcal{K}$ is denoted by $bd\mathcal{K}= \{\mathbf x \in \mathds{R}^k: \|\mathbf x\|_2 = \mathbf x^T \mathbf 1_k\}$, and its dual cone takes the form as $\mathcal{K}^\ast  = \{\mathbf x \in \mathds{R}^k: \mathbf x^T \mathbf 1_k \geq \sqrt{k-1}\|\mathbf x\|_2\}$. Some useful facts of dual cones from \cite{donoho2004does}: (i) $cone (\mathbf A)^\ast = \{\mathbf x \in \mathds{R}^p: \mathbf x^T \mathbf A \geq 0\}$; (ii) if $\mathcal A$ and $\mathcal {\bar A}$ are convex cones, and $\mathcal A \subseteq \mathcal {\bar A}$, then $\mathcal {\bar A}^\ast \subseteq \mathcal A^\ast$.


The true $\mathbf C$, $\mathbf W$, and $\mathbf U$ are denoted by $\tC$,$\tW$,$\tU$, respectively; $\hat{\mathbf C}_n$ is the estimator obtained from $F_{n\times d}(\mathbf C;\mathbf X)$. $\hat{\mathbf W}_n$ is a valid estimator for the mixing matrix in $\mathds{R}^{k\times d}$ which we will construct in Lemma~\ref{lem:dis_upper_bound}
such that $\mathbf {\hat W}_n \geq 0$, $\mathbf {\hat W}_n^T \mathbf 1_k  = \mathbf 1_d$. $\epsilon_n = C_0 \sqrt{\frac{\log (n\vee d)}{n}}$ is a small quantity used to measure the convergence rates. Here $C_0$ in $\epsilon_n$ is a positive constant independent of $n$ and $d$. 

Throughout, we use symbols like  $C$, $C^\prime$, $C''$, $C'''$, $C^\ast$, $C_i$, $C'_i, i=1, 2, \dots$, and $ D_1$, $D_2$ as generic notations for large absolute numbers, whose exact values may vary from part to part. Unless stated otherwise, these constants are all independent of $n$ and $d$.

\subsection{Proof of Theorem \ref{thm:idf}}\label{supp:proof_thm_idf}

The following lemmas are useful in the proof of Proposition \ref{prop:sub_idf}. Their proofs are given in Appendix \ref{app:prf_lemmas}.


\begin{lemma}\label{lem:vol_det}
For a full column rank matrix $\mathbf C  \in \mathds{R}^{V\times k}$,
$$|\conv(\mathbf C)| = \frac{\sqrt{\det(\mathbf C^T\mathbf C)}}{h \cdot (k-1)!} ,$$
where $h$ is the perpendicular distance from the origin to the hyperplane $\textnormal{aff}(\mathbf C)$. In particular, we have 
$$\frac{|\conv(\mathbf C)|}{|\conv(\mathbf {\bar C})|} = \frac{\sqrt{\det(\mathbf C^T\mathbf C)}}{\sqrt{\det(\mathbf {\bar C}^T\mathbf {\bar C})}},$$
if $\textnormal{aff}(\mathbf C) = \textnormal{aff}(\mathbf {\bar C})$.
\end{lemma}


\begin{lemma}\label{lem:W_sigma_min}
If $\mathbf W \in \mathds{R}^{k\times d}$ satisfies Condition (S1), then $\mathbf W$ is of rank $k$ (full row rank), and
$\sigma_{\min}^+ (\mathbf W) \geq \frac{1}{k}.$ 
\end{lemma}

\subsection*{}

We first show that Condition (S1) guarantees that $\conv(\mathbf C)$ has the minimal volume.

\begin{proposition}\label{prop:sub_idf}
If $\mathbf W$ satisfies Condition (S1) and $\mathbf C$ is of rank $k$ (full column rank), then  $|\conv(\mathbf{\bar C})| \geq |\conv(\mathbf C)|$ must hold for any other set of parameters $(\mathbf {\bar C}, \mathbf{\bar W})$ satisfying $\mathbf C\mathbf W = \mathbf {\bar C}\mathbf {\bar W}$.
\end{proposition}

\begin{proof}[Proof of Proposition \ref{prop:sub_idf}] By Lemma \ref{lem:W_sigma_min}, $\mathbf W \mathbf W^T \in \mathds{R}^{k\times k}$ is invertible. Define 
$$\mathbf B_{k \times k} := \mathbf {\bar W}\mathbf W^{T}(\mathbf W^{ }\mathbf W^{ T})^{-1}.$$ 
Then $\mathbf C  = \mathbf {\bar C} \mathbf B$. Note that
$$ \mathbf B^T\mathbf 1_k = (\mathbf W\mathbf W^{ T} )^{-1}\mathbf W \mathbf {\bar W}^T \mathbf 1_k =   (\mathbf W \mathbf W^{T} )^{-1}\mathbf W \mathbf 1_d,$$
which is the solution of the least square (LS) problem $\min_{\mathbf x \in \mathds{R}^k}\|\mathbf 1_d - \mathbf x^T \mathbf W \|_2 $. Since $\|\mathbf 1_d - \mathbf 1_k^T \mathbf W \|_2 =0$ achieves the minimum, the unique LS solution is given by $\mathbf 1_k$, i.e., 
\begin{equation} \label{eq:B_sum_1}
\mathbf B^T\mathbf 1_k = \mathbf 1_k.
\end{equation}
Thus, columns of $\mathbf{\bar C}$ are convex combination of columns of $\mathbf{C}$, which implies $\text{aff}(\mathbf C) = \text{aff}(\mathbf {\bar C})$. By Lemma  \ref{lem:vol_det}, we have
\begin{align*}
   \frac{|\conv(\mathbf {\bar C})|}{ |\conv(\mathbf C)|} &= \sqrt{\frac{\det(\mathbf{\bar C}^T \mathbf{\bar C})} {\det(\mathbf C^T \mathbf C)}} = \sqrt{ \frac{\det(\mathbf{\bar C}^T \mathbf{\bar C})} {\det(\mathbf B^T\mathbf{\bar C}^T \mathbf{\bar C}\mathbf B)}} = \frac{1}{|\det({\mathbf B})|} . 
\end{align*}
Therefore, it suffices to show 
\begin{equation} \label{detB-less-1}
|\det(\mathbf B)| \leq 1.
\end{equation}

We first show that for any row of $\mathbf B$, we have $\mathbf B(f,:) \in cone(\mathbf W)^*\subseteq \mathcal{K}$. Since $\mathbf C \mathbf W = \mathbf {\bar C} \mathbf B \mathbf W  = \mathbf {\bar C}\mathbf {\bar W}$ and $\mathbf {\bar C}^{T}\mathbf {\bar C} \in \mathds{R}^{k\times k}$ is invertible, we have
\begin{align*}
     \mathbf B \mathbf W = (\mathbf {\bar C}^{T}\mathbf {\bar C} )^{-1}\mathbf {\bar C} ^T \mathbf {\bar C} \mathbf B \mathbf W  & =  (\mathbf {\bar C}^{T}\mathbf {\bar C} )^{-1} \mathbf {\bar C}^{T}\mathbf {\bar C}\mathbf {\bar W}= \mathbf {\bar W}.
\end{align*}
Because $\mathbf {\bar W} \geq \mathbf 0_{k\times d}$, we obtain that, for any row of $\mathbf B$, $\mathbf B(f,:) \in \mathds{R}^k$,
$$\mathbf B^T(f,:)\mathbf W  = \mathbf {\bar W}^T(f,:) \geq \mathbf 0.$$
That is, $\mathbf B(f,:) \in cone(\mathbf W)^\ast$, which consequently implies that
     \begin{equation}\label{eq:B_in_K}
         \|\mathbf B(f,:)\|_2 \leq \mathbf B(f,:)^T\mathbf 1_k.
     \end{equation}

Combining \eqref{detB-less-1}, \eqref{eq:B_in_K} and the Hadamard Inequality and Inequality of Arithmetic and Geometric means (AM-GM), we can show  \eqref{detB-less-1} as follows: 
 \begin{align}
|\det(\mathbf B)| \stackrel{Hadamard's}{\leq}   \prod_{f=1}^k \|\mathbf B(f,:)\|_2 \stackrel{\eqref{eq:B_in_K}}{\leq}  \prod_{f=1}^k \mathbf B^T(f,:)\mathbf 1_k  \stackrel{AM-GM}{\leq} & \left(\frac{\sum_{f=1}^k \mathbf B^T(f,:){\mathbf 1_k}}{k}\right)^k\nonumber \\
 = &\left(\frac{\sum \mathbf B^T\mathbf 1_k}{k}\right)^k\stackrel{\eqref{eq:B_sum_1}}{=}  1. \label{eq:B_larger_1}
\end{align}
\end{proof}

Next, we give the proof of Theorem \ref{thm:idf}.

\begin{proof}[Proof of Theorem \ref{thm:idf}.]
Suppose $\mathbf C \mathbf W  = \mathbf {\bar C}\mathbf {\bar W}$ and $|\conv(\mathbf {\bar C})| \leq |\conv(\mathbf C)|$. Following the notation of the proof of Proposition \ref{prop:sub_idf}, we aim to show that $\mathbf B$ is a permutation matrix.

To complete the proof, we only need to  verify the  following three conditions on $\mathbf B$.
\begin{enumerate}[label=(1.\roman*),ref=(1.\roman*)]
    \item\label{con:B_in_bd} Any row of $\mathbf B$ belongs to $bd\mathcal{K} \bigcap cone(\mathbf W)^\ast$, i.e., 
    $$\mathbf B(f,:)\in \{\lambda \mathbf e_s: s=1,\cdots,k, \ \lambda \geq 0\},\forall f\in[k].$$
    
    \item\label{con:B_row_sum_1}  Any row sum of $\mathbf B$ is one, which, along with \ref{con:B_in_bd}, implies 
     $$\mathbf B(f,:)\in \{\lambda \mathbf e_s: s=1,\cdots,k, \ \lambda \geq 0\},\forall f\in[k].$$
     
    \item\label{con:detB_1}  $\det(\mathbf B) = 1$. Along with the previous two conditions, it implies $$\{\mathbf B(1,:), \mathbf B(2,:), \cdots, \mathbf B(k,:)\} = \{\mathbf e_1,\mathbf e_2, \cdots, \mathbf e_k\};$$
that is, $\mathbf B$ must be a permutation matrix.
\end{enumerate}

First, by the condition $|\conv(\mathbf {\bar C})| \leq |\conv(\mathbf C)|$ and Proposition \ref{prop:sub_idf}, we have $|\conv(\mathbf {\bar C})| = |\conv(\mathbf C)|$, or equivalently $\det(\mathbf B) =1$,  i.e., \ref{con:detB_1} holds. 

Consequently, all inequalities in (\ref{eq:B_larger_1}) become equalities. Specifically,
\begin{equation}\label{eq:b_f_boundary_k}
    \|\mathbf B(f, :)\|_2 = \mathbf B^T(f, :)\mathbf 1_k = 1,\, \forall f \in [k],
\end{equation}
which implies that the row sums of $\mathbf B$ are all 1's, i.e., \ref{con:B_row_sum_1} holds.

The above equation \eqref{eq:b_f_boundary_k} also implies that $\mathbf B(f,:)$ is on the boundary of $\mathcal{K}$, $\mathbf B(f,:) \in bd \mathcal K$. Together with the fact that $\mathbf B(f,:)$ is in $cone(\mathbf W)^\ast$ (proved in the proof of Proposition \ref{prop:sub_idf}), it implies that \ref{con:B_in_bd} holds.
\end{proof}

\subsection{Proof of Theorem \ref{thm:main} }\label{app:pf_thm_2}


The sketch of this proof is as follows:

\textit{Step 1}: We first show that with high probability, all true word frequency vectors, columns of  $\mathbf U^{0}$, are close to the estimated convex polytope $\conv(\hatCn)$. More specifically, we show in Lemma \ref{lem:dis_upper_bound} that there exists a $k \times d$ column-stochastic matrix
\footnote{We say a matrix is column-stochastic, if its entries are non-negative and columns sum to one.}  $\mathbf {\hat W}_n$ such that
\begin{equation} \label{app:eq:U0}
\mathbf{U}^{0} = \tC \tW  =  \mathbf{\hat C}_n\mathbf {\hat W}_n + \mathbf E_n
\end{equation}
and $\max_i \| \mathbf E_n(:,i)\|_2 \le C\epsilon_n$.

\textit{Step 2}: 
We then work with a subset of $s$ documents. Let $\mathbf W^0_1\in \mathds R^{k\times s}$ be the collection of the $s$ columns of $\tW$ that are $(\alpha, \beta)$-SS; let 
$\mathbf {\hat W}_{n1}$ and $\mathbf E_{n1}$ be the corresponding sub-matrices of 
$\mathbf {\hat W}_n$ and $\mathbf E_n$, respectively.
As a consequence of (\ref{app:eq:U0}), we have
\begin{equation}\label{eq:U_CW_dist}
    \tC \mathbf{W}^0_1 = \mathbf {\hat C}_n \mathbf{\hat W}_{n1} + \mathbf E_{n1}. 
\end{equation}
We can upper bound the estimation error by the summation of the following two terms:
\begin{align}\label{eq: error_two_terms_origin}
    \dis(\hatCn,\tC)  \leq  \|\mathbf E_{n1} {\mathbf{W}^0_1}^{T} (\mathbf{W}^0_1 {\mathbf{W}^0_1}^{T} )^{-1}\|_2 + \min_{\mathbf \Pi} \sqrt{k}  \|\mathbf B - \mathbf \Pi\|_2,
\end{align}
where $\mathbf B = \mathbf {\hat W}_{n1} {\mathbf{W}^0_1}^{T} (\mathbf{W}^0_1 {\mathbf{W}^0_1}^{T} )^{-1}$. By Lemma \ref{lem:dis_upper_bound}, the first term is upper bounded.

\textit{Step 3}: 
We show that for all  $f = 1,\cdots, k$, $\mathbf B(f,:)$ satisfies:
\begin{align}\label{eq:to_show_step2}
\mathbf B(f,:) \in [cone(\mathbf W_1^0)^*]^{C_1 \epsilon_n } \bigcap [bd \mathcal{K}]^{C_1 \epsilon_n }
\end{align}
Then by the definition of $(\alpha, \beta)$-SS, $\mathbf B(f,:)$'s are all close to indicator vectors. Using Lemma~\ref{lem:B_close_Pi} and letting $\alpha=C_{1}\epsilon_{n}$, we can prove that the matrix $\mathbf B$ is close to a permutation matrix.
So the second term in \eqref{eq: error_two_terms_origin} can be bounded.
Putting all the steps together, we obtain that with high probability,
$$\dis(\mathbf{\hat C}_n, \tC) \leq D_1 \sqrt{\frac{\log (n\vee d)}{n}} + D_2 \beta.$$
~\\

In the following, we provide the details of the above-mentioned steps.



\begin{proof}[Proof of Theorem \ref{thm:main}]
$ $\newline
\textit{Step 1}:
The following lemma is useful; its proof is given in Appendix \ref{app:prf_lemmas}. 

\begin{lemma}\label{lem:dis_upper_bound}
With probability at least $(1-3/(n\vee d)^{c})^d$, there exists a matrix $\hat{\mathbf W}_n\in \mathds{R}^{k\times d}$ satisfying $\mathbf {\hat W}_n \geq 0$, $\mathbf {\hat W}_n^T \mathbf 1_k  = \mathbf 1_d$ such that 
\begin{equation*}
     \mathbf U^0 =  \tC \tW  = \mathbf{\hat C}_n\mathbf {\hat W}_n  + \mathbf E_n 
\end{equation*}
and each column of $\mathbf E_n $ satisfies
\begin{equation}\label{eq:u_error}
\|\mathbf E_n(:, i) \|_2\leq C\epsilon_n
\end{equation}
for all $i = 1,\cdots,d$, where $c,C > 0$ are constants independent of $n$ and $d$.
\end{lemma}
~\\

\textit{Step 2}: 
By Lemma \ref{lem:dis_upper_bound}, we have
$$
   \tC \mathbf{W}^0_1 = \mathbf {\hat C}_n \mathbf{\hat W}_{n1} + \mathbf E_{n1},
$$
and $\|\mathbf E_{n1} \|_{2} \leq C \sqrt{s}\epsilon_n$. 
~\\

Let $\mathbf B = \mathbf {\hat W}_{n1} {\mathbf{W}_1^0}^{T} (\mathbf{W}_1^0 {\mathbf{W}_1^0}^{T} )^{-1}$. Then,
\begin{equation}\label{eq:C_CB}
    \tC =  \mathbf {\hat C}_n \mathbf{\hat W}_{n1} {\mathbf{W}_1^0}^{T} (\mathbf{W}_1^0 {\mathbf{W}_1^0}^{T} )^{-1} + \mathbf E_{n1} {\mathbf{W}_1^0}^{T} (\mathbf{W}_1^0 {\mathbf{W}_1^0}^{T} )^{-1} = \mathbf {\hat C}_n \mathbf B + \mathbf {\tilde E}_{n1}
\end{equation}
where  $\mathbf {\tilde E}_{n1} = \mathbf E_{n1} {\mathbf{W}_1^0}^{T} (\mathbf{W}_1^0 {\mathbf{W}_1^0}^{T} )^{-1} $.  
We can bound $\|\mathbf {\tilde E}_{n1}\|_2 $ by 
$$\|\mathbf {\tilde E}_{n1}\|_2 \leq {[ \sigma^{+}_{\min} (\mathbf{W}_1^0)]}^{-1} \|\mathbf E_{n1}\|_2  \leq k\cdot C \sqrt{s}\epsilon_n
 = C^\prime \sqrt{s}\epsilon_n.$$
Then, we have
\begin{align*}
    \dis(\hatCn,\tC) = \min_{\mathbf \Pi}\|\mathbf {\hat C}_n \mathbf \Pi - \tC \|_2&  =   \min_{\mathbf \Pi}\|\mathbf {\hat C}_n \mathbf \Pi -  \mathbf {\hat C}_n \mathbf B - \mathbf {\tilde E}_{n1}\|_2 \\ 
    & \leq \min_{\mathbf \Pi} \|\mathbf {\hat C}_n  \|_2 \|\mathbf B - \mathbf \Pi\|_2 + \|\mathbf {\tilde E}_{n1}\|_2 \\
     & \leq  \min_{\mathbf \Pi} \sqrt{k}  \|\mathbf B - \mathbf \Pi\|_2 + \|\mathbf {\tilde E}_{n1}\|_2 
\end{align*}
~\\

\textit{Step 3}:
Now, it suffices to show that for some permutation matrix $\mathbf \Pi$,
$$\|\mathbf B - \mathbf \Pi\|_2 \leq C''  \beta.$$
We will use the following Lemma \ref{lem:B_close_Pi} to prove the above inequality. The proof of Lemma \ref{lem:B_close_Pi} is deferred to Appendix \ref{app:prf_lemmas}.
\begin{lemma}\label{lem:B_close_Pi}
For a matrix $\mathbf B \in \mathds{R}^{k\times k}$, if it satisfies the following conditions 
\begin{enumerate}[label=(2.\roman*),ref=(2.\roman*)]
    \item\label{con:B_col_sum_1_3}  $\mathbf B^T \mathbf 1_k = \mathbf 1_k$;
    \item\label{con:B_l2_3}  $\|\mathbf B\|_2 \leq M$;
    \item\label{con:B_in_bd_3} any row of $\mathbf B$ belongs to $[bd\mathcal{K}]^{\alpha} \bigcap [cone(\mathbf W^0_1)^\ast]^{\alpha}$, so that 
    $$\mathbf B(f,:)\in \{\lambda \mathbf e_l + \bm \epsilon: l=1,\cdots,k, \ \lambda \geq 0, \bm \|\bm \epsilon\|_2 \leq \beta\lambda\}, f=1,\cdots,k;$$
\end{enumerate}
then there exists a permutation matrix $\mathbf \Pi$, such that
$$\|\mathbf B - \mathbf \Pi\|_2 \leq C'''M \beta,$$
where $C'''$ is a constant independent of $n$ and $d$.
\end{lemma}

Next, we verify the conditions in Lemma \ref{lem:B_close_Pi}. 

Firstly, the proof of \ref{con:B_col_sum_1_3} $\mathbf B^T \mathbf 1_k = \mathbf 1_k$ is similar to the proof of Proposition \ref{prop:sub_idf} equation \eqref{eq:B_sum_1}, so we omit it here. 

Secondly, \ref{con:B_l2_3} holds because $$\|\mathbf B\|_2 
\leq \sigma_{\max}(\mathbf{\hat W}_{n1}) [\sigma^{+}_{\min}(\mathbf W_1^0)]^{-1} 
\leq \sqrt{s}\|\mathbf{\hat W}_{n1}\|_1 [\sigma^{+}_{\min}(\mathbf W_1^0)]^{-1}  
\leq \sqrt{s} \cdot k = M.$$

Thirdly, to prove \ref{con:B_in_bd_3}, it suffices to verify the followings hold for any $f\in[k]$, 
\begin{enumerate}
    \item $\mathbf B(f,:) \in [cone(\mathbf W_1^0)^\ast]^{C_1 \epsilon_n}$, i.e.,
    \begin{equation}\label{eq:B_1}
        \mathbf B(f,:)^T \mathbf W_1^0 \geq - C_1\epsilon_n\|\mathbf B(f,:)\|_2\mathbf 1_s.
    \end{equation}
    \item $\mathbf B(f,:) \in [bd \mathcal{K}]^{C_1\epsilon_n}$, i.e.,
    \begin{align}
        \| \mathbf  B(f,:)\|_2 - \mathbf B(f,:)^T \mathbf 1_k \leq   C_1\epsilon_n \|\mathbf B(f,:)\|_2, \label{eq:B_21}\\
        \| \mathbf  B(f,:)\|_2 - \mathbf B(f,:)^T \mathbf 1_k \geq  - C_1\epsilon_n \|\mathbf B(f,:)\|_2,\label{eq:B_22}
    \end{align}
\end{enumerate}




Now, we proceed to verify (\ref{eq:B_1}), (\ref{eq:B_21}) and (\ref{eq:B_22}). The following lemma is useful; its proof is given in Appendix \ref{app:prf_lemmas}.
\begin{lemma}\label{lem:det_C_bound}
\begin{equation}
        |\det(\mathbf {\hat C}^T_n \mathbf {\hat C}_n )| \leq (1 + C' \epsilon_n ) |\det({\tC}^T \tC )|\label{eq:vol_error} 
\end{equation}
where $C'> 0$ is a constant.
\end{lemma}

Since
\begin{align}
    \det(\mathbf {\hat C}_n^T \mathbf {\hat C}_n)   & = \det(\mathbf B^{-T}(\tC - \mathbf{\tilde E}_{n1})^T(\tC  - \mathbf{\tilde E}_{n1} )\mathbf B^{-1}))\nonumber\\
     & = |\det(\mathbf B^{-1})|^2\det(\mathbf C^{0 T} \mathbf C^{0} - \mathbf C^{0 T}\mathbf {\tilde E}_{n1}-\mathbf {\tilde E}_{n1}^T \mathbf C^{0}  +   \mathbf {\tilde E}_{n1}^T \mathbf {\tilde E}_{n1})\nonumber\\
    & = |\det(\mathbf B^{-1})|^2 \det(\mathbf C^{0 T} \mathbf C^{0})\det\left( \mathbf I - \mathbf F_n\right),
    \label{eq: Cn_lower_bound}
\end{align}
where $\mathbf F_n = (\mathbf C^{0 T} \mathbf C^{0})^{-1}\mathbf C^{0 T}\mathbf {\tilde E}_{n1}+ (\mathbf C^{0 T} \mathbf C^{0})^{-1}\mathbf {\tilde E}_{n1}^T \mathbf C^{0} - (\mathbf C^{0 T} \mathbf C^{0})^{-1} \mathbf {\tilde E}_{n1}^T \mathbf {\tilde E}_{n1}$. Then $\|\mathbf F_n\|_2 \leq C_5 \epsilon_n$. 
We order the singular values $\sigma_i$ of $\mathbf I - \mathbf F_n$ as $\sigma_1\leq \sigma_2\leq\cdots \leq \sigma_k$. By Weyl's inequality in matrix theory \citep{weyl1912asymptotische}, $|1-\sigma_i|\leq \|\mathbf F_n\|_2\leq C_5 \epsilon_n$ for all $ i=1,\cdots, k$.
Therefore 
\begin{equation}\label{eq: detI-Fn_lower_bound}
\det\left( \mathbf I - \mathbf F_n\right) = \prod_{i=1}^k\sigma_i\geq (1-C_5 \epsilon_n)^k \geq 1- k C_5 \epsilon_n.
\end{equation}
By (\ref{eq: Cn_lower_bound}) and (\ref{eq: detI-Fn_lower_bound}), we have \begin{equation}\label{eq: detCn_lower_bound}
\det(\mathbf {\hat C}_n^T \mathbf {\hat C}_n) \geq |\det(\mathbf B^{-1})|^2 \det(\mathbf C^{0 T} \mathbf C^{0})\left(1 -C'_5 \epsilon_n \right) 
\end{equation}
By \eqref{eq:vol_error}  and \eqref{eq: detCn_lower_bound}, we have 
\begin{equation}\label{eq:det_lower_bound}
    |\det(\mathbf B)| \geq 1 - C_6 \epsilon_n.
\end{equation}

\begin{itemize}

\item \textit{Verify \eqref{eq:B_1}.}

Right-multiplying $\mathbf W^0$ on both sides of \eqref{eq:C_CB}, we have
\begin{align*}
     \mathbf {\hat C}_n\mathbf{\hat W}_{n1} + \mathbf E_{n1} =  \tC \mathbf{W}_1^0  =   \mathbf {\hat C}_n\mathbf B  \mathbf{W}_1^0 +\mathbf E_{n1}  {\mathbf{W}_1^0}^{T} (\mathbf{W}_1^0{\mathbf{W}_1^0}^{T})^{-1}\mathbf{W}_1^0.
\end{align*}
Then, left-multiply $(\mathbf {\hat C}_n^T\mathbf {\hat C}_n)^{-1}\mathbf {\hat C}_n^T$ on both sides of the above equation:
\begin{eqnarray}
\mathbf {\hat W}_{n1} +(\mathbf {\hat C}_n^T\mathbf {\hat C}_n)^{-1}\mathbf {\hat C}_n^T\mathbf E_{n1} & =& \mathbf B\mathbf{W}_1^0  +(\mathbf {\hat C}_n^T\mathbf {\hat C}_n)^{-1}\mathbf {\hat C}_n^T \mathbf E_{n1}  {\mathbf{W}_1^0}^{T} (\mathbf{W}_1^0{\mathbf{W}_1^0}^{T})^{-1}\mathbf{W}_1^0 \nonumber\\
  \mathbf B\mathbf{W}_1^0  &=&  \mathbf {\hat W}_{n1} + (\mathbf{\hat C}_n^T \mathbf {\hat C}_n)^{-1}\mathbf{\hat C}_n^T \mathbf E_{n1} ( \mathbf I - {\mathbf{W}_1^0}^{T} (\mathbf{W}_1^0{\mathbf{W}_1^0}^{T})^{-1}\mathbf{W}_1^0) \nonumber \\
  &\geq &  - C_7\epsilon_n, \label{eq:W_larger_0}
\end{eqnarray}

The last inequality holds because $ \mathbf {\hat W}_{n1} \geq 0$ and 
\begin{align}
    &\|(\mathbf{\hat C}_n^T \mathbf {\hat C}_n)^{-1}\mathbf{\hat C}_n^T \mathbf E_{n1} ( \mathbf I - {\mathbf{W}_1^0}^{T} (\mathbf{W}_1^0{\mathbf{W}_1^0}^{T})^{-1}\mathbf{W}_1^0) \|_F \nonumber\\
    \leq &
    \sqrt{k}\|(\mathbf{\hat C}_n^T \mathbf {\hat C}_n)^{-1}\mathbf{\hat C}_n^T \mathbf E_{n1} ( \mathbf I - {\mathbf{W}_1^0}^{T} (\mathbf{W}_1^0{\mathbf{W}_1^0}^{T})^{-1}\mathbf{W}_1^0) \|_2\nonumber\\
    \leq &
    \sqrt{k} \cdot\|(\mathbf{\hat C}_n^T \mathbf {\hat C}_n)^{-1}\mathbf{\hat C}_n^T\|_2 \cdot\|\mathbf E_{n1}\|_2 \cdot\| \mathbf I - {\mathbf{W}_1^0}^{T} (\mathbf{W}_1^0{\mathbf{W}_1^0}^{T})^{-1}\mathbf{W}_1^0\|_2 \nonumber\\
    \leq &
    \sqrt{k} \cdot {[ \sigma^{+}_{\min} (\mathbf{\hat C}_n)]}^{-1} \cdot C\sqrt{s}\epsilon_n \cdot 1\nonumber\\
    \leq & C'\sqrt{s}\epsilon_n,
\end{align}
where in the last inequality we use the fact that $\sigma^{+}_{\min}  (\mathbf{\hat C}_n)$ is lower-bounded by a positive constant. That is because by \eqref{eq: detCn_lower_bound},
\begin{align}\label{eq: detCn_lower_bound_new}
\det(\mathbf {\hat C}_n^T \mathbf {\hat C}_n) 
&\geq |\det(\mathbf B^{-1})|^2 \det(\mathbf C^{0 T} \mathbf C^{0})\left(1 -C'_5 \epsilon_n \right) \nonumber\\
&\geq \|\mathbf B\|_2^{-2k}\det(\mathbf C^{0 T} \mathbf C^{0})\left(1 -C'_5 \epsilon_n \right)\nonumber\\
&\geq M^{-2k} \det(\mathbf C^{0 T} \mathbf C^{0})\left(1 -C'_5 \epsilon_n \right)\nonumber\\
&\geq \frac{\det(\mathbf C^{0 T} \mathbf C^{0})}{2\cdot M^{2k}} 
\end{align}
At the same time,
\begin{align}\label{detCn_upper_bound}
\det(\mathbf {\hat C}_n^T \mathbf {\hat C}_n) 
\leq \|\mathbf{\hat C}_n\|_2^{2(k-1)}{[ \sigma^{+}_{\min} (\mathbf{\hat C}_n)]}^{2} \leq k^{k-1} {[ \sigma^{+}_{\min} (\mathbf{\hat C}_n)]}^{2}
\end{align}
Combining \eqref{eq: detCn_lower_bound} and \eqref{detCn_upper_bound}, we get a lower bound for $\sigma^{+}_{\min} (\mathbf{\hat C}_n)$.


\item \textit{Verify \eqref{eq:B_21}.}

Since $\mathbf 1_k^T \mathbf W_1^0 = \mathbf 1_s^T$, by \eqref{eq:W_larger_0}, we have 
\begin{align}\label{eq: 12.5}
(\mathbf B + C_7 \epsilon_n \mathbf 1_{k\times k})\mathbf W_1^0 = \mathbf B\mathbf W_1^0 + C_7 \epsilon_n \mathbf 1_{k\times s} \geq 0,
\end{align}
which implies that for any row of $\mathbf B$, $\mathbf B(f,:)$,
$$(\mathbf B(f,:)+ C_7 \epsilon_n\mathbf 1_k) \in cone(\mathbf W_1^0)^\ast = \{\mathbf x: \mathbf x^T\mathbf W_1^0 \geq \mathbf 0\} \subseteq \mathcal{K} = \{\mathbf x: \|\mathbf x\|_2 \leq \mathbf x^T \mathbf 1\},$$
where we use the condition (S1) in the definition of SS condition.
$(\mathbf B(f,:)+ C_7 \epsilon_n\mathbf 1_k) \in \{\mathbf x: \|\mathbf x\|_2 \leq \mathbf x^T \mathbf 1\}$ implies that
\begin{align}\label{eq: B_upper_bound}
    \|\mathbf B(f,:)+ C_7 \epsilon_n\mathbf 1_k\|_2 & \leq (\mathbf B(f,:)+ C_7 \epsilon_n\mathbf 1_k)^T \mathbf 1_k \nonumber \\ 
    \|\mathbf B(f,:)\|_2 & \leq \mathbf B(f,:)^T \mathbf 1_k  + C_8 \epsilon_n. 
\end{align}

\item \textit{Verify \eqref{eq:B_22}.}

By Hadamard's inequality, Inequality of AM-GM, and \eqref{eq:det_lower_bound}, we have
\begin{align}
    &\left(\frac{1}{k}\sum_{f=1}^k \|\mathbf B(f,:)\|_2\right)^k  \stackrel{AM-GM}{\geq}  \prod_{f=1}^k \|\mathbf B(f,:)\|_2 \stackrel{Hadamard's}{\geq} |\det(\mathbf B)|\stackrel{ \eqref{eq:det_lower_bound}}{\geq} 1- C_6 \epsilon_n. \label{eq:am-gm-hadamard}
\end{align}
Consequently,
\begin{align}
   \frac{1}{k} \sum_{p=1}^k \|\mathbf B(p,:)\|_2 & \geq  (1- C_6 \epsilon_n)^{1/k} \nonumber \\
   \frac{1}{k} \sum_{p=1}^k \|\mathbf B(p,:)\|_2 & \geq  (1- C_6 \epsilon_n)^{1/k} \cdot \frac{1}{k} \sum_{p=1}^k [\mathbf B(p,:)^T\mathbf 1_k] & \text{by~} \mathbf B^T \mathbf 1_k = \mathbf 1_k \nonumber \\
    \sum_{p=1}^k \|\mathbf B(p,:)\|_2 & \geq  \sum_{p=1}^k [\mathbf B(p,:)^T\mathbf 1_k] - C_9 \epsilon_n  \nonumber \\
     \|\mathbf B(f,:) \|_2 +  \sum_{p\neq f}^k \|\mathbf B(p,:)\|_2 & \geq \mathbf B(f,:)^T\mathbf 1_k +\sum_{p\neq f}^k \mathbf B(p,:)^T\mathbf 1_k - C_9 \epsilon_n  \nonumber\\
    \|\mathbf B(f,:) \|_2 & \geq \mathbf B(f,:)^T\mathbf 1_k - \sum_{p\neq f}^k [\|\mathbf B(p,:)\|_2 - \mathbf B(p,:)^T\mathbf 1_k ] - C_9 \epsilon_n \nonumber\\
     \|\mathbf B(f,:) \|_2 & \geq \mathbf B(f,:)^T\mathbf 1_k - (k-1)C_8\epsilon_n  - C_9 \epsilon_n  & \text{by \eqref{eq: B_upper_bound}} \nonumber\\
    & \geq  \mathbf B(f,:)^T\mathbf 1_k - C_{10}\epsilon_n,\quad \forall f= 1,\cdots,k \label{eq: B_lower_bound}
\end{align}

\item  \textit{Check $\|\mathbf B(f, :)\|_2$ is lower-bounded.}

Now we show that, $\|\mathbf B(f, :)\|_2$ is lower-bounded, using Inequality of AM-GM and (\ref{eq: B_upper_bound}), 
\begin{align} \label{eq: B_l2_lower_bound}
 \|\mathbf B(f,:) \|_2 \left(\frac{1}{k-1} \sum_{p\neq f} \|\mathbf B(p,:)\|_2 \right)^{k-1} &\stackrel{AM-GM}{\geq} \|\mathbf B(f,:) \|_2 \prod_{p\neq f}\|\mathbf B(p,:)\|_2 &  \nonumber\\  
  \|\mathbf B(f,:) \|_2\left(\frac{ \sum_{p\neq f} [\mathbf B(p,:)^T \mathbf 1  + C_8 \epsilon_n] }{k-1}\right)^{k-1} &\geq  \|\mathbf B(f,:) \|_2 \prod_{p\neq f}\|\mathbf B(p,:)\|_2 & \text{by~} (\ref{eq: B_upper_bound}) \nonumber\\  
    \|\mathbf B(f,:) \|_2 \left(\frac{k - \|\mathbf B(f,:) \|_2 + kC_8\epsilon_n}{k-1}\right)^{k-1} &\geq 1 - C_6\epsilon_n &  \text{by~}\eqref{eq:am-gm-hadamard}\nonumber \\  
    (1 + C^{\prime}_8\epsilon_n)\|\mathbf B(f,:) \|_2 \left(\frac{k}{k-1}\right)^{k-1} & \geq 1 - C_6\epsilon_n &  \nonumber\\
    \|\mathbf B(f,:) \|_2 & \geq e^{-1} (1 - C_8^{\prime\prime} \epsilon_n) & \text{by~} \left(1 + \frac{1}{x}\right)^x \leq e \nonumber\\
   & \geq e^{-1}/2 \quad \forall f= 1,\cdots,k.
\end{align}
\item Now we put all the previous derivations together. From \eqref{eq: 12.5} and \eqref{eq: B_l2_lower_bound}, we have (\ref{eq:B_1}) holds,
$$ \mathbf B(f,:)^T \mathbf W_1^0 \geq - C_7\epsilon_n\mathbf 1_s\geq - 2 e\cdot C_7  \|\mathbf B(f,:)\|_2 \epsilon_n \mathbf 1_s.$$

Similarly, from \eqref{eq: B_upper_bound}, \eqref{eq: B_lower_bound} and \eqref{eq: B_l2_lower_bound},  we have
\begin{eqnarray*}
    \mathbf B(f,:)^T\mathbf 1_k - 2e\cdot C_{10}  \|\mathbf B(f,:)\|_2\epsilon_n & \leq \|\mathbf B(f,:)\|_2 & \leq  \mathbf B(f,:)^T \mathbf 1_k  + 2e\cdot C_8   \|\mathbf B(f,:)\|_2 \epsilon_n.
\end{eqnarray*}
Therefore, (\ref{eq:B_21}) and (\ref{eq:B_22}) hold.

 \end{itemize}
 \end{proof}

\subsection{Proof of Theorem~\ref{thm:sample}}\label{app:pf_thm_3}
 

\begin{proof}

This proof consists of two major steps:

\textit{Step 1}: We apply Chernoff bound to show that with probability at least $1 - D_1' s/d$, for any $\mathbf w_i^\sharp$, there exists at least one sample $\mathbf w_{(i)}^1$, such that 
$$\|\mathbf w_{(i)}^1 - \mathbf w_i^\sharp\|_2 \leq r_d,\;\; \forall i = 1,\cdots, s,$$
where $r_d = \left(\frac{\log d}{d} \right)^{\frac{1}{k-1}}$.

\textit{Step 2}: Let $\mathbf W_1^0 = \{\mathbf w^1_{(i)}\}_{i=1}^s$, 
$\mathbf {\hat W}_n = \{\mathbf {\hat w}_{n (i)}\}_{i=1}^d$, and $\mathbf B = \mathbf{\hat W}_{n1}\mathbf W_1^{0T} (\mathbf W_1^{0}\mathbf W_1^{0T} )^{-1}$. 
We show with probability at least $1 - D_2'd/(n\vee d)^c$,  for all $f = 1,\cdots, k$, $\mathbf B(f,:)$ satisfies:
\begin{align}\label{eq:to_show_step2_dist}
\mathbf B(f,:) \in [cone(\mathbf W^0_1)^*]^{C_1^\prime \epsilon_n} \bigcap [bd \mathcal{K}]^{C_1^\prime \epsilon_n}
\end{align}
Then using the conclusion from Theorem \ref{thm:main}, we get the desired bound.


~\\

In the following, we provide the details of the above-mentioned steps.~\\

\textit{Step 1}:
Let $X_i$ denote a random variable representing the number of documents falls into the ball $B(\mathbf w^\sharp_i, r_d)$ ($r_d \leq r_0$) in a sample of size $d$ drawn from $\mathcal{P}$,
$$X_i \sim \text{Binomial}(d, p_i)$$
where $p_i \geq (k-1)!\cdot c_0 \cdot r_d^{k-1}$. Since $\mathcal{P}$ is an $(\alpha, \beta)$-SS distribution, we have
\begin{align}\label{eq:lower_p_i}
p_{i}=\mathbb{P}\left(\|\mathbf{w}-\mathbf{w}_{i}^{\sharp}\|_{2} \leq r_{d}\right) \geq(k-1) ! \cdot c_{0} \cdot r_{d}^{k-1}=C_{3} \frac{\log d}{d}
\end{align}
According to Chernoff bound, for $0<\delta<1$,
$$
\mathbb{P}\left(X_{i} \leq(1-\delta) C_{3} \log d\right) \stackrel{\eqref{eq:lower_p_i}}{\leq} \mathbb{P}\left(X_{i} \leq(1-\delta) d p_{i}\right) \stackrel{\text { Chernoff }} \leq \exp \left(-\frac{\delta^{2} d p_{i}}{2}\right)\stackrel{\eqref{eq:lower_p_i}}{\leq} \exp \left(-\delta^{2} C_{3} \log d / 2\right).
$$
Therefore, when $d$ is large enough, such that for some $0<\delta_{0}<1,\left(1-\delta_{0}\right) C_{3} \log d \geq \frac{1}{2}$, we have
$$
\mathbb{P}\left(X_{i} \leq \frac{1}{2}\right) \leq \exp \left(-\delta_{0}^{2} C_{3} \log d / 2\right)=D_{1}^{\prime} \frac{1}{d}, \quad \forall i=1, \cdots, s
$$
Then, we can bound the probability of the event $\left\{\min_{i=1,\cdots, s} X_i \leq \frac{1}{2}\right\}$,
$$P\left(\min_{i=1,\cdots, s} X_i \leq \frac{1}{2}\right) \leq \sum_{i=1}^s P\left(X_i \leq \frac{1}{2}\right) \leq  D'_1 \frac{s}{d}. $$
In other words, with probability at least $1 - D'_1 s/d$, there exist $s$ different samples $\mathbf w^1_{(1)}, \cdots, \mathbf w^1_{(s)}$, such that $\|\mathbf w^1_{(i)} - \mathbf w^\sharp_i\|_2 \leq r_d$. 
~\\

\textit{Step 2}:
Denote $\mathbf W_1^{0} = \{\mathbf w^1_{(i)}\}_{i=1}^s$, 
$\tW = \left(\mathbf W_1^{0}, \mathbf W_2^{0}\right)\in \mathds{R}^{k\times d}$, and $\mathbf {\hat W}_n = \{\mathbf {\hat w}_{n(i)}\}_{i=1}^d$. We have
\begin{align*}
     \mathbf C^0 \mathbf W^0  = \mathbf{\hat C}_n \mathbf{\hat W}_n + \mathbf E_n
\end{align*}
Therefore,
\begin{align}\label{eq:dist_from_before}
     \mathbf C^0 \mathbf W_1^{0}  = \mathbf{\hat C}_n \mathbf{\hat W}_{n1} + \mathbf E_{n1},\;\; \mathbf W_1^{0}  = \mathbf W^\sharp + \mathbf E_d ^\prime,
\end{align}
where $\mathbf{\hat W}_{n1}$ and $\mathbf E_{n1}$ are the collections of the corresponding columns from $\mathbf{\hat W}_n$ and $\mathbf E_n$ respectively. Moreover,
$\|\mathbf E_{n1}(:,j)\|_2 \leq C_3 \epsilon_n$ and $ \|\mathbf E^\prime_{d}(:,j)\|_2 \leq r_d$, for all $ j = 1, \cdots, s$.

Now we show that 
$[cone(\mathbf W_1^{0})^\ast]^{ \alpha-r_d}=  \{\mathbf x: \mathbf x^T \mathbf W_1^{0} \geq -(\alpha-r_d)\|\mathbf x\|_2\} \subseteq
[cone(\mathbf W^\sharp)^\ast]^{ \alpha}=  \{\mathbf x: \mathbf x^T \mathbf W^\sharp \geq -\alpha\|\mathbf x\|_2\}$. For any $\mathbf x \in \mathds{R}^{k}$ and $\mathbf x^T \mathbf W_1^{0} \geq -(\alpha-r_d)\|\mathbf x\|_2$, 
$$-\alpha\|\mathbf x\|_2 \leq \mathbf x^T \mathbf W_1^{0} -r_d\|\mathbf x\|_2 \leq \mathbf x^T \mathbf W_1^{0} + \mathbf x^T(\mathbf W^\sharp-\mathbf W_1^{0}) = \mathbf x^T \mathbf W^\sharp,$$
where in the second inequality we apply \eqref{eq:sample_close} and Cauchy–Schwarz inequality.
Therefore, by definition, if $\mathbf W^\sharp$ is $(\alpha, \beta)$-SS, $\mathbf W_1^{0}$ is $(\alpha-r_d, \beta)$-SS.

Back to our case, since $\mathbf W^\sharp$ is $(C_1'\sqrt{\frac{\log (n\vee d)}{n}}+r_d, \beta)$-SS, $\mathbf W_1^{0}$ is $(C_1'\sqrt{\frac{\log (n\vee d)}{n}}, \beta)$-SS. Then by Theorem \ref{thm:main}, we obtain that with
probability at least $1-D_{1}^{\prime} s / d-D_2'd/(n\vee d)^c$, 
$$\dis(\mathbf{\hat C}_n, \tC) \leq D_3' \sqrt{\frac{\log (n\vee d)}{n}} + D_4' \beta.$$

\end{proof}


\newpage
\section{Proofs of Technical Lemmas and Propositions}\label{app:prf_lemmas}
In this section, we provide proofs of propositions and all technical lemmas. From now on, we use $\tu{}$ to denote the true word frequency; $\hat{\mathbf u}$ the the sample word frequency; $\mathbf{\tilde u} = \mathbf {\hat C}_n \mathbf {\hat w}$ the estimated word frequency in $\conv(\mathbf {\hat C}_n)$. We use the superscript $(i)$ to denote the $i$-th document. For example, $\mathbf u^{0(i)}$ denotes the true word frequency of the $i$-th document and $\mathbf x^{(i)}$ denotes the observation of the $i$-th document. We write $\tU = \tC\tW = (\mathbf u^{0(1)},\cdots,\mathbf u^{0(d)})\in \mathds{R}^{V\times d}$ and $\mathbf{\tilde U}_n = \mathbf {\hat C}_n \mathbf {\hat W}_n =  (\mathbf {\tilde u}^{(1)},\cdots,\mathbf {\tilde u}^{(d)})\in \mathds{R}^{V\times d}$. 

We use $f^{(i)}(\mathbf u)$ as a shorthand notation of $f_n(\mathbf u; \mathbf x^{(i)})$. By Pinsker's inequality, we have
\begin{equation}\label{eq:pinskers}
    \frac{f^{(j)}(\mathbf u)}{f^{(j)}(\mathbf {\hat u}^{(j)})} \leq \exp\left( - \frac{n}{2}\|\mathbf{\hat u}^{(j)} - \mathbf u\|_2^2\right), 
\end{equation} 
for any $\mathbf u \in \simplex^{V-1}.$ By the reverse Pinsker's inequality \citep{gotze2019higher}, we have
\begin{equation}\label{eq:inv_pinskers}
    \frac{f^{(j)}(\mathbf u)}{f^{(j)}(\mathbf {\hat u}^{(j)})} \geq \exp\left( - C_6n\|\mathbf{\hat u}^{(j)} - \mathbf u\|_2^2\right), 
\end{equation} 
where $C_6 = (\min_{i\in [V]} u_i)^{-1}$ depends on the minimum element of $\mathbf u$.

\subsection{Proof of Lemma \ref{lem:vol_det}}


\begin{proof}
Write $\mathbf{C} = (\mathbf c_1, \cdots, \mathbf c_k)\in \mathds{R}^{V\times k}$ and $\mathbf{\tilde C}=(\mathbf{\tilde c}_1,  \cdots, \mathbf{\tilde c}_{k-1}) \in \mathds{R}^{V\times (k-1)}$, where $\mathbf{\tilde c}_j = \mathbf c_{j} - \mathbf c_k$ with $j \in [k-1].$ Write $$\mathbf G = \mathbf C^T \mathbf C, \;\;\mathbf{\tilde G} = \mathbf {\tilde C}^T\mathbf {\tilde C}. $$
The volume of the $k$-dimensional parallelepiped spanned by $\mathbf c_1, \cdots, \mathbf c_k \in \mathds{R}^V$  is given by $\sqrt{\det(\mathbf G)}$ \citep{boyd2004convex}. Therefore $\sqrt{\det(\mathbf{\tilde G})} =(k-1)! |\conv(\mathbf C)| $, since $\sqrt{\det(\mathbf{\tilde G})}$ measures the volume of the $(k-1)$-dimensional parallelepiped spanned by columns of $\mathbf {\tilde C}$ in $\mathds{R}^V$, which is $(k-1)!$ times larger than the volume of $\conv(\mathbf C)$. It suffices to show that 
$$h^2\det(\mathbf{\tilde G}) = \det(\mathbf{G}).$$

Denote by $\mathbf v$ the perpendicular vector to $\text{aff}(\mathbf C)$, represented as $$\mathbf v = \sum_{j=1}^{k-1} t_j(\mathbf c_j - \mathbf c_k) + \mathbf c_k,$$
so that 
$$(\mathbf c_j - \mathbf c_k)^T \mathbf v = 0\;\;\text{and}\;\;\mathbf c^T_k \mathbf v = \|\mathbf v\|_2 = h^2.$$
Further, we construct a system of $k$ linear equations for $k$ unknowns, $t_1, \cdots, t_{k-1}, h^2$,
\begin{align}
    \sum_{j=1}^{k-1} t_j \mathbf{\tilde c}_i^T\mathbf{\tilde c}_j & = - \mathbf{\tilde c}_i^T\mathbf{ c}_k,\; i = 1, \cdots, k-1 \\
     \sum_{j=1}^{k-1} t_j \mathbf{\tilde c}_j^T\mathbf{c}_k - h^2 & = - \mathbf{ c}_k^T\mathbf{c}_k.
\end{align}
By Cramer's rule, we have
\begin{align*}
    h^2&  = \frac{1}{\det \left(\left[\begin{array}{ *{4}{c} }
    & & &  0 \\
    & & & \vdots  \\
    \multicolumn{3}{c}
      {\raisebox{\dimexpr\normalbaselineskip+.7\ht\strutbox-.5\height}[0pt][0pt]
        {\scalebox{1.5}{$\mathbf{\tilde G}$}}} &0  \\
      \mathbf{\tilde c}_1^T \mathbf c_k  & \cdots &  \mathbf{\tilde c}_{k-1}^T \mathbf c_k & -1 
  \end{array}\right]\right)}
  \det\left(\left[\begin{array}{ *{4}{c} }
    & & &  -\mathbf{\tilde c}_1^T\mathbf c_k \\
    & & & \vdots  \\
    \multicolumn{3}{c}
      {\raisebox{\dimexpr\normalbaselineskip+.7\ht\strutbox-.5\height}[0pt][0pt]
        {\scalebox{1.5}{$\mathbf{\tilde G}$}}} &-\mathbf{\tilde c}_{k-1}^T\mathbf c_k  \\
      \mathbf{\tilde c}_1^T \mathbf c_k & \cdots & \mathbf{\tilde c}_{k-1}^T \mathbf c_k & -\mathbf{ c}_{k}^T\mathbf c_k 
  \end{array}\right]\right)
\end{align*} 
   %
Then, we see that for the denominator,
\begin{align*}
    & \det \left(\left[\begin{array}{ *{4}{c} }
    & & &  0 \\
    & & & \vdots  \\
    \multicolumn{3}{c}
      {\raisebox{\dimexpr\normalbaselineskip+.7\ht\strutbox-.5\height}[0pt][0pt]
        {\scalebox{1.5}{$\mathbf{\tilde G}$}}} &0  \\
      \mathbf{\tilde c}_1^T \mathbf c_k & \cdots & \mathbf{\tilde c}_{k-1}^T \mathbf c_k & -1 
  \end{array}\right]\right) = -\det(\mathbf{\tilde G})\\
\end{align*}
The numerator is
\begin{align*}
   &  - \det\left(\left[\begin{array}{ *{4}{c} }
     \mathbf{\tilde c}_1^T \mathbf{\tilde c}_1& \cdots & \mathbf{\tilde c}_1^T \mathbf{\tilde c}_{k-1} &  \mathbf{\tilde c}_1^T\mathbf c_k \\
    \vdots&\ddots   &\vdots & \vdots  \\
    \mathbf{\tilde c}_{k-1}^T \mathbf{\tilde c}_1& \cdots & \mathbf{\tilde c}_{k-1}^T \mathbf{\tilde c}_{k-1} &  \mathbf{\tilde c}_{k-1}^T\mathbf c_k \\
      \mathbf{\tilde c}_1^T \mathbf c_k & \cdots & \mathbf{\tilde c}_{k-1}^T \mathbf c_k & \mathbf{ c}_{k}^T\mathbf c_k 
  \end{array}\right]\right)\\
  \xlongequal{\text{add last column to others}} & - \det\left(\left[\begin{array}{ *{4}{c} }
     \mathbf{\tilde c}_1^T \mathbf{c}_1& \cdots & \mathbf{\tilde c}_1^T \mathbf{ c}_{k-1} &  \mathbf{\tilde c}_1^T\mathbf c_k \\
    \vdots&\ddots   &\vdots & \vdots  \\
    \mathbf{\tilde c}_{k-1}^T \mathbf{ c}_1& \cdots & \mathbf{\tilde c}_{k-1}^T \mathbf{ c}_{k-1} &  \mathbf{\tilde c}_{k-1}^T\mathbf c_k \\
      \mathbf{ c}_1^T \mathbf c_k & \cdots & \mathbf{ c}_{k-1}^T \mathbf c_k & \mathbf{ c}_{k}^T\mathbf c_k 
  \end{array}\right]\right) \\
  \xlongequal{\text{add last row to others}} &- \det\left(\left[\begin{array}{ *{4}{c} }
     \mathbf{c}_1^T \mathbf{c}_1& \cdots & \mathbf{ c}_1^T \mathbf{ c}_{k-1} &  \mathbf{ c}_1^T\mathbf c_k \\
    \vdots&\ddots   &\vdots & \vdots  \\
    \mathbf{ c}_{k-1}^T \mathbf{ c}_1& \cdots & \mathbf{ c}_{k-1}^T \mathbf{ c}_{k-1} &  \mathbf{ c}_{k-1}^T\mathbf c_k \\
      \mathbf{ c}_1^T \mathbf c_k & \cdots & \mathbf{ c}_{k-1}^T \mathbf c_k & \mathbf{ c}_{k}^T\mathbf c_k 
  \end{array}\right]\right) \\
  = & - \det(\mathbf G).
\end{align*} 
\end{proof}

\subsection{Proof of Lemma \ref{lem:W_sigma_min}}
\begin{proof}
We will show that for any $\mathbf x \in \mathds{R}^k, \mathbf x\neq \mathbf 0$, there exists $\bm \beta\in\mathds R^d$ such that 
\begin{equation}\label{eq:x_Wbeta}
    \frac{\mathbf x^T(\mathbf W\bm\beta)}{\|\mathbf x\|_2} \geq \frac{1}{k},\, \text{and}\, \|\bm\beta\|_2 \leq 1.
\end{equation}
Therefore, 
$$\frac{\|\mathbf x^T\mathbf W\|_2}{\|\mathbf x\|_2}\geq \frac{1}{\|\bm\beta\|_2}\frac{\mathbf x^T\mathbf W\bm\beta}{\|\mathbf x\|_2}\geq \frac{1}{k}.$$
~\\

In the following, we will find $\bm\beta$ satisfying \eqref{eq:x_Wbeta}.

First, decompose $\mathbf x$ as
$$\mathbf x = \frac{\lambda}{k}\mathbf 1_k + \bm\gamma,$$
for some $\lambda\in\mathds R$ and $\bm\gamma \in \mathds R^k$ such that $\bm\gamma^T \mathbf 1_k = 0.$

Second, let
$$\mathbf y = \frac{sign(\lambda)}{k}\cdot\mathbf 1_k + \frac{1}{\sqrt{k(k-1)}\|\bm\gamma\|_2}\cdot \bm\gamma $$
where $sign(\cdot)$ is the sign function. Next we verify that
$$\frac{\mathbf x^T\mathbf y}{\|\mathbf x\|_2} \geq \frac{1}{k}.$$
This is because
\begin{eqnarray*}
    \mathbf x^T\mathbf y &=& \frac{|\lambda|}{k} + \frac{1}{\sqrt{k(k-1)}}\|\bm\gamma\|_2 \\
    \|\mathbf x\|_2 &=& \sqrt{\frac{\lambda^2}{k} + \|\bm \gamma\|_2^2} \leq \begin{cases}
|\lambda| &\mbox{if } |\lambda|\geq \sqrt{\frac{k}{k-1}}\|\bm\gamma\|_2 \\
\sqrt{\frac{k}{k-1}}\|\gamma\|_2 &\mbox{if } |\lambda|< \sqrt{\frac{k}{k-1}}\|\bm\gamma\|_2
\end{cases}, 
\end{eqnarray*}
and
$$
    \frac{\mathbf x^T\mathbf y}{\|\mathbf x\|_2}  \geq  \begin{cases}
\frac{1}{k} + \frac{1}{\sqrt{k(k-1)}}\frac{\|\bm\gamma\|_2}{|\lambda|} &\mbox{if } |\lambda|\geq \sqrt{\frac{k}{k-1}}\|\bm\gamma\|_2 \\
\frac{|\lambda|}{\|\bm\gamma\|_2}\frac{\sqrt{k-1}}{k\sqrt{k}} + \frac{1}{k} &\mbox{if } |\lambda|< \sqrt{\frac{k}{k-1}}\|\bm\gamma\|_2.
\end{cases}
$$

Third, we verify that $[sign(\lambda)\cdot\mathbf y] \in [\simplex^{k-1}\bigcap \mathcal K^\ast]$. This is because
$$[sign(\lambda)\cdot\mathbf y]^T\mathbf 1_k = \frac{1}{k}\mathbf 1^T_k\mathbf 1_k  = 1,$$
$$\|\mathbf y\|_2 = \sqrt{\frac{1}{k} + \frac{1}{k(k-1)}} = \frac{1}{\sqrt{k-1}}.$$
Since $\mathcal K^\ast \subseteq cone(\mathbf W)$, we have
$$[\simplex^{k-1}\bigcap\mathcal K^\ast] \subseteq [\simplex^{k-1}\bigcap cone(\mathbf W)] = \{\mathbf x\in \simplex^{k-1}: \mathbf x=\mathbf W\bm\lambda, \bm \lambda\geq 0\}=\conv(\mathbf W).$$
Therefore, $[sign(\lambda)\cdot\mathbf y] \in \conv(\mathbf W)$, meaning that there exists $\bm\beta' \in \simplex^{d-1}$ such that 
$$\mathbf y = sign(\lambda)\cdot \mathbf W\bm\beta' =\mathbf W[sign(\lambda)\bm\beta'] = \mathbf W\bm\beta ,$$ 
and $$\|\bm\beta\|_2 = \|\bm\beta'\|_2 \leq \bm\beta'^T\mathbf 1_k = 1.$$

\end{proof}

\subsection{Proof of Lemma \ref{lem:dis_upper_bound}}
We arrange the proof as Lemma \ref{lem:F_lower_bound}  and Lemma \ref{lem:dis_upper_bound1}. First, in Lemma \ref{lem:F_lower_bound}, we derive a lower bound for the integrated likelihood function, $F_{n\times d}(\mathbf {\hat C}_n; \mathbf X)$. Then, we prove equation (\ref{eq:u_error}) in Lemma \ref{lem:dis_upper_bound1}.
~\\

We first define the $\delta$-enlargement convex polytope below, which is useful later in the proof.

\begin{definition}[$\delta$-enlargement convex polytope]\label{def:delta_enlarge}
For a convex polytope, $\conv(\mathbf C) \subseteq \mathds{R}^V$, with $k$ linearly independent vertices $\mathbf C = \{\mathbf c_f\}_{f=1}^k\in \mathds{R}^{V\times k}$. The \textbf{$\bm \delta$-enlargement convex polytope} of $\conv(\mathbf C)$, denoted as $\conv(\mathbf C^\delta)$, is defined such that each column of $\mathbf C^\delta$,
 $$\mathbf c_{f}^\delta = \left(1 + \rho(\mathbf{C})\delta \right) (\mathbf c_f - \mathbf {\bar c}) + \mathbf {\bar c}, \quad \forall f = 1,\cdots k,$$
  where $\rho(\mathbf{C}) = \frac{k} {\sigma^+_{\min} (\mathbf C)}$, and $\mathbf {\bar c} = \frac{1}{k}\sum_{f=1}^k \mathbf c_f \in \mathds{R}^V$ is the center of the $k$ columns of $\mathbf C$. $\mathbf C^\delta$ is called the \textbf{$\bm \delta$-enlargement matrix} of $\mathbf C$.
\end{definition}

\begin{proposition}\label{prop:enlargement}
$\conv(\mathbf C^\delta)$ satisfies the following properties.
\begin{enumerate}
    \item It composes of $k$ vertices, $\mathbf C^\delta  = \{\mathbf c_f^\delta\}_{f=1}^k\in \mathds{R}^{V\times k}$;
    
    \item $|\conv(\mathbf C^\delta)|=\left(1 + \rho(\mathbf{\mathbf C})\delta \right)^{k-1}|\conv(\mathbf{C})|. $
\end{enumerate}
\end{proposition}

~\\

\begin{lemma}\label{lem:F_lower_bound} 
With probability at least $\left(1-3\cdot (n\vee d)^{-c}\right)^d$, the integrated likelihood is lower-bounded:
$$F_{n\times d}(\mathbf{\hat C}_n; \mathbf X) \geq C\cdot A_{n,d}\cdot (n\vee d)^{-C_{10}d},$$
where $A_{n,d} := \prod_{i=1}^d f_n(\mathbf{\hat u}^{(i)}; \mathbf x^{(i)})$ and $C$, $C_{10}$ are constants. 
\end{lemma}

\begin{proof}
The integrated likelihood function can be written as  
\begin{align*}
    F_{n\times d}(\mathbf C; \mathbf X) = & \prod_{i=1}^d \frac{1}{|\conv(\mathbf C)|} \int_{\conv(\mathbf C)} f_n(\mathbf u; \mathbf x^{(i)}) d \mathbf u \\
    = & \prod_{i=1}^d  f_n(\mathbf{\hat u}^{(i)}; \mathbf x^{(i)}) \int_{\conv(\mathbf C)} \frac{1}{|\conv(\mathbf C)|} \frac{f_n(\mathbf u; \mathbf x^{(i)})}{f_n(\mathbf {\hat u}^{(i)}; \mathbf x^{(i)})}  d \mathbf u \\
    = & A_{n,d} \cdot  \prod_{i=1}^d  \int_{\conv(\mathbf C)} \frac{1}{|\conv(\mathbf C)|} \frac{f^{(i)}(\mathbf u)}{f^{(i)}(\mathbf {\hat u}^{(i)})}  d \mathbf u 
\end{align*}
where $f^{(i)}(\mathbf u)$ is a shorthand notation of $f_n(\mathbf u; \mathbf x^{(i)})$.


 
  By \citet{devroye1983equivalence}, for each document $i$, it holds with probability at least $1-3\cdot e^{-cx^2}$ that $\|\mathbf u^{0 (i)} - \mathbf {\hat u}^{(i)}\|_2  \leq\frac{5\sqrt{c}x}{\sqrt{n}}$ for all $x>0$. By a simple union bound argument, we have that with probability at least $(1-3\cdot(n\vee d)^{-c})^d$, $\|\mathbf u^{0 (i)} - \mathbf {\hat u}^{(i)}\|_2 \leq 5\sqrt{c}\cdot\sqrt{\frac{\log(n\vee d)}{n}} =: C_1\sqrt{\frac{\log(n\vee d)}{n}}$,
 for any $i\in [d]$, by choosing $x$ to be a large multiple of $\sqrt{\log(n
\vee d)}$. Let $\mathcal{B}(\mathbf u^{0 (i)}; C_1 \epsilon_n)$ denote the Euclidean ball centered at $\mathbf u^{0 (i)}$ with radius $C_1\epsilon_n$. Consequently, with high probability, for any $\mathbf{u}\in \mathcal{B}(\mathbf u^{0 (i)}; C_1 \epsilon_n)$,
 \begin{align}\label{eq:dist_close_in_ball}
     \|\mathbf{u}-\mathbf {\hat u}^{(i)}\|_2\leq \|\mathbf{u}-\mathbf { u}^{0(i)}\|_2 + \|\mathbf { u}^{0(i)}-\mathbf {\hat u}^{(i)}\|_2\leq 2C_1\epsilon_n.
 \end{align}
 
 Next, by the definition of MLE, we have
 \begin{align}
    F_{n\times d}(\mathbf{\hat C}_n; \mathbf X) & \geq F_{n\times d}(\tC; \mathbf X) \nonumber \\
    & =  \frac{A_{n,d}}{|\conv(\tC)|^d}\cdot  \prod_{i=1}^d  \int_{\conv(\tC)}  \frac{f^{(i)}(\mathbf u)}{f^{(i)}(\mathbf {\hat u}^{(i)})}  d \mathbf u \nonumber \\
    & \geq  \frac{A_{n,d}}{|\conv(\tC)|^d}\cdot  \prod_{i=1}^d  \int_{\conv(\tC)\bigcap \mathcal{B}(\mathbf u^{0(i)};  C_1\epsilon_n)}  \frac{f^{(i)}(\mathbf u)}{f^{(i)}(\mathbf {\hat u}^{(i)})}  d \mathbf u \nonumber \\
    & \geq  \frac{A_{n,d}}{|\conv(\tC)|^d}\cdot  \prod_{i=1}^d  \int_{\conv(\tC)\bigcap \mathcal{B}(\mathbf u^{0(i)};  C_1\epsilon_n)}  
    \exp \left (-C_6 n \|\mathbf{u}-\mathbf {\hat u}^{(i)}\|_2^2 \right ) d \mathbf u \label{pf:lemmaB1:KLbound} \\ 
    & \geq \frac{A_{n,d}}{|\conv(\tC)|^d}\cdot \prod_{i=1}^d \left[C_8 (C_1\epsilon_n)^{k-1} \cdot \exp(-C_6 n (2C_1\epsilon_n)^2) \right] \label{pf:lemmaB1:dist:bound}\\
    & \geq C\cdot A_{n,d}\cdot (n\vee d)^{-C_{10}d}. \nonumber
\end{align}
Inequality \eqref{pf:lemmaB1:KLbound} follows from the reverse Pinsker's inequality (\ref{eq:inv_pinskers}) since the columns of $\tC$ are interior points in $\simplex^{V-1}$, i. Inequality \eqref{pf:lemmaB1:dist:bound} follows from  \eqref{eq:dist_close_in_ball}. 

\end{proof}

\begin{definition}[Distance between a vector and a convex polytope]\label{def:distance}
The distance between a vector $\mathbf x$ and a convex polytope $\conv(\mathbf{C})$ is defined as 
$$d(\mathbf x, \conv(\mathbf{C})) = \min_{\mathbf y \in \conv(\mathbf C)} \| \mathbf x - \mathbf y \|_2.$$
\end{definition}


\begin{lemma}\label{lem:dis_upper_bound1}
With probability at least $\left(1-3\cdot (n\vee d)^{-c}\right)^d$, we have
\begin{align}\label{eq:u0i_close_to_convcn}
d(\mathbf u^{0(i)}, \conv(\mathbf{\hat C}_n))  \leq C\epsilon_n
\end{align}
for any $i \in [d]$.
Therefore, there exists a matrix $\mathbf {\hat W}_n$, such that  $\mathbf {\hat W}_n \geq 0$, $\mathbf {\hat W}_n^T \mathbf 1_k  = \mathbf 1_d$, and 
$$\mathbf U^0 = \tC \mathbf W^0 = \mathbf{\hat C}_n\mathbf {\hat W}_n + \mathbf E_n  =  \mathbf {\tilde U}_n + \mathbf E_n $$
and $\max_i \|\mathbf E_n(:, i) \|_{2} \leq   C\epsilon_n$. Here constants $c$ and  $C$ are independent of $n$ and $d$.
\end{lemma}

\begin{proof}
We prove the lemma by contradiction. Suppose the $i$-th document violates (\ref{eq:u0i_close_to_convcn}):
$$d(\mathbf u^{0(i)}, \conv(\mathbf{\hat C}_n))  \geq C\epsilon_n.$$
First, we claim that there exist at least $C_1d$ columns of $\tU$ such that 
$$d(\mathbf u^{0(i)}, \conv(\mathbf{\hat C}_n))  \geq C_2C\epsilon_n,$$
where $C_1, C_2\in (0,1)$ are constants independent of $n$ and $d$. We prove this claim at the end. 

Then, by \citet{devroye1983equivalence}, with probability at least $(1-3\cdot(n\vee d)^{-c})^d$, we have $\|\mathbf u^{0 (i)} - \mathbf {\hat u}^{(i)}\|_2 \leq O\left(\sqrt{\frac{\log(n\vee d)}{n}}\right)$
hold for all $i = 1, \cdots, d$.
By making the constant $C$ large enough, we have
$$d(\mathbf {\hat u}^{(j)}, \conv(\mathbf{\hat C}_n))  \geq (C_2C-1)\epsilon_n.$$
Therefore,
\begin{align*}
    F_{n \times d}(\mathbf{\hat C}_n; \mathbf X) = &  A_{n,d}  \prod_{i=1}^d \int_{\conv(\mathbf{\hat C}_n)} \frac{1}{|\conv(\mathbf{\hat C}_n)|}  \frac{f^{(i)}(\mathbf u)}{f^{(i)}(\mathbf {\hat u}^{(i)})} d\mathbf u \\
    \leq & \frac{A_{n,d}}{|\conv(\mathbf{\hat C}_n)|^d} \prod_{i=1}^d\int_{\conv(\mathbf{\hat C}_n)} \exp\left( - \frac{n}{2}\|\mathbf{\hat u}^{(i)} - \mathbf u\|_2^2\right) d\mathbf u \\
     = & A_{n,d} \prod_{i=1}^d\exp\left( - \frac{n}{2}\|\mathbf{\hat u}^{(i)} - \mathbf u^{*(i)}\|_2^2\right)\\
   \leq & A_{n,d} \cdot \exp\left( - \frac{n}{2}\sum_{i=1}^d d^2(\mathbf {\hat u}^{(i)}, \conv(\mathbf{\hat C}_n))\right) \\
   \leq & A_{n,d} \cdot \exp\left(- \frac{n}{2} \cdot C_1 d \cdot (C_2 C-1)^2\epsilon_n^2\right) \\
   = & A_{n,d} \cdot (n\vee d)^{- \frac{1}{2} C_1 (C_2C-1)^2C^2_0 d},
   \end{align*}
where the first inequality follows \eqref{eq:pinskers} and the second inequality is due to the mean value theorem for integrals with $\mathbf u^{*(i)}$'s being  some points in $\conv(\mathbf{\hat C}_n)$.
By choosing $C$ large enough, we can make 
$$ F_{n \times d}(\mathbf{\hat C}_n; \mathbf X) \leq A_{n,d} \cdot (n\vee d)^{-(C_{10}+1)d},$$
which contradicts with Lemma \ref{lem:F_lower_bound}. So we conclude that 
$$d(\mathbf u^{0(i)}, \conv(\mathbf{\hat C}_n))  \leq C\epsilon_n$$
for all $i=1,\cdots,d$.

It remains to prove the claim we made at the  beginning. When $\text{aff}(\tC)$ is parallel to $\text{aff}(\mathbf{\hat C}_n)$, the claim is trivial by making $C_2$ small. When $\text{aff}(\tC)$ is not parallel to $\text{aff}(\mathbf{\hat C}_n)$, again we prove it by contradiction. Suppose there are at least $(1-C_1)d$ columns of $\tU$ such that 
$$d(\mathbf u^{0(j)}, \conv(\mathbf{\hat C}_n))  \leq C_2C\epsilon_n$$
and let $\mathcal{S}$ be their column index set.  

Denote $r$ as the distance from $\mathbf u^{0(i)}$ to the intersection of $\text{aff}(\tC)$ and $\text{aff}(\mathbf{\hat C}_n)$, i.e.,
\begin{align*}
r = d(\mathbf u^{0(i)}, \text{aff}(\tC)\bigcap \text{aff}(\mathbf{\hat C}_n)),
\end{align*}
where $\mathbf u^{0(i)}$ is the vector such that $d(\mathbf u^{0(i)}, \conv(\mathbf{\hat C}_n))  \geq C\epsilon_n$. 
Since $d(\mathbf u^{0(j)}, \conv(\mathbf{\hat C}_n))  \leq C_2C\epsilon_n$ for all $j\in \mathcal{S}$, we know that 
$$
d(\mathbf u^{0(j)}, \text{aff}(\tC)\bigcap \text{aff}(\mathbf{\hat C}_n))\leq \frac{C_2C\epsilon_n}{C\epsilon_n}\cdot r = C_2r,\; \forall j\in \mathcal{S}.
$$
At the same time, 
$$
r-C_2r \leq \max_{j\in \mathcal{S}}\|\mathbf u^{0(i)}- \mathbf u^{0(j)}\|_2 \leq \max_{ i,j\in [k] }\|\mathbf C^{0(i)}- \mathbf C^{0(j)}\|_2.
$$
Since the RHS is a constant, we know that $r$ is upper bounded.

Let $\mathbf{b}_n$ be the unit normal vector of $\text{aff}(\tC)\bigcap \text{aff}(\mathbf{\hat C}_n)$ on the hyperplane $\text{aff}(\tC)$. Since $\mathbf{b}_n \in \text{aff}(\tC)$, there exists $\bm{\lambda}_n\in \mathds{R}^k$ and $\bm{\lambda}_n^T \mathbf{1}_k = 1$ such that $\mathbf{b}_n = \tC\bm{\lambda}_n$.

On the one hand, the variance of all $\mathbf u^{0(i)}$'s on the direction of $\mathbf{b}_n$ can be upper bounded:
\begin{align}\label{eq:var_upper}
Var_{\mathbf{b}_n}(\tU) \leq \frac{1}{d}\left[(1-C_1)d\cdot C_2^2r^2 + C_1d \cdot r^2\right] = \left((1-C_1) C_2^2 + C_1  \right)r^2
\end{align}
On the other hand, since the minimum eigenvalue of $\mathbf{W}_c \mathbf{W}_c^T$ is lower bounded, we have
\begin{align}\label{eq:var_lower}
Var_{\mathbf{b}_n}(\tU) &\geq \frac{1}{d}\mathbf{b}_n^T \mathbf{U}_c \mathbf{U}_c^T\mathbf{b}_n = \frac{1}{d}\mathbf{b}_n^T \tC\mathbf{W}_c \mathbf{W}_c^T\mathbf{C}^{0T}\mathbf{b}_n \nonumber\\
& \geq C_3 \|\mathbf{C}^{0T}\mathbf{b}_n\|_2^2 = C_3 \cdot\bm{\lambda}_n^T\mathbf{C}^{0T} \tC\mathbf{C}^{0T}\tC\bm{\lambda}_n \nonumber\\
&\geq C_3 \left[\sigma_{\min}^{+}(\tC)\right]^4\|\bm{\lambda}_n\|_2^2 \nonumber\\&
\geq C_3 \left[\sigma_{\min}^{+}(\tC)\right]^4\frac{1}{k}
\end{align}
In \eqref{eq:var_upper}, by choosing the constants $C_1$ and $C_2$ small enough, we can make 
$$
\left((1-C_1) C_2^2 + C_1  \right)r^2 < C_3 \left[\sigma_{\min}^{+}(\tC)\right]^4\frac{1}{k}.
$$
Therefore, we get a contradiction from \eqref{eq:var_upper} and \eqref{eq:var_lower}, which finishes the proof of the claim.

As a conclusion, let $ \mathbf {\tilde u}^{(i)}=\argmin_{\mathbf u\in \conv(\hatCn)}d(\mathbf u, \mathbf u^{0(i)})$ for $i = 1, \cdots, d$ and 
 $\mathbf {\tilde U}_n = \{\mathbf {\tilde u}^{(1)},\cdots, \mathbf {\tilde u}^{(d)}\}$, then we have shown that w.h.p. 
 $\|\mathbf u^{0(i)} -\mathbf {\tilde u}^{(i)}\|_2 \leq  C\epsilon_n$. 
Further, by the definition of $\conv(\mathbf{\hat C}_n)$, there exists $\mathbf {\hat w}^{(i)} \in \simplex^{k-1}$, such that $\mathbf {\tilde u}^{(i)} = \mathbf{\hat C}_n\mathbf {\hat w}^{(i)}$, for any $i=1,\cdots,d$. Let $\mathbf {\hat W}_n = \{\mathbf {\hat w}^{(1)}, \cdots, \mathbf {\hat w}^{(d)}\}$, we have $\mathbf{\hat C}_n\mathbf {\hat W}_n = \mathbf {\tilde U}_n $ and 
$$\|\mathbf E_n(:,i) \|_{2} = \|\mathbf u^{0(i)} -\mathbf{\hat C}_n\mathbf {\hat w}^{(i)}\|_2 = \|\mathbf u^{0(i)} -\mathbf {\tilde u}^{(i)}\|_2  \leq  C \epsilon_n.$$
\end{proof}

\subsection{Proof of Lemma \ref{lem:B_close_Pi}}
\begin{proof}
Since $\mathbf B(f,:) = \lambda_f \mathbf e_{(f)} + \bm \epsilon_f, \|\bm \epsilon_f\|_2 \leq  \lambda_f \beta$ and $\|\mathbf B\|_2 \leq M$, we can bound $\lambda_f$ by $C_2 M$.
\begin{align*}
M \geq \|\mathbf B \|_2  \geq  \|\mathbf B(f,:) \|_2 \geq  \|\lambda_f \mathbf e_{(f)}\|_2 - \|\bm \epsilon_f\|_2 \geq \lambda_f - \beta \lambda_f, \quad  f = 1,\cdots, k.
\end{align*}
$$\lambda_f \leq \frac{M}{1-\beta} \leq C_2 M,\quad f = 1, \cdots, k.$$

We write $\mathbf T = (\lambda_1 \mathbf e_{(1)}, \cdots \lambda_k \mathbf e_{(k)})^T$, $\mathbf E = (\bm \epsilon_1, \cdots, \bm \epsilon_k)^T$, such that $\mathbf T + \mathbf E = \mathbf B$.  


Next, we show that the column sums of $\mathbf T$ are close to $1$, using the fact that the column sums of $\mathbf B$ are all 1's.
\begin{align}\label{eq:T_close_1}
  \sum_{s=1}^k\left|\sum_{f=1}^k \mathbf T(f,s) - 1\right|  & = \sum_{s=1}^k\left| \sum_{f=1}^k \mathbf T(f,s) - \sum_{f=1}^k \mathbf B(f,s) \right| \nonumber\\
  & \leq \sum_{s=1}^k \sum_{f=1}^k \left|\mathbf B(f,s) - \mathbf T(f,s)\right|  = \sum_{f=1}^k \sum_{s=1}^k \left|\mathbf B(f,s) - \mathbf T(f,s)\right| \nonumber\\
    & =  \sum_{f=1}^k \left\|\mathbf B(f,:) - \mathbf T(f,:)\right\|_1 \leq \sqrt{k}\sum_{f=1}^k \|\mathbf B(f,: ) - \mathbf T(f, :)\|_2\nonumber \\
    & = \sqrt{k}\sum_{f=1}^k \|\bm \epsilon_f\|_2 
    \leq \sqrt{k}  \sum_{f=1}^k \lambda_f\beta 
    \leq C_3  M \beta.
\end{align}

Let $\mathbf \Pi = (\mathbf e_{(1)}, \cdots \mathbf e_{(k)})^T$.  Then $\mathbf \Pi$ must be a permutation matrix. Otherwise, there exists at least one column $p$, such that all the entries in the $p$-th column of $\mathbf \Pi$ are $0$, i.e., $ \mathbf e_{(1),p} =\cdots  \mathbf e_{(k),p} = 0$, where $\mathbf e_{(f),p}$ denotes the $p$-th element in $\mathbf e_{(f)}$.  Then the sum of $p$-th column of $\mathbf T$  is $0$, i.e., $\sum_{f=1}^k \mathbf T(f,p) = \sum_{f=1}^k \lambda_f \mathbf e_{(f),p} = 0$, which contradicts with (\ref{eq:T_close_1}). 

Furthermore, since $\mathbf T = \mathbf \Pi \cdot \text{diag}(\lambda_1, \cdots, \lambda_k)$ and $\mathbf \Pi$ is a permutation matrix,  each column of $\mathbf T$ should include one and only one of $\lambda_1, \cdots, \lambda_k$, so that
 $$\sum_{s=1}^k \left|\sum_{f=1}^k \mathbf T(f,s) - 1\right|  = \sum_{f=1}^k |\lambda_f -1 |   \leq C_3 M\beta.$$
 
Consequently,
$$\|\mathbf B - \mathbf \Pi \|_2 \leq \|\mathbf T - \mathbf \Pi\|_2 + \|\mathbf B - \mathbf T \|_2 \leq \sum_{f=1}^k |\lambda_f -1| + \|\mathbf E\|_2 \leq  C'_3 M \beta.$$
where the last inequality holds because 
$$\|\mathbf E\|_2 \leq \|\mathbf E\|_F = \left(\sum_{f=1}^k \|\bm \epsilon_f\|_2^2\right)^{\frac{1}{2}}\leq \sqrt{k}C_2 M\beta$$
\end{proof}


\subsection{Proof of Lemma \ref{lem:det_C_bound}}


Recall that $\epsilon_n = C_0\sqrt{\frac{\log (n\vee d)}{n}}$ where $C_0 > 0$ is a constant. We aim to show that
\begin{align*}
    |\det(\mathbf {\hat C}^T_n \mathbf {\hat C}_n )| &\leq (1 + C'' \epsilon_{n} ) |\det({\tC}^T \tC )|.
\end{align*}

Let $\text{aff}(\hatCn)$ and $\text{aff}(\tC)$ be the $(k-1)$-dim hyperplanes obtained by expanding $\conv(\hatCn)$ and $\conv(\tC)$, respectively. 
By Lemma \ref{lem:vol_det}, 
$$\frac{|\det(\mathbf {\hat C}^T_n \mathbf {\hat C}_n )|}{|\det({\tC}^T \tC )|} = \frac{\hat h_n}{ h^0}\cdot \frac{|\conv(\hatCn)|}{|\conv(\tC)|}.$$
where $\hat h_n$ is the perpendicular distance from the origin to $\text{aff}(\hatCn)$, and $h^0$ is the perpendicular distance from the origin to $\text{aff}(\tC)$. 

Therefore, it suffices to show the following two inequalities,
\begin{equation}\label{eq:h_close}
    \hat h_n \leq (1+C_1 \epsilon_n) h^0,
\end{equation}
and 
\begin{equation}\label{eq:vol_close}
    |\conv(\mathbf {\hat C}_n )| \leq (1 + C_2 \epsilon_n ) |\conv(\tC )|.
\end{equation}
~\\

We first prove the projection matrix associated with $\text{aff}(\hatCn)$ converges to the one associated with $\text{aff}(\tC)$ in the order of $\epsilon_n$ in Lemma \ref{lem:project_error}. Then  \eqref{eq:h_close} is proved in Corollary \ref{lem:h_diff}, as a special case of Lemma \ref{lem:project_error}.
~\\

To compare $|\conv(\mathbf {\hat C}_n )|$ and $|\conv(\tC )|$, we introduce three more convex polytopes: 
\begin{itemize}
    \item $\conv((\tC)^{\gamma\epsilon_n})$, an enlarged convex polytope of $\conv(\tC)$ (defined in Definition \ref{def:delta_enlarge}). Here $\gamma>0$ is a constant.
    \item $\conv(\mathbf C^\sharp)$, the projection of $\conv((\tC)^{\gamma\epsilon_n})$ on $\text{aff}(\hatCn)$. 
    \item $\conv(\smallC)$, the smallest $k$-vertex convex polytope on $\text{aff}(\mathbf {\hat C}_n) \bigcap \simplex^{V-1}$ containing \\$\setballs =  \conv(\mathbf{\hat C}_n) \bigcap \left\{\bigcup_{i=1}^d \mathcal{B}(\mathbf u^{0 (i)}; C_4\epsilon_n)\right\}$. Here $\mathcal{B}(\mathbf u^{0 (i)}; C_4 \epsilon_n)$ is the Euclidean ball centered at $\mathbf u^{0 (i)}$ with radius $C_4\epsilon_n$. The formal definition is given in Definition \ref{def:smallest_C}.
\end{itemize}

We then prove \eqref{eq:vol_close}  by the following steps.
\begin{enumerate}
    \item  In Lemma \ref{lem:F_lower_bound_2} and Lemma \ref{lem:vol_upper_bound}, we show 
    \begin{equation}\label{eq:hatCn_Csharp}
       \left(1-\frac{1}{n}\right) |\conv(\mathbf{\hat C}_n)| \leq |\conv(\smallC)|.
    \end{equation}
    
    \item In Lemma \ref{lem:Csharp_inside_simplex} to Lemma \ref{lem:proj_inside}, we show that $\conv(\mathbf C^\sharp)$ is a $k$-vertex convex polytope within $\simplex^{V-1}$ containing $\setballs$. Therefore, by the definition of $\conv(\mathbf C^\ast)$, we have 
    \begin{equation}\label{eq:smallC_Csharp}
    |\conv(\smallC)| \leq |\conv(\mathbf C^\sharp)|.
    \end{equation}
    \item In Lemma \ref{lem:vol_upper_bound2}, we prove \eqref{eq:vol_close} by summarizing the the above inequalities, i.e.,
    \begin{align*}
       \left(1-\frac{1}{n}\right) |\conv(\mathbf{\hat C}_n)| & \stackrel{\eqref{eq:hatCn_Csharp}}{\leq}|\conv(\smallC)|
        \stackrel{\eqref{eq:smallC_Csharp}}{\leq}|\conv(\mathbf C^\sharp)|
        \stackrel{\text{Definition of~} \conv(\mathbf C^\sharp)}{\leq}|\conv((\tC)^{\gamma\epsilon_n})|\\
        &  \stackrel{\text{Proposition~} \ref{prop:enlargement}}{\leq} \left(1 + \rho(\tC)\gamma\epsilon_n\right)^{k-1}|\conv(\tC)|.
    \end{align*} 
\end{enumerate}

~\\

Next, we provide detailed proof.
~\\

First, we show that the projection matrix and any projected vector of $\text{aff}(\mathbf{\hat C}_n)$ converges to the ones of $\text{aff}(\tC)$ in the order of $ \sqrt{\frac{\log (n\vee d)}{n}}$. 

Let $(\mathbf u^{(1)}, \cdots, \mathbf u^{(k)})$ be any $k$ linearly independent vectors from $\text{aff}(\mathbf C)$. Then, the \textbf{projection matrix} of $\text{aff}(\mathbf C)$ can be written as
\begin{align}\label{eq:proj_matrix}
\mathbf P_{\mathbf C}= \mathbf {U}^\prime(\mathbf {U}^{\prime T}\mathbf { U}^\prime)^{ -1} \mathbf { U}^{^\prime T},
\end{align}
where $\mathbf {U}^\prime = \left(\mathbf u^{(2)} - \mathbf u^{(1)}, \cdots, \mathbf u^{(k)} -\mathbf u^{(1)}\right)$.  

For any vector $\mathbf y$, its \textbf{projection onto $\text{aff}(\mathbf C)$}  is given by
\begin{equation}\label{eq:y_projection}
    \mathbf {\hat y}_{{\mathbf C}} = \mathbf P_{\mathbf C}(\mathbf y -\mathbf u^{(1)}) +  \mathbf u^{(1)} =\mathbf P_{\mathbf C}\mathbf y + (\mathbf I - \mathbf P_{\mathbf C}) \mathbf u^{(1)}. 
\end{equation}

\begin{lemma}\label{lem:project_error}
\begin{equation}\label{eq:proj_mtr_bound}
    \| \mathbf P_{\hatCn} - \mathbf P_{\tC}\|_2 \leq C\epsilon_n
\end{equation}
\begin{equation}\label{eq:diff_projection_bound}
    \|\mathbf {\hat y}_{\hatCn} - \mathbf {\hat y}_{\tC}\|_2 \leq C\|\mathbf y\|_2\epsilon_n + C^\prime \epsilon_n,
\end{equation}
for any $\mathbf y \in \mathds{R}^V$, where $C$ and $C'$ are positive constants.

\end{lemma}

\begin{proof}
By assumption (A3), let $(i_1,\cdots, i_k)$ denote the index set of the columns of $\mathbf W^{0\ast}$ in
$\tW$, where the $k$ columns of $\mathbf W^{0\ast}$ are affinely independent and have minimum positive singular value lower bounded. 
Let $\mathbf U^{0\ast} =  \tC\mathbf W^{0\ast}= (\mathbf u^{0(i_1)}, \cdots, \mathbf u^{0(i_k)})$ and $\mathbf {\tilde U}^\ast_n = (\mathbf {\tilde u}^{(i_1)}, \cdots, \mathbf{\tilde u}^{(i_k)})$, where $ \mathbf {\tilde u}^{(i)}=\argmin_{\mathbf u\in \conv(\hatCn)}d(\mathbf u, \mathbf u^{0(i)})$ is the projection of $\mathbf u^{0(i)}$ onto $\conv(\hatCn)$.

By Lemma \ref{lem:dis_upper_bound1}, we have
$$\|\mathbf U^{0\ast}  - \mathbf {\tilde U}^\ast_n \|_2 
\leq \|\mathbf U^{0\ast}  - \mathbf {\tilde U}^\ast_n \|_F 
= \left(\sum_{j=1}^k\|\mathbf u^{0(i_j)} -\mathbf {\tilde u}^{(i_j)}\|_2^2\right)^{\frac{1}{2}} 
\leq C_2 \epsilon_n,$$
and 
$$\| \mathbf u^{0(i_1)} - \mathbf {\tilde u}^{(i_1)} \|_2 
\leq C_3 \epsilon_n
.$$

By \eqref{eq:proj_matrix}, we have
$$\mathbf P_{\tC} = \mathbf {U}^{0\prime}(\mathbf { U}^{0\prime T}\mathbf { U}^{0\prime})^{-1} \mathbf {U}^{0\prime T}, \quad \mathbf P_{\hatCn} = \mathbf {\tilde U}^\prime_n(\mathbf {\tilde U}^{\prime T}_n \mathbf {\tilde U}^\prime_n)^{-1} \mathbf {\tilde U}^{\prime T}_n $$
where $\mathbf {U}^{0\prime} = \mathbf U^{0\ast} \mathbf Q$, $\mathbf {\tilde U}^\prime_n = \mathbf{\tilde U}^{\ast}_n \mathbf Q$, and   
$\mathbf Q_{k\times (k-1)} = \left[ \begin{array}{cc}
 - \mathbf 1_{k-1} &\mathbf I_{k-1}  \\
\end{array}
\right]^T.$

By Weyl's inequality in matrix theory \citep{weyl1912asymptotische},
$$\sigma^+_{\min} (\mathbf U^{0\prime}) - \sigma^+_{\min} (\mathbf{\tilde U}_n^\prime) \leq \|\mathbf U^{0\prime}- \mathbf{\tilde U}_n^\prime\|_2 \leq \|\mathbf U^{0\ast}  - \mathbf {\tilde U}^\ast_n \|_2\|\mathbf Q \|_2 \leq C_2'\epsilon_n.$$
Therefore,
\begin{align}\label{min_singular_close}
    \sigma^+_{\min} (\mathbf{\tilde U}_n^\prime) 
    \geq \sigma^+_{\min} (\mathbf U^{0\prime})-C_2'\epsilon_n
    \geq \frac{\sigma^+_{\min} (\mathbf U^{0\prime})}{2}.
\end{align}
Moreover,
\begin{align}\label{eq:singular_lower_bound}
    \sigma^+_{\min} (\mathbf U^{0\prime}) = \sigma^+_{\min} (\mathbf U^{0\ast}\mathbf Q)
    = \sigma^+_{\min} (\tC\mathbf W^{0\ast}\mathbf Q)
    \geq \sigma^+_{\min} (\tC)  \sigma^+_{\min}(\mathbf W^{0\ast}) \sigma^+_{\min}(\mathbf Q) 
    \geq C_3.
\end{align}
So the columns of $\mathbf {\tilde U}^\ast_n$ are also affinely independent. 

According to Davis-Kahan theorem \citep{chen2016perturbation,davis1970rotation}, we have
\begin{align*}
    \|\mathbf P_{\hatCn} - \mathbf P_{\tC}\|_2  & \stackrel{Davis-Kahan}{\leq}  \max{\left(\frac{1}{\sigma^+_{\min} (\mathbf{\tilde U}_n^\prime)},\frac{1}{\sigma^+_{\min} (\mathbf U^{0\prime}) } \right)}
    \|\mathbf{\tilde U}_n^\prime - \mathbf U^{0\prime}\|_2\\
    &\leq  \max{\left(\frac{1}{\sigma^+_{\min} (\mathbf{\tilde U}_n^\prime)},\frac{1}{\sigma^+_{\min} (\mathbf U^{0\prime}) } \right)}
    \|\mathbf{\tilde U}_n^\ast - \mathbf U^{0\ast}\|_2\|\mathbf Q\|_2\\
    &  \leq 
    C_4 \epsilon_n,
\end{align*}
where the last inequality is due to \eqref{min_singular_close} and \eqref{eq:singular_lower_bound}.

Finally, for any $\mathbf y \in \mathds{R}^V$, 
\begin{align*}
    \|\mathbf {\hat y}_{\hatCn} - \mathbf {\hat y}_{\tC} \|_2 
    & \leq  \|\mathbf P_{\hatCn}- \mathbf P_{\tC}\|_2\|\mathbf y\|_2  + \|\mathbf P_{\hatCn}- \mathbf P_{\tC}\|_2\|\|\mathbf u^{0(i_1)} \|_2 + \| \mathbf u^{0(i_1)} - \mathbf {\tilde u}^{(i_1)} \|_2\\
 & \leq  C \|\mathbf y\|_2 \epsilon_n  + C' \epsilon_n 
\end{align*}
\end{proof}

\begin{corollary}\label{lem:h_diff}
Denote the perpendicular distance between origin, $\mathbf 0 = (0,0,\cdots, 0)$, and  $\text{aff}(\tC)$ by $h^0$, and the perpendicular distance between origin and $\text{aff}(\mathbf {\hat C}_n)$ by $\hat h_n$. The followings hold,
\begin{enumerate}
    \item $|\hat h_n - h^0 | \leq C'\epsilon_n,$
    \item $h_0 > C''.$
\end{enumerate}
where $C'$ and $C''$ are positive constants.
\end{corollary}

\begin{proof}
The perpendicular distance of $\text{aff}(\mathbf C)$ is the length of the projected vector of $\mathbf 0$ on $\text{aff}(\mathbf C)$. Specifically,
$$\hat h_n= \|\mathbf {\hat 0}_{\hatCn} \|_2,\quad h^0 = \|\mathbf {\hat 0}_{\tC} \|_2, $$

Therefore,
$$\left|\hat h_n - h^0 \right| =  \left|\|\mathbf {\hat 0}_{\hatCn}\|_2 -  \|\mathbf {\hat 0}_{\tC} \|_2 \right| 
    \leq \|\mathbf {\hat 0}_{\hatCn} - \mathbf {\hat 0}_{\tC} \|_2 \leq   (C\|\mathbf 0\|_2 \epsilon_n + C'\epsilon_n) \leq C'\epsilon_n.$$
    
Furthermore, since $\mathbf {\hat 0}_{\tC}$ is on $\text{aff}(\tC)$, we can represent $\mathbf {\hat 0}_{\tC}$ by $\tC\mathbf w_h$ for some $\mathbf w_h \in \simplex^{k-1}$.  
$$h^0 = \|\mathbf {\hat 0}_{\tC}\|_2 \geq \sigma^+_{\min}(\tC) \|\mathbf w_h\|_2 \geq \sigma^+_{\min}(\tC)\|\mathbf w_h\|_1/\sqrt{k} = \sigma^+_{\min}(\tC)/\sqrt{k}.$$
\end{proof}

With the result from Lemma \ref{lem:dis_upper_bound1}, in the following Lemma \ref{lem:F_lower_bound_2}, we show that most of the mass of $f^{(i)}(\mathbf u)$ on $\conv(\mathbf {\hat C}_n)$ is concentrated on $\conv(\mathbf {\hat C}_n)\bigcap \mathcal{B}(\mathbf u^{0(i)}; C_4\epsilon_n)$.
\begin{lemma}\label{lem:F_lower_bound_2}
For any $i\in [d]$,
\begin{align*}
    & \int_{\conv(\mathbf {\hat C}_n)\bigcap \mathcal{B}(\mathbf u^{0(i)};  C_4\epsilon_n)} \frac{f^{(i)}(\mathbf u)}{f^{(i)}(\mathbf {\hat u}^{(i)})} d\mathbf u \geq (1 - \frac{1}{ n}) \int_{\conv(\mathbf {\hat C}_n)}\frac{f^{(i)}(\mathbf u)}{f^{(i)}(\mathbf {\hat u}^{(i)})} d\mathbf u.
\end{align*}
\end{lemma}

\begin{proof}
It suffices to show that for  any $i\in [d]$,
\begin{align}\label{eq:to_show_lm_b4}
\int_{\conv(\mathbf{\hat C}_n)\bigcap \mathcal{B}^C(\mathbf u^{0(i)}; C_4\epsilon_n)} \frac{f^{(i)}(\mathbf u)}{f^{(i)}(\mathbf {\hat u}^{(i)})} d\mathbf u \leq   \frac{1}{n} \int_{\conv(\mathbf{\hat C}_n)}\frac{f^{(i)}(\mathbf u)}{f^{(i)}(\mathbf {\hat u}^{(i)})} d\mathbf u.
\end{align}

For the LHS of \eqref{eq:to_show_lm_b4},
\begin{align*}
&\int_{\conv(\mathbf{\hat C}_n)\bigcap \mathcal{B}^C(\mathbf u^{0(i)}; C_4\epsilon_n)} \frac{f^{(i)}(\mathbf u)}{f^{(i)}(\mathbf {\hat u}^{(i)})} d\mathbf u \\
\leq & \int_{\conv(\mathbf{\hat C}_n)\bigcap \mathcal{B}^C(\mathbf u^{0(i)}; C_4\epsilon_n)}
\exp\left( - \frac{n}{2}\|\mathbf{\hat u}^{(i)} - \mathbf u\|_2^2\right) d\mathbf u\\
\leq & \exp\left(- \frac{n}{4}(C_4-1)^2\epsilon_n^2\right)
\int_{\text{aff}(\mathbf{\hat C}_n)} \exp\left( - \frac{n}{4}\|\mathbf{\hat u}^{(i)} - \mathbf u\|_2^2\right) d\mathbf u\\
=& \exp\left(- \frac{n}{4}(C_4-1)^2\epsilon_n^2\right)
\exp\left(- \frac{n}{4} d^2(\mathbf{\hat u}^{(i)}, \text{aff}(\mathbf{\hat C}_n)) \right)
\int_{\text{aff}(\mathbf{\hat C}_n)} \exp\left( - \frac{n}{4}\|\mathbf{\hat u}^{(i)}_{\mathbf{\hat C}_n} - \mathbf u\|_2^2\right) d\mathbf u\\
\leq & \exp\left(- \frac{n}{4}(C_4-1)^2\epsilon_n^2\right)\cdot 1\cdot \frac{C_5}{n^{\frac{k-1}{2}}},
\end{align*}
where the first inequality is due to the Pinsker's inequality \eqref{eq:pinskers}, the third inequality is from the normalizing constant for a multivariate Gaussian distribution.

 For the integration in the RHS of \eqref{eq:to_show_lm_b4},
\begin{align*}
&\int_{\conv(\mathbf{\hat C}_n)}\frac{f^{(i)}(\mathbf u)}{f^{(i)}(\mathbf {\hat u}^{(i)})} d\mathbf u \\
\geq& \int_{\conv(\mathbf {\hat C}_n)\bigcap \mathcal{B}(\mathbf u^{0(i)};  \sqrt{2}C\epsilon_n)} \frac{f^{(i)}(\mathbf u)}{f^{(i)}(\mathbf {\hat u}^{(i)})} d\mathbf u\\
\geq& \int_{\conv(\mathbf {\hat C}_n)\bigcap \mathcal{B}(\mathbf u^{0(i)};  \sqrt{2}C\epsilon_n)} \exp\left( - C_7n\|\mathbf{\hat u}^{(i)} - \mathbf u\|_2^2\right) d\mathbf u\\
\geq& C_8 (C\epsilon_n)^{k-1} \cdot\exp\left( - C_7n (\sqrt{2}C+1)^2\epsilon_n^2 \right)
\end{align*}
where $C$ is the constant from \eqref{eq:u0i_close_to_convcn}. 
Since $\mathbf u^{0(i)}$'s are interior points in $\simplex^{V-1}$, when $n$ is large enough, the second inequality follows from the reverse Pinsker's inequality \eqref{eq:inv_pinskers}.

By choosing $C_4$ large enough, we can ensure 
$$
\exp\left(- \frac{n}{4}(C_4-1)^2\epsilon_n^2\right) \frac{C_5}{n^{\frac{k-1}{2}}}
\leq \frac{1}{n}\cdot C_8 C^{k-1}\epsilon_n^{k-1} \exp\left( - C_7n  (\sqrt{2}C+1)^2\epsilon_n^2 \right)
$$
Consequently, we have
$$
\int_{\conv(\mathbf{\hat C}_n)\bigcap \mathcal{B}^C(\mathbf u^{0(i)}; C_4\epsilon_n)} \frac{f^{(i)}(\mathbf u)}{f^{(i)}(\mathbf {\hat u}^{(i)})} d\mathbf u \leq   \frac{1}{n} \int_{\conv(\mathbf{\hat C}_n)}\frac{f^{(i)}(\mathbf u)}{f^{(i)}(\mathbf {\hat u}^{(i)})} d\mathbf u.
$$
\end{proof}

\begin{definition}\label{def:smallest_C}
Define $\smallC = [\mathbf c_1^*, \cdots, \mathbf c_k^*] \in \mathds{R}^{V\times k}$, such that, $\mathbf c_f^* \in \text{aff}(\mathbf{\hat C}_n)\bigcap \simplex^{V-1},$ $\forall f\in[k]$, and $\conv(\smallC)$ is {\bf the smallest (volume) convex polytope with $k$ vertices on  $\text{aff}(\mathbf{\hat C}_n)\bigcap \simplex^{V-1}$} that contains the set $\setballs = \conv(\mathbf{\hat C}_n) \bigcap \left\{\bigcup_{i=1}^d \mathcal{B}(\mathbf u^{0 (i)}; C_4\epsilon_n)\right\}$.
\end{definition}

Note that $\conv(\mathbf{\hat C}_n)$ is a convex polytope with $k$ vertices on  $\text{aff}(\mathbf{\hat C}_n)\bigcap \simplex^{V-1}$ containing $\setballs$. So $\smallC$ must exist and it satisfies $|\conv(\smallC)| \leq|\conv(\mathbf{\hat C}_n)|$. In the following lemma, we show that $|\conv(\mathbf{\hat C}_n)|$ cannot be much larger than $|\conv(\smallC)|$. 
\begin{lemma}\label{lem:vol_upper_bound}

$$\left(1- \frac{1}{n}\right)|\conv(\mathbf{\hat C}_n)| \leq |\conv(\smallC)|$$
\end{lemma}

\begin{proof}
Since $\mathbf c_f^* \in \simplex^{V-1}, \, \forall f \in [k]$, $\smallC$ is a valid parameter of $F_{n\times d}(\mathbf C; \mathbf X) $, and $F_{n\times d}(\mathbf{\hat C}_n; \mathbf X) \geq F_{n\times d}(\smallC; \mathbf X)$. From Lemma \ref{lem:F_lower_bound_2}, we have 
\begin{align*}
    &  \frac{A_{n,d}}{|\conv(\mathbf{\hat C}_n)|^d} \cdot \prod_{i=1}^d \int_{\conv(\mathbf{\hat C}_n)\bigcap \mathcal{B}(\mathbf u^{0(i)}; C_4\epsilon_n)} \frac{f^{(i)}(\mathbf u)}{f^{(i)}(\mathbf {\hat u}^{(i)})} d\mathbf u  \\
    \geq & \left(1 - \frac{1}{ n}\right)^d F_{n\times d}(\mathbf{\hat C}_n; \mathbf X) \geq  \left(1 - \frac{1}{ n}\right)^d F_{n\times d}(\smallC; \mathbf X)\\
    = & \left(1 - \frac{1}{ n}\right)^d  \frac{A_{n,d}}{|\conv(\smallC)|^d} \cdot \prod_{i=1}^d \int_{\conv(\smallC)} \frac{f^{(i)}(\mathbf u)}{f^{(i)}(\mathbf {\hat u}^{(i)})} d\mathbf u  \\
    \geq & \left(1 - \frac{1}{ n}\right)^d  \frac{A_{n,d}}{|\conv(\smallC)|^d} \cdot \prod_{i=1}^d \int_{\conv(\mathbf{\hat C}_n) \bigcap \mathcal{B}(\mathbf u^{0(i)}; C_4\epsilon_n)} \frac{f^{(i)}(\mathbf u)}{f^{(i)}(\mathbf {\hat u}^{(i)})} d\mathbf u  ,
\end{align*}
where the last inequality is due to the definition of $\conv(\smallC)$.
Therefore,
$$\frac{1}{|\conv(\mathbf{\hat C}_n)|^d}  \geq \frac{\left(1 - \frac{1}{n}\right)^d}{|\conv(\smallC)|^d},$$
$$\left(1- \frac{1}{n}\right)|\conv(\mathbf{\hat C}_n)|\leq |\conv(\smallC)|.$$

\end{proof}

Next, we compare $|\conv(\tC)|$ and $|\conv(\smallC)|$. We will construct an enlarged convex polytope, $\conv((\tC)^{\gamma\epsilon_n})$, and then project it to $\text{aff}(\mathbf {\hat C}_n)$ to obtain a projected convex polytope, $|\conv(\mathbf C^\sharp)|$. We will show that $|\conv(\mathbf C^\sharp)|$ contains the set $\setballs$, so that  $   |\conv(\smallC)| \leq |\conv(\mathbf C^\sharp)|$. \\

\begin{definition}
Let $\mathbf C^\sharp = (\mathbf c^\sharp_1, \cdots, \mathbf c^\sharp_k) \in \mathds{R}^{V\times k}$ such that $\mathbf c^\sharp_f$ is the projected vector of the $f$-th vertex of $\conv\left((\mathbf{C}^0)^{\gamma  \epsilon_n}\right)$ on $\text{aff}(\mathbf {\hat C}_n)$, $\forall f= [k]$. Here $\gamma>0$ is a constant.
\end{definition}

\begin{lemma}\label{lem:Csharp_inside_simplex}
When $n$ is large enough, $\conv(\mathbf C^\sharp)$ is in $\simplex^{V-1}$.
\end{lemma}

\begin{proof}
It suffices to show that for any $f\in[k]$, (1) ${\mathbf c_f^{\sharp T}} \mathbf 1_V = 1 $ and (2) $\mathbf c_f^\sharp \geq 0$.

By the definition of $\text{aff}(\hatCn)$, $\mathbf c_f^\sharp= \hatCn \bm{\lambda}_f$ and $\bm{\lambda}_f^T\mathbf 1_k=1$. Therefore, (1) holds because
$${\mathbf c_f^{\sharp T}} \mathbf 1_V =  \bm{\lambda}_f^T \hatCn^T\mathbf 1_V = \bm{\lambda}_f^T \mathbf 1_k = 1.$$

By Lemma \ref{lem:project_error}, we have
\begin{align*}
\|\mathbf c^\sharp_f - (\mathbf{c}_f^0)^{\gamma  \epsilon_n}\|_2 \leq C\epsilon_n.
\end{align*}

Therefore, to show (2), it suffices to verify that $(\mathbf{c}_f^0)^{\gamma  \epsilon_n} > C_1,$ for any $f\in[k]$. 

Note that $\mathbf c^0_1,\cdots, \mathbf c^0_k \geq C_2,$ since $\mathbf c^0_1,\cdots, \mathbf c^0_k$ are strict inner points of $\simplex^{V-1}$.  

By the definition of the enlarged convex polytope (Definition \ref{def:delta_enlarge}), we have 
\begin{align*}
    (\mathbf{c}_f^0)^{\gamma \epsilon_n} & = \left(1 + \rho\gamma\epsilon_n \right) (\mathbf{c}_f^0 - \mathbf {\bar c}^0) + \mathbf {\bar c}^0 \\
    & = \left(1 + \rho\gamma\epsilon_n \right) \mathbf{c}_f^0 -  \rho\gamma\epsilon_n \mathbf {\bar c}^0\\
    & \geq (1 +  \rho\gamma\epsilon_n) C_2-  \rho\gamma\epsilon_n\geq  C_1,
\end{align*}
where $\rho = \rho(\tC)$ and $\mathbf {\bar c}^0 = \frac{1}{k}\sum_{f=1}^k \mathbf{c}_f^0$.
\end{proof}

Next, we want to prove that $\conv(\mathbf C^\sharp)$ contains the set $\setballs$. We first study the property of the boundary points of the $\delta$-enlargement convex polytope in Lemma \ref{lem:bd_dis_bound} and show that the distance between any boundary point and the original convex polytope is at least $\delta$. Using this fact,  we know that any boundary point of $\conv(\mathbf C^\sharp)$ is at least $\gamma\epsilon_n$ away from any $\mathbf u^{0(i)} \in \conv(\tC)$. By letting $\gamma$ large enough, we can have $\conv(\mathbf C^\sharp)$ contain the set $\setballs'_i = \mathcal{B}(\mathbf u^{0(i)}, C_4\epsilon_n)\bigcap \text{aff}(\mathbf {\hat C}_n)$ for any $i \in [d]$. Therefore, $\conv(\mathbf C^\sharp)$ contains the set $\setballs' = \bigcup_{i=1}^d\setballs'_i$, which is a superset of the set $\setballs$.  The detailed proof is in Lemma \ref{lem:proj_inside}.
\begin{lemma}\label{lem:bd_dis_bound}
 For any point $\mathbf x$ on the boundary of $\conv(\mathbf C^\delta)$,
$$\delta \leq d(\mathbf x, \conv(\mathbf C))  \leq \kappa(\mathbf C)k\delta.$$
where $\kappa(\mathbf C) = \frac{\sigma_{\max}(\mathbf C)}{\sigma^+_{\min}(\mathbf C)}$ is the conditional number of $\mathbf C$.
\end{lemma}

\begin{proof}
By the definition of $\delta$-enlargement in Definition \ref{def:delta_enlarge}, 
\begin{align*}
    \mathbf x =& \sum_{f=1}^k \alpha_f \mathbf{c}_f^\delta = \sum_{f=1}^k \alpha_f [(1 + \rho \delta)(\mathbf{c}_f - \mathbf{\bar c})  + \mathbf{\bar c}]\\
    = & \sum_{f=1}^k \alpha_f (1 + \rho \delta)[\mathbf{c}_f  - \frac{1}{k}\sum_{s=1}^k\mathbf{c}_s] + \sum_{f=1}\alpha_f\mathbf{\bar c} \\
     = & \sum_{f=1}^k \alpha_f (1 + \rho \delta)\mathbf{c}_f  - \frac{1 + \rho \delta}{k}\sum_{s=1}^k\mathbf{c}_s + \mathbf{\bar c} \\
    = & \sum_{f=1}^k (1 + \rho \delta)\left(\alpha_f -\frac{1}{k}\right)\mathbf{c}_f  + \mathbf{\bar c} 
\end{align*}
where $\rho =  \frac{k} {\sigma^+_{\min} (\mathbf C)}$, and $\bm \alpha = \{\alpha_f\}_{f=1,\cdots,k} \in \simplex^{k-1}$. Since $\mathbf x$ is a boundary point, there exists at least one $f\in[k]$, such that $\alpha_f = 0$. WLOG, we assume $\alpha_k = 0$.  

Let $\bm \alpha^\prime = \bm \alpha - \frac{1}{k}\mathbf 1$. We have $$\mathbf x = \mathbf C (1 + \rho\delta)\bm\alpha'+  \mathbf{\bar c}.$$

At the same time, any point $\mathbf y$ in $\conv(\mathbf C)$ can be represented by $$\mathbf y = \mathbf{C}\bm{\beta} = \sum_{f=1}^k\beta_f (\mathbf{c}_f - \mathbf{\bar c}) +\mathbf{\bar C} = \mathbf C \bm \beta'  +\mathbf{\bar c},$$ 
where $\bm \beta= \{\beta_f\}_{f=1,\cdots,k} \in \simplex^{k-1}$ and $\bm \beta^\prime = \bm \beta - \frac{1}{k}\mathbf 1$.

Now we can measure the distance between the boundary point $\mathbf x$ and $\conv(\mathbf C)$,
\begin{align*}
    d(\mathbf x,\conv(\mathbf C)) & = \min_{\mathbf y \in \conv(\mathbf C)}\|\mathbf x - \mathbf y\|_2 = \min_{\bm \beta \in \simplex^{V-1}}\|\mathbf{ C} [(1+\rho\delta)\bm \alpha^\prime - \bm \beta^\prime]\|_2 \\
    & \geq  \sigma^+_{\min}(\mathbf{C}) \cdot \min_{\bm \beta \in \simplex^{V-1}}\|(1+\rho\delta)\bm \alpha^\prime - \bm \beta^\prime\|_2.
\end{align*}

Write $\bm \eta = (1+\rho\delta)\bm \alpha^\prime - \bm \beta^\prime = (1+\rho\delta)\bm \alpha - \bm \beta - \frac{\rho\delta}{k}\bm 1. $ Then, the $k$-th element of $\bm \eta $ is 
$$\eta_k = (1+\rho\delta)\alpha_k - \beta_k - \frac{\rho\delta}{k} = 0 - \beta_k - \frac{\rho\delta}{k} \leq  - \frac{\rho\delta}{k} $$
because we assume $\alpha_k = 0$, and $\bm\beta \in \simplex^{(k-1)}$ so that $\beta_k \geq 0$. Then, we obtain the lower bound,
$$ d(\mathbf x,\conv(\mathbf C)) \geq \sigma^+_{\min}(\mathbf{C}) \cdot \min_{\bm\eta }\|\bm \eta \|_2 \geq \sigma^+_{\min}(\mathbf{C}) \cdot \min_{\bm \eta }|\eta_k|\geq \sigma^+_{\min}(\mathbf{C})\cdot \frac{\rho\delta}{k} = \delta.$$

For the upper bound, let $\bm \beta^\prime = \bm \alpha^\prime$, we have 
\begin{align*}
    \min_{\mathbf y \in \conv(\mathbf C)}\|\mathbf x - \mathbf y\|_2&\leq \|\mathbf{C} [(1+\rho\delta)\bm \alpha^\prime - \bm \alpha^\prime]\|_2 \leq  \sigma_{\max}(\mathbf{C}) \cdot \rho\delta \cdot \|\bm \alpha^\prime\|_2\\
    & = \sigma_{\max}(\mathbf{C})\rho\delta  \sqrt{\|\bm \alpha\|_2^2 - \frac{1}{k}} \leq  \sigma_{\max}(\mathbf{C})\rho\delta  \sqrt{1 - \frac{1}{k}} \\
    & \leq \kappa(\mathbf C)k\delta.
\end{align*}
\end{proof}

\begin{lemma}\label{lem:proj_inside}


For some $\gamma > 0$, $\conv(\mathbf C^\sharp)$ covers the set  $\setballs' = \left\{\bigcup_{i=1}^d \mathcal{B}(\mathbf u^{0(i)}, C_4\epsilon_n)\right\} \bigcap \text{aff}(\mathbf {\hat C}_n)$, i.e., 
$$\setballs \subseteq \setballs' \subseteq \conv(\mathbf C^\sharp).$$
\end{lemma}

\begin{proof}
Since $\conv(\mathbf C^\sharp)$ is a closed and simply connected region, it suffices to show that for any boundary point, $\mathbf x \in bd\conv(\mathbf C^\sharp)$, 
$$\min_{i=1\cdots,d}\|\mathbf x - \mathbf u^{0(i)}\|_2 \geq C_4\epsilon_n.$$ 

In fact, any boundary point of $\conv(\mathbf{C}^\sharp)$, $\mathbf x \in bd\conv(\mathbf{C}^\sharp)$, is projected from a boundary point of $\conv(\mathbf{C}^0)^{\gamma\epsilon_n}$, denoted as $\mathbf y \in  bd(\conv(\mathbf{C}^0)^{\gamma\epsilon_n})$. Then, 
$$\mathbf x = \mathbf P_{\mathbf {\tilde U}_n} \mathbf y + (\mathbf I - \mathbf P_{\mathbf {\tilde U}_n})\mathbf {\tilde u}^{(k)}. $$
By Lemma \ref{lem:bd_dis_bound}, we have $\forall \mathbf y\in bd(\conv(\mathbf{C}^0)^{\gamma\epsilon_n})$,
\begin{equation}\label{eq:dis_y_bound}
    \|\mathbf y - \mathbf u^{0(i)}\|_2 \geq d(\mathbf y, \conv(\mathbf{C}^0)) \geq \gamma \epsilon_n.\quad \forall i = 1,\cdots, d.
\end{equation}

Denote the projected point of $\mathbf u^{0(i)}$ on $\text{aff}(\mathbf{\hat C}_n) = \text{aff}(\mathbf{\tilde U}_n)$, as $\mathbf{\hat{u}}_{\mathbf{\tilde U}_n}^{0(i)}$. We have
\begin{align*}
    \|\mathbf x - \mathbf u^{0(i)}\|_2  = & \|(\mathbf x -  \mathbf{\hat{u}}_{\mathbf{\tilde U}_n}^{0(i)} ) + (\mathbf{\hat{u}}_{\mathbf{\tilde U}_n}^{0(i)}- \mathbf u^{0(i)})  \|_2\\
    \geq  & \|\mathbf x -  \mathbf{\hat{u}}_{\mathbf{\tilde U}_n}^{0(i)} \|_2
    =  \|\mathbf P_{\mathbf{\tilde U}_n}(\mathbf y - \mathbf u^{0(i)})\|_2
    \\
    = & \|\mathbf P_{\tU}(\mathbf y - \mathbf u^{0(i)}) + (\mathbf P_{\mathbf{\tilde U}_n}-\mathbf P_{\tU})(\mathbf y - \mathbf u^{0(i)})\|_2 
    \\
    \geq & \|\mathbf y - \mathbf u^{0(i)}\|_2 -  \|\mathbf P_{\mathbf{\tilde U}_n}-\mathbf P_{\tU}\|_2 \|\mathbf y - \mathbf u^{0(i)}\|_2 
    \\
      \geq & (1-  C_5\epsilon_n )  \|\mathbf y - \mathbf u^{0(i)}\|_2
      &\text{by Lemma \ref{lem:project_error} \eqref{eq:proj_mtr_bound}}\\ 
    \geq & (1-  C_5\epsilon_n ) \gamma\epsilon_n 
    \quad \forall i =1,\cdots, d. &\text{by~\eqref{eq:dis_y_bound}}
\end{align*}
By letting $\gamma$ large enough, we can make that for any $ \mathbf x \in bd\conv(\mathbf{C}^\sharp)$,
$$ \|\mathbf x - \mathbf u^{0(i)}\|_2 \geq (1-  C_5\epsilon_n ) \gamma\epsilon_n 
\geq C_4 \epsilon_n \quad \forall i =1,\cdots, d.$$
\end{proof}
~\\

Now we are ready to piece together all the useful results and conclude the following inequalities.
\begin{lemma}\label{lem:vol_upper_bound2}
$$|\conv(\mathbf{\hat C}_n)| \leq \left(1 + C'\epsilon_n\right)|\conv(\mathbf{C}^0)|$$
$$ |\det(\mathbf {\hat C}^T_n \mathbf {\hat C}_n )| \leq (1 + C'' \epsilon_n ) |\det({\tC}^T \tC )|$$
\end{lemma}

\begin{proof}
Since the $\delta$-enlargement is an affine transformation, by Proposition \ref{prop:enlargement},
$$|\conv((\mathbf{C}^0)^{\gamma\epsilon_n})|=\left(1 + \rho(\mathbf{C}^0)\gamma\epsilon_n\right)^{k-1}|\conv( \mathbf{C}^0)| \leq (1 +C_1\epsilon_n)|\conv( \mathbf{C}^0)|. $$

Then, by Lemma \ref{lem:Csharp_inside_simplex} and Lemma \ref{lem:proj_inside}$, \conv(\mathbf C^\sharp)$ is on $\text{aff}(\mathbf {\hat C}_n)\bigcap \simplex^{V-1}$ covering $\setballs$. And by Definition \ref{def:smallest_C}, $\conv(\smallC)$ is the smallest $k$-vertex convex polytope on $\text{aff}(\mathbf {\hat C}_n)\bigcap \simplex^{V-1}$ covering $\setballs$. So we have ,
$$|\conv(\smallC)|\leq |\conv(\mathbf C^\sharp)|\leq |\conv((\mathbf{C}^0)^{\gamma\epsilon_n})|.$$
The last inequality holds because $\mathbf C^\sharp$ is a projection of $(\mathbf{C}^0)^{\gamma\epsilon_n}$ on $\text{aff}(\mathbf {\hat C}_n)$, so that any side length of $\conv(\mathbf C^\sharp)$ is shorter than the corresponding one of $(\mathbf{C}^0)^{\gamma\epsilon_n}$, i.e.,
$$\|\mathbf c^\sharp_f- \mathbf c^\sharp_p\|_2 =  \|\mathbf P_{\mathbf {\tilde U}_n}[(\mathbf{c}^0_f)^{\gamma\epsilon_n}- (\mathbf{c}^0_p)^{\gamma\epsilon_n}]\|_2\leq\|(\mathbf{c}^0_f)^{\gamma\epsilon_n}- (\mathbf{c}^0_p)^{\gamma\epsilon_n}\|_2 $$

Together with Lemma \ref{lem:vol_upper_bound}, we obtain,
\begin{align*}
    |\conv(\mathbf{\hat C}_n)|& \leq\left(1 +\frac{1}{n-1}\right) |\conv(\smallC)|\\
    & \leq \left(1 +\frac{1}{n-1}\right)\left(1 + C_1 \epsilon_n\right)|\conv(\mathbf{C}^0)|\\
    & \leq \left(1 + C' \epsilon_n\right)|\conv(\mathbf{C}^0)|.
\end{align*}

Furthermore, by Lemma \ref{lem:vol_det}, we have
$$\frac{|\conv(\mathbf {\hat C}_n)|}{|\conv(\mathbf C^0)|} = \frac{h^0}{\hat h_n}\cdot \frac{\sqrt{\det(\mathbf {\hat C}_n^T\mathbf {\hat C}_n)}}{\sqrt{\det{({\tC}^T\tC})}}\leq {1 + C'\epsilon_n}.$$

By Corollary \ref{lem:h_diff}, we obtain
$$\frac{\sqrt{\det(\mathbf {\hat C}_n^T\mathbf {\hat C}_n)}}{\sqrt{\det{({\tC}^T\tC})}}\leq  \frac{\hat h_n}{h^0}\cdot({1 + C'\epsilon_n}) \leq 1+ C''\epsilon_n.$$

\end{proof}

\subsection{Proof of Proposition \ref{prop:ab_ss_property} }\label{app:pf_prop_3}
\begin{proof}
\begin{itemize}
\item[(i)] If $\alpha \geq \alpha'$, $\beta \leq \beta'$ and $\mathbf W$ is $(\alpha, \beta)$-SS,
\begin{align*}
& [cone(\mathbf W)^{\ast}]^{\alpha'} \bigcap [bd \mathcal K]^{\alpha'} \subseteq [cone(\mathbf W)^{\ast}]^\alpha \bigcap [bd \mathcal K]^{\alpha}\\
 & \subseteq \{\mathbf x: \|\mathbf x - \lambda \mathbf e_f \|_2 \leq \beta\lambda, \lambda \geq 0 \} 
\subseteq \{\mathbf x: \|\mathbf x - \lambda \mathbf e_f \|_2 \leq \beta'\lambda, \lambda \geq 0 \}.
\end{align*}
Then $\mathbf W$ is $(\alpha', \beta')$-SS.

 \item[(ii)] 
 If $\conv(\mathbf W)\subseteq \conv(\mathbf{\bar{W}})$, 
 \begin{align}\label{eq:prop3_1}
     cone(\mathbf{\bar{W}})^{\ast} \subseteq cone(\mathbf W)^{\ast} \subseteq \mathcal K.
  \end{align}
Also, since 
$$[cone(\mathbf W)^{\ast}]^\alpha = \left\{\mathbf x: \mathbf x^T \mathbf W \geq -\alpha\|\mathbf x\|_2\right\} = \left\{\mathbf x: \mathbf x^T \mathbf w \geq -\alpha\|\mathbf x\|_2, \forall \mathbf w\in \conv(\mathbf W)\right\},$$
 we have 
  \begin{align}\label{eq:prop3_2}
    [cone(\mathbf{\bar{W}})^{\ast}]^\alpha \subseteq [cone(\mathbf{W})^{\ast}]^\alpha.
  \end{align}
By \eqref{eq:prop3_1}, \eqref{eq:prop3_2} and the definition, if $\mathbf W$ is $(\alpha, \beta)$-SS, $\mathbf{\bar{W}}$ is also $(\alpha, \beta)$-SS.
 \item[(iii)] The proof is trivial by definition.
\end{itemize}
\end{proof}

\subsection{Proof of Proposition~\ref{prop:suff_SS_new} }\label{app:pf_prop_6}
The following lemma is helpful in the proof of Proposition~\ref{prop:suff_SS_new}.
\begin{lemma}\label{lem:l1_and_l2_norm}
For any $\mathbf{x} \in \mathds{R}^k$, if 
$$
\|\mathbf x\|_1 - \|\mathbf x\|_2 \leq \epsilon\|\mathbf x\|_1
$$
for some $\epsilon\in [0, \frac{1}{2k}]$,
then there exists one element $x_i$ of $\mathbf{x}$ such that 
$$
\|\mathbf{x}-x_i\mathbf{e}_i \|_2\leq 4\sqrt{k-1}\epsilon\cdot |x_i|.
$$
\end{lemma}
\begin{proof}
It suffices to show the lemma holds for $\mathbf x\geq 0$ and $\|\mathbf x\|_1 = 1$. Now suppose there exist some elements of $\mathbf x$, say $x_1$, such that 
$$x_1\geq 2\epsilon\;\; \textit{and}\;\; x_1\leq \frac{1}{2}.$$
Since $\|\mathbf{x}_{-1}\|_2$ is a convex function, under the convex constraint 
$$\left\{(x_2, \cdots, x_k):\sum_{j=2}^k x_j = 1-x_1, x_j \geq 0, j = 2,\cdots, k\right\},$$
it is maximized on the vertex of the constraint. Therefore,
$\|\mathbf{x}_{-1}\|_2 \leq 1-x_1$, which leads to
\begin{align*}
\|\mathbf x\|_1 - \|\mathbf x\|_2
&\geq 1-\sqrt{(1-x_1)^2+x_1^2}\\
&= 1-(1-x_1)\sqrt{1+\left(\frac{x_1}{1-x_1}\right)^2}\\
&> 1-(1-x_1)\left[1+\frac{1}{2}\left(\frac{x_1}{1-x_1}\right)^2\right] \\
&= x_1-\frac{x_1^2}{2(1-x_1)}\\
&\geq x_1-x_1^2\geq \frac{x_1}{2}\geq \epsilon,
\end{align*}
where the second inequality is because $\sqrt{1+t}< 1+\frac{t}{2}$ for $t > 0$. 
So we get a contradiction. Consequently, there is no element in $[2\epsilon, \frac{1}{2}]$. Since there is at least one element that is larger than or equal to $\frac{1}{k}$ and $2\epsilon \leq \frac{1}{k}$, at least one element is larger than or equal to $2\epsilon$.
At the same time, there is at most one element that is larger than $\frac{1}{2}$, so there must be exactly one element that is larger than $\frac{1}{2}$. Let the element be $x_i$, then all other elements are less than $2\epsilon$. 
Therefore,
\begin{align*}
\|\mathbf{x}-x_i\mathbf{e}_i \|_2\leq \sqrt{k-1}\cdot 2\epsilon\leq \sqrt{k-1}\cdot 2\epsilon\cdot 2x_i
\end{align*}
\end{proof}



Now we are ready to present the proof of Proposition~\ref{prop:suff_SS_new}.
\begin{proof}


By Proposition \ref{prop:ab_ss_property}(ii), it suffices to prove for the case when all $x_{ij} = m\in \left[0, \frac{1}{k}\right)$.
Denote 
\begin{align*}
    A_{\beta_\epsilon} &= \left\{\mathbf x: \|\mathbf x - \lambda \mathbf e_f \|_2 \leq \beta_\epsilon\lambda, \lambda \geq 0 \right\},\\
    B_\epsilon &= [bd \mathcal K]^{\epsilon} = \left\{\mathbf x:  |\|\mathbf x\|_2 - \mathbf x^T \mathbf 1_k| \leq \epsilon  \|\mathbf x\|_2 \right\},\\ 
    C_\epsilon &= [cone(\tW)^{\ast}]^\epsilon = \left\{\mathbf x: \mathbf x^T \tW \geq -\epsilon\|\mathbf x\|_2\right\}.
\end{align*}
Then it suffices to show that there exist $\beta_\epsilon \to 0$ when $\epsilon \to 0$ such that $B_\epsilon\bigcap C_\epsilon \subseteq A_{\beta_\epsilon}$ for $\epsilon>0$.

Without loss of generality, we assume $\|\mathbf x\|_2 = 1$. For any $\mathbf x\in B_\epsilon\bigcap C_\epsilon$,  We consider the following three cases:

\textit{Case (i).}
When all elements of $\mathbf x$ are nonnegative. Since $\mathbf x\in B_\epsilon$, $\|\mathbf x\|_1 = \sum_{i=1}^k x_i \leq 1+\epsilon$. Then
\begin{align}\label{eq: l2_and_l1_old}
  \|\mathbf x\|_1 - \|\mathbf x\|_2\leq \epsilon \leq \epsilon\|\mathbf x\|_1.  
\end{align}

\textit{Case (ii).}
When there exist at least negative two elements in $\mathbf x$. Suppose one of the negative elements is $x_1$. Denote 
$$\mathcal P = \left\{i: x_i\geq 0\right\},\; s = \sum_{i\in \mathcal P}x_i \geq 0.$$ 
And 
$$\mathcal N = \left\{i: x_i<0, i\neq 1\right\},\; t = \sum_{i\in \mathcal N}x_i < 0.$$ 
Since $\mathbf x\in C_\epsilon$, for all $i \in \mathcal N$, 
$$m x_1 + (1-m) x_i \geq -\epsilon,$$
which implies 
$$
0> x_i \geq  -\frac{\epsilon}{1-m}
$$
and
\begin{align}\label{eq: t_ineq}
0> t = \sum_{i\in \mathcal N}x_i  > -\frac{k}{1-m}\epsilon.
\end{align}
Also, pick any $i\in \mathcal N$,
$$
-\epsilon \leq (1-m) x_1 + m x_i < (1-m) x_1.
$$
Therefore,
\begin{align}\label{eq: x_1_ineq}
0> x_1  > -\frac{\epsilon}{1-m}.
\end{align}
By \eqref{eq: t_ineq} and \eqref{eq: x_1_ineq},
\begin{align}\label{eq: l1_and_ele_sum}
\sum_{i=1}^k x_i = x_1 + s + t > \left(|x_1|-\frac{2\epsilon}{1-m}\right) + s + \left(|t| -\frac{2k}{1-m}\epsilon\right) = \|\mathbf x\|_1 - \frac{2k+2}{1-m}\epsilon.
\end{align}
Since $\mathbf x\in B_\epsilon$, 
\begin{align}\label{eq: l2_and_ele_sum}
\sum_{i=1}^k x_i \leq 1+\epsilon.
\end{align}
By \eqref{eq: l1_and_ele_sum} and \eqref{eq: l2_and_ele_sum},
$$
\|\mathbf x\|_1 \leq 1+\left(1+\frac{2k+2}{1-m}\right)\epsilon,
$$
i.e.,
\begin{align}\label{eq: l2_and_l1}
\|\mathbf x\|_1 - \|\mathbf x\|_2 \leq \left(1+\frac{2k+2}{1-m}\right)\epsilon \leq \left(1+\frac{2k+2}{1-m}\right)\epsilon\cdot \|\mathbf x\|_1.
\end{align}

\textit{Case (iii).} 
When there exists only one negative element in $\mathbf x$. Suppose the negative element is $x_1$.
Without loss of generality, we assume $x_k \geq |x_j|$ for all $j = 1, \cdots, k-1$.  

Denote 
$$r = \sum_{i=2}^{k-1}x_i \geq 0.$$ 
Since $\mathbf x\in B_\epsilon$, $\sum_{i=1}^k x_i \leq 1+\epsilon$. 
At the same time, $\|\mathbf x\|_2 = 1$. Combining these two expressions, we have
\begin{align}\label{eq: 1st_ineq}
0\leq \sum_{i=1}^{k-1} x_i^2 + \left(\sum_{i=1}^{k-1} x_i\right)^2 - 2(1+\epsilon)\sum_{i=1}^{k-1} x_i + \epsilon^2 + 2\epsilon
\end{align}
In \eqref{eq: 1st_ineq}, applying the fact that 
$$\sum_{i=2}^{k-1}x_i^2 \leq \left(\sum_{i=2}^{k-1}x_i\right)^2 = r^2,$$
we have
\begin{align}\label{eq: 2nd_ineq}
0\leq 2x_1^2-2(1+\epsilon)x_1 + 2r^2-2(1+\epsilon)r+\epsilon^2 +2\epsilon.
\end{align}
Since $\mathbf x\in C_\epsilon$, for all $i = 2,\cdots, k-1$, 
$$(1-m)x_1 + m x_i \geq -\epsilon,$$
which implies 
\begin{align}\label{eq: x_1_ineq_new}
0> x_1 \geq  -\frac{m}{(k-2)(1-m)}r-\frac{\epsilon}{1-m}.
\end{align}
From \eqref{eq: 2nd_ineq} and \eqref{eq: x_1_ineq_new}, we derive that 
\begin{align}\label{eq: 3rd_ineq}
0\leq & 2\left[\left(\frac{m}{(k-2)(1-m)}\right)^2+1\right]r^2 + \left[\frac{4m}{(k-2)(1-m)^2}\epsilon+\frac{2m}{(k-2)(1-m)}(1+\epsilon)-2-2\epsilon\right]r \nonumber\\
& + \left[\frac{2}{(1-m)^2}\epsilon^2+\frac{2}{1-m}\epsilon+\frac{2}{1-m}\epsilon^2+\epsilon^2+2\epsilon\right]
\end{align}
Since $\mathbf x\in B_\epsilon$, 
$$
1-\epsilon \leq \sum_{i=1}^k x_i = x_1 + r + x_k \leq r + x_k \leq (k-1)x_k,
$$
so
\begin{align}\label{x_k_lower}
    x_k\geq \frac{1-\epsilon}{k-1}.
\end{align}
Since $\mathbf x\in C_\epsilon$,
$$(1-m)x_1 + m x_k \geq -\epsilon.$$
So 
\begin{align}\label{x_1_and_x_k}
    x_1 \geq -\frac{m}{1-m}x_k -\frac{\epsilon}{1-m}.
\end{align}
Therefore,
\begin{align*}
    1+\epsilon &\geq \sum_{i=1}^k x_i = x_1 + r + x_k \nonumber\\
    &\stackrel{\eqref{x_1_and_x_k}}{\geq} -\frac{m}{1-m}x_k -\frac{\epsilon}{1-m} + r + x_k \nonumber\\
    &\stackrel{\eqref{x_k_lower}}{\geq} \left(1-\frac{m}{1-m}\right)\frac{1-\epsilon}{k-1}-\frac{\epsilon}{1-m} + r.
\end{align*}
In other words,
\begin{align}\label{r_upper}
    r \leq \left(1-\frac{1-2m}{(k-1)(1-m)}\right) + \left(1+\frac{1}{1-m}+\frac{1-2m}{(k-1)(1-m)}\right)\epsilon
\end{align}
By \eqref{eq: 3rd_ineq} and \eqref{r_upper}, we get $r < 30\epsilon$ when $\epsilon$ is small enough. 
Then combining with \eqref{eq: x_1_ineq_new},
$$
0> x_1 \geq  -\frac{1}{1-m}\left(\frac{30m}{k-2}+1\right)\epsilon.
$$
Consequently,
\begin{align}\label{eq: l1_and_ele_sum_new}
\sum_{i=1}^k x_i = x_1 + r + x_k 
& > \left[|x_1|-\frac{2}{1-m}\left(\frac{30m}{k-2}+1\right)\epsilon\right] + r + x_k \nonumber \\ 
& = \|\mathbf x\|_1 - \frac{2}{1-m}\left(\frac{30m}{k-2}+1\right)\epsilon.
\end{align}
Since $\mathbf x\in B_\epsilon$, 
\begin{align}\label{eq: l2_and_ele_sum_new}
\sum_{i=1}^k x_i \leq 1+\epsilon.
\end{align}
By \eqref{eq: l1_and_ele_sum_new} and \eqref{eq: l2_and_ele_sum_new},
$$
\|\mathbf x\|_1 \leq 1+\left[1+\frac{2}{1-m}\left(\frac{30m}{k-2}+1\right)\right]\epsilon,
$$
i.e.,
\begin{align}\label{eq: l2_and_l1_new}
\|\mathbf x\|_1 - \|\mathbf x\|_2 
\leq \left[1 + \frac{2}{1-m}\left(\frac{30m}{k-2}+1\right)
\right]\epsilon 
\leq \left[1+\frac{2}{1-m}\left(\frac{30m}{k-2}+1\right)\right]\epsilon\cdot \|\mathbf x\|_1.
\end{align}

~\\
~\\
Finally, combining the three cases above, by \eqref{eq: l2_and_l1_old}, \eqref{eq: l2_and_l1} and \eqref{eq: l2_and_l1_new}, for any $\mathbf x\in B_\epsilon\bigcap C_\epsilon$, 
\begin{align}\label{eq: l2_and_l1_final}
\|\mathbf x\|_1 - \|\mathbf x\|_2 
\leq \left[1 + \max\left\{\frac{2k+2}{1-m}, \frac{2}{1-m}\left(\frac{30m}{k-2}+1\right)\right\}\right]\epsilon \cdot \|\mathbf x\|_1.
\end{align}
Then by Lemma \ref{lem:l1_and_l2_norm}, we know that $B_\epsilon\bigcap C_\epsilon \subseteq A_{\beta_\epsilon}$ for all $\epsilon>0$ and $\epsilon$ small, where 
$$\beta_\epsilon = 4\sqrt{k-1} \cdot \left[1 + \max\left\{\frac{2k+2}{1-m}, \frac{2}{1-m}\left(\frac{30m}{k-2}+1\right)\right\}\right]\epsilon \to 0$$
when $\epsilon \to 0$.
\end{proof}

\subsection{Proof of Proposition~\ref{prop:const_SS_columns} }\label{app:pf_prop_4}
\begin{proof}
By Step 1 of the proof of Theorem~\ref{thm:sample} in Section~\ref{app:pf_thm_3}
we know that: if the probability density function satisfies
\begin{equation}\label{eq:SS_density}
    \mathds P(\|\mathbf w - \mathbf w_i^\sharp\|_2 \leq r) \geq (k-1)!\cdot c_0 \cdot r^{k-1},\quad \forall\, 0 <r \leq r_0
\end{equation}
for the $s$ distinct points $\mathbf w^\sharp_1, \cdots, \mathbf w^\sharp_s$ in its support and some positive constants $r_0$, $c_0$, then with probability at least $1 - C_1 s/d$, for any $\mathbf w_i^\sharp$, there exists at least one sample $\mathbf w_{(i)}^1$, such that 
\begin{equation}\label{eq:sample_close}
\|\mathbf w_{(i)}^1 - \mathbf w_i^\sharp\|_2 \leq r_d,\;\; \forall i = 1,\cdots, s,
\end{equation}
where $r_d = \left(\frac{\log d}{d} \right)^{\frac{1}{k-1}}$. 

Now we show that 
$$[cone(\mathbf W^0_1)^\ast]^{ \alpha-r_d}=  \{\mathbf x: \mathbf x^T \mathbf W^0_1 \geq -(\alpha-r_d)\|\mathbf x\|_2\} \subseteq
[cone(\mathbf W^\sharp)^\ast]^{ \alpha}=  \{\mathbf x: \mathbf x^T \mathbf W^\sharp \geq -\alpha\|\mathbf x\|_2\}.$$
For any $\mathbf x \in \mathds{R}^{k}$ and $\mathbf x^T \mathbf W^0_1 \geq -(\alpha-r_d)\|\mathbf x\|_2$, 
$$-\alpha\|\mathbf x\|_2 \leq \mathbf x^T \mathbf W^0_1 -r_d\|\mathbf x\|_2 \leq \mathbf x^T \mathbf W^0_1 + \mathbf x^T(\mathbf W^\sharp-\mathbf W^0_1) = \mathbf x^T \mathbf W^\sharp,$$
where in the second inequality we apply \eqref{eq:sample_close} and Cauchy–Schwarz inequality.
Therefore, by definition, if $\mathbf W^\sharp$ is $(\alpha, \beta)$-SS, $\mathbf W^0_1$ is $(\alpha-r_d, \beta)$-SS.

We pick $\mathbf w^\sharp_i$ to be the vertex $\mathbf e_i$ of $\simplex^{k-1}$ for $i = 1,\cdots, k$. Apparently, when the density function is uniformly larger than a constant on neighborhoods of $\mathbf e_i$'s, \eqref{eq:SS_density} holds.
Let 
$$\alpha_0 = C_2 \sqrt{\frac{\log (n\vee d)}{n}} + r_d,$$
then by Proposition~\ref{prop:suff_SS_new}, $\mathbf W^\sharp$ is $\left(\alpha_0, C_3 \alpha_0 \right)$-SS for all $n$ and $d$ if $\frac{\log d}{n}\to 0$. As a result, $\mathbf W_1^0$ is $(C_2 \sqrt{\frac{\log (n\vee d)}{n}}, C_3 \alpha_0)$-SS.

When $d\geq C n^{\frac{k-1}{2}}$, we have $$\alpha_0 = C_2 \sqrt{\frac{\log (n\vee d)}{n}} + r_d \leq C_2' \sqrt{\frac{\log (n\vee d)}{n}}.$$
Then $\mathbf W^0_1$ is $(C_2 \sqrt{\frac{\log (n\vee d)}{n}}, C_3' \sqrt{\frac{\log (n\vee d)}{n}})$-SS for all $n$ and $d$ if $\frac{\log d}{n}\to 0$, which finishes the proof.
\end{proof}

\newpage

{\colorblue
\section{Additional Simulations and Experiments}\label{app:add_sim}
\subsection{Convergence of the Estimation}\label{app:conver}
We use the Monte Carlo simulation to show the convergence of the integrated likelihood $ F_{n\times d}(\mathbf C)$ and the MLE $\hatCn$.
Consider a simple setup: $k = V = 3,d=6,$
$$\tC = \begin{bmatrix} 
2/3 & 1/6 & 1/6 \\
1/6 & 2/3 & 1/6\\
1/6 & 1/6 & 2/3 \\
\end{bmatrix}, \quad \tW = \begin{bmatrix}
5/6 & 0 & 1/6 & 5/6 & 1/6 & 0 \\ 
1/6 & 5/6 & 0 & 0 & 5/6 & 1/6 \\
0 & 1/6 & 5/6 & 1/6 & 0 & 5/6 \\
\end{bmatrix},$$
and the sample size is set to be $n = 60, 600,6000,60000$. 

In the experiment, we consider the ``noiseless'' data, i.e., $\mathbf X = n\tC\tW$.  We compare the integrated likelihood among candidate $\mathbf C$'s taking the following form:
\begin{align}
    \label{mat:c_appendix}
 \begin{bmatrix}
c & (1-c)/2 & (1-c)/2\\
(1-c)/2 & c & (1-c)/2\\
(1-c)/2 & (1-c)/2 & c\\
\end{bmatrix},
\end{align}
with $c$ taking values from $[0.5,1]$. We use Monte Carlo method to evaluate the integrated likelihood \eqref{eq:lkh_fn}: 
$$\hat{F}_{n\times d, T}(\mathbf C) \approx  \prod_{i=1}^d  \left[\frac{1}{T} \sum_{t=1}^T f_n(\mathbf x^{(i)}| \mathbf u = \mathbf C\mathbf w_t)\right],$$
where $\mathbf w_1,\cdots, \mathbf w_T$ are i.i.d. random samples from $\text{Dir}_k (\mathbf 1)$ and $T = 50,000.$

\begin{figure}
\begin{minipage}{.6\textwidth}
    \centering
    \includegraphics[width=0.99\textwidth]{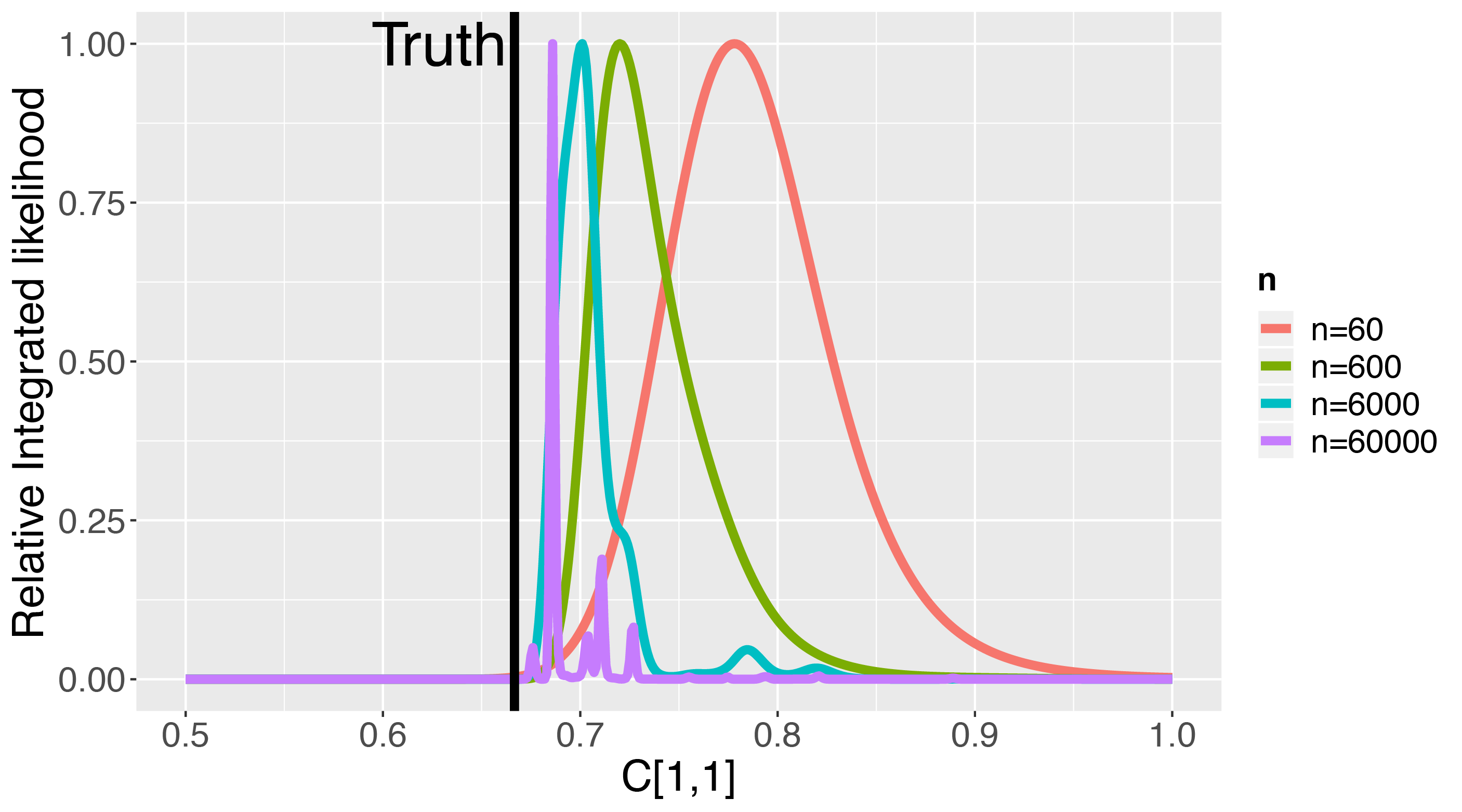}
    \label{fig:lkh_vs_C}
\end{minipage}
\begin{minipage}{.39 \textwidth}
\centering
\includegraphics[width=0.8\textwidth]{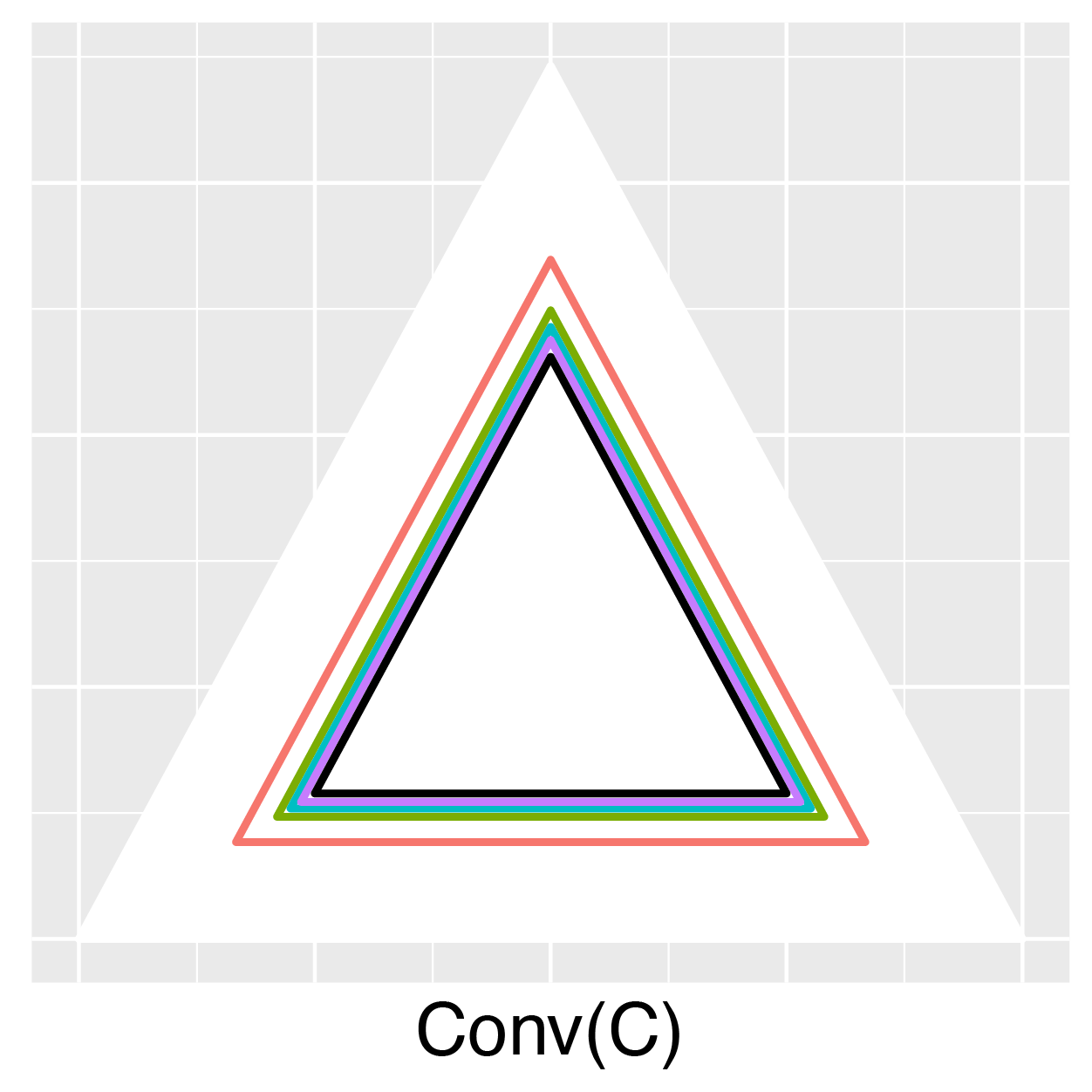}
\end{minipage}
\caption{Results of the experiment in Section \ref{app:conver}. Left: the relative integrated likelihood of ``noiseless'' data. Right: $\conv(\tC)$ and the optimal $\conv(\mathbf C)$'s under different $n$. The white triangle represents $\simplex^2$; the smallest black triangle is $\conv(\tC)$; other colored triangles represent the $\conv(\mathbf C)$'s that maximize $\hat{F}_{n\times d, T}(\mathbf C)$ under different $n$'s. The legend in the middle is shared by both plots.}
   \label{fig:sim_3}
\end{figure}

The left plot of Figure \ref{fig:sim_3} shows $\hat{F}_{n\times d, T}(\mathbf C)/\max_{\mathbf C}\hat{F}_{n\times d, T}(\mathbf C)$, the relative value of the estimated integrated likelihood. As $n$ increases, the peak of the likelihood approach the truth (i.e., $c=2/3$): the optimal $c$ values that maximize $\hat{F}_{n\times d, T}(\mathbf C)$ for $n=60,600,6000,60000$, are $0.778, 0.720, 0.701, 0.686$, respectively. The small fluctuations in the curves of $n=6000,60000$ are possibly due to numeric issues.
The right plot of Figure \ref{fig:sim_3} displays $\conv(\tC)$ and the optimal $\conv(\mathbf C)$'s for different $n$.


\subsection{Comparison with Other Methods}\label{app:comparison}

In this section, we provide additional simulation studies to compare the proposed method (MCMC-EM) with several existing approaches: 
Anchor Free (AnchorF) \citep{huang2016anchor}, Geometric Dirichlet Means (GDM) \citep{yurochkin2016geometric}, and two MCMC algorithms based on Gibbs sampler (Gibbs) \citep{griffiths2004finding} and  based on partially collapsed Gibbs sampler (pcLDA) \citep{magnusson2018sparse, terenin2018polya}.

The basic simulation setup is as follows: $V = 1200$, $d = 1000$, $n=1000$ and $k = 5$,  columns of $\mathbf{C}$ are generated from $Dir_V(0.1)$ and columns of $\mathbf{W}$ are from $Dir_k(0.1)$. For our MCMC-EM algorithm, the number of MCMC samples is 20 without burn-in. The EM algorithm stops after 50 iterations; For each  simulation, we run  the EM algorithm 12 times in parallel with different randomly-initialized parameters and report the result with the highest likelihood value. All hyper-parameters
are set as default, except that the prior over mixing weights in Gibbs and pcLDA is set to be uniform, same as ours.

\begin{figure}[h!]
\centering
\captionsetup{width=.88\linewidth}
  \includegraphics[width=.74\linewidth]{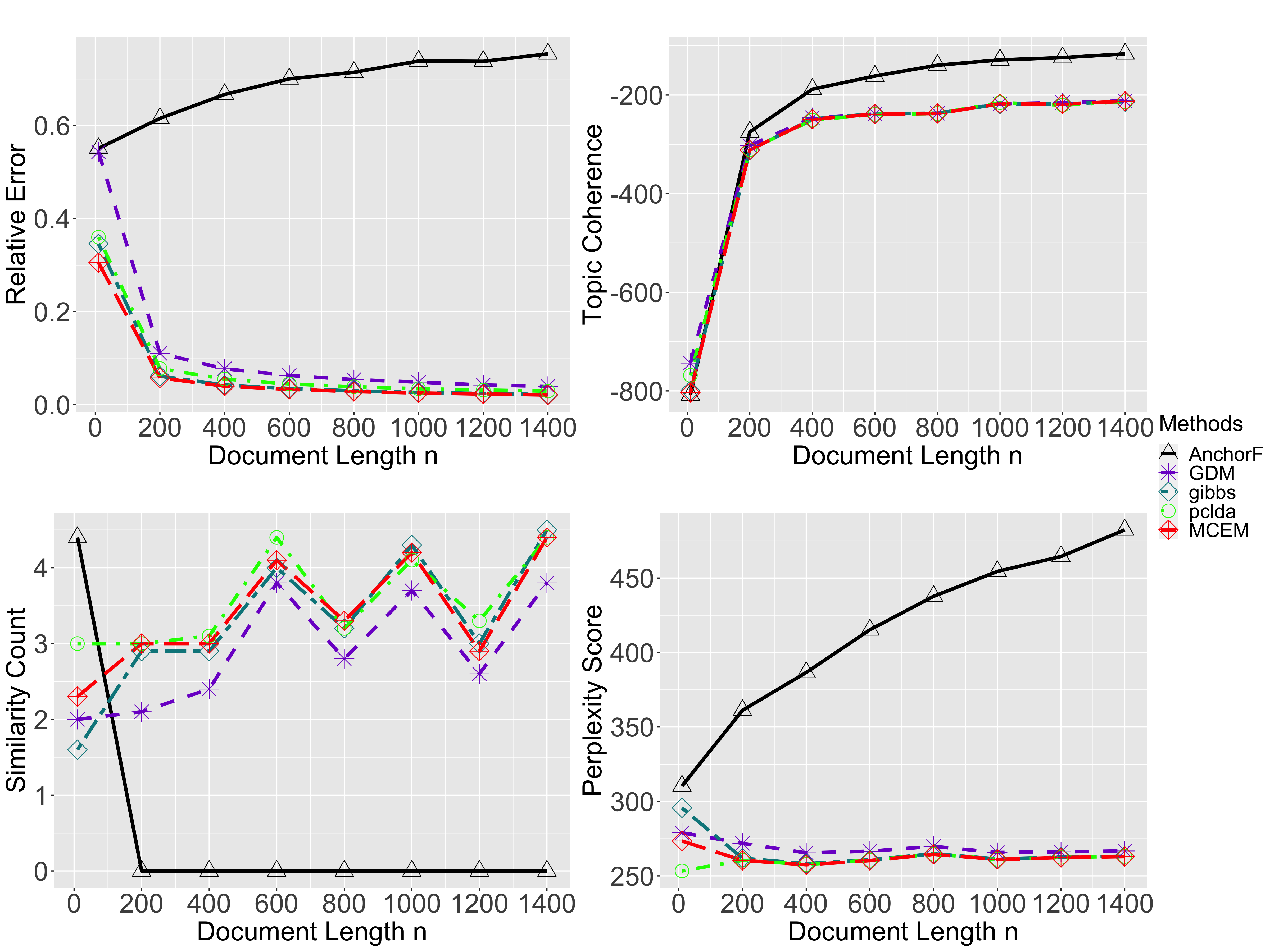}
\caption{Comparison with existing methods when document length varies ($n=10,200,400,\cdots,1400$).
}
\label{sim:p1}
\end{figure}

\begin{figure}[h!]
\centering
\captionsetup{width=.88\linewidth}
  \includegraphics[width=.74\linewidth]{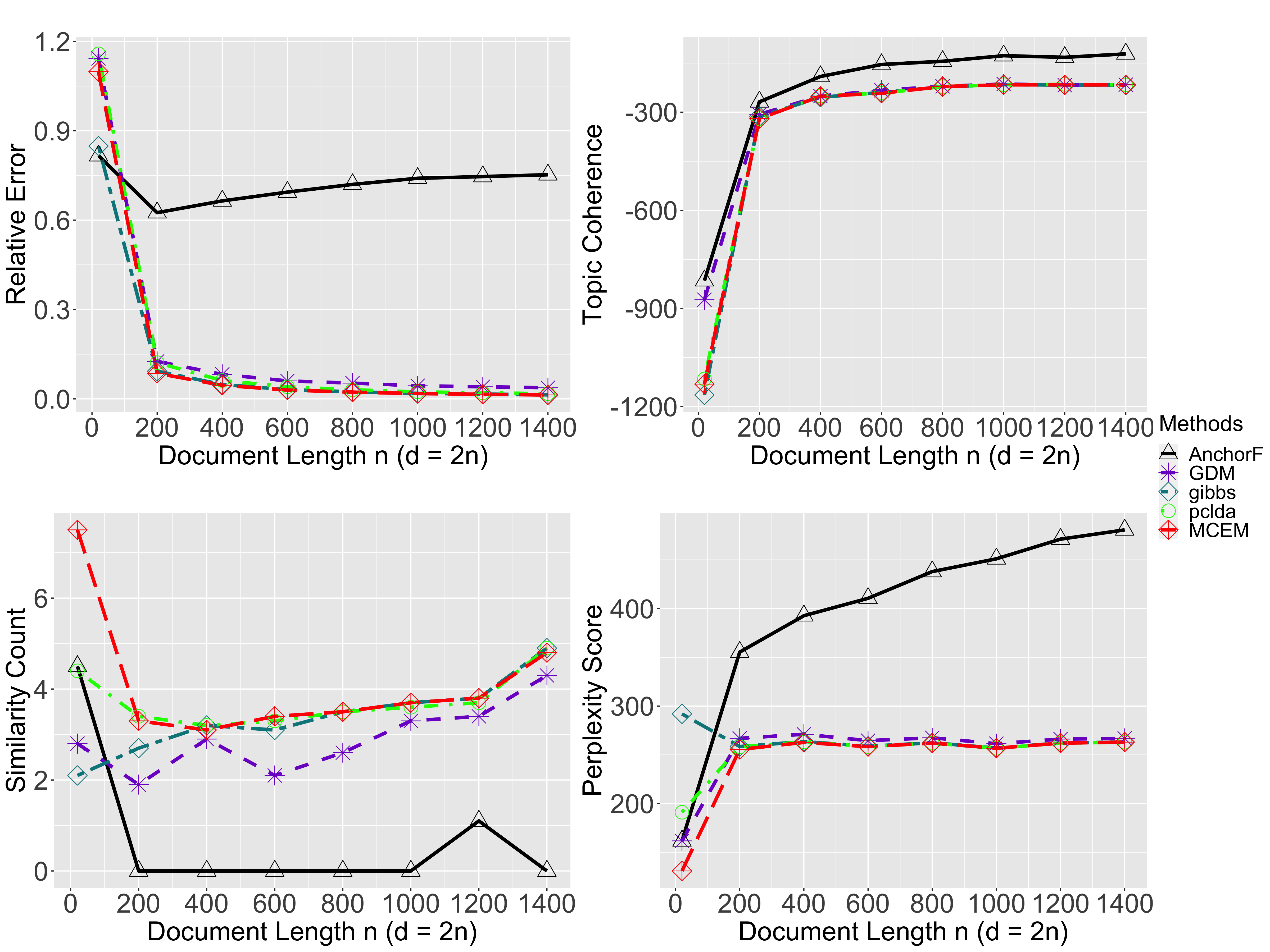}
\caption{Comparison with existing methods when both document length $n$ and number of documents $d$ vary ($n=20,200,400,\cdots,1400$, $d=2n$).
}
\label{sim:p2}
\end{figure}

\begin{figure}[h!]
\centering
\captionsetup{width=.88\linewidth}
  \includegraphics[width=.74\linewidth]{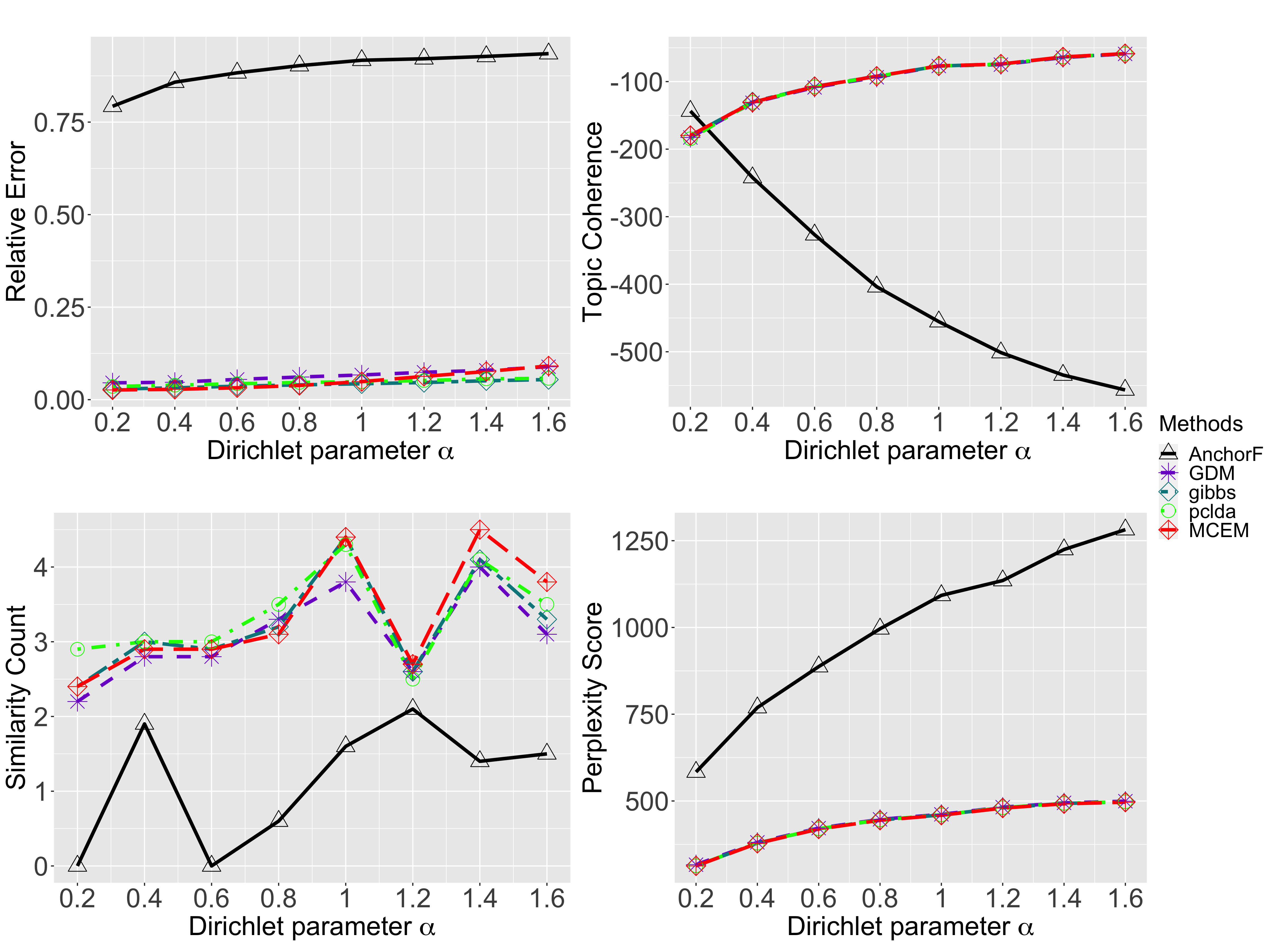}
\caption{Comparison with existing methods when the Dirichlet parameter $\alpha$ varies.
Columns of $\mathbf{W}\sim Dir_k(\alpha)$ with $\alpha=0.2,0.4,\cdots,1.6$. Identity matrix $\mathbf{I}_k$ is appended to the randomly sampled matrix $\mathbf{W}$ to ensure model identifiability.}
\label{sim:p3}
\end{figure}

We evaluate the performance by the following four metrics: 
\begin{itemize}
\item \emph{Relative Error} is defined by $\min_{\mathbf \Pi}\|\mathbf {\hat C}\mathbf\Pi -\mathbf C  \|_F/\|\mathbf C\|_F$, where $\mathbf\Pi$ is a permutation matrix. 
\item \emph{Topic Coherence} is used to measure the single-topic quality, defined as 
$$\sum_{l=1}^k\sum_{v_1, v_2 \in \mathcal V_l} \log \left(\frac{\text{freq}(v_1, v_2) + \epsilon} {\text{freq}(v_2)}\right)$$
where $\mathcal V_l$ is the leading 20 words for topic $l$, $\text{freq}(v_1, v_2)$, $\text{freq}(v_2)$ are the co-occurrence count of word $v_1$ and word $v_2$ and the occurrence counts of word $v_2$, respectively, and $\epsilon$ is a small constant added to avoid numerical issue. Generally, the higher the topic coherence is, the better the quality of the mined topics is.
\item \emph{Similarity Count} is used to measure similarity between topics \citep{arora2013practical,huang2016anchor}, which is obtained simply by adding up the overlapped words across $\mathcal V_l$.
$$\sum_{l_1 < l_2} \sum_{v_1 \in V_{l_1}, v_2 \in V_{l_2}} \idfn (v_1 = v_2).$$
    It focuses on the relationship between mined topics while the topic coherence measures the one within each topic. A smaller similarity count means the mined topics are more distinguishable.
\item \emph {Perplexity Score} measures the goodness of fit of the fitted model to the data. It is the multiplicative inverse of the likelihood normalized by the number of words. Sometimes the perplexity score is calculated on the hold-out data. Here, for simplicity, we use the one based on the training data (the whole dataset),
    $$\sqrt[\leftroot{-3}\uproot{3}\sum_{i=1}^d n_i]{\frac{1}{\prod_{i=1}^d  f_{n_i}(\mathbf x^{(i)}|\mathbf {\hat C}, \mathbf {\hat w}^{(i)})}}.$$
    For a fixed $k$, a smaller perplexity score implies a better fit of the model.
\end{itemize}

We investigate the performance of those methods (i) when  document length $n$ varies, (ii) when both document length $n$ and number of documents $d$ varies, and (iii) when the parameter $\alpha$ of the Dirichlet distribution we use to generate  $\mathbf{W}$ varies. Results are reported in Figure \ref{sim:p1} to Figure \ref{sim:p3}; each metric reported in those  plots is the average over 10 repetitions. Below we summarize our findings: 
\begin{enumerate}[label=(\roman*)]
\item MCMC-EM, GDM, Gibbs and pcLDA perform very similarly in these three simulation settings in terms of four different evaluation metrics.
That is because MCMC-EM, Gibbs and pcLDA have the same objective function and GDM is also a likelihood-based approach.
MCMC-EM has the best relative error and perplexity score in most experiments of the first two settings;
\item Estimators of MCMC-EM, GDM, Gibbs and pcLDA converge very quickly as $n$ increases or as both $n$ and $d$ increase. Their performance is stable as the Dirichlet parameter $\alpha$ increases;
\item The eigenvalue decomposition-based approach AnchorF has better similarity count than other methods in most experiments. However, it performs much worse than others in terms of relative error and perplexity score in almost all experiments.  
The topic coherence of AnchorF is slightly better than the others in the first two settings, but decreases sharply as the Dirichlet parameter $\alpha$ increases in the third setting.
\end{enumerate}

In Table \ref{tab:time}, we report the computation time of our MCMC-EM algorithm and other methods for the experiment in Fig.~\ref{sim:p3} ($V=1200$, $d=n=1000$ and $k=5$).
For our MCMC-EM algorithm, the number of MCMC samples is 20 without burn-in. 
The EM algorithm stops after 50 iterations. 
The results show that the computation time of our MCMC-EM algorithm is comparable with the other methods. 
Our code, which is currently partially implemented in C++, could run faster if being fully implemented in C++; in comparison, the publicly available codes of the competing methods have been mostly highly optimized.
\begin{table}[!ht]
    \centering
    \begin{tabular}{l|lllll}
    \hline
        Method & AnchorF & GDM & Gibbs & pcLDA & $\text{MC}^2$-EM \\ \hline
        Time/s & 6.93 & 0.27 & 82.20 & 34.83 & 49.15 \\ \hline
    \end{tabular}
    \caption{Computational time of the MCMC-EM algorithm and other methods ($V=1200$, $d=n=1000$,  $k=5$).}
    \label{tab:time}
\end{table}


\subsection{Selecting the Number of Topics}\label{app:find_k}
In practice, the number of topics $k$ is unknown. Below we propose a procedure to select $k$ based on the "effective rank" of the sample term-document matrix $\hat{\mathbf{U}}$ reflected in the spectrum. 
   
In Theorem 2, the topic matrix $\mathbf{C}$ is assumed to have full rank; consequently, the true term-document matrix $\mathbf{U}=\mathbf{C} \mathbf{W}$ has rank $k$. By Weyl's inequality \citep{weyl1912asymptotische}, the singular values of the sample term-document matrix $\hat{\mathbf{U}}$ are expected to be close to those of $\mathbf{U}$. Similar to the elbow method used in selecting the number of components in clustering analysis and in PCA, we plot the ordered singular values of $\hat{\mathbf{U}}$ versus its index, and then select $k$ by detecting the location of a significant drop of the curve. 

To test our procedure, we conducted a simulation study where $k=5$, $V = 1200$, $d = 1000$ and $n = 50$. Columns of $\mathbf{C}$ are randomly generated from $Dir_{V}(0.1)$ and columns of $\mathbf{W}$ are randomly generated from $Dir_{k}(0.1)$. We repeated the experiment  10 times and the results are shown in Fig.~\ref{fig:vol_k}. 
From the figure we can see that there is a sudden drop between the 5th and the 6th largest singular values. And the singular values after the 6th one are stable.
So, we would set $k=5$, which agrees with the underlying truth. 
\begin{figure}[h!]
    \centering
    \captionsetup{width=1\linewidth}
    \includegraphics[width=0.8\textwidth]{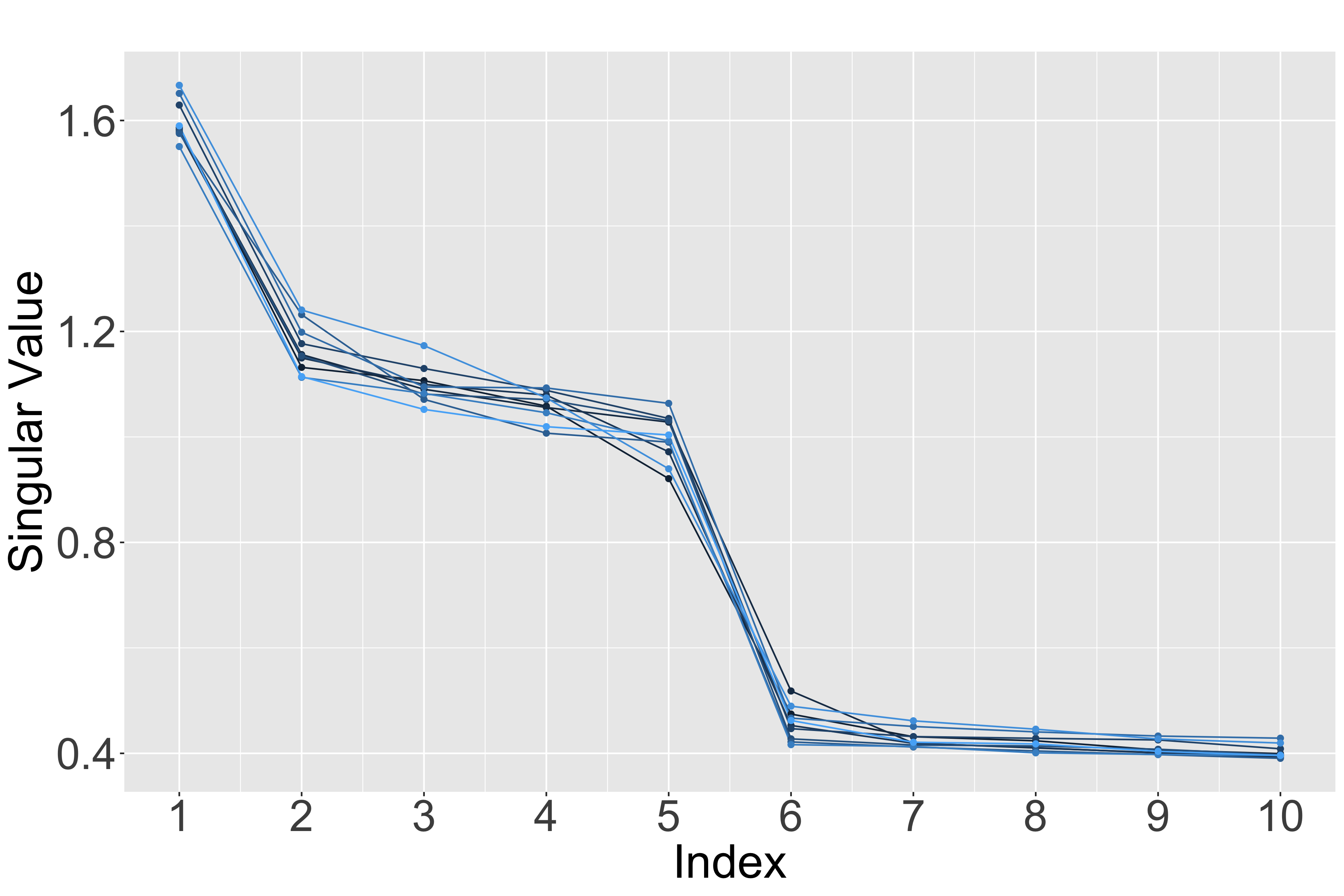}
    \caption{Singular values plot of sample term-document matrices. In 10 repetitions of the experiments, $k=5$, $V = 1200$, $d = 1000$ and $n = 50$. Columns of $\mathbf{C}$ and  $\mathbf{W}$ are generated Dirichlet distributions.}
    \label{fig:vol_k}
\end{figure}

We also apply the approach to the two text data used in the paper -- the NIPS and the Daily Kos datasets. The singular values plots are in Figure \ref{fig:singular_real}. For the NIPS dataset, there is a drop between 5th and 6th largest singular values. For the Daily Kos dataset, there is a drop between 7th and 8th largest singular values. So we choose 5 and 7 as the recommended number of topics for the NIPS and the Daily Kos datasets, respectively. 

\newpage

\begin{figure}[h!]
\centering
\begin{subfigure}{.8\textwidth}
  \centering
  \includegraphics[width=\linewidth]{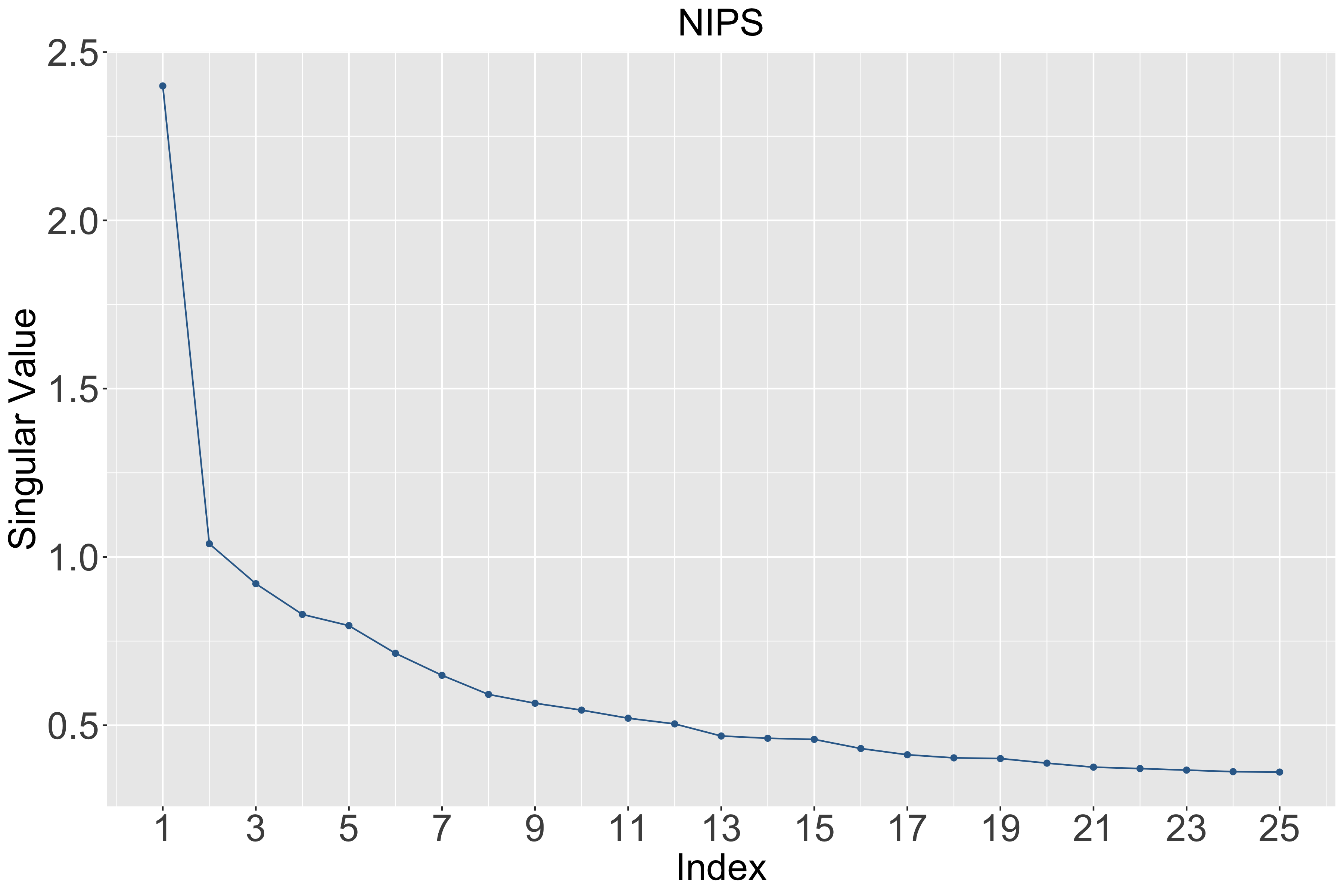}
\end{subfigure}
\begin{subfigure}{.8\textwidth}
  \centering
  \includegraphics[width=\linewidth]{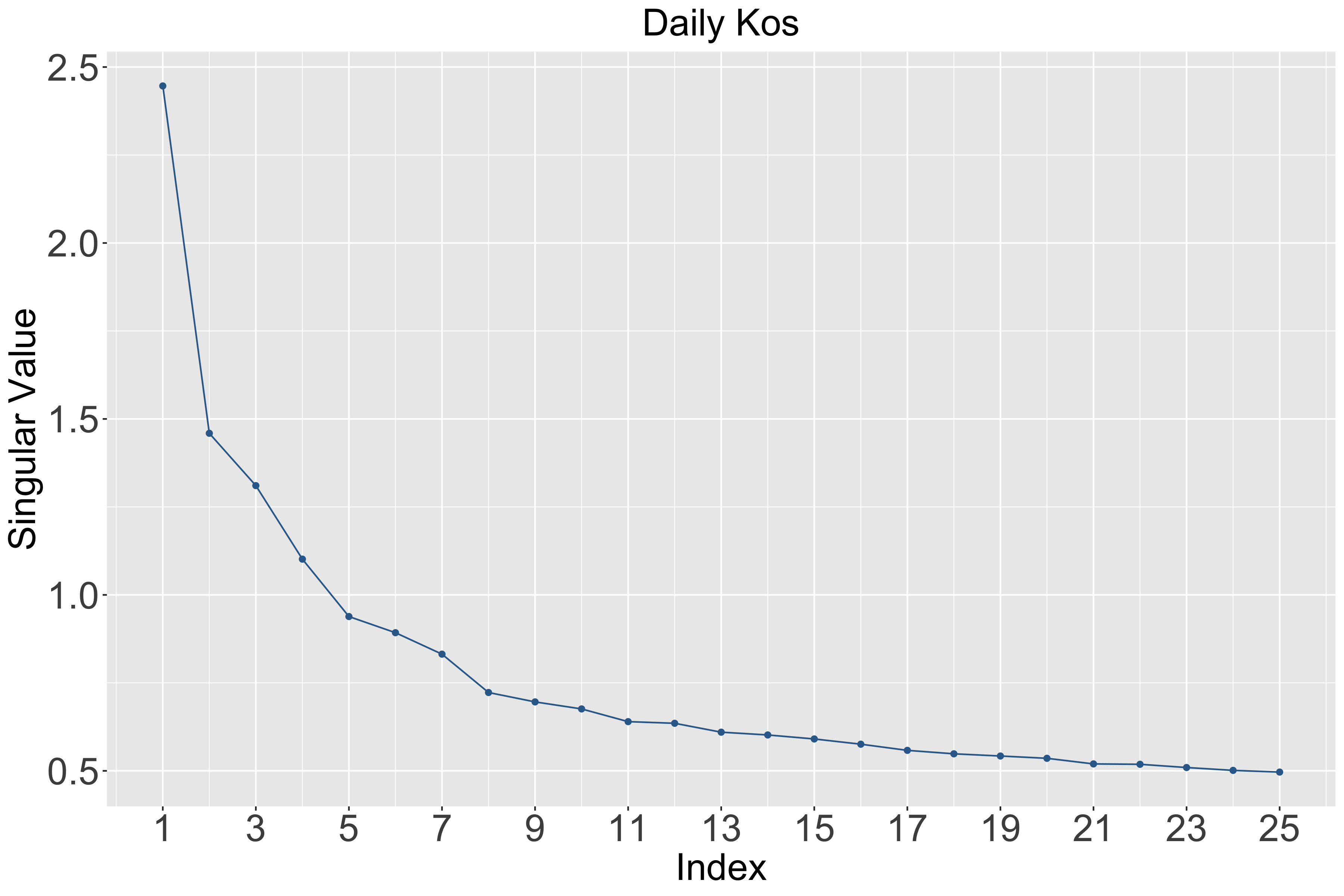}
\end{subfigure}
\caption{Singular values plot of the NIPS and the Daily Kos datasets.}
\label{fig:singular_real}
\end{figure}

}

\newpage

\section{Estimated Topics for the NIPS Dataset}\label{app:nips_topwords}
The NIPS dataset is originally from \cite{perrone2016poisson} and is accessible on UCI Machine Learning Repository\footnote{\href{ https://archive.ics.uci.edu/ml/datasets/NIPS+Conference+Papers+1987-2015}{https://archive.ics.uci.edu/ml/datasets/NIPS+Conference+Papers+1987-2015}}.
It contains $V=11463$ words and $d=5811$ NIPS conference papers published between 1987 and 2015, with an average document length of 1902. In this section, we display the top 10 words of mined topics output by our MCMC-EM algorithm at $k=5, 10, 15, 20$.

\begin{landscape}
\begin{table}[h]
\centering
\scalebox{0.9}{
\begin{tabular}{lllll}
\hline
Topic 1  & Topic 2   & Topic 3      & Topic 4        & Topic 5   \\
\hline
network  & algorithm & model        & training       & learning  \\
neural   & matrix    & models       & learning       & algorithm \\
input    & function  & data         & data           & state     \\
time     & problem   & distribution & set            & time      \\
model    & data      & inference    & image          & function  \\
networks & set       & using        & features       & value     \\
figure   & error     & parameters   & feature        & policy    \\
neurons  & linear    & prior        & classification & set       \\
output   & let       & likelihood   & using          & action    \\
system   & theorem   & bayesian     & images         & optimal  \\
\hline
\end{tabular}}
\caption{Mined topics for NIPS at $k=5$}
~\\
\centering
\scalebox{0.9}{
\begin{tabular}{llllllllll}
\hline
Topic 1     & Topic 2        & Topic 3     & Topic 4  & Topic 5    & Topic 6  & Topic 7   & Topic 8   & Topic 9   & Topic 10     \\
\hline
image       & learning       & data        & network  & graph      & neurons  & state     & algorithm & matrix    & model        \\
images      & training       & noise       & networks & algorithm  & model    & learning  & function  & kernel    & models       \\
object      & data           & time        & neural   & tree       & time     & policy    & theorem   & data      & distribution \\
model       & classification & model       & learning & set        & neural   & time      & bound     & problem   & data         \\
figure      & set            & figure      & training & node       & neuron   & action    & let       & linear    & gaussian     \\
features    & features       & test        & input    & clustering & input    & value     & learning  & algorithm & inference    \\
using       & class          & error       & output   & nodes      & spike    & function  & loss      & method    & likelihood   \\
visual      & feature        & performance & layer    & number     & figure   & algorithm & case      & sparse    & parameters   \\
recognition & label          & estimate    & units    & problem    & activity & reward    & functions & methods   & bayesian     \\
objects     & using          & using       & hidden   & time       & stimulus & optimal   & error     & vector    & prior     \\
\hline
\end{tabular}}
\caption{Mined topics for NIPS at $k=10$}
\end{table}
\end{landscape}

\begin{landscape}
\begin{table}
\centering
\begin{tabular}{lllllllll}
\hline
Topic 1      & Topic 2        & Topic 3      & Topic 4      & Topic 5     & Topic 6   & Topic 7   & Topic 8  \\
\hline
neurons      & learning       & model        & learning     & model       & graph     & network   & state    \\
model        & data           & models       & algorithm    & time        & nodes     & neural    & learning \\
time         & training       & distribution & bound        & data        & node      & networks  & policy   \\
neuron       & classification & inference    & loss         & models      & graphs    & input     & action   \\
spike        & set            & bayesian     & bounds       & figure      & structure & learning  & value    \\
neural       & features       & data         & probability  & human       & network   & output    & states   \\
activity     & feature        & parameters   & error        & prediction  & edge      & training  & time     \\
stimulus     & class          & likelihood   & theorem      & task        & networks  & units     & reward   \\
cells        & label          & latent       & let          & subjects    & edges     & layer     & function \\
figure       & using          & posterior    & regret       & target      & random    & hidden    & control  \\
\hline
\hline
Topic 9      & Topic 10       & Topic 11     & Topic 12     & Topic 13    & Topic 14  & Topic 15  &          \\
\hline
gaussian     & image          & tree         & algorithm    & kernel      & matrix    & function  &          \\
data         & images         & time         & algorithms   & data        & sparse    & problem   &          \\
distribution & object         & speech       & gradient     & space       & norm      & set       &          \\
function     & features       & using        & time         & points      & rank      & functions &          \\
mean         & model          & trees        & optimization & distance    & problem   & theorem   &          \\
noise        & objects        & source       & methods      & clustering  & matrices  & let       &          \\
variance     & using          & algorithm    & method       & linear      & data      & convex    &          \\
error        & feature        & signal       & step         & dimensional & analysis  & case      &          \\
estimate     & recognition    & node         & cost         & point       & algorithm & following &          \\
estimation   & figure         & used         & convergence  & kernels     & convex    & given     &    \\
\hline
\end{tabular}
\caption{Mined topics for NIPS at $k=15$}
\end{table}
\end{landscape}

\begin{landscape}
\begin{table}
\centering
\begin{tabular}{llllllllll}
\hline
Topic 1      & Topic 2    & Topic 3  & Topic 4      & Topic 5       & Topic 6        & Topic 7      & Topic 8     & Topic 9  & Topic 10    \\
\hline
theorem      & network    & neurons  & distribution & function      & graph          & model        & image       & policy   & matrix      \\
bound        & networks   & model    & posterior    & points        & tree           & data         & images      & learning & data        \\
let          & learning   & spike    & bayesian     & point         & nodes          & models       & object      & action   & rank        \\
probability  & training   & neuron   & sampling     & functions     & node           & parameters   & features    & value    & matrices    \\
bounds       & neural     & time     & prior        & space         & algorithm      & distribution & model       & state    & norm        \\
proof        & input      & activity & inference    & error         & set            & likelihood   & recognition & reward   & low         \\
sample       & layer      & neural   & process      & case          & graphs         & gaussian     & feature     & function & algorithm   \\
lemma        & units      & cells    & variational  & approximation & variables      & variables    & using       & optimal  & dimensional \\
following    & output     & input    & data         & mean          & structure      & log          & objects     & actions  & pca         \\
distribution & hidden     & stimulus & model        & given         & edge           & mixture      & vision      & decision & analysis    \\
\hline
\hline
Topic 11     & Topic 12   & Topic 13 & Topic 14     & Topic 15      & Topic 16       & Topic 17     & Topic 18    & Topic 19 & Topic 20    \\
\hline
algorithm    & kernel     & noise    & time         & algorithm     & learning       & features     & clustering  & system   & model       \\
algorithms   & distance   & signal   & state        & optimization  & training       & feature      & word        & memory   & human       \\
online       & data       & using    & model        & gradient      & classification & data         & cluster     & network  & figure      \\
regret       & kernels    & speech   & states       & problem       & class          & set          & words       & neural   & task        \\
learning     & space      & filter   & sequence     & function      & data           & regression   & clusters    & figure   & target      \\
set          & metric     & source   & system       & convex        & examples       & method       & data        & input    & subjects    \\
time         & learning   & signals  & markov       & methods       & label          & selection    & language    & control  & brain       \\
number       & based      & sparse   & transition   & convergence   & set            & number       & set         & time     & information \\
problem      & using      & coding   & dynamics     & method        & classifier     & problem      & model       & output   & subject     \\
bound        & similarity & basis    & sequences    & solution      & error          & methods      & means       & systems  & experiment \\
\hline
\end{tabular}
\caption{Mined topics for NIPS at $k=20$}
\end{table}
\end{landscape}

\newpage
\section{Estimated Topics for the KOS Dataset}\label{app:kos_topwords}
The Daily Kos dataset is accessible on UCI Machine Learning Repository Bag of Words Database\footnote{\href{ https://archive.ics.uci.edu/ml/machine-learning-databases/bag-of-words/}{https://archive.ics.uci.edu/ml/machine-learning-databases/bag-of-words/}}, and its original source is \href{dailykos.com}{dailykos.com}, a group blog and internet forum focused on the Democratic Party and liberal American politics.
The KOS dataset contains $V=6906$ words and $d=3430$ Daily Kos blog entries, with an average document length of 67. In this section, we display the top 10 words of mined topics output by our MCMC-EM algorithm at $k=5, 10, 15, 20$.
\begin{landscape}
\begin{table}[h]
\centering
\scalebox{0.9}{
\begin{tabular}{lllll}
\hline
Topic 1  & Topic 2   & Topic 3      & Topic 4        & Topic 5   \\
\hline
republican & kerry      & bush      & november    & iraq           \\
senate     & bush       & president & poll        & war            \\
house      & dean       & people    & house       & bush           \\
party      & poll       & kerry     & republicans & administration \\
democrats  & percent    & media     & governor    & military       \\
campaign   & edwards    & bushs     & senate      & american       \\
democratic & democratic & time      & electoral   & president      \\
elections  & voters     & campaign  & polls       & iraqi          \\
race       & primary    & general   & account     & people         \\
state      & polls      & years     & vote        & officials     \\
\hline
\end{tabular}}
\caption{Mined topics for KOS at $k=5$}
~\\
\centering
\scalebox{0.9}{
\begin{tabular}{llllllllll}
\hline
Topic 1     & Topic 2        & Topic 3     & Topic 4  & Topic 5    & Topic 6  & Topic 7   & Topic 8   & Topic 9   & Topic 10     \\
\hline
kerry      & november    & senate     & general    & house          & bush           & iraq      & campaign   & bush           & people       \\
dean       & poll        & race       & election   & bush           & kerry          & war       & media      & tax            & political    \\
poll       & house       & house      & bush       & committee      & president      & military  & party      & administration & america      \\
edwards    & republicans & elections  & states     & white          & bushs          & iraqi     & democratic & years          & issue        \\
primary    & governor    & republican & republican & national       & administration & american  & million    & jobs           & rights       \\
percent    & senate      & democrats  & voters     & delay          & general        & troops    & money      & year           & time         \\
clark      & electoral   & state      & state      & texas          & cheney         & soldiers  & time       & health         & marriage     \\
democratic & account     & district   & party      & administration & war            & saddam    & people     & percent        & conservative \\
polls      & polls       & gop        & vote       & court          & iraq           & officials & political  & million        & politics     \\
results    & vote        & democratic & nader      & report         & john           & people    & democrats  & economy        & gay         \\
\hline
\end{tabular}}
\caption{Mined topics for KOS at $k=10$}
\end{table}
\end{landscape}

\begin{landscape}
\begin{table}[]
\centering
\begin{tabular}{lllllllll}
\hline
Topic 1      & Topic 2        & Topic 3      & Topic 4      & Topic 5     & Topic 6   & Topic 7   & Topic 8  \\
\hline
bush     & party      & senate     & bush      & years         & law         & poll    & people     \\
news     & campaign   & race       & kerry     & people        & republicans & bush    & campaign   \\
national & money      & elections  & president & abu           & court       & percent & convention \\
john     & million    & republican & general   & policy        & rights      & kerry   & media      \\
kerry    & democratic & state      & bushs     & blades        & republican  & voters  & nader      \\
war      & house      & house      & john      & meteor        & marriage    & polls   & speech     \\
bushs    & democrats  & democrats  & campaign  & ghraib        & issue       & results & tom        \\
general  & delay      & district   & kerrys    & american      & state       & numbers & ballot     \\
service  & candidates & seat       & debate    & government    & gay         & polling & time       \\
campaign & committee  & gop        & election  & international & political   & lead    & party     \\
\hline
\hline
Topic 9      & Topic 10       & Topic 11     & Topic 12     & Topic 13    & Topic 14  & Topic 15  &          \\
\hline
election    & dean       & bush           & media     & iraq     & bush           & november    &  \\
vote        & edwards    & tax            & people    & war      & administration & poll        &  \\
general     & kerry      & jobs           & time      & iraqi    & president      & house       &  \\
voting      & primary    & administration & ive       & military & iraq           & account     &  \\
voters      & democratic & year           & community & troops   & house          & governor    &  \\
republicans & clark      & health         & blog      & american & intelligence   & senate      &  \\
voter       & iowa       & economy        & political & soldiers & white          & polls       &  \\
republican  & gephardt   & years          & dkos      & forces   & cheney         & electoral   &  \\
ohio        & lieberman  & billion        & read      & killed   & report         & republicans &  \\
oct         & jan        & states         & ill       & baghdad  & war            & vote        & \\
\hline
\end{tabular}
\caption{Mined topics for KOS at $k=15$}
\end{table}
\end{landscape}

\begin{landscape}
\begin{table}[]
\centering
\begin{tabular}{llllllllll}
\hline
Topic 1      & Topic 2    & Topic 3  & Topic 4      & Topic 5       & Topic 6        & Topic 7      & Topic 8     & Topic 9  & Topic 10    \\
\hline
news     & military & iraq     & tax     & party      & people    & iraq     & law       & administration & house       \\
media    & abu      & war      & billion & democratic & life      & war      & court     & bush           & republicans \\
john     & women    & saddam   & years   & campaign   & american  & iraqi    & marriage  & white          & delay       \\
campaign & ghraib   & troops   & year    & political  & political & baghdad  & gay       & house          & republican  \\
fox      & rumsfeld & united   & federal & democrats  & country   & killed   & rights    & intelligence   & democrats   \\
debate   & people   & iraqi    & cuts    & candidates & years     & american & amendment & president      & committee   \\
press    & american & american & budget  & candidate  & family    & soldiers & federal   & report         & senate      \\
sunday   & defense  & bush     & energy  & election   & white     & military & issue     & commission     & gop         \\
national & health   & country  & plan    & campaigns  & america   & forces   & state     & officials      & elections   \\
mccain   & war      & military & health  & dnc        & politics  & city     & legal     & security       & bill      \\
\hline
\hline
Topic 11     & Topic 12   & Topic 13 & Topic 14     & Topic 15      & Topic 16       & Topic 17     & Topic 18    & Topic 19 & Topic 20    \\
\hline
poll    & dean       & november    & bush      & media      & bush      & million     & bush     & bush         & senate     \\
percent & edwards    & poll        & cheney    & people     & kerry     & money       & states   & democrats    & race       \\
kerry   & kerry      & house       & general   & time       & president & campaign    & state    & republicans  & elections  \\
bush    & primary    & governor    & kerry     & ive        & bushs     & candidates  & kerry    & jobs         & republican \\
polls   & clark      & account     & service   & ill        & john      & raised      & nader    & republican   & state      \\
voters  & democratic & electoral   & national  & blog       & general   & house       & general  & democratic   & seat       \\
results & iowa       & republicans & bushs     & bloggers   & campaign  & dkos        & florida  & president    & district   \\
polling & gephardt   & senate      & guard     & night      & kerrys    & fundraising & election & conservative & democrats  \\
numbers & lieberman  & polls       & military  & convention & george    & donors      & vote     & job          & gop        \\
lead    & poll       & vote        & president & email      & debate    & time        & ohio     & reagan       & candidate \\
\hline
\end{tabular}
\caption{Mined topics for KOS at $k=20$}
\end{table}
\end{landscape}

\newpage
\section{Mined meta states for the taxi-trip dataset}\label{app:taxi_trips}

\begin{figure}[H]
    \centering
    \includegraphics[width=0.9\textwidth]{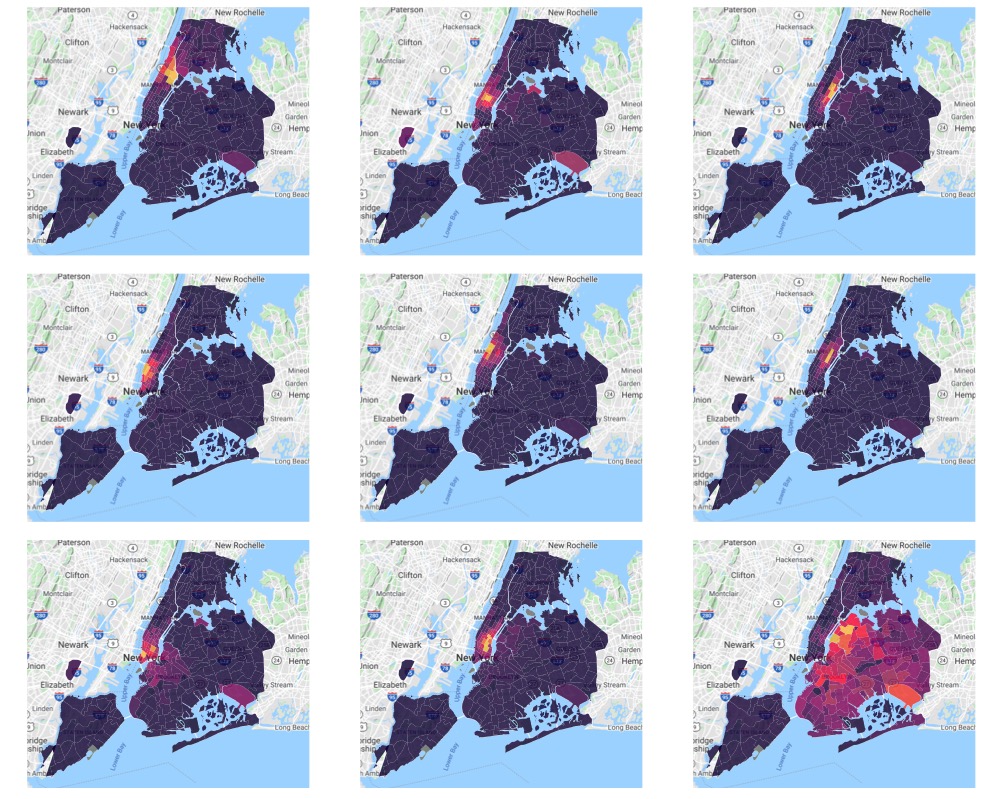}
    \caption{Estimation of disaggregation distributions for NYC taxi-trip data for $k=9$: $\mathbf {\hat C}_1, \mathbf {\hat C}_2, \cdots, \mathbf {\hat C}_9 \in \mathds{R}^V$, where $\mathbf {\hat C}_l = \mathds P(X_{t+1}|Z_t = l)$.}
    \label{fig:taxi_k9_C}
\end{figure}

\begin{figure}[H]
    \centering
    \includegraphics[width=0.9\textwidth]{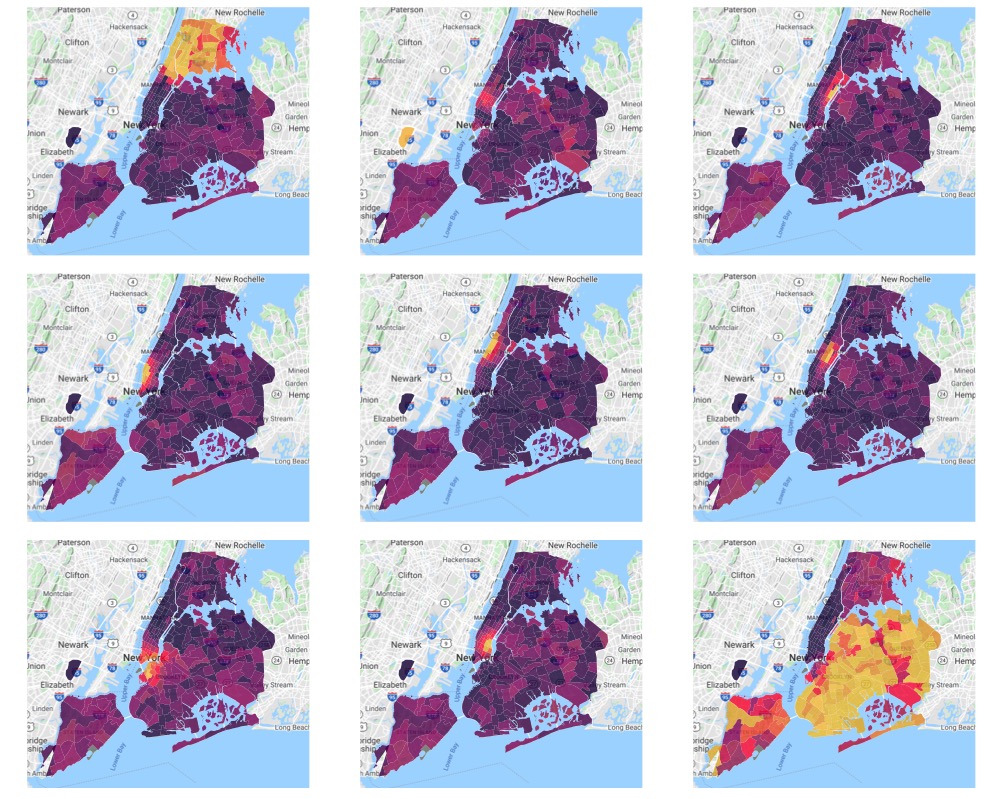}
    \caption{Estimation of aggregation distributions for NYC taxi-trip data for $k=9$: $\mathbf {\hat W}_1, \mathbf {\hat W}_2, \cdots, \mathbf {\hat W}_9 \in \mathds{R}^V$, where $\mathbf {\hat W}_l = \mathds P(Z_t = l|X_t)$.}
\end{figure}

\end{document}